\documentclass[11pt]{article}
\usepackage{latexsym}
\usepackage{amsmath}
\usepackage{amsmath}
\usepackage{amssymb}
\usepackage{amsthm}
\usepackage{epsfig}
\usepackage{xcolor}
\usepackage{graphicx}
\usepackage{bm}
\usepackage{enumitem}
\usepackage{mathtools}
\mathtoolsset{showonlyrefs}
\usepackage{graphicx}
\usepackage{tikz}
\usepackage{mathrsfs}
\usepackage[toc,page]{appendix}
\usepackage{graphicx}
\usepackage{bbm}
\usepackage{ytableau}
\usepackage{tkz-berge}
\usepackage[linesnumbered, ruled, vlined]{algorithm2e}
\usepackage[top=1in, bottom=1in, left=1in, right=1in]{geometry}

\usepackage{array}
\usepackage{multirow}

\usepackage{hyperref}
\hypersetup{
    colorlinks,
    linkcolor={blue!80!black},
    citecolor={green!50!black},
}
\colorlet{linkequation}{blue}
\usepackage{authblk}

\renewcommand{\P}{\mathbb{P}}
\newcommand{\E}{\mathbb{E}}
\newcommand{\Var}{{\rm Var}}
\newcommand{\Cov}{\text{Cov}}
\newcommand{\cN}{\mathcal{N}}

\newcommand{\R}{\mathbb{R}}
\newcommand{\C}{\mathbb{C}}
\newcommand{\N}{\mathbb{N}}
\renewcommand{\S}{\mathbb{S}}

\newcommand{\eps}{\varepsilon} 

\DeclareMathOperator{\id}{id}

\renewcommand{\d}{\textup{d}}

\newcommand{\<}{\langle}
\renewcommand{\>}{\rangle}

\newcommand{\diag}{\text{diag}}

\newcommand{\op}{{\rm op}}

\def\sT{{\mathsf T}}

\DeclareMathOperator*{\argmin}{arg\,min}

\newtheoremstyle{customplain}
  {\topsep}        
  {\topsep}        
  {\itshape}       
  {}               
  {\bfseries}      
  {.}              
  {.5em}           
  {}   

\theoremstyle{customplain}
\newtheorem{theorem}{Theorem}
\newtheorem*{theorem*}{Theorem}
\newtheorem{lemma}{Lemma}

\newtheorem{assumption}{Assumption}

\newtheorem{proposition}{Proposition}

\newtheoremstyle{uprightremark}      
  {\topsep}                          
  {\topsep}                          
  {\normalfont}                      
  {}                                 
  {\bfseries}                        
  {.}                                
  {.5em}                             
  {}                                 

\theoremstyle{uprightremark}
\newtheorem{remark}{Remark}[section]

\DeclareSymbolFont{rsfs}{U}{rsfs}{m}{n}
\DeclareSymbolFontAlphabet{\mathscrsfs}{rsfs}

\def\bA{{\boldsymbol A}}
\def\bB{{\boldsymbol B}}
\def\bC{{\boldsymbol C}}
\def\bD{{\boldsymbol D}}
\def\bE{{\boldsymbol E}}

\def\bG{{\boldsymbol G}}

\def\bH{{\boldsymbol H}}

\def\bI{\mathbf{I}}
\def\bJ{{\boldsymbol J}}
\def\bK{{\boldsymbol K}}

\def\bM{{\boldsymbol M}}

\def\bQ{{\boldsymbol Q}}
\def\bR{{\boldsymbol R}}
\def\bS{{\boldsymbol S}}
\def\bT{{\boldsymbol T}}

\def\bW{{\boldsymbol W}}
\def\bX{{\boldsymbol X}}

\def\ba{{\boldsymbol a}}
\def\bb{{\boldsymbol b}}
\def\be{{\boldsymbol e}}
\def\boldf{{\boldsymbol f}}
\def\bg{{\boldsymbol g}}

\def\bk{{\boldsymbol k}}

\def\bu{{\boldsymbol u}}
\def\bv{{\boldsymbol v}}
\def\bw{{\boldsymbol w}}
\def\bx{{\boldsymbol x}}
\def\by{{\boldsymbol y}}
\def\bz{{\boldsymbol z}}

\def\bbeta{{\boldsymbol \beta}}

\def\btheta{{\boldsymbol \theta}}

\def\bPsi{{\boldsymbol \Psi}}
\def\bPhi{{\boldsymbol \Phi}}
\def\bSigma{{\boldsymbol \Sigma}}

\def\cR{\mathcal{R}}

\def\sfs{{\sf s}}

\def\de{{\rm d}}
\def\Tr{{\rm Tr}}

\def\de{{\rm d}}
\def\Unif{{\rm Unif}}

\def\cV{{\mathcal V}}
\def\cG{{\mathcal G}}

\def\cP{{\mathcal P}}

\def\cT{{\mathcal T}}

\def\cL{{\mathcal L}}
\def\cF{{\mathcal F}}
\def\cE{{\mathcal E}}
\def\cS{{\mathcal S}}

\def\cV{{\mathcal V}}
\def\cG{{\mathcal G}}

\def\cP{{\mathcal P}}

\def\cT{{\mathcal T}}
\def\cH{{\mathcal H}}
\def\cA{{\mathcal A}}

\def\Unif{{\sf Unif}}

\def\proj{{\mathsf P}}

\def\proj{{\mathsf P}}

\def\Unif{{\sf Unif}}

\def\proj{{\mathsf P}}

\def\proj{{\mathsf P}}

\def\cE{{\mathcal E}}
\def\cD{{\mathcal D}}

\def\cF{{\mathcal F}}
\def\cS{{\mathcal S}}

\def\He{{\rm He}}

\def\de{{\rm d}}
\def\Unif{{\rm Unif}}

\def\cE{{\mathcal E}}

\def\bA{{\boldsymbol A}}
\def\btheta{{\boldsymbol \theta}}

\def\cT{{\mathcal T}}
\def\cV{{\mathcal V}}

\def\diag{{\rm diag}}
\def\bS{{\boldsymbol S}}

\def\bD{{\boldsymbol D}}
\def\bPsi{{\boldsymbol \Psi}}

\def\bb{{\boldsymbol b}}

\def\bR{{\boldsymbol R}}

\def\bC{{\boldsymbol C}}

\def\ind{\mathbbm{1}}

\def\boldf{\boldsymbol{f}}

\def\cB{\mathcal{B}}

\def\Im{{\rm Im}}

\def\cZ{\mathcal{Z}}

\def\cJ{\mathcal{J}}

\def\boldf{\boldsymbol{f}}

\def\cB{\mathcal{B}}

\def\Im{{\rm Im}}

\def\Cov{{\sf Cov}}

\def\cK{\mathcal{K}}

\def\dist{{\rm dist}}

\def\frob{\mathrm{F}}

\newcommand{\mmid}{\,\|\,}
\newcommand{\comp}{\mathsf c}

\newcommand{\deq}{:=}

\newcommand{\sft}{\mathsf{t}}
\DeclareMathOperator{\KL}{KL}

\def\bOmega{\bm{\Omega}}

\newcommand{\sfz}{\mathsf{z}}

\title{Generalization, memorization, and overfitting \\for diffusion models trained in the lazy high-dimensional regime} 

\author[1]{Hugo Latourelle-Vigeant}
\author[1]{Sinho Chewi}
\author[2]{Aram-Alexandre Pooladian}
\author[3]{John Sous}
\author[1]{Theodor Misiakiewicz}

\affil[1]{Department of Statistics and Data Science, Yale University}
\affil[2]{Institute for Foundations of Data Science, Yale University}
\affil[3]{Department of Applied Physics, Yale University}

\date{\today}

\allowdisplaybreaks

\begin{document}

\maketitle 

\begin{abstract}
Modern score-based generative models have achieved remarkable empirical success in high-dimensional tasks such as image, audio, and video synthesis. These models reduce distribution learning to a sequence of regression problems that, if solved exactly on finite data, would ultimately reproduce the training samples. Their ability to generalize must therefore arise from the implicit or explicit regularization during training. In this work, we develop a generative counterpart to the theory of benign overfitting and algorithmic regularization for overparameterized neural networks in the supervised lazy-training regime. We study denoising score matching in a vector-valued reproducing kernel Hilbert space with an inner-product kernel. In the proportional high-dimensional regime $n\asymp d$, we derive exact risk trajectories under gradient flow training. These trajectories exhibit three phases governed by qualitatively distinct estimators: a spectral estimator that generalizes, a pure-noise score with localized peaks that interpolate the training objective, and an empirical Bayes estimator that memorizes the data. We then analyze how these estimators combine along the reverse-time SDE and characterize the distribution of the resulting samples. The analysis reveals familiar mechanisms from supervised learning, including kernel linearization and self-induced regularization from the nonlinear part of the kernel, but also reveals a distinct phenomenology specific to generative modeling. 

\end{abstract}

\tableofcontents

\clearpage

\section{Introduction}

The goal of generative modeling is to learn, from $n$ independent samples drawn from an unknown distribution $\pi_0$, a procedure that generates new samples distributed approximately as $\pi_0$. Modern generative models, most notably diffusion models, have achieved striking success across a wide range of domains, including image, audio, video, and scientific data generation \cite{ho2020denoising,rombach2022high,kong2020diffwave,watson2023novo,abramson2024accurate}. Remarkably, these models routinely operate in high-dimensional regimes, where the ambient dimension $d$ ranges from thousands to millions and the number of training samples $n$ is comparable to, or even smaller than, $d$. In such regimes, classical nonparametric density estimation is hopeless, as minimax rates deteriorate exponentially with dimension. Their success must therefore rely on implicit or explicit regularization induced by the model architecture and training algorithm. A central challenge for a statistical theory of generative modeling is therefore to characterize what these models learn under practical training algorithms, through which mechanisms, and over what training and data scales.

Diffusion models~\cite{sohl2015deep,song2019generative,ho2020denoising,song2021score} are particularly attractive for such a theory, because they reduce generative modeling to a family of familiar regression tasks. We briefly review this reduction below\footnote{Throughout this paper, we consider the variance-preserving formulation of diffusion models~\cite{song2021score}.}. Starting from the data distribution \(\pi_0\), consider the Ornstein--Uhlenbeck process
\begin{equation}\label{eq:OU_flow}
\de \bx_t = - \bx_t \,\de t + \sqrt{2}\, \de \bB_t, \qquad \bx_0 \sim \pi_0,
\end{equation}
where $(\bB_t)_{t\ge 0}$ is a standard Brownian motion in $\R^d$ independent of $\bx_0$. This process progressively transforms the data distribution into Gaussian noise. Its marginal distribution $\bx_t \sim \pi_t$ admits the explicit representation 
\begin{equation}\label{eq:OU_flow-marginal}
\bx_t \mid \bx_0 \overset{\de}{=} \rho_t \bx_0 + \sigma_t \bz, \qquad \rho_t \deq  e^{-t}, \quad \sigma_t \deq  \sqrt{1-e^{-2t}}, \quad \bz \sim \cN(0,\bI_d), 
\end{equation}
where $\bz$ is independent of $\bx_0$.  As $t \to \infty$, $\pi_t$ converges exponentially fast to the standard Gaussian distribution. Consequently, for a sufficiently large terminal time $T$ such that $\pi_T\approx\cN(0,\bI_d)$, approximate samples from $\pi_0$ can be generated by initializing $\bx_T\sim\cN(0,\bI_d)$ and simulating from $T$ to $0$ the reverse-time SDE~\cite{anderson1982reverse,haussmann1986time}
\begin{equation}\label{eq:intro-reverse_SDE}
    \d\bx_t = \bigl(-\bx_t - 2\nabla \log \pi_t(\bx_t)\bigr)\,\d t + \sqrt{2}\,\d\bar{\bB}_t,\qquad \bx_T \sim \cN(0,\bI_d),
\end{equation}
where $\bar{\bB}_t$ is a standard reverse-time Brownian motion. 

Simulating \eqref{eq:intro-reverse_SDE} requires estimating the score $\nabla\log\pi_t$ at every noise level $t$ from the training samples ${(\bx_{0,i})}_{i\leq n}$. By Tweedie's formula, this can be accomplished by writing the score estimator as $\widehat \bS_t \deq  \widehat \boldf_t/\sigma_t$ and fitting the denoiser $\widehat \boldf_t$ over a model class $\cF_d$ by minimizing the empirical denoising score matching (DSM) objective~\cite{hyvarinen2005estimation,vincent2011connection}:
\begin{equation}\label{eq:intro-DSM-empirical-risk}
\widehat \boldf_t = \argmin_{\boldf \in \cF_d} \frac{1}{d n} \sum_{i=1}^n \E_{\bz \sim \cN(0,\bI_d)}\left[\bigl\|\boldf(\rho_t \bx_{0,i} + \sigma_t \bz) + \bz\bigr\|_2^2\right].
\end{equation}
Diffusion models thus reduce generative modeling to a sequence of supervised learning problems indexed by the noise level $t$, whose solutions are subsequently composed through the stochastic dynamics \eqref{eq:intro-reverse_SDE} to produce the generated distribution $\widehat{\pi}_{\mathrm{gen}}$.

Despite its algorithmic simplicity, the statistical properties of denoising score matching are far from straightforward and, at first sight, appear incompatible with generalization. Indeed, when the model class is sufficiently rich---as is typically the case for overparameterized neural networks---and the empirical objective \eqref{eq:intro-DSM-empirical-risk} is minimized exactly, the resulting estimator coincides with the empirical Bayes denoiser for the Gaussian-smoothed empirical distribution
\begin{equation}\label{eq:intro-empirical-Bayes-denoiser}
    \widehat \boldf^\star_t = \sigma_t \nabla \log \widehat\mu_{d,n}, \qquad   \widehat\mu_{d,n}  \deq  \frac{1}{n} \sum_{i=1}^n \cN(\rho_t \bx_{0,i},\sigma_t^2 \bI_d).
\end{equation}
As $t\to 0^+$, and hence $\sigma_t \to 0$, $\widehat\mu_{d,n}$ concentrates around the training samples, and the reverse-time dynamics reduce to nearest-neighbor retrieval.  Exact minimization of \eqref{eq:intro-DSM-empirical-risk} thus leads to memorization and not generalization~\cite{pidstrigach2022score,biroli2024dynamical,baptista2025memorization}.

In practice, however, trained diffusion models often generalize: they produce samples that are genuinely new while sharing global and local statistics of the data. Empirically, memorization recedes as the training set grows~\cite{kadkhodaie2024generalization,zhang2023emergence,gu2023memorization}, sets in only after training for long times~\cite{bonnaire2026diffusion,favero2025bigger}, and, most strikingly, models can overfit their training objective while continuing to generate novel samples~\cite{kaiser2026diffusion}. What prevents the learned score from fitting the empirical Bayes denoiser? Which components of the model class and of the training algorithm act as implicit regularizers? And what distribution does the sampler generate? A growing body of work has investigated these questions from complementary perspectives, reviewed in Section~\ref{sec:related}. Our aim in this paper is to develop a high-dimensional theory that makes the underlying mechanisms explicit and delineates several concepts that are often conflated in the literature: overfitting the objective, memorizing the samples, generalizing to the population objective, and generalizing to the target distribution (generated distribution $\widehat{\pi}_{\mathrm{gen}}$ close to $\pi_0$).

In this paper, we focus on the canonical setting of kernel methods with inner-product kernels, which describes wide neural networks trained in the lazy regime~\cite{jacot2018neural,chizat2019lazy} and for which the supervised learning phenomenology is understood in fine detail. The comparison to the supervised setting is instructive in both directions: several mechanisms carry over to generative modeling, while others lead to vastly different conclusions. We first review the main insights from supervised learning and then summarize our results.

\subsection{The lessons from supervised learning}
A central puzzle in deep learning is why highly expressive models trained by empirical risk minimization (ERM), often without explicit regularization, nevertheless generalize. Consider a regression problem with $n$ samples $(\bx_i,y_i)_{i \leq n}$ where $\bx_i \in \R^d$ and $y_i = f_\star (\bx_i) +\eps_i$ with mean-zero, independent label noise $\eps_i$. Given a class $\cF$ of functions $f :\R^d \to \R$, the empirical risk minimizer is 
\begin{equation}\label{eq:intro-ERM-regression}
    \widehat f = \argmin_{f \in \cF}  \frac{1}{n} \sum_{i = 1}^n (f (\bx_i) - y_i)^2.
\end{equation}
Without capacity control, minimizing the empirical risk fits not only the signal $f_\star$ but also the noise in the observed labels, and classical theory predicts poor out-of-sample behavior. Yet overparameterized models trained in practice often generalize well. A decade of work has isolated two key mechanisms to resolve this apparent paradox: \emph{benign overfitting}~\cite{bartlett2020benign,liang2020just,ghorbani2021linearized,tsigler2023benign,hastie2022surprises} and \emph{algorithmic regularization}~\cite{soudry2018implicit,ali2019continuous}.

These mechanisms are understood particularly sharply for neural networks trained in the so-called \emph{lazy regime}~\cite{jacot2018neural,chizat2019lazy}, in which the dynamics of a wide, fully connected network are effectively those of a kernel method with an inner-product kernel\footnote{More precisely, the neural tangent kernel under Gaussian initialization takes the form $h( \< \bx,\bx'\>/d, \| \bx\|_2/\sqrt d, \| \bx'\|_2/ \sqrt d)$. For simplicity, we neglect its dependence on the norms.}
\begin{equation}\label{eq:intro-inner-product-kernel}
    K (\bx,\bx') = h (\< \bx,\bx'\> / d).
\end{equation}
The corresponding kernel ridge estimator admits the explicit representation
\[
\widehat f = \argmin_{f \in \cH}  \Big\{ \frac{1}{n} \sum_{i = 1}^n (f (\bx_i) - y_i)^2 + \frac{\kappa}{n}\, \| f \|_{\cH}^2 \Big\} = \bk_n^\sT ( \bK + \kappa \bI_n)^{-1} \by,
\]
where $\cH$ is the reproducing kernel Hilbert space (RKHS) associated with $K$, $\bK = (K(\bx_i,\bx_j))_{i,j \in [n]}$ is the kernel matrix, and $\bk_n = (K(\cdot, \bx_i))_{i \in [n]}$ is the vector of kernel evaluations at a test point. The empirical risk minimizer \eqref{eq:intro-ERM-regression} over $\cH$ corresponds to the interpolation limit $\kappa \to 0^+$, in which the estimator fits the training labels exactly whenever the kernel matrix is invertible.

Consider now the high-dimensional proportional regime $n \asymp d$. Suppose that the covariates have mean zero, covariance $\E[\bx\bx^\sT] = \bSigma$, and satisfy appropriate concentration assumptions, such as Assumption~\ref{ass:input_dist}. For simplicity of presentation, suppose also that $h(0) = h''(0) = 0$ and $h'(0)=1$.  Several important phenomena emerge in this setting.
\begin{description}
    \item[Kernel linearization.] The kernel matrix is approximated in operator norm by a linear function of the sample covariance matrix, together with an additive diagonal correction \cite{el2010spectrum}:
    \[
    \bK = \bX^\sT\bX/d + \gamma_d \bI_n  + o_{\| \cdot \|_\op,\P} (1) , \qquad \gamma_d = h ( \Tr(\bSigma)/d) -  \frac{\Tr(\bSigma)}{d},
    \]
    where $\bX = [\bx_{1}, \ldots, \bx_n ] \in \R^{d \times n}$. At a fresh test point $\bx$, the kernel vector similarly satisfies $\bk_n (\bx) \approx \bX^\sT \bx/d$. The nonlinear part of the kernel survives on the training set only through the diagonal correction $\gamma_d\bI_n$, whereas the kernel behaves linearly at a typical test point. This follows from high-dimensional near-orthogonality, $\< \bx, \bx '\>/d = O_{\P}(d^{-1/2})$, for independent samples.

    \item[Benign overfitting.] At a test point, the fitted model is asymptotically equivalent to the linear ridge predictor 
    \begin{equation}\label{eq:supervised-additive-ridge}
    \widehat f (\bx) \approx  \bx^\sT \bX^\sT \bigl(  \bX^\sT \bX + d(\kappa + \gamma_d ) \bI_n \bigr)^{-1} \by,
    \end{equation}
     where the nonlinear part of the kernel acts as an additive \emph{self-induced ridge} $\gamma_d$~\cite{liang2020just,bartlett2021deep,misiakiewicz2024six}. Thus, even in the interpolation limit $\kappa=0^+$, the predictor retains a nonvanishing effective regularization and can overfit benignly. The interpolating model decomposes into a regular, linear component that captures the signal and generalizes, and a localized, spiky component that fits the training residuals while remaining invisible at test time.

    \item[Limitations of the kernel regime.] Although the RKHS of $K$ may be universal at any fixed $d$, its effective expressive power in high dimension is sharply limited. In the proportional regime $n\asymp d$, an isotropic inner-product kernel can learn only the linear component of the target; more generally, with \(n\asymp d^{\ell}\) samples it recovers at most a degree-\(\ell\) polynomial approximation~\cite{ghorbani2021linearized,mei2022hypercontractivity}. Escaping this polynomial barrier requires feature learning or architectural structure such as locality and weight sharing~\cite{ghorbani2020neural,misiakiewicz2022learning}.
\end{description}

Given the explanatory success of these results, and the reduction of diffusion models to regression through the objective~\eqref{eq:intro-DSM-empirical-risk}, it is natural to ask for an analogous theory in generative modeling. Three differences, however, make denoising score matching a genuinely distinct problem.

First, the regression problem \eqref{eq:intro-ERM-regression} is replaced by the denoising problem \eqref{eq:intro-DSM-empirical-risk} whose unrestricted empirical minimizer is the empirical Bayes denoiser $\widehat{\boldf}^{\star}_t$ of \eqref{eq:intro-empirical-Bayes-denoiser}, not the population score of $\pi_t$. Around this empirical target, the regression problem is effectively \emph{noiseless}: we show that the residual denoising noise is exponentially small in $d$ (Proposition~\ref{prop:bayes_opt_denoiser_linearization}), and the classical rationale for early stopping---avoiding overfitting the noise---does not apply. Second, interpolation---driving the training objective to zero---is not the benign endpoint it can be in supervised learning: the exact minimizer is the memorizing score. Whatever generalization occurs must arise from the \emph{gap} between the trained model $\widehat{\boldf}_t$ and the empirical minimizer $\widehat{\boldf}^{\star}_t$. Third, a diffusion model is not one regression problem but a sequence of them, indexed by the noise level $t$, and the fitted denoisers are composed through the nonlinear, random dynamics \eqref{eq:intro-reverse_SDE}. Small score risk at each level transfers to the generated distribution by now-standard sampling theory~\cite{chen2023sampling,benton2024nearly}; but in the regimes of interest here the score risk is \emph{not} small, and the generated distribution must be understood directly from the estimator geometry along the reverse trajectories.

These differences make the theory of denoising score matching both conceptually and technically distinct from its supervised counterpart, and, as we shall see, they overturn several of its conclusions.

\subsection{Summary of main results: lazy training for generative modeling}

We study denoising score matching over the vector-valued RKHS induced by the inner-product kernel \eqref{eq:intro-inner-product-kernel}, with an independent model at each noise level. We work in the proportional high-dimensional regime $n \asymp d$ and assume that the data distribution is centered with covariance $\bSigma$ and satisfies a log-Sobolev inequality. Our primary estimator, denoted by $\widehat{\boldf}_{\sft}$, is obtained by running gradient flow on the empirical objective from zero initialization for training time \(\sft\) (reserving \(t\) for diffusion time). This gradient-flow formulation idealizes the practical training of diffusion models by stochastic gradient descent with fresh noise resampled at each step. A parallel theory for the ridge estimator, which exhibits the same phenomenology under the correspondence $\kappa = d/\sft$, is developed in Appendix~\ref{app:ridge}.

We establish sharp characterizations for the score-estimator trajectories, their training and test DSM errors, and the sampling distribution obtained by combining these estimators through the reverse diffusion process. In particular, our theory provides the following analogues of the phenomena described above in supervised learning.

\begin{description}
    \item[Kernel linearization.] With $\cF_d$ the vector-valued RKHS of an inner-product kernel, the objective~\eqref{eq:intro-DSM-empirical-risk} is a kernel regression problem in which the $n \times n$ kernel matrix is replaced by an infinite-dimensional kernel operator $\widehat{\cK}$ on the noisy empirical space, with the $n$ data points replaced by $n$ Gaussian \emph{tubes} $\{\rho_t\bx_{0,i}+\sigma_t\bz\}$. We prove an operator-level analogue of El Karoui's linearization (Proposition~\ref{prop:linearization_kernel_op}): $\widehat\cK$ is approximated by an effective operator $\widehat\cK_{\mathrm{eff}}$ that is linear except for a per-sample correction of strength 
    \[
    \delta_n \deq  \frac{h(\rho_t^2\tau_\bSigma)-\rho_t^2\tau_\bSigma}{n}, \qquad \tau_\bSigma \deq  \Tr(\bSigma)/d,
    \]
    a \emph{self-induced regularization} induced by the nonlinear part of the kernel. Unlike in supervised learning, this correction is \emph{not} a ridge penalty: it acts by averaging the estimator over each tube. Thus, the learned score effectively decomposes into a linear map and localized nonlinear bumps around each training sample. These bumps allow the estimator to interpolate the residuals in the DSM objective while being invisible at test time and generation.

    \item[Benign overfitting?] We delineate three related but distinct notions: ($O$) \emph{overfitting}, in which the empirical denoising objective is driven to zero; ($M$) \emph{memorization}, in which the generated distribution collapses onto the training samples; and ($G^\comp$) \emph{failing to generalize} to the population distribution. By propagating the score estimators through the reverse dynamics~\eqref{eq:intro-reverse_SDE}, we show that $(O) \not\Rightarrow (M)$: even when the training loss vanishes, the generative model need not reproduce training samples. At the same time, the onsets of $(O)$ and $(G^\comp)$ nearly coincide: the training loss vanishes by \emph{forgetting the covariance of the data} and the test-time score degrades to the score of pure noise. This supports the assertion that overfitting in diffusion models is generally not benign \cite{farghly2026benign}.
    
    \item[Limitations of the kernel regime.] For lazily trained denoisers, the generated distribution is asymptotically Gaussian, with a covariance we characterize exactly as a function of the covariance $\bSigma$, the sample ratio $n/d$, the self-induced regularization $\delta_n$, and the training and sampling schedules. This is the generative analogue of the polynomial barrier: at proportional scales, lazy diffusion models generate at most a Gaussian approximation to the data. Escaping it requires feature learning, architectural structure, or genuinely low-dimensional data, which suggests natural next steps for future research \cite{kadkhodaie2024generalization,kamb2024analytic,SclFavWya25PhaseTransition,cui2026solvable,bardone2026theory}.
\end{description}

Although prior work has also studied denoising score matching asymptotics in the proportional regime \cite{CuiZde23Denoising,GeoVeiMac25Manifold,bonnaire2026diffusion,wang2026analytical,wang2026random,MerGol26LinDiff}, we emphasize that the conclusions of the second and third bullet points above require going beyond studying the DSM objectives in isolation and ask about the reverse dynamics~\eqref{eq:intro-reverse_SDE} directly. This is a key novelty of our analysis. In particular, our results allow us to study the interplay between  \textit{generalization} to the population objective and \textit{generalization} to the target distribution. In Figure \ref{fig:covariance_evolution}, we compare several early-stopping schedules for gradient-flow-trained score estimators and evaluate $\KL(\pi_0\mmid \widehat{\pi}_{\mathrm{gen}})$ at different stages of the reverse SDE. We find that, when the reverse SDE is stopped early, a training schedule with suboptimal test error can yield a slightly smaller KL divergence than the optimal early-stopping schedule, which minimizes the population DSM objective along the gradient-flow trajectory at every noise level. We believe that clarifying the relationship between score-matching performance and generation quality is an important direction for future work, and that future theory should study the generated distribution directly rather than focus exclusively on the DSM loss.

\subsubsection{Algorithmic regularization via early stopping}

At the center of our investigation is a precise characterization of the gradient-flow trajectory and the implicit regularization induced by early stopping. This analysis identifies three distinct optimization regimes:
\begin{description}
    \item[Linear regime $\sft = O(d)$.] On this timescale, the kernel operator linearizes and the learned denoiser decomposes into a global linear component and localized bumps around the training samples, corresponding to two spectral channels. The first arises from the linear part of the kernel and relaxes on the timescale $d$; through this channel, the global component learns the spectrum of the sample covariance, top directions first. The second arises from the nonlinear part of the kernel and relaxes on the timescale $1/\delta_n\asymp n$; through this channel, the model begins to fit the per-sample bumps. We derive exact deterministic equivalents for the training and test errors along the entire trajectory (Theorem~\ref{thm:gf_train_test_equivalents_moderate_time}).  The test risk is minimized at $\sft^\star \asymp d$, where the estimator implements a linear shrinkage of the noisy sample covariance. Past $\sft^\star$, the fit overshoots: interpolation forces the linear component toward the pure-noise score, and the model progressively forgets the covariance it had learned.
    
    \item[Polynomial regime $\sft = d^{1+\Theta(1)}$.] On polynomially longer timescales, the kernel's high-degree components fit increasingly sharp localized bumps around the training samples, and the empirical denoising loss vanishes. Note that interpolation here is far more demanding than in supervised learning, since the model must interpolate an entire function on $n$ Gaussian tubes, not $n$ scalar labels. At the same time, at a typical test point, the estimator is simply the pure-noise linear score and the population risk plateaus at $\tfrac{\rho_t^2}{\sigma_t^2}\tau_\bSigma$ (Theorem~\ref{thm:gf_train_test_equivalents_extended_time}). The model has overfitted its objective; it has \emph{not} memorized its samples: the reverse process, as we discuss below, still generates a Gaussian law.

    \item[Super-polynomial regime $\sft=d^{\omega(1)}$.] Only on super-polynomial timescales does gradient flow converge to the empirical Bayes denoiser  (Theorem~\ref{thm:ridgeless_main}). The population risk doubles to $2\frac{\rho_t^2}{\sigma_t^2}\tau_{\bSigma}$ and the reverse process collapses onto the training set. This agrees with the empirical observation that memorization in diffusion models emerges only at very long training times~\cite{bonnaire2026diffusion,favero2025bigger,kaiser2026diffusion}.  
\end{description}
The qualitative behavior of the train and test error trajectories at fixed diffusion time $t$ is summarized in the stylized diagram of Figure~\ref{fig:risk_three_regimes}.

This hierarchy mirrors the observation that diffusion models learn coarse, global statistics before fitting increasingly fine, and eventually sample-specific, details~\cite{wang2026analytical,bardone2026theory,bonnaire2026diffusion,favero2025bigger}. Its distinctive feature is the separation of the onset of interpolation (polynomial time) from the onset of memorization (super-polynomial time): between the two lies an entire regime in which the training loss is vanishing, the score risk is large, and the model generates neither the data distribution nor the training set.

\begin{figure}[t]
    \centering
    \definecolor{ProBlue}{RGB}{31, 119, 180}
    \definecolor{ProRed}{RGB}{214, 39, 40}

    \begin{tikzpicture}[x=1cm, y=1.1cm, >=stealth]

        \fill[gray!8] (4.8, 0) rectangle (8.6, 5.0);  
        \fill[gray!16] (8.6, 0) rectangle (12.5, 5.0); 
        
        \draw[dotted, thick, gray] (4.8, 0) -- (4.8, 5.0);
        \draw[dotted, thick, gray] (8.6, 0) -- (8.6, 5.0);

        \node[text=gray!80!black, font=\footnotesize\bfseries, anchor=base] at (2.4, 4.58) {Generalization regime};
        \node[text=gray!80!black, font=\footnotesize\bfseries, anchor=base] at (6.7, 4.58) {Overfitting regime};
        \node[text=gray!80!black, font=\footnotesize\bfseries, anchor=base] at (10.55, 4.58) {Memorization regime};
        \node[text=gray!80!black, font=\footnotesize, anchor=base] at (2.4, 4.22) {$\sft = O(d)$};
        \node[text=gray!80!black, font=\footnotesize, anchor=base] at (6.7, 4.22) {$\sft = d^{1+\Theta(1)}$};
        \node[text=gray!80!black, font=\footnotesize, anchor=base] at (10.55, 4.22) {$\sft = d^{\omega(1)}$};

        \draw[->, line width=1pt] (-0.2, 0) -- (12.9, 0) node[below left] {$\sft$};
        \draw[->, line width=1pt] (0, -0.2) -- (0, 5.0);

        \draw[dashed, gray, line width=1pt] (0, 1.8) node[left, black] {$\frac{\rho^2}{\sigma^2}\tau_{\bSigma}$} -- (12.5, 1.8);
        \draw[dashed, gray, line width=1pt] (0, 3.6) node[left, black] {$2\frac{\rho^2}{\sigma^2}\tau_{\bSigma}$} -- (12.5, 3.6);
        \draw[dashed, gray, line width=1pt] (0, 1.0) node[left, black] {$1$} -- (12.5, 1.0);

        \draw[line width=1.5pt, ProRed, densely dashed] (0, 1.0) 
            .. controls (0.8, 1.0) and (1.65, 0.5) .. (2.2, 0.5)
            .. controls (2.45, 0.5) and (2.60, 0.5) .. (2.70, 0.5)
            .. controls (2.95, 0.5) and (4.20, 0.0) .. (5.10, 0.0)
            -- (12.5, 0.0)
            node[right, align=left, yshift=0.4cm] {$\cL(\widehat{\boldf}_{\sft};0)$};

        \draw[line width=1.5pt, ProBlue] (0, 1.0) 
            .. controls (0.8, 1.0) and (1.65, 0.5) .. (2.2, 0.5)
            .. controls (2.45, 0.5) and (2.75, 0.5) .. (3.0, 0.5)
            .. controls (3.45, 0.5) and (4.20, 1.8) .. (5.10, 1.8)
            -- (8.6, 1.8)
            .. controls (9.25, 1.8) and (9.5, 3.6) .. (10.4, 3.6)
            -- (12.5, 3.6)
            node[right, align=left] {$\cR(\widehat{\boldf}_{\sft})$};

        \draw[dashed, line width=1pt] (2.5, 0.5) -- (2.5, 0) node[below] {$\sft^\star \asymp d$};

    \end{tikzpicture}
    \caption{Stylized evolution of the empirical denoising loss $\cL(\widehat{\boldf}_{\sft};0)$ and population risk $\cR(\widehat{\boldf}_{\sft})$ along gradient flow at fixed noise level $(\rho,\sigma)$, with $\tau_\bSigma = \Tr(\bSigma)/d$. As training time increases, the score estimator transitions from a spectral shrinkage estimator at $\sft = O(d)$ (Theorem~\ref{thm:gf_train_test_equivalents_moderate_time}), to fitting the pure-noise score at polynomial times $\sft = d^{1 +\Theta(1)}$, with localized bumps that drive the training loss to zero while remaining invisible on typical test inputs (Theorem~\ref{thm:gf_train_test_equivalents_extended_time}), and finally to the empirical Bayes denoiser which memorizes the training samples at super-polynomial times (Theorem~\ref{thm:ridgeless_main}). In this figure, generalization means that the train and test objectives are close.}
\label{fig:risk_three_regimes}
\end{figure}
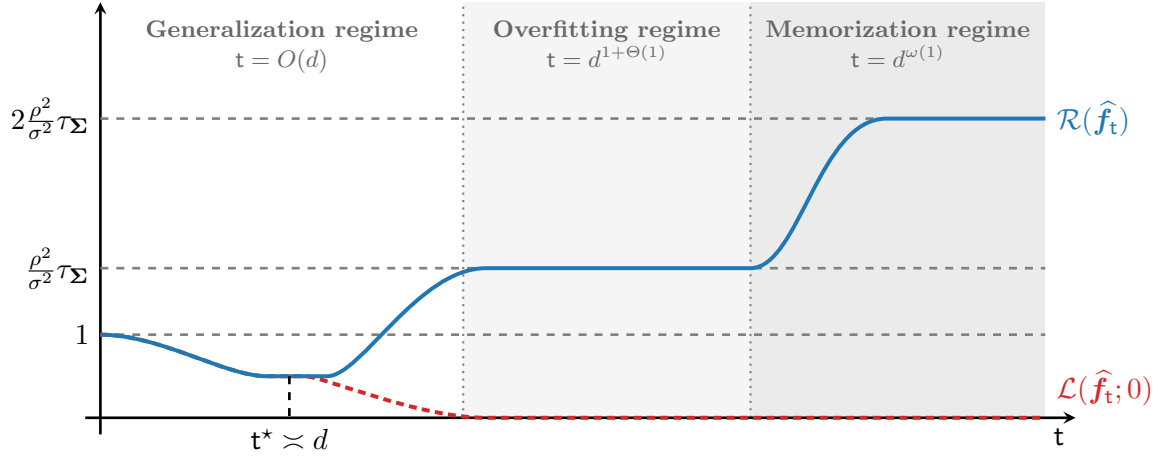

\subsubsection{Generation: delocalization and the Gaussian barrier}

We next analyze how the learned denoisers compose along the reverse diffusion process~\eqref{eq:intro-reverse_SDE}.

\begin{description}
    \item[Linearization of the reverse SDE.] Our main result (Theorem~\ref{thm:kl_bound_empirical_linearized_reverse}) is a pathwise linearization: for any training schedule $\sft_t \le d^{1+c}$ with $c<1/6$ (covering the generalization phase and the onset of overfitting), the Kullback--Leibler divergence between the path measures induced by the learned scores and their linearized surrogates vanishes asymptotically. The mechanism is high-dimensional delocalization: with high probability, a reverse trajectory has normalized overlap $O(d^{-1/2})$ with \emph{every} training sample, uniformly over the trajectory, while the fitted nonlinear bumps activate only within narrow neighborhoods of those samples. Thus, along a typical generative trajectory, the learned kernel scores are indistinguishable from their linearized counterparts.

    \item[Gaussian inductive bias.] As a consequence of this pathwise linearization, the generated distribution is asymptotically Gaussian, with covariance given by an explicit deterministic equivalent (Theorem~\ref{thm:effective_covariance_bound}): a Marchenko--Pastur-type map applied to $\bSigma$, composed with an ODE that tracks the training schedule and sampler early stopping. The generated law depends on $\pi_0$ only through its first two moments, a \emph{universality principle} for lazy diffusion models. Sample-specific nonlinear corrections can interpolate the training objective, but they are invisible to typical reverse trajectories and therefore do not affect the limiting generated law. This provides a mechanism for the effectiveness of linear and Gaussian approximations to trained diffusion models reported in~\cite{wang2023hidden,wang2024unreasonable,LiDaiQu24Gaussian}. All information in $\pi_0$ beyond second order is lost in the lazy-training model, even though the underlying function class is universal and the trained model may interpolate the empirical denoising objective.
\end{description}

Our results delimit what isotropic lazy architectures can achieve in proportional high-dimensional regimes. Non-Gaussian generation requires leaving this regime through feature learning, architectural anisotropy---such as convolution, attention, or locality---or genuinely low-dimensional data structure. This is the generative analogue of the polynomial barrier for kernel methods in supervised learning: when \(n\asymp d\), lazy supervised models learn at most linear functions of their inputs, while the lazy diffusion models studied here generate at most Gaussian distributions.

\subsection{Related work}
\label{sec:related}

Score matching was introduced by Hyv\"arinen~\cite{hyvarinen2005estimation} as a way to fit unnormalized models by matching gradients of log-densities rather than normalizing constants, and Vincent~\cite{vincent2011connection} established its denoising formulation as the regression~\eqref{eq:intro-DSM-empirical-risk}. Diffusion-based generative modeling originates with Sohl-Dickstein et al.~\cite{sohl2015deep}; its modern score-based formulations were developed by Song and Ermon~\cite{song2019generative}, Ho, Jain, and Abbeel~\cite{ho2020denoising}, and Song et al.~\cite{song2021score}, with design refinements surveyed in~\cite{karras2022elucidating}. 

A growing body of research develops sampling and statistical guarantees for diffusion models.
A first line takes a score estimate of prescribed $L^2$ accuracy as given and controls the resulting sampler: the generated law is close to $\pi_0$ whenever the score risk is small, with polynomial dependence on the dimension, with recent improvements to linear dependence on the intrinsic dimensionality of the data~\cite{chen2023sampling,lee2023convergence,benton2024nearly,conforti2025kl,de2022convergence, Chen+26HighAccDiffusion}. A second line of work studies the task of score \emph{estimation} itself: neural network estimators achieve minimax rates for densities with smoothness or low-dimensional structure~\cite{oko2023diffusion,chen2023score,zhang2024minimax}, and optimal score estimation is closely tied to empirical Bayes smoothing~\cite{wibisono2024optimal,dou2024optimal}. These frameworks presuppose that training succeeds at controlling the population score risk, typically as $n \to \infty$ at fixed (or effectively low) dimension, with statistical rates subject to the curse of dimensionality. By contrast, in the proportional regime $n \asymp d$, consistent score estimation is impossible and our paper instead characterizes what the training algorithm actually learns in this regime.  
This requires analyzing the reverse process pathwise rather than a reduction to risk bounds.

That diffusion models can replicate training data was established by~\cite{carlini2023extracting,somepalli2023diffusion,somepalli2023understanding}. The transition between memorization and generalization has since been mapped along several axes, such as dataset size~\cite{yoon2023diffusion,gu2023memorization,zhang2023emergence,kadkhodaie2024generalization}, training time~\cite{bonnaire2026diffusion,favero2025bigger,kaiser2026diffusion}, model capacity~\cite{buchanan2026edge,ye2026provable}, and local properties of the underlying data density~\cite{merger2026local,ross2025geometric}. On the theoretical side, a line of work in statistical physics analyzes the reverse process driven by the \emph{exact} empirical score: it identifies speciation and collapse transitions along the reverse dynamics and shows that avoiding collapse onto training points requires $n$ to be exponential in the dimensions $d$~\cite{biroli2023generative,biroli2024dynamical,RayAmb23SymBreaking,achilli2025memorization,achilli2024losing,li2024critical}. 

Closest to this paper is a rapidly developing line of work on exactly solvable models for diffusion. For \emph{linear} denoisers, prior work has characterized spectral bias laws in the training dynamics~\cite{wang2026analytical}, derived random matrix deterministic equivalents for the learned denoiser \cite{wang2026random}, and analyzed generalization dynamics~\cite{MerGol26LinDiff,Hal25OneStep,GhaAkhHas26Denoisers,PieGal25Gaussian}. For \emph{random features} and shallow architectures, prior work obtained exact DSM learning curves with memorization phase diagrams~\cite{GeoVeiMac26RandomFeatures,GeoMac26Manifold}, and sharp analyses of trained shallow generative models on structured targets~\cite{Cui+24LearningGenerative,CuiPehLu25Solvable,CuiZde23Denoising}. Most relevant to our three-phase picture, \cite{bonnaire2026diffusion} identified two training timescales---a generalization time and a memorization time growing linearly in $n$---analytically in a random feature model and empirically in different models; see also~\cite{Vas25GenVar,ye2026provable,buchanan2026edge,wu2025taking} for complementary mechanisms such as approximation-capacity separations and the effects of learning rates.

 Relative to this literature, our advances are: (i) a universal kernel class---an infinite-dimensional model that genuinely can interpolate and memorize---handled through a new operator-level linearization with a non-ridge self-induced regularization, for general anisotropic data under a log-Sobolev condition rather than Gaussian data; (ii) exact deterministic equivalents along the gradient flow trajectory that includes the overshoot by which the trained model unlearns the covariance; (iii) the resulting separation between interpolation (polynomial time) and memorization (super-polynomial time), a distinction that cannot arise in models not rich enough to interpolate; and (iv) a pathwise analysis of the reverse SDE driven by the trained \emph{nonlinear} score, where prior generated-law characterizations concern models that are linear by construction. Furthermore, our results sharpen and provide a more transparent interpretation of the two-timescale phenomenology in~\cite{bonnaire2026diffusion,favero2025bigger}. Finally, \cite{farghly2026benign} proves in a general framework that small train and population DSM losses are incompatible unless $n$ is exponential in $d$---an impossibility statement consistent with our exact plateau; our results sharpen the picture in the kernel regime by showing that this score-level harm coexists with unharmed generation.

Finally, the supervised-learning phenomena that we take as a starting point are developed in~\cite{zhang2016understanding,belkin2019reconciling,bartlett2020benign,tsigler2023benign,hastie2022surprises,ghorbani2021linearized,muthukumar2020harmless,bartlett2021deep} and references therein, such as interpolation with benign test error, its spectral characterizations, and the decomposition of interpolants into regular and spiky parts. For kernels specifically, the high-dimensional linearization of inner-product kernel random matrices goes back to~\cite{el2010spectrum,cheng2013spectrum} and the self-induced regularization of kernel interpolation and the polynomial approximation barrier to~\cite{liang2020just,ghorbani2021linearized,mei2022hypercontractivity}.  
On the technical side we rely on deterministic equivalents for sample covariance resolvents~\cite{couillet2011random,hachem2007deterministic,latourelle2026dyson,misiakiewicz2024non,fan2026anisotropic}. Our operator-level linearization extends the kernel matrix theory to the infinite-dimensional designs generated by denoising objectives.

\subsection{Notation and organization}
\label{sec:notation}

We write $\|\cdot\|_2$ for the Euclidean norm, $\|\cdot\|_\op$, $\|\cdot\|_\frob$, $\|\cdot\|_\ast$ for the operator, Frobenius, and nuclear norms, and $\<\cdot,\cdot\>$ for the Euclidean inner product. For a positive semidefinite matrix $\bSigma$ we set $\tau_{\bSigma} \deq  \Tr(\bSigma)/d$.  We use $f(d) \lesssim g(d)$ to denote that there exists a constant $C > 0$, depending only on the problem parameters and the constants in the assumptions, such that $f(d) \leq C g(d)$ for all sufficiently large $d$; $f \asymp g$ means $f \lesssim g \lesssim f$. An event $\cE$ holds \emph{with very high probability} (w.v.h.p.) if its failure probability is at most $\P (\cE^c) \leq e^{-d^{c}}$ for some constant $c > 0$ and all large $d$. We denote $o_{d,\P} (\cdot), O_{d,\P}(\cdot)$ for the standard small-o and big-O in probability.  For notational convenience, we introduce a notion of \emph{stochastic domination} with very high probability. For sequences of non-negative random variables $X = \{X_d\}$ and $Y = \{Y_d\}$, we say that $X$ is stochastically dominated by $Y$, denoted by $X \prec Y$, if for every $\epsilon > 0$ there exist constants $c > 0$ and $d_0$ such that 
\begin{equation}\label{eq:stochastic-domination-def}
\P(X_d > d^{\epsilon} Y_d) \leq e^{-d^{c}} , \qquad \text{for all $d \geq d_0$}.
\end{equation}

Throughout, $C, c, c', \ldots$ denote positive constants that may change from line to line. When the diffusion time \(t\) is fixed, we allow constants to depend on \(t\), notably through \(\rho_t,\sigma_t\). For a probability measure $\mu$, $L^2(\mu)$ is the space of square-integrable functions with norm $\|\cdot\|_{L^2(\mu)}$; $\KL(\cdot \mmid \cdot)$ denotes the Kullback--Leibler divergence. We reserve $t$ for the diffusion time and $\sft, \sfs$ for the training (optimization) time.

\paragraph*{Organization.} Section~\ref{sec:setting} sets up kernel denoising score matching: the empirical objective and its memorizing optimum, the RKHS model class, and the gradient flow and ridge estimators. Section~\ref{sec:gf} analyzes the training dynamics at a fixed noise level: the linearization and localized-bump structure (Section~\ref{sec:linearization}), and the three phases (Sections~\ref{sec:gf-phase-I}--\ref{sec:gf-endpoint}). Section~\ref{sec:reverse} analyzes generation: the delocalization of reverse trajectories, the pathwise linearization of the trained score, and the Gaussian characterization of the generated distribution. Finally, we conclude with future research directions in Section~\ref{sec:discussion}. Appendix~\ref{app:ridge} develops the parallel ridge theory; the remaining appendices contain the proofs.

\section{Kernel denoising score matching}
\label{sec:setting}

This section introduces the setting that we consider throughout this paper. Section \ref{sec:DSM} describes the denoising score matching objective, Section \ref{sec:empirical-objective} introduces the empirical objective and its memorizing optimum, Section \ref{sec:inner-product-kernels} describes the RKHS model class with inner-product kernels, and Section \ref{sec:operators-GF} introduces the kernel operators and the gradient flow and ridge estimators.

\subsection{Denoising score matching}
\label{sec:DSM}


Let $\pi_0 \in \mathscr{P}(\R^d)$ be the data distribution we wish to sample from. We progressively corrupt samples from $\pi_0$ with Gaussian noise using the Ornstein--Uhlenbeck forward process~\eqref{eq:OU_flow}. Let $\pi_t \deq  \mathrm{Law}(\bx_t)$ denote the resulting distribution at time $t$, whose explicit form is given in~\eqref{eq:OU_flow-marginal}. To generate new samples, we simulate the reverse-time SDE~\eqref{eq:intro-reverse_SDE} from a terminal time $T>0$, chosen sufficiently large so that $\pi_T \approx \cN(0,\bI_d)$, back to time zero. This requires an estimate of the marginal score $\nabla\log\pi_t$ at every noise level $t \in [0,T]$. One therefore trains a time-dependent score estimator $\bS (\cdot;t) : \R^d \to \R^d$. Exploiting the identity
\begin{equation}
    \nabla_{\bx_t}\log \pi_{t\mid0}(\bx_t\mid\bx_0) = -(\bx_t-\rho_t\bx_0)/\sigma_t^2 = -\bz/\sigma_t,
\end{equation}
the score can be learned by regression on the injected noise. After rescaling the score estimator $\boldf_t = \sigma_t \bS (\cdot ; t)$ and with the standard weighting $\omega(t) = \sigma_t^2$, the population DSM objective is
\begin{equation}\label{eq:intro-pop-risk}
    \cR(\boldf) \deq  \frac{1}{d}\int_0^T
    \E_{\bx_0,\bz} \left[
        \bigl\|\boldf_t(\rho_t\bx_0+\sigma_t\bz)+\bz\bigr\|_2^2
    \right]\de t ,
\end{equation}
where the $1/d$ normalization keeps the risk of order one as $d \to \infty$. By Tweedie's formula~\cite{robbins1992empirical,efron2011tweedie}, its minimizer is the (rescaled) Bayes denoiser
\[
    \boldf_t^\star(\bx)=-\E[\bz\mid \bx_t=\bx]=\sigma_t\nabla\log\pi_t(\bx).
\]

\paragraph*{Decoupling across diffusion times.}  In this paper, we model the family of scores $\{\boldf_t\}_{t \in [0,T]}$ by a separate RKHS function at each noise level, so that the objective~\eqref{eq:intro-pop-risk} decouples across time and each regression problem can be trained and analyzed separately. Abusing notation slightly, we fix $t \in (0,T]$, write $\rho=\rho_t$ and $\sigma=\sigma_t$ with $\rho^2 + \sigma^2 = 1$, and define the fixed-time population risk of a denoiser $\boldf:\R^d\to\R^d$ as
\begin{equation}\label{eq:pop-risk-fixed-t}
    \cR(\boldf) = \frac{1}{d}\E_{(\bx_0, \bz)\sim \pi_0\otimes \cN(0,\bI_d)}\left[\bigl\|\boldf(\rho\bx_0 + \sigma\bz) + \bz\bigr\|_2^2\right].
\end{equation}
Section~\ref{sec:gf} studies this fixed-time problem, treating $t$ (and hence $(\rho,\sigma)$) as a constant as $n,d \to \infty$; Section~\ref{sec:reverse} reassembles the estimators trained at the different noise levels into the reverse SDE. Note that, in practice, the score is often modeled by a single neural network with time as an additional input, which couples the regression problems across $t$. We comment further on this simplifying assumption in Section~\ref{sec:discussion}.

\subsection{The empirical objective and its memorizing optimum}
\label{sec:empirical-objective}

In practice, we observe $n$ i.i.d.\ samples $\bx_{0,1},\ldots,\bx_{0,n} \sim \pi_0$ from the target distribution. The empirical DSM objective replaces $\pi_0$ by the empirical measure and the score estimator $\boldf$ is obtained by minimizing
\begin{equation}\label{eq:empirical_risk}
    \cL(\boldf) \deq  \frac{1}{dn}\sum_{i=1}^n \E_{\bz}\left[\bigl\|\boldf(\rho\bx_{0,i} + \sigma\bz) + \bz\bigr\|_2^2\right],
\end{equation}
over a function class $\cF \subseteq \{ \boldf: \R^d \to \R^d\}$.  Define the noisy empirical and population measures
\begin{equation}
    \widehat{\pi}_{d,n} \deq  \Unif \left(\{\bx_{0,i}\}_{i=1}^{n}\right)\otimes \cN(0,\bI_d),\qquad\qquad \pi_d \deq  \pi_0 \otimes \cN(0,\bI_d),
\end{equation}
and let $\widehat{\mu}_{d,n}$ and $\mu_d$ be their pushforwards under $(\bx,\bz)\mapsto \rho\bx+\sigma\bz$; thus $\widehat\mu_{d,n}$ is the Gaussian mixture 
\begin{equation}
\widehat\mu_{d,n} \deq  \frac1n\sum_{i=1}^n\cN(\rho\bx_{0,i},\sigma^2\bI_d),
\end{equation} 
a \emph{random} measure, while $\mu_d$ is the deterministic noisy marginal. Over an unrestricted function class, the minimizer of $\cL(\cdot)$ is the Bayes denoiser of $\widehat\mu_{d,n}$
\begin{equation}\label{eq:emp_bayes_opt_denoiser}
    \widehat{\boldf}^\star (\bu) \deq  -\E_{(\bx,\bz)\sim \widehat{\pi}_{d,n}}\left[\bz\mid \rho\bx + \sigma\bz = \bu\right] = \frac{1}{\sigma}\sum_{k=1}^{n}\omega_k(\bu) (\rho\bx_{0,k} - \bu),
\end{equation}
with softmax weights
\begin{equation}\label{eq:emp_bayes_opt_denoiser_weights}
    \omega_k(\bu) \deq  \frac{\exp\left(-\frac{1}{2\sigma^2}\|\rho\bx_{0,k} - \bu\|^2\right)}{\sum_{j=1}^{n}\exp\left(-\frac{1}{2\sigma^2}\|\rho\bx_{0,j} - \bu\|^2\right)}.
\end{equation}
This is the \emph{empirical Bayes denoiser}: the exact optimum of the training objective \eqref{eq:empirical_risk}  over any sufficiently rich function class, at every noise level. Once inserted into the reverse SDE~\eqref{eq:intro-reverse_SDE}, this score drives the sampler back to the training set as $t \rightarrow 0^+$ (that is, $\sigma \to 0^+$), essentially performing nearest-neighbor retrieval from the training data.  This paper will be interested in understanding the extent to which the RKHS model class and the training algorithm can avoid this memorization and generalize to the population Bayes denoiser $\boldf^\star$.

\subsection{Inner-product kernels and the model class}
\label{sec:inner-product-kernels}

We estimate each coordinate of the denoiser in the RKHS $\cH$ of a positive definite \emph{inner-product kernel}
\begin{equation}\label{eq:def-inner-product-kernel}
    K(\bx,\bx') = h(\<\bx,\bx'\>/d), \qquad h:\R\to\R .
\end{equation}
More precisely, $\cH$ is the completion of the space $\mathrm{span}\{K(\cdot,\bx):\bx\in\R^d\}$ under the inner product $\<K(\cdot,\bx),K(\cdot,\by)\>_\cH = K(\bx,\by)$, and satisfies the reproducing property $f(\bx) = \<f, K(\cdot,\bx)\>_{\cH}$~\cite{steinwart2008support}. The vector-valued class is the product space $\cH^d$ with norm 
\begin{equation}
    \|\boldf\|_{\cH^d}^2 = \sum_{i=1}^d\|f_i\|_{\cH}^2.
\end{equation}

This class is especially interesting in the context of diffusion models for two reasons. First, when $h$ has strictly positive Taylor coefficients, the RKHS is \emph{universal}: it is dense in $L^2(\widehat\mu_{d,n})$ (and $L^2 (\mu_d)$ under mild assumptions on $\pi_0$; see Appendix \ref{app:interpolation-limit}) at any fixed $d$, hence expressive enough to represent the empirical Bayes denoiser~\eqref{eq:emp_bayes_opt_denoiser} itself; nothing in the model class forbids memorization. Second, inner-product kernels are naturally connected to wide neural networks trained by gradient descent in the lazy regime~\cite{jacot2018neural,chizat2019lazy}, and their high-dimensional behavior in supervised regression is understood in fine detail~\cite{el2010spectrum,ghorbani2021linearized,mei2022hypercontractivity,misiakiewicz2024six}; the comparison between that theory and ours is one of the main goals of the present paper.

We assume throughout that $K$ is continuous and diagonally integrable with respect to $\mu_d$ and $\widehat{\mu}_{d,n}$ (almost surely), meaning that $\E_{\bu\sim \nu}[K(\bu,\bu)]<\infty$ for $\nu\in \{\mu_d,\widehat{\mu}_{d,n}\}$, so that all operators and risks below are well-defined.
It will be useful to introduce a \textit{ridge-regularized} version of the empirical risk \eqref{eq:empirical_risk}:
\begin{equation}\label{eq:empirical_risk_RKHS}
    \cL(\boldf;\kappa) \deq  \frac{1}{dn}\sum_{i=1}^n \E_{\bz}\left[\bigl\|\boldf(\rho\bx_{0,i} + \sigma\bz) + \bz\bigr\|_2^2\right] + \frac{\kappa}{d^2}\|\boldf\|_{\cH^d}^2,
\end{equation}
where $\kappa \geq 0$ is the ridge parameter. Here, the $1/d^2$ scaling is added to match the ridge term to the empirical risk in the high-dimensional regime $n \asymp d$. When $\cH$ is universal, the minimizer of \eqref{eq:empirical_risk_RKHS} converges as $\kappa \to 0^+$ to the empirical Bayes denoiser \eqref{eq:emp_bayes_opt_denoiser}. 

\paragraph*{Decoupling across coordinates.}
Fix an orthonormal basis $\{\bv_i\}_{i=1}^d$ of $\R^d$ and decompose $\boldf(\bx) = \sum_{i=1}^d f_{\bv_i}(\bx)\bv_i$ with $f_{\bv_i}\deq \<\bv_i,\boldf\>$. A direct computation shows $\sum_{i=1}^d \|f_{\bv_i}\|_{\cH}^2 = \|\boldf\|_{\cH^d}^2$ for any orthonormal basis, and since the losses are separable across coordinates,
\begin{equation}\label{eq:emp_risk_decoupled}
    \cL(\boldf;\kappa) = \frac{1}{d}\sum_{i=1}^d \cL_{\bv_i}(f_{\bv_i};\kappa),\qquad \cL_{\bv}(f;\kappa) \deq  \frac{1}{n}\sum_{j=1}^n \E_{\bz}\left[\bigl|f(\rho\bx_{0,j} + \sigma\bz) + \<\bv,\bz\>\bigr|^2\right] + \frac{\kappa}{d}\|f\|_{\cH}^2,
\end{equation}
and similarly $\cR(\boldf) = \frac{1}{d}\sum_{i=1}^d \cR_{\bv_i}(f_{\bv_i})$ with $\cR_{\bv}(f) \deq  \E_{(\bx_0,\bz)}[|f(\rho\bx_0 + \sigma\bz) + \<\bv,\bz\>|^2]$. In the study of the gradient flow and ridge estimators, it will be convenient to fix an arbitrary unit vector $\bv$ and analyze the scalar regression problem for the coordinate function $f_{\bv}$; overall risks are recovered by averaging over an orthonormal basis and the union bound.

\subsection{Kernel operators, gradient flow, and ridge regression}
\label{sec:operators-GF}

To define our estimators, we introduce the following evaluation operators between $\cH$ and the noisy $L^2$ spaces. Define $\cS:\cH \to L^2(\widehat{\pi}_{d,n})$ and $\cT:\cH \to L^2(\pi_d)$ by
\begin{equation}\label{eq:forward_operators}
    (\cS f)(\bx,\bz) \deq  f(\rho\bx + \sigma\bz), \qquad\quad (\cT f)(\bx,\bz) \deq  f(\rho\bx + \sigma\bz),
\end{equation}
acting on the empirical and population spaces respectively. By the reproducing property, their adjoints are the integral operators
\begin{equation}\label{eq:adjoint_operators}
    \cS^\ast g = \E_{(\bx,\bz)\sim \widehat{\pi}_{d,n}}\left[K(\rho\bx + \sigma\bz,\cdot)\,g(\bx,\bz)\right], \qquad\cT^\ast g = \E_{(\bx,\bz)\sim \pi_d}\left[K(\rho\bx + \sigma\bz,\cdot)\,g(\bx,\bz)\right].
\end{equation}
It is occasionally useful to factor these operators through the noisy marginal spaces: $\cS = \widehat{\cP} \circ \cS_0$, where $\cS_0: \cH \hookrightarrow L^2(\widehat{\mu}_{d,n})$ is the natural inclusion and $\widehat{\cP}: L^2(\widehat{\mu}_{d,n}) \to L^2(\widehat{\pi}_{d,n})$, $(\widehat{\cP} g)(\bx,\bz) \deq  g(\rho\bx+\sigma\bz)$, is the (isometric) pullback; similarly $\cT = \cP \circ \cT_0$ with $\cT_0: \cH \hookrightarrow L^2(\mu_d)$ and the population pullback $\cP$. We write 
\begin{equation}
    \widehat{\cK} \deq  \cS \cS^\ast : L^2(\widehat{\pi}_{d,n}) \to L^2(\widehat{\pi}_{d,n}), \qquad 
    \cK \deq  \cT \cS^\ast : L^2(\widehat{\pi}_{d,n}) \to L^2(\pi_d),
\end{equation}
 for the empirical kernel operator and the train-to-population transfer operator respectively. The operator $\widehat\cK$ is self-adjoint and positive semidefinite on $L^2(\widehat\pi_{d,n})$; it plays the role of the kernel matrix $\bK/n$ of supervised learning, with the crucial difference that it acts on the infinite-dimensional space $L^2 (\widehat \pi_{d,n})$. The operator $\cK$ maps functions on the empirical space to their extensions on the population space and is the operator through which train-time quantities are evaluated at test time.
 
 Decomposing the coordinate-wise empirical risk around its unrestricted minimizer (the empirical Bayes denoiser \eqref{eq:emp_bayes_opt_denoiser}) gives $\cL_{\bv}(f;\kappa) = \cL_{\bv}^{\mathrm{red}}(f;\kappa) + \cL_{\bv}^{\mathrm{irr}}$, where
\begin{equation}\label{eq:train_loss_decomposition}
    \cL_{\bv}^{\mathrm{red}}(f;\kappa) \deq  \|\cS f - \widehat{f}^\star_{\bv}\|^2_{L^2(\widehat{\pi}_{d,n})} + \frac{\kappa}{d}\|f\|_{\cH}^2,\qquad \cL_{\bv}^{\mathrm{irr}} \deq  \E_{(\bx,\bz)\sim \widehat{\pi}_{d,n}}\left[\bigl|\<\bv,\bz\> + \widehat{f}^\star_{\bv}(\rho\bx + \sigma\bz)\bigr|^2\right],
\end{equation}
and we identify $\widehat f^\star_\bv$ with its pullback to $L^2(\widehat\pi_{d,n})$. Let us emphasize that the decomposition \eqref{eq:train_loss_decomposition} is different from the standard bias-variance decomposition of supervised learning, which is taken around the population Bayes predictor. Here the decomposition is around the empirical Bayes denoiser $\widehat{f}^\star_{\bv}$, the unrestricted minimizer of the empirical risk. The reducible part $ \cL_{\bv}^{\mathrm{red}}$ is the error that optimization over the function class can remove. The irreducible part $\cL^{\mathrm{irr}}_\bv$ is the intrinsic noise of denoising under the empirical mixture; as we will see (Proposition~\ref{prop:bayes_opt_denoiser_linearization}), in high dimensions it is exponentially small---the empirical DSM problem is effectively noiseless, in contrast with noisy supervised learning.

\paragraph*{Gradient flow.}
Our primary estimator is the gradient flow of the (unregularized) empirical risk in the RKHS geometry, started from zero: with $\sft \geq 0$ as the training time,
\begin{equation}\label{eq:gradient_flow}
    \frac{\partial \widehat{f}_{\sft, \bv}}{\partial \sft} = -\cS^\ast (\cS \widehat{f}_{\sft,\bv} - \widehat{f}^\star_{\bv}), \qquad \widehat f_{0,\bv} = 0,
\end{equation}
which is the functional counterpart of training a lazily parametrized network by gradient descent on~\eqref{eq:empirical_risk_RKHS} with $\kappa=0$. Since $\|\cS\|_\op^2 = \|\widehat{\cK}\|_\op \leq \E_{\widehat\pi_{d,n}}[K(\rho\bx+\sigma\bz,\rho\bx+\sigma\bz)] < \infty$, the dynamics~\eqref{eq:gradient_flow} are globally well-posed, and applying $\cS$ yields the closed-form training-space solution
\begin{equation}\label{eq:gf_solution}
    \widehat{f}_{\sft, \bv} = \cS^\ast \int_0^{\sft} \exp(-(\sft-\sfs)\widehat{\cK})\widehat{f}^\star_{\bv}\, \de \sfs,
    \qquad
    \cS\widehat{f}_{\sft,\bv} = \bigl(\id_{L^2(\widehat{\pi}_{d,n})} - \exp(-\sft \widehat{\cK})\bigr)\widehat{f}^\star_{\bv}.
\end{equation}
Thus, gradient flow applies to the empirical Bayes target the spectral filter $\lambda \mapsto 1 - e^{-\sft\lambda}$ of the empirical kernel operator: components of $\widehat f^\star_\bv$ on eigenfunctions of $\widehat\cK$ with eigenvalue $\lambda$ are fitted on the timescale $1/\lambda$. The entire phenomenology of Section~\ref{sec:gf} unfolds from the spectrum of $\widehat\cK$ and the correlation of $\widehat f^\star_\bv$ with its eigenvectors. The reducible train error and the test error along the trajectory are
\begin{equation}\label{eq:gf_train_red_loss}
    \cL_{\bv}^{\mathrm{red}}(\widehat{f}_{\sft,\bv};0) = \big\|\exp(-\sft \widehat{\cK})\widehat{f}^\star_{\bv}\big\|_{L^2(\widehat{\pi}_{d,n})}^2,
\end{equation}
\begin{equation}\label{eq:gf_test_error}
    \cR_{\bv}(\widehat{f}_{\sft,\bv}) =  \big\|\cK\widehat{\cK}^{\dagger}\bigl(\id_{L^2(\widehat{\pi}_{d,n})} - \exp(-\sft \widehat{\cK})\bigr)\widehat{f}^\star_{\bv} - \sfz_\bv\big\|_{L^2(\pi_d)}^2,
    \qquad \sfz_\bv(\bx,\bz) \deq  -\<\bv,\bz\>,
\end{equation}
where $\widehat\cK^\dagger$ is the Moore--Penrose pseudoinverse. Throughout, coordinate flows are assembled into the vector-valued estimator $\widehat \boldf_{\sft} \deq  \sum_{i=1}^d \widehat f_{\sft,\bv_i}\bv_i$ (independent of the chosen basis, by linearity of the flow), and analogously $\widehat \boldf_{\kappa}$ for the ridge estimator~\eqref{eq:opt_denoiser} and $\widehat{\boldsymbol f}^{\star}$ for the vector-valued empirical Bayes denoiser~\eqref{eq:emp_bayes_opt_denoiser}.

\paragraph*{Ridge regression.}
In addition to gradient flow, we also study the \emph{kernel ridge regression} estimator. Minimizing $\cL_\bv(\cdot;\kappa)$ over $\cH$ for $\kappa>0$ gives the kernel ridge estimator, characterized by the normal equation $(\cS^\ast \cS + \frac{\kappa}{d} \id_{\cH})\widehat{f}_{\kappa,\bv} = \cS^\ast \widehat{f}^\star_{\bv}$, i.e.,
\begin{equation}\label{eq:opt_denoiser}
    \widehat{f}_{\kappa,\bv} = \cS^\ast\left(\widehat{\cK} + \frac{\kappa}{d}\id_{L^2(\widehat{\pi}_{d,n})}\right)^{-1}\widehat{f}^\star_{\bv},
\end{equation}
with train error
\begin{equation}\label{eq:opt_train_red_loss}
    \cL^{\mathrm{red}}_{\bv}(\widehat{f}_{\kappa,\bv};0) = \frac{\kappa^2}{d^2}\big\|\big(\widehat{\cK} + \frac{\kappa}{d}\id_{L^2(\widehat{\pi}_{d,n})}\big)^{-1}\widehat{f}^\star_{\bv}\big\|_{L^2(\widehat{\pi}_{d,n})}^2
\end{equation}
and test error
\begin{equation}\label{eq:opt_test_red_loss}
    \cR_{\bv}(\widehat{f}_{\kappa,\bv}) = \big\|\cK\big(\widehat{\cK} + \frac{\kappa}{d}\id_{L^2(\widehat{\pi}_{d,n})}\big)^{-1}\widehat{f}^\star_{\bv} - \sfz_\bv\big\|_{L^2(\pi_d)}^2 .
\end{equation}
Comparing the spectral filters $1-e^{-\sft\lambda}$ and $\lambda/(\lambda + \kappa/d)$ shows that gradient flow at time $\sft$ and ridge at penalty \(\kappa_{\mathrm{eff}} = d/\sft\) apply comparable shrinkage to eigencomponents. This is the classical early-stopping--ridge correspondence~\cite{yao2007early,raskutti2014early,ali2019continuous}, and it can be made quantitative along the whole trajectory (Lemma~\ref{lem:gf_ridge_comparison}). We use ridge in two ways: as a source of intuition and direct comparison with the supervised learning literature---each phase of gradient flow has a ridge analogue at the corresponding $\kappa_{\mathrm{eff}}$---and as a technical device in the proofs. The full ridge theory (exact risk asymptotics for constant, vanishing, and zero penalty) is developed in Appendix~\ref{app:ridge} and mirrors, statement by statement, the gradient flow results presented next.

\section{Learning the score along the gradient flow trajectory}
\label{sec:gf}

Fix a diffusion time $t \in (0,T]$ and write $(\rho,\sigma) = (\rho_t,\sigma_t)$. This section characterizes the gradient flow estimator $\widehat\boldf_\sft$ across all training times $\sft$, and identifies three phases with distinct statistical behavior: a \emph{generalization} phase at times $\sft \lesssim d^{1+c}$, in which the estimator behaves as a regularized linear model (Section~\ref{sec:gf-phase-I}); an \emph{overfitting} phase at polynomial times, in which the training loss vanishes while the test risk is pinned at the risk of the pure-noise score (Section~\ref{sec:gf-phase-II}); and a \emph{memorization} phase at super-polynomial times, at which the estimator finally reaches the empirical Bayes denoiser (Section~\ref{sec:gf-endpoint}). The mechanism behind the first two phases is a linearization of the kernel operator, developed in Section~\ref{sec:linearization}, which splits the estimator into a global linear component and localized nonlinear corrections around the training data.

Throughout, we work in the proportional high-dimensional regime and under a standard concentration condition on the data distribution.

\begin{assumption}[High-dimensional proportional limit]\label{ass:high_dim}
    The sample size $n$ and the data dimension $d$ grow to infinity, $n, d \to \infty$, such that their ratio converges to a strictly positive constant $n/d \to r \in (0, \infty)$.
\end{assumption}

\begin{assumption}[Regularity of the data distribution]\label{ass:input_dist}
    The data distribution $\pi_0 \in \mathscr{P}(\R^d)$ satisfies the following conditions:

    \begin{enumerate}[label={\normalfont(\alph*)}]
        \item \textbf{Mean and covariance:} The first and second moments of $\pi_0$ satisfy
        \[
            \E_{\bx\sim \pi_0}[\bx]=0,\qquad \E_{\bx\sim \pi_0}[\bx\bx^\sT]=\bSigma,
        \]
        and $\liminf_{d\to\infty}\frac{\Tr(\bSigma)}{d} \geq c>0$
        for some constant $c > 0$. \label{ass:moments_input_dist}

        \item \textbf{Logarithmic Sobolev inequality:} There exists a constant $C_{\mathrm{LSI}}\geq 1$ such that $\pi_0$ satisfies the log-Sobolev inequality
        \[
        \mathrm{Ent}_{\pi_0}(f^2) \leq 2 C_{\mathrm{LSI}} \E_{\pi_0}[\|\nabla f\|_2^2]
        \]
        for every $f\in W^{1,2}(\pi_0)$,\footnote{Here, $W^{1,2}(\pi_0)$ denotes the Sobolev space with respect to $L^2(\pi_0)$, i.e.\ the set of functions $f:\R^d\to\R$ satisfying $\|f\|_{L^2(\pi_0)}<\infty$ and $\|\|\nabla f\|_2\|_{L^2(\pi_0)}<\infty$.} where $\mathrm{Ent}_{\pi_0}(f^2) \deq  \E_{\pi_0}[f^2 \log f^2] - \E_{\pi_0}[f^2] \log \E_{\pi_0}[f^2]$.
    \end{enumerate}
\end{assumption}

Assumption~\ref{ass:input_dist}.(b)  should be viewed as a high-dimensional concentration assumption and
it is imposed mainly for convenience: we expect the results to extend to other standard high-dimensional models, for instance $\bx=\bSigma^{1/2}\bu$ with $\bu$ having independent sub-Gaussian entries. Apart from this concentration requirement, Assumption~\ref{ass:input_dist} constrains only the first two moments of $\pi_0$ and the data distribution may be multimodal and have intricate higher-order dependencies. Nevertheless, we will see that the RKHS estimators studied here depend only on the covariance $\bSigma$, and therefore do not learn  higher-order structure in the data distribution. Throughout, we write 
\begin{equation}
\tau_\bSigma \deq  \Tr(\bSigma)/d \geq c > 0,
\end{equation} 
for the average eigenvalue of $\bSigma$. In particular, under Assumption \ref{ass:input_dist}.(b),  \(\tau_\bSigma\le \|\bSigma\|_{\op}\le C_{\mathrm{LSI}}\). We arrange the \(n\) training inputs as the columns of the data matrix \(\bX\deq [\bx_{0,1},\ldots,\bx_{0,n}]\in\R^{d\times n}\), and denote their empirical covariance matrix by \(\widehat{\bSigma} \deq  \bX\bX^\sT/n\).

\subsection{Linearization of the kernel and localized bumps}
\label{sec:linearization}

In the moderate-time regime $\sft \lesssim d^{1 +c}$ with $c>0$ small enough, we assume the following\footnote{The different results in this section impose different additional conditions on the kernel. These conditions are introduced primarily to simplify the analysis and are not meant to be minimal or necessary conditions. We emphasize that all three assumptions are compatible and are satisfied for example by the kernels in Remark \ref{rmk:admissible-kernels}.}.

\begin{assumption}[Kernel regularity]\label{ass:kernel_func_constant_kappa}
    The kernel is a positive definite inner-product kernel of the form $K(\bx,\bx') = h(\<\bx,\bx'\>/d)$. The kernel function $h: \R \to \R$ is three times continuously differentiable, with $h(0) = h''(0) = 0$ and $h'(0) = 1$. Furthermore, there exists a constant $C\geq 1$ such that $|h'''(x)|\leq Ce^{C|x|}$ for all $x\in \R$. 
\end{assumption}

The condition $h'(0) = 1$ is a normalization convention and is without loss of generality. For simplicity, we also assume that $h(0)=h''(0)=0$, which eliminates the constant function spike from the kernel linearization. Remark~\ref{rmk:constant-contribution} explains how this assumption can be removed.

The key phenomenon we use in the first two phases is the \emph{linearization} of the kernel on \emph{delocalized} pairs of inputs. On pairs of inputs whose overlap $\<\bu,\bu'\>/d$ is $o_d(1)$, as is typical for independent noisy samples, the nonlinear kernel is indistinguishable from its linearization $h(x) \approx x$; in particular, this is always the case for new independent data at test time. The nonlinearity of $h$ survives only on \emph{coincident} inputs $\bu = \rho \bx + \sigma \bz$ and $\bu' = \rho \bx + \sigma \bz'$ that share the same data point $\bx$: there the overlap concentrates around $\rho^2 \tau_{\bSigma}$ and $h(\<\bu,\bu'\>/d) = h(\rho^2\tau_{\bSigma}) + o_d(1)$ retains the full nonlinearity of the kernel. Thus, in the high-dimensional linear regime  $n \asymp d$, the kernel operators behave as the \emph{effective} operators $\widehat{\cK}_{\mathrm{eff}}:L^2(\widehat{\pi}_{d,n}) \to L^2(\widehat{\pi}_{d,n})$ at train time, and $\cK_{\mathrm{eff}}:L^2(\widehat{\pi}_{d,n}) \to L^2(\pi_d)$ at test time, defined by
\begin{equation}\label{eq:K_hat_op}
    \widehat{\cK}_{\mathrm{eff}}f(\bx,\bz) \deq  \E_{(\bx',\bz')\sim \widehat{\pi}_{d,n}}\left[\frac{\<\rho\bx+\sigma\bz,\rho\bx'+\sigma\bz'\>}{d}f(\bx',\bz')\right] + \delta_n\,\E_{\bz'\sim \cN(0,\bI_d)}[f(\bx,\bz')],
\end{equation}
and
\begin{equation}
    \cK_{\mathrm{eff}}f(\bx,\bz) \deq  \E_{(\bx',\bz')\sim \widehat{\pi}_{d,n}}\left[\frac{\<\rho\bx+\sigma\bz, \rho\bx'+\sigma\bz'\>}{d}f(\bx',\bz')\right],
\end{equation}
where we denoted
\begin{equation}\label{eq:delta_n}
    \delta_n \deq  \frac{h(\rho^2\tau_\bSigma) - \rho^2\tau_\bSigma}{n}.
\end{equation}
The first term in~\eqref{eq:K_hat_op} is the operator of the \emph{linear} kernel $\<\bx,\bx'\>/d$; the second, present only on the empirical space, averages over the noise at a fixed training sample and is weighted by the kernel excess curvature $h(\rho^2\tau_\bSigma)-\rho^2\tau_\bSigma$ at that scale. Following \cite{ghorbani2021linearized}, we refer to $\delta_n$ as the denoising \emph{self-induced regularization} from the nonlinear part of the kernel. Positive definiteness of \(K\), together with
\(h(0)=0\) and \(h'(0)=1\), ensures that \(\delta_n\geq 0\) (see Remark~\ref{rmk:delta_n_nonnegativity}).

\begin{proposition}[Kernel operator linearization]\label{prop:linearization_kernel_op}
    Suppose that Assumptions~\ref{ass:high_dim},~\ref{ass:input_dist}, and~\ref{ass:kernel_func_constant_kappa} hold. Then, $\|\widehat{\cK} - \widehat{\cK}_{\mathrm{eff}}\|_{\op}\vee \|\cK - \cK_{\mathrm{eff}}\|_{\op}\prec d^{-\frac{3}{2}}$.
\end{proposition}

Proposition~\ref{prop:linearization_kernel_op} is the analogue, for denoising score matching, of the classical linearization of $n \times n$ inner-product kernel random matrices in the regime $n\asymp d$~\cite{el2010spectrum}. The key distinction is that our result holds at the operator level on the noisy, \emph{infinite-dimensional} empirical space $L^2(\widehat\pi_{d,n})$, where $\widehat\pi_{d,n}$ is a mixture of $n$ Gaussians, rather than $n$ points. In particular, the linearization must hold uniformly over functions of the Gaussian noise directions, requiring careful control of the kernel's power-series expansion; see the proof in Appendix~\ref{app:lin}.

Thus, in the proportional regime and at moderate training times, the sole effect of the kernel's nonlinearity is to add a per-sample regularization $\delta_n\E_{\bz'}[f(\bx,\bz')]$ to the effective linear operator. This is the DSM analogue of the additive self-induced ridge regularization identified for kernel interpolation in supervised learning~\cite{liang2020just,ghorbani2021linearized,misiakiewicz2024six}. However, the self-induced regularization here is not a ridge penalty (a multiple of the identity) on \(L^2(\widehat\pi_{d,n})\): it acts through the conditional noise-averaging operator \(f(\bx_{0,i},\bz)\mapsto \E_{\bz'}[f(\bx_{0,i},\bz')]\). 

To understand what the effective operator implies for the estimator, it is instructive to write down the RKHS and empirical risk minimization problem associated with $\widehat{\cK}_{\mathrm{eff}}$: the effective kernel function is
\begin{equation}\label{eq:def_K_eff_kernel_func}
    K_{\mathrm{eff}}(\bx_{0,i},\bz;\bx_{0,j},\bz') = \frac{1}{d}\<\rho\bx_{0,i}+\sigma\bz,\rho\bx_{0,j}+\sigma\bz'\> + n\delta_n \ind_{i = j}, 
\end{equation}
with associated RKHS (with $\bb=0$ when $\delta_n=0$)
\begin{equation}\label{eq:def_H_eff}
    \begin{aligned}
    &~\cH_{\mathrm{eff}} \deq   \Big\{f_{\bw,\bb} \in L^2(\widehat{\pi}_{d,n}) \; : \;  f_{\bw,\bb} (\bx_{0,i},\bz) = \<\bw, \rho \bx_{0,i} + \sigma \bz \> + b_i,  \; \bw \in \R^d, \bb \in \R^n\Big\}, \\
    &~\| f_{\bw,\bb} \|_{\cH_{\mathrm{eff}}}^2 \deq   d \| \bw \|_2^2 + \frac{1}{n \delta_n} \| \bb \|_2^2 . 
    \end{aligned}
\end{equation}
For simplicity, consider the ridge regression estimator $\widehat f_{\kappa,\bv}$, which is related to the gradient flow estimator at $\sft = d/\kappa$. For moderate regularization $\kappa \gtrsim d^{-c}$ (that is, $\sft \lesssim d^{1 + c}$), the ridge empirical minimization problem \eqref{eq:emp_risk_decoupled} is well approximated by
\begin{equation}\label{eq:effective-train-risk-linear}
    \begin{aligned}
        \cL_{\bv} (\widehat f_{\kappa,\bv};\kappa) = &~ \min_{f \in \cH} \cL_{\bv}(f;\kappa) 
        \approx  \min_{f \in \cH_{\mathrm{eff}} } \frac{1}{n}\sum_{i=1}^n \E_{\bz}\left[\bigl|f(\rho\bx_{0,i} + \sigma\bz) + \<\bv,\bz\>\bigr|^2\right]  + \frac{\kappa}{d} \|f\|_{\cH_{\mathrm{eff}}}^2\\
        =&~ \min_{\bw \in \R^d} \frac{1}{n}\sum_{i=1}^n \min_{b_i \in \R} \Big\{\E_{\bz}\left[\bigl|\< \bw, \rho\bx_{0,i} + \sigma\bz\> + b_i + \<\bv,\bz\>\bigr|^2\right] + \frac{\kappa}{d \delta_n} b_i^2\Big\} + \kappa \|\bw\|_2^2,
    \end{aligned}
\end{equation}
and the test error in this regime (with $(\widehat \bw,\widehat \bb)$ the minimizer of the ERM over $\cH_{\mathrm{eff}}$) is
\begin{equation}\label{eq:effective-test-risk-linear}
    \cR_{\bv} ( \widehat f_{\kappa,\bv}) \approx \E_{(\bx,\bz) \sim \pi_d} \left[\bigl|\< \widehat \bw,\rho\bx + \sigma\bz\> + \<\bv,\bz\>\bigr|^2\right],
\end{equation}
where 
\begin{equation}\label{eq:solution-effective-ridge}
\widehat \bb = - \frac{d\delta_n}{\kappa + d\delta_n}\rho \bX^\sT \widehat \bw,\qquad   \widehat \bw = - \sigma \Big[ \kappa \bI_d + \sigma^2 \bI_d+ \frac{\rho^2 \kappa}{\kappa + d \delta_n} \widehat \bSigma\Big]^{-1}  \bv.
\end{equation}

Three features of the effective problem \eqref{eq:effective-train-risk-linear}--\eqref{eq:effective-test-risk-linear} and its minimizer \eqref{eq:solution-effective-ridge} drive the phenomenology in the next two subsections:
\begin{itemize}
    \item[(a)] The kernel nonlinearity effectively behaves as a per-sample offset $b_i$. In the original RKHS, $b_i$ corresponds to a `localized bump' on the tube around $\bx_{0,i}$ that allows the model to fit the conditional denoising target $-\E[\<\bv,\bz\>| \bu =\rho \bx_{0,i} + \sigma \bz]$.  The fitted model $\widehat f_{\kappa,\bv}$ is thus a linear combination of a global linear predictor $\<\widehat \bw, \rho \bx + \sigma \bz\>$ and $n$ localized nonlinear corrections $b_i$.
    The inverse curvature $(h(\rho^2 \tau_\bSigma) - \rho^2 \tau_\bSigma)^{-1}$ plays the role of a ridge penalty on these offsets.

    \item[(b)] For $\kappa = o_d(1)$ (but still large enough for the linearization approximation to hold) and $\delta_n >0$, the minimizer \eqref{eq:solution-effective-ridge} becomes $\widehat \bw = - \bv / \sigma + o_{d,\P}(1)$ and $b_i = - \rho \< \widehat \bw, \bx_{0,i}\>+ o_{d,\P}(1)$, and achieves vanishing train error $\cL_\bv (\widehat f_{\kappa,\bv}) = o_{d,\P}(1)$ and test error $\cR_{\bv}(\widehat f_{\kappa,\bv}) = \tfrac{\rho^2}{\sigma^2} \tau_{\bSigma}+ o_{d,\P}(1)$. The linear part corresponds to the score of the pure noise component $-\<\bv,\bz\>$ and the bumps cancel the resulting bias $-\rho \<\bv,\bx_{0,i} \>/\sigma$ on each tube. In fact, this solution (and test error) persists for all polynomially vanishing $\kappa= d^{-C}$ (equivalently, polynomial time $\sft = d^{1+C}$), and memorization happens only for superpolynomially vanishing $\kappa$ (superpolynomial $\sft$).
    
    \item[(c)] At test time, the RKHS score model reduces effectively to a linear score, even though $\widehat f_{\kappa,\bv}$ is nonlinear. Moreover, unlike in the supervised setting \eqref{eq:supervised-additive-ridge}, the kernel nonlinear contribution $\delta_n$ does not act as an additive ridge regularizer. Instead, it shrinks the empirical covariance matrix $\widehat \bSigma$ by the factor $\kappa/(\kappa + d \delta_n)$. Thus, whenever $\delta_n >0$, the solution $\widehat \bw$ in \eqref{eq:solution-effective-ridge} loses its dependence on $\widehat \bSigma$ as $\kappa \to 0$. This shrinkage also explains the observation following Theorem \ref{thm:gf_train_test_equivalents_moderate_time} that larger $\delta_n$ lead to an earlier onset of overfitting and covariance forgetting; see Figure \ref{fig:delta_n_and_c_effective_flow} for a numerical illustration.
\end{itemize}

We further discuss the consequences of the kernel linearization for estimation in Section \ref{sec:gf-phase-I}.

\begin{proposition}[Vanishing irreducible error]\label{prop:bayes_opt_denoiser_linearization}
    Suppose that Assumptions~\ref{ass:high_dim} and~\ref{ass:input_dist} hold. Let $\bv\in\R^d$ with $\|\bv\|_2=1$. Then there exists a constant $c>0$ such that
    \[
        \|\widehat{f}^\star_\bv - \sfz_\bv\|_{L^2(\widehat{\pi}_{d,n})} \leq e^{-d^{c}}
    \]
    with very high probability, where $\sfz_\bv(\bx,\bz) = -\<\bv,\bz\>$.
\end{proposition}

Two consequences follow immediately from this proposition (proved in Appendix~\ref{app:lin}). First, the irreducible part of the training risk $\cL^{\mathrm{irr}}_\bv$ in~\eqref{eq:train_loss_decomposition} is exponentially small: the empirical regression problem is effectively noiseless around the empirical Bayes denoiser. Second, the informative part of $\widehat f^\star_\bv$ (where it deviates from $-\<\bv,\bz\>$, encoding the locations $\bx_{0,i}$) lives at exponentially small $L^2(\widehat\pi_{d,n})$-scale. The entire training signal about the data distribution is carried not by the target $\widehat f^\star_\bv$ but by the design geometry $\{\rho\bx_{0,i}+\sigma\bz\}$ on which the regression is performed. This is in sharp contrast with supervised learning and underlies the different phenomenology of DSM.

\subsection{Phase I: linearized dynamics and generalization}
\label{sec:gf-phase-I}

We now describe the gradient flow trajectory on timescales $\sft \lesssim d^{1+c}$ for a small constant $c > 0$. On these timescales the dynamics are governed by the linearized operator $\widehat\cK_{\mathrm{eff}}$, and the risks concentrate on deterministic equivalents expressed through the spectrum of the sample covariance. 

The deterministic equivalents are built from standard self-consistent resolvent approximations. Let $\C_+ \deq  \{z \in \C : \Im(z) > 0\}$, and define $\bM: \C_+ \to \{\bA\in \C^{d \times d}:\Im[\bA]\succ 0\}$ and $\mu: \C_+ \to \C_+$ as the unique analytic solutions of
\begin{equation}\label{eq:resolvent_fixed_point}
    \bM(z) = \left( \frac{\bSigma}{1 + \mu(z)} - z\bI_d \right)^{-1}, \qquad \mu(z) = \frac{1}{n}\Tr\left(\bSigma \bM(z)\right).
\end{equation}
Since $\bM$ is a matrix-valued Herglotz function~\cite{gesztesy2000matrix}, there exists a compactly supported matrix-valued probability measure $\bOmega$ on Borel subsets of $\R$ such that
\begin{equation}\label{eq:resolvent_fixed_point_measure}
    \bM(z) = \int_\R \frac{1}{\lambda - z}\, \bOmega(\de \lambda);
\end{equation}
$\bOmega$ is the deterministic equivalent of the spectral measure of the sample covariance $\widehat{\bSigma}$, resolved along the eigenspaces of $\bSigma$~\cite{couillet2011random,latourelle2026dyson}.

In the linearized regime, the gradient flow dynamics decouple along this spectral decomposition, and in each spectral direction $\lambda$, they reduce to a two-dimensional linear system coupling the signal component of the estimator (the coefficient of $\<\bx_0,\cdot\>$) with a `noise component' (associated to $\<\bz,\cdot\>$); the nonlinearity of the kernel enters only through the self-induced regularization $\delta_n$. For every \(x \geq 0\), define
\begin{equation}\label{eq:gf_lin_profiles}
        \psi(x,\sft)
        \deq -\frac{1}{\sigma}\be_2^\sT
        \left(\bI_2-e^{-\sft\bQ(x)}\right)\be_2, \qquad \bQ(x)\deq 
        \begin{bmatrix}
            \dfrac{\rho^2x}{d}+\delta_n
            & \dfrac{\rho\sigma\sqrt{x}}{d}
            \\[0.75em]
            \dfrac{\rho\sigma\sqrt{x}}{d}
            & \dfrac{\sigma^2}{d}
        \end{bmatrix},
\end{equation}
and the resulting deterministic equivalents for the train and test errors
\begin{equation}\label{eq:train_test_error_equivalents_gf}
    \begin{aligned}
     \cL_{\mathrm{det}}^{(\sft)} \deq &~  \int_\R \ell_{\mathrm{det}}^{(\sft)}(\lambda)\,\frac{1}{d}\Tr(\bOmega(\de \lambda)), \qquad\qquad \ell_{\mathrm{det}}^{(\sft)}(x)
        \deq \big\|e^{-\sft\bQ(x)}\be_2\big\|_2^2,\\
    \cR_{\mathrm{det}}^{(\sft)} \deq &~ \int_\R \psi(\lambda,\sft)^2\,\frac{1}{d}\Tr(\bSigma_\sigma \bOmega(\de \lambda)) + 2\sigma\int_\R \psi(\lambda, \sft)\, \frac{1}{d}\Tr(\bOmega(\de \lambda)) + 1,
    \end{aligned}
\end{equation}
where $\bSigma_\sigma \deq  \rho^2\bSigma + \sigma^2\bI_d$ is the covariance of the noisy data and $\be_2 = (0,1)^\sT$.
We denote by \(\gamma_-(x)\leq \gamma_+(x)\) the eigenvalues of \(\bQ(x)\), which are the relaxation rates of the two-dimensional linear system in spectral direction \(x\), and by \(\ell_{\mathrm{det}}^{(\sft)}(x)\) and \(\psi(x,\sft)\) the corresponding profiles for the training and test errors, respectively. Remark~\ref{rmk:gf_explicit_profiles} gives explicit expressions for these quantities.

On a new test sample, the estimator at time $\sft$ acts as the \emph{linear} shrinkage $\widehat{\boldf}_{\sft}(\bu) \approx \psi (\widehat \bSigma, \sft)\, \bu$ where $\psi(\cdot,\sft)$ acts spectrally. The expression $\cR_{\mathrm{det}}^{(\sft)} $ is then obtained by replacing $\widehat \bSigma$ by its deterministic equivalent. 


\begin{figure}[htb]
    \centering
    \includegraphics[width=1\textwidth]{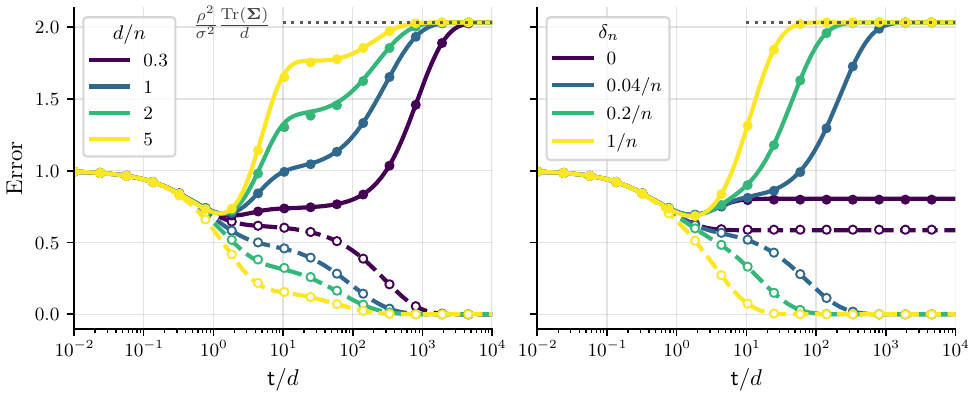}
    \caption{Training and test errors as functions of the normalized gradient-flow time $\sft/d$ for $n=240$, diffusion time $t=0.2$, and Gaussian data $\pi_0=\cN(0,\bI_d)$. Dashed and solid lines correspond, respectively, to the training and test errors of the finite-dimensional linearized estimator obtained by solving~\eqref{eq:effective-train-risk-linear} with gradient flow. Open and filled circles show the corresponding deterministic equivalents from Theorem~\ref{thm:gf_train_test_equivalents_moderate_time}. The dotted horizontal line indicates the asymptotic test-error plateau $\tfrac{\rho^2}{\sigma^2}\frac{\Tr(\bSigma)}{d}$. 
    \textbf{Left:} We use $h(x)=x+x^4/12$ and vary $d/n$ while holding $n$, and hence $\delta_n$, fixed. This changes the relaxation rates $\gamma_-$, shifting the onset of overfitting.
    \textbf{Right:} We fix $d=120$ and vary the effective linearization parameter $\delta_n$. When $\delta_n=0$, the estimator is linear and cannot interpolate the training loss. By contrast, whenever $\delta_n>0$, the training error converges to zero and the test error approaches the dotted interpolation plateau. Because $\gamma_-$ increases with $\delta_n$, larger $\delta_n$ leads to an earlier transition to overfitting.}
    \label{fig:delta_n_and_c_effective_flow}
\end{figure}

\begin{theorem}[Phase I: deterministic equivalents along gradient flow]\label{thm:gf_train_test_equivalents_moderate_time}
    Suppose that Assumptions~\ref{ass:high_dim}, \ref{ass:input_dist}, and \ref{ass:kernel_func_constant_kappa} hold.  Then, for every $\epsilon, D>0$, there exists a constant $C > 0$ such that for every $\sft \geq 0$, 
    \[
         |\cL(\widehat{\boldf}_{\sft};0) - \cL_{\mathrm{det}}^{(\sft)}| \lesssim  \left(1 + \frac{\sft}{d}\right) d^{\epsilon-\frac{1}{2}},\qquad  |\cR(\widehat{\boldf}_{\sft}) - \cR_{\mathrm{det}}^{(\sft)}| \lesssim  \left(1 + \frac{\sft}{d}\right)^3 d^{\epsilon-\frac{1}{2}}
    \]
    with probability at least $1-Cd^{-D}$. In particular, if $\sft \leq d^{1 + \alpha}$ for a constant $\alpha <1/6$ and $\eps < 1/2 - 3 \alpha$, then the right hand sides are $o_d(1)$.
\end{theorem}

The proof is given in Appendix~\ref{app:gf_train_test_equivalents_moderate_time}. We next highlight the main phenomenology of gradient flow at moderate times:
\begin{itemize}
    \item[(a)]     The relaxation rates satisfy $\gamma_+(x) \approx \tfrac{\rho^2 x + \sigma^2}{d} + \delta_n$ and $\gamma_-(x) \approx \tfrac{\sigma^2\delta_n}{d\,\gamma_+(x)}$. Each spectral direction of the sample covariance is learned on the timescale $1/\gamma_+(\lambda) \asymp d/(\rho^2\lambda + \sigma^2 + d\delta_n)$: top principal directions first, bulk directions later, and the timescales are modulated by the noise level---at high noise ($\rho \to 0$) all directions relax together, while at low noise the training schedule follows the spectral hierarchy of $\widehat\bSigma$. This was already described for purely linear models in~\cite{wang2026analytical,wang2026random}. Here we show that, to leading order, the kernel nonlinearity simply shifts the timescales by the self-induced regularization $\delta_n$.

    \item[(b)]  For the linear kernel ($\delta_n = 0$ and $\gamma_- \equiv 0$), gradient flow converges to the linear ridgeless fit without ever interpolating. On the other hand, any kernel curvature $h(\rho^2\tau_\bSigma) - \rho^2\tau_\bSigma > 0$ opens a second channel with relaxation rate $\gamma_-$, proportional to the self-induced regularization $\delta_n$, through which the training error continues to decrease toward zero. The data eigendirections are learned at timescale $1/\gamma_+ \approx d$, while the localized bumps are fitted at timescale $1/\gamma_- \approx \delta_n^{-1} \approx n$. The separation between the two families of timescales is the mechanism behind the two-stage training curves that was observed and analytically described for random-feature models in \cite{bonnaire2026diffusion}. Here, we sharpen this picture by characterizing how this timescale separation depends on the data covariance $\bSigma$ and the kernel curvature $\delta_n$.

    \item[(c)]     At any fixed $\lambda$ in the bulk, the shrinkage symbol $\psi(\lambda,\sft)$ is monotonically decreasing in $\sft$, from $\psi(\lambda,0)=0$ to
    \[
        \lim_{\sft \to \infty}\psi(\lambda,\sft) = -\frac{1}{\sigma} \qquad (\delta_n > 0).
    \]
    Along the way, it crosses the pointwise-optimal shrinkage $\psi^{\mathrm{opt}}(\lambda) = -\frac{\sigma\,\Tr(\bOmega(\de\lambda))}{\Tr(\bSigma_\sigma \bOmega(\de\lambda))}$, which minimizes the integrand of $\cR_{\mathrm{det}}^{(\sft)}$ in~\eqref{eq:train_test_error_equivalents_gf}. The contribution to the test risk of each eigendirection therefore descends toward the risk of the best linear denoiser and then \emph{climbs} as the fit overshoots toward the constant $-1/\sigma$. The limit $-\bu/\sigma^2$ of $\widehat{\boldf}_{\sft}/\sigma$ is the rescaled score of the pure-noise distribution $\cN(0, \sigma^2\bI_d)$. Trained to convergence within the linearized description, the model \emph{forgets the covariance it had learned} and treats all of its input as noise. Substituting $\psi \equiv -1/\sigma$ into~\eqref{eq:train_test_error_equivalents_gf} gives $\cR =  \frac{\rho^2}{\sigma^2}\tau_{\bSigma} (1 + o(1))$, which is the interpolation plateau of Phase II described below. The optimal early-stopping window is the interval between the covariance-learning descent and this overshoot; its optimum is located at $\sft^\star \asymp d$, with a constant determined by $\bSigma$, $\sigma$, and $\delta_n$ through~\eqref{eq:train_test_error_equivalents_gf}. In particular, since $\gamma_-$ is nondecreasing in $\delta_n$, larger curvature shortens the good generalization window: more strongly nonlinear kernels overfit earlier.
\end{itemize}

We illustrate the train and test error dynamics during Phase I in Figure~\ref{fig:delta_n_and_c_effective_flow}. 

\subsection{Phase II: vanishing error and overfitting}
\label{sec:gf-phase-II}

Beyond the timescale $d^{1+c}$, the operator-level kernel linearization used in Theorem~\ref{thm:gf_train_test_equivalents_moderate_time} no longer applies. Nevertheless, its approximation remains valid for the explicit gradient flow solution~\eqref{eq:gf_solution}. We establish this directly by showing that the kernel nonlinearity fits \emph{localized bumps} around the individual training samples. Our analysis of this regime requires the following additional condition on the kernel.

\begin{assumption}[Kernel high-degree coefficients]\label{ass:kernel_func_vanishing_kappa}
    The kernel is an inner-product kernel of the form $K(\bx,\bx') = h(\<\bx,\bx'\>/d)$. The base function $h: \R \to \R$ admits a power series expansion $h(x)=\sum_{k=0}^{\infty}\alpha_k x^k$ with non-negative coefficients $\alpha_k\geq 0$ for all $k \geq 0$, with $\alpha_0=\alpha_2=0$ and $\alpha_1=1$. Furthermore, there exists an even integer $k_0 \geq 6$ and a constant $c>0$ such that $\alpha_k > c$ for all $k \in \{k_0, \dots, 2k_0\}$.

    Finally, there exists a constant $\zeta\in (0,1)$ such that the higher-order coefficients satisfy the summability condition
    \begin{equation}
        \sum_{k=3}^{\infty}k^2\alpha_k C_{\zeta,k}^2 <\infty,\qquad C_{\zeta,k} \deq  (2C_{\mathrm{LSI}}^2((4k - \ln \zeta + 2) k)^2)^{\frac{k}{2}} (4k - \ln \zeta).
    \end{equation}
\end{assumption}

Under this assumption, we show that using the kernel's high-degree components, one can fit a family of \emph{localized bump functions} (Appendix~\ref{app:localized_bumps}): elements $f^{\mathrm{loc}}_{m,\bv} \in \cH$ which approximate the target $-\<\bv,\bz\>$ \emph{on the union of the training tubes} while remaining bounded in RKHS norm.

\begin{proposition}[Localized interpolants]\label{prop:localized_bumps_properties}
    Suppose that Assumptions~\ref{ass:high_dim},~\ref{ass:input_dist}, and~\ref{ass:kernel_func_vanishing_kappa} hold. Let $\bv\in\R^d$ be such that $\|\bv\|_2=1$. Let $k_0 \geq 6$ be the even integer guaranteed by Assumption~\ref{ass:kernel_func_vanishing_kappa} and denote $k_0 = 2m+6$. Then, there exists $f^{\mathrm{loc}}_{m,\bv} : \R^d \to \R$ such that
    \[
         \cL_{\bv}(f^{\mathrm{loc}}_{m,\bv};0) \prec  d^{-k_0+2},\qquad  \|f^{\mathrm{loc}}_{m,\bv}\|_\cH^2 \prec d.
    \]
\end{proposition}

It is worth emphasizing that this result is much stronger than the interpolation of $n$ labels in supervised learning. The empirical DSM objective integrates over the noise: driving it to zero requires matching the target \emph{as a function} of a mixture of $n$ $d$-dimensional Gaussian vectors, i.e., interpolating in $L^2(\widehat\mu_{d,n})$ rather than at finitely many points. 

Combining Proposition~\ref{prop:localized_bumps_properties} with a comparison between gradient flow and ridge along the trajectory (Lemma~\ref{lem:gf_ridge_comparison}) yields the description of the second phase.

\begin{theorem}[Phase II: overfitting]\label{thm:gf_train_test_equivalents_extended_time}
    Suppose that Assumptions~\ref{ass:high_dim},~\ref{ass:input_dist}, and~\ref{ass:kernel_func_vanishing_kappa} hold. Suppose further that there exists an $\epsilon>0$ such that $d^{1+\epsilon}\lesssim \sft\lesssim d^{k_0-1}$, where $k_0$ is given in Assumption~\ref{ass:kernel_func_vanishing_kappa}. Then, for every constant $c \in (0, (\epsilon \wedge 1)/2)$, there exist constants $C, c' > 0$ such that
    \[
        \cL(\widehat{\boldf}_{\sft};0)\prec \frac{d}{\sft}=o_d(1),
    \]
    and
    \[
        \left| \cR(\widehat{\boldf}_{\sft}) - \frac{\rho^2}{\sigma^2 } \frac{\Tr(\bSigma)}{d} \right| \lesssim (C_\zeta \vee 1) d^{c-\frac{\epsilon \wedge 1}{2}}  + \sqrt{1+(C_\zeta \vee 1) d^{c-\frac{\epsilon \wedge 1}{2}}}\,d^{\frac{c}{2}-\frac{\epsilon \wedge 1}{4}}
    \]
    with probability at least $1 - \zeta - e^{-d^{c'}}$ for all $d \geq C$, where $C_\zeta$ is defined in~\eqref{eq:C_delta_def}.
\end{theorem}

The proof is given in Appendix~\ref{app:gf_train_test_equivalents_extended_time}. In the polynomial-time regime of Phase II, the training error continues to vanish at rate $d/\sft$, while the population risk plateaus at $\frac{\rho^2}{\sigma^2}\tau_{\bSigma}$, which is the risk of the rescaled pure-noise score $-\bu/\sigma^2$. The gap between the two is carried by the localized bumps which are invisible to the population risk. The model has overfitted to the training data, but the overfitting is confined to an exponentially thin neighborhood of the training tubes, and the population risk is determined by the pure-noise score, not by the covariance fit.  Whether the overfitting is benign is set by the noise level $\rho_t^2/\sigma_t^2 = 1/\sigma_t^2 - 1$: at high noise the plateau is close to the best predictor and overfitting is benign, while at low noise the plateau exceeds the risk of the zero estimator and overfitting is detrimental. Note that in either case, the plateau is flat: throughout the polynomial window, further training changes the estimator only inside the tubes and leaves the population risk unchanged.

\subsection{Phase III: the memorization endpoint}\label{sec:gf-endpoint}

The plateau of Phase II describes a model that has interpolated the training objective while behaving, at test points, as the pure-noise score: overfitted, but far from the empirical Bayes denoiser. The final phase of training, which occurs at super-polynomial times, is the memorization phase. We characterize the gradient flow endpoint in fixed dimension, and we give a closed-form expression for its population risk in the high-dimensional limit. In this phase, we assume that the kernel has all higher-degree coefficients positive.

\begin{assumption}[Kernel high-degree tail]\label{ass:universal-kernel}
    There exists an even integer $k_0 \geq 0$ such that \(K(\bx,\by)=h(\< \bx,\by\>/d)\) with \(h(x)=\sum_{k=0}^\infty \alpha_k x^k\), \(\alpha_k\geq 0\) for every \(k\geq 0\), and \(\alpha_k>0\) for all \(k \geq k_0\).
\end{assumption}

Under this assumption, the gradient flow estimator $\widehat \boldf_\sft$ converges to the empirical Bayes denoiser $\widehat \boldf^\star$ in $L^2(\widehat\pi_{d,n})$ as $\sft \to \infty$. This convergence, however, does not automatically imply convergence of the test error (in $L^2(\pi_d)$): we show in Appendix \ref{sec:ridgeless-instability} that, without further assumptions, the gradient flow estimator can develop a tail instability as $\sft \to \infty$, with $\cR (\widehat \boldf_\sft)$ diverging along the trajectory.

Instead, we add a mild post-processing step that prevents such instabilities for any kernel satisfying Assumption \ref{ass:universal-kernel} and any $\pi_0$ with bounded second moment. Let $A$ be a compact convex set containing all scaled samples $\{\rho \bx_{0,i}\}_{i \in [n]}$, and let $\proj_A$ denote the Euclidean projection onto $A$. We define the post-processed estimator by
\begin{equation}\label{eq:gf-endpoint-postprocessing}
\widetilde \boldf_\sft (\bu) \deq  \frac{\proj_A D_{\widehat \boldf_\sft} (\bu) - \bu}{\sigma},
\end{equation}
where $D_{\widehat \boldf_\sft} (\bu) \deq  \bu + \sigma \widehat \boldf_\sft (\bu)$. This is a mild post-processing step; its form is chosen to preserve the empirical Bayes denoiser interpolation limit, that is, $\widehat \boldf^\star (\bu) = (\proj_A D_{\widehat \boldf^\star} (\bu) - \bu)/\sigma$.

\begin{proposition}[Gradient flow converges to the empirical Bayes denoiser]\label{prop:gf_fixed_d_endpoint}
Fix $d,n$, the training samples \(\{\bx_{0,i}\}_{i=1}^{n}\), and $\sigma >0$. Let
$\widehat \boldf_{\sft}$ solve the unregularized gradient flow~\eqref{eq:gradient_flow}
from initialization $0$. Under Assumptions~\ref{ass:input_dist} and~\ref{ass:universal-kernel}, $\cH$ is dense in $L^2(\widehat{\mu}_{d,n})$ and $L^2 (\mu_d)$, and the gradient flow converges to the empirical Bayes denoiser:
\[
    \cS_0\widehat \boldf_{\sft}
    \longrightarrow
    \widehat \boldf^\star
    \qquad\text{in }L^2(\widehat\mu_{d,n};\R^d)
    \qquad\text{as }\sft\to\infty.
\]
Furthermore,
\[
    \liminf_{\sft\to\infty}\cR(\widehat \boldf_{\sft})
    \geq
    \cR(\widehat \boldf^\star), \qquad  \lim_{\sft\to\infty}\cR (\widetilde\boldf_{\sft})
    = \cR(\widehat \boldf^\star).
\]
\end{proposition}

This proposition holds for any $d,n$. Next, we characterize the risk of the empirical Bayes denoiser in the linear high-dimensional scaling regime. 

\begin{theorem}[The memorization infinite-time limit]\label{thm:ridgeless_main}
    Assume Assumptions \ref{ass:high_dim} and~\ref{ass:input_dist}, and $\sigma >0$ hold.  Then, the risk of the empirical Bayes denoiser $\widehat \boldf^\star$ is given by
    \begin{equation}
    \left|\cR(\widehat \boldf^\star)
    - 2 \frac{\rho^2}{\sigma^2} \frac{\Tr(\bSigma)}{d}\right| \prec \frac{1}{\sqrt{d}}.
    \end{equation}
    If we further assume Assumption \ref{ass:universal-kernel}, the post-processed gradient flow estimator \eqref{eq:gf-endpoint-postprocessing} satisfies
    \begin{equation}
        \left|\lim_{\sft \to \infty} \cR(\widetilde \boldf_{\sft}) - 2 \frac{\rho^2}{\sigma^2} \frac{\Tr(\bSigma)}{d}\right| \prec \frac{1}{\sqrt{d}}.
    \end{equation}
\end{theorem}

The proofs of the above results are given in Appendix~\ref{app:interpolation-limit}. The test risk of the empirical Bayes denoiser is \emph{twice} the interpolation plateau of Phase II, and it is worth explaining where this doubling comes from. Off-sample, at a fresh noisy point $\bu = \rho\bx + \sigma\bz$, the softmax weights of the empirical score undergo a winner-take-all collapse: $\sum_i \omega_i(\bu)^2 \to 1$ (Lemma~\ref{lem:wta-under-lsi}), so the empirical Bayes denoiser output is $(\rho\bx_{0,i^\star(\bu)}-\bu)/\sigma$  with $\bx_{0,i^\star(\bu)}$ being the single training point  nearest to $\bu/\rho$. Its error thus has two components of identical size, $\|\bx\|^2/d \approx \tau_{\bSigma}$ and $\|\bx_{0,i^\star}\|^2/d \approx \tau_{\bSigma}$, which are near orthogonal in high dimensions, while Phase II's plateau only has the first component.

\section{Generation: the implicit bias of the reverse process}
\label{sec:reverse}

Section~\ref{sec:gf} characterizes the score model learned by the DSM regression problem at each noise level, with sharp asymptotic train and test errors. In generative modeling, however, we are interested in the distribution produced by composing the learned denoisers $\{\widehat \boldf_t\}_{t \in (0,T]}$ through the reverse dynamics~\eqref{eq:intro-reverse_SDE}. The standard reductions from score accuracy to sampling accuracy \cite{chen2023sampling,benton2024nearly} require the $L^2$ score error to be small, whereas in our setting, the score error is of constant order and these bounds are not very informative. This section instead analyzes the reverse SDE driven by the trained score directly, focusing on the early-stopped (non-memorized) regime of gradient flow. We show that, \emph{with high probability along the reverse process, only the linearized part of the learned score is ever evaluated}. The generated distribution is effectively Gaussian, with covariance only depending on the data through the empirical covariance $\widehat \bSigma$. The nonlinear part of the score---the bumps that overfit the samples---is invisible to the sampler.

\subsection{The learned reverse process and its linearized approximation}
\label{sec:reverse-empirical-linearized}

Reintroducing the diffusion time, suppose that for each $t \in (0,T]$ we train by gradient flow at the noise level $(\rho_t,\sigma_t)$ for time $\sft_t$, and write $\widehat \boldf_t\deq \widehat \boldf_{\sft_t,t}$ and $\widehat \bS_t\deq \widehat \boldf_t/\sigma_t$. The \emph{learned reverse process} is
\begin{equation}\label{eq:emp_reverse_SDE}
    \d\widehat{\bx}_t = -\bigl(\widehat{\bx}_t + 2\widehat{\bS}_t(\widehat{\bx}_t)\bigr)\,\d t + \sqrt{2}\,\d\bar{\bB}_t, \qquad \widehat{\bx}_T \sim \cN(0, \bI_d),
\end{equation}
run backward on $[t_0, T]$, where $t_0 \in (0,T)$ is the usual early-stopping time of the sampler and $\bar\bB_t$ is a Brownian motion independent of the training data. We stress that the drift of~\eqref{eq:emp_reverse_SDE} is a \emph{nonlinear} random field: it is built from the trained gradient flow estimator, including the bumps. We assume the training-time schedule \(t\mapsto\sft_t\) is deterministic and piecewise continuous on \([t_0,T]\).



For the pathwise comparison, we replace the trained gradient-flow score by a linearized gradient-flow solution. For every \(t>0\), let \(\widetilde{\cK}_{t}:L^2(\widehat\pi_{d,n})\to L^2(\widehat\pi_{d,n})\) be a modification of the effective empirical kernel operator \(\widehat\cK_{\mathrm{eff}}\) at noise level \(t\) that replaces the normalized trace \(\tau_\bSigma\) in the self-induced regularization \(\delta_n\) by the sample norm of the input:
\begin{equation}\label{eq:sample_corrected_effective_operator}
    \widetilde{\cK}_t f(\bx_{0,i},\bz)
    \deq \frac{1}{n}\sum_{j=1}^n \E_{\bz'}\left[
    \frac{\<\rho_t\bx_{0,i}+\sigma_t\bz,
    \rho_t\bx_{0,j}+\sigma_t\bz'\>}{d}
    f(\bx_{0,j},\bz')\right]+\widetilde\delta_{i,t}\E_{\bz'}[f(\bx_{0,i},\bz')],
\end{equation}
where
\begin{equation}
    \widetilde\delta_{i,t} \deq  \frac{h(\rho_t^2\|\bx_{0,i}\|_2^2/d)-\rho_t^2\|\bx_{0,i}\|_2^2/d}{n}
\end{equation}
for \(i\in[n]\). By Remark~\ref{rmk:delta_n_nonnegativity}, applied at
\(\rho^2_t\|\bx\|^2/d\geq0\), one has \(\widetilde\delta_{i,t}\geq0\) for all \(i\in[n]\) and \(t>0\).
Then, for a unit vector $\bv$, define the \emph{linearized score estimator} coordinate-wise as
\begin{equation}\label{eq:linearized_score_estimator}
    \widetilde{S}_{t,\bv}(\bx)
    \deq \frac1{\sigma_t}\left[\cS^\ast_{\mathrm{eff},t}
    \int_0^{\sft_t}\exp(-\sfs\widetilde{\cK}_{t})\sfz_{\bv}\,\de\sfs\right](\bx),
\end{equation}
where $\cS^\ast_{\mathrm{eff},t}:L^2(\widehat\pi_{d,n})\to\cH$ is the effective adjoint,
\begin{equation}
    \cS^\ast_{\mathrm{eff},t}(g) \deq  \E_{(\bx,\bz)\sim \widehat{\pi}_{d,n}}\left[\frac{\< \cdot\,, \rho_t \bx + \sigma_t \bz\>}{d}\,g(\bx,\bz)\right].
\end{equation}
Unwinding the definitions, we show in Lemma~\ref{lem:reverse_linearized_check_score_representation} that \(\widetilde\bS_t(\bx)=\widetilde\bPsi_t\bx/\sigma_t\) for some bounded negative semidefinite matrix \(\widetilde\bPsi_t\in \R^{d\times d}\) depending on \(t\) and the training data \(\{\bx_{0,i}\}_{i=1}^n\).

The \emph{linearized reverse process} is defined as
\begin{equation}\label{eq:linearized_reverse_SDE}
    \d\widetilde{\bx}_t = -\bigl(\widetilde{\bx}_t + 2\widetilde{\bS}_t(\widetilde{\bx}_t)\bigr)\,\d t + \sqrt{2}\,\d\bar{\bB}_t, \qquad \widetilde{\bx}_T \sim \cN(0, \bI_d),
\end{equation}
coupled to~\eqref{eq:emp_reverse_SDE} through the same Brownian motion. Since its drift is linear and its initialization is Gaussian, $\widetilde\bx_t$ is, conditionally on the training data, a Gaussian process.

\subsection{Delocalization and pathwise linearization}
\label{sec:pathwise-delocalization}

The fitted score $\widehat \bS_t$ and its linearization $\widetilde \bS_t$ differ substantially only where the kernel's nonlinearity is active: within the thin tubes around the training points, where the overlaps $\<\bx, \bx_{0,i}\>/d$ are of order one. The linearized process, however, is a Gaussian process in dimension $d$ with bounded covariance; against any \emph{fixed} direction, its overlaps are of order $d^{-1/2}$. We prove (Lemma~\ref{lem:linearized_SDE_bounds}) that this delocalization holds uniformly along the trajectory and over all $n$ training directions:
\[
    \E\left[\sup_{t \in [t_0,T]} \max_{i \in [n]} |\<\widetilde{\bx}_t, \bx_{0,i}\>|^k \;\Big|\; \bX\right]^{1/k} \prec \sqrt{d} \qquad \text{for every } k \in \N .
\]
Thus, with high probability, the sampling trajectory remains nearly orthogonal to every training sample, with overlap of order $d^{-1/2}$. The sampler never visits the `bumps' where the model memorized, so the process driven by the nonlinear score is indistinguishable from the linear surrogate. The following theorem makes this heuristic precise on path measures. To simplify the proofs, we assume a stronger condition on the kernel than Assumption \ref{ass:kernel_func_constant_kappa}.

\begin{assumption}[Kernel regularity for reverse process]\label{ass:kernel_reverse}
      The kernel is a positive definite inner-product kernel of the form $K(\bx,\bx') = h(\<\bx,\bx'\>/d)$. The kernel function $h: \R \to \R$ is five times continuously differentiable, with $h(0) = 0$, $h'(0)= 1$, and $h^{(k)} (0) = 0$ for $k\in\{2,3,4\}$. Furthermore, there exists a constant $C\geq 1$ such that $|h^{(5)}(x)|\leq Ce^{C|x|}$ for all $x\in \R$. 
\end{assumption}

Let \(\R_\infty^d\deq\R^d\cup\{\infty\}\) be the one-point compactification of \(\R^d\). Conditionally on \(\bX\), let \(\widehat{\bx}\) denote the unique maximal solution of the empirical reverse SDE~\eqref{eq:emp_reverse_SDE}, extended beyond its explosion time by making $\infty$ absorbing. Let \(\widetilde{\bx}\) be the unique global solution of the linearized reverse SDE~\eqref{eq:linearized_reverse_SDE}, viewed as an \(\R_\infty^d\)-valued process. We denote their conditional path laws on \([t_0,T]\) by \(\P_{\widehat{\bx}}(\cdot\mid\bX)\) and \(\P_{\widetilde{\bx}}(\cdot\mid\bX)\), respectively. Both are probability measures on \(C([t_0,T];\R_\infty^d)\), as the linearized process~\eqref{eq:linearized_reverse_SDE} is nonexplosive and any explosive empirical trajectory of~\eqref{eq:emp_reverse_SDE} is continuous in the topology of the one-point compactification after being absorbed at $\infty$.

\begin{theorem}[Pathwise linearization of the reverse process]
\label{thm:kl_bound_empirical_linearized_reverse}
    Let \(T>0\) and \(t_0\in(0,T)\) be constants, and let
    \(t\mapsto\sft_t\) be a deterministic piecewise-continuous schedule.
    Suppose that Assumptions~\ref{ass:high_dim},
    \ref{ass:input_dist}, and~\ref{ass:kernel_reverse} hold.
    Then
    \[
        \KL\left(\P_{\widetilde{\bx}}(\cdot\mid\bX)\mmid\P_{\widehat{\bx}}(\cdot\mid\bX)\right)
        \prec
        \sup_{t\in[t_0,T]}\sft_t^2d^{-5}\left(1+\sft_t^2+\sft_t^4d^{-2}\right).
    \]
\end{theorem}

The proof is given in Appendix~\ref{app:reverse_process}. In particular, if \(\sup_{t\in[t_0,T]}\sft_t\lesssim d^{1+c}\) for some \(0\le c<1/6\), then\footnote{This restriction on $\sft_t$ is an artifact of our proof technique, and we expect the path linearization to hold more broadly for score estimators trained throughout Phase II. Moreover, the range $\sft_t\lesssim d^{1+c}$ already includes a regime in which the training loss vanishes polynomially in $d$.}
\[
    \KL\left(
    \P_{\widetilde{\bx}}(\cdot\mid\bX)
    \mmid
    \P_{\widehat{\bx}}(\cdot\mid\bX)
    \right)
    \prec d^{-1+6c}=o_{d,\P}(1).
\]
Thus, the conditional KL divergence converges to zero in probability. Nonlinear overfitting is invisible to the sampler, and the generated distribution is effectively Gaussian with covariance determined by the linearized score. 

\subsection{The Gaussian generated distribution}
\label{sec:generated-dist}

\begin{figure}[t]
    \centering
    \includegraphics[width=1\textwidth]{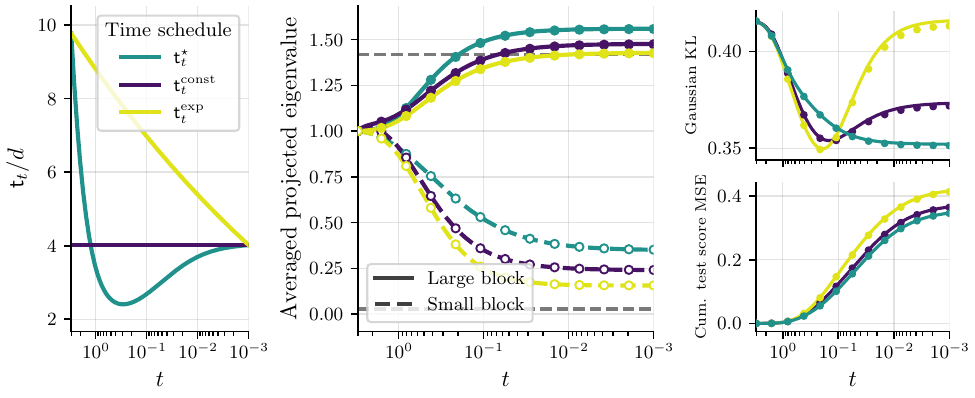}
    \caption{Recovery of two covariance scales along the linearized reverse process for $d=720$, $n=800$, $T=3$, $t_0=10^{-3}$, $h(x)=x+0.005x^5$ and $\pi_0=\cN(0,\bSigma)$ where \(\bSigma = \diag(\frac{17}{12}\bI_{504},\frac{1}{36}\bI_{216})\) has condition number \(51\) and \(\tau_\bSigma=1\).
    \textbf{Left:}  A schedule $t\mapsto\sft_t$ specifies the gradient-flow training time in~\eqref{eq:gradient_flow} at every diffusion time $t$. We consider the pointwise DSM-optimal schedule $\sft_t^\star\in\arg\min_{\sft/d\in[10^{-1},10^2]} \cR_{\mathrm{det},t}^{(\sft)}$. We compare it with the constant schedule $\sft_t^{\mathrm{const}}=\sft_{t_0}^\star$ and an exponential schedule $\sft_t^{\mathrm{exp}}$ interpolating between $\sft_T^\star$ and $\sft_{t_0}^\star$ on a logarithmic diffusion-time scale.
    \textbf{Middle:} Solid and dashed lines show, respectively, the average projected covariances \(\Tr(\proj_{\pm}\widetilde{\bSigma}_{t})/\operatorname{rank}(\proj_{\pm})\) along the dynamics in~\eqref{eq:covariance_ODE_linearized_reverse}, where \(\proj_+\) and \(\proj_-\) denote the projections onto the large- and small-eigenvalue eigenspaces of \(\bSigma\), respectively. Filled and open circles show the corresponding deterministic predictions from Theorem~\ref{thm:effective_covariance_bound}, while gray dashed lines mark the population eigenvalues.
    \textbf{Upper right:}
    Solid curves show the conditional normalized Gaussian divergence $\KL(\cN(0,\bSigma)\mmid\cN(0,\widetilde{\bSigma}_t))/d$. Markers show a natural spectral plug-in prediction using~\eqref{eq:s_t_x_ode_def}.
    \textbf{Lower right:} Solid curves show the cumulative population score MSE \(\int_t^T(\cR_{s}(\widehat{\boldf}^{\mathrm{lin}}_{\sft_s})-\cR_s^\star)/\sigma_s^2\d s\), where \(\widehat{\boldf}^{\mathrm{lin}}_{\sft_s}\) is the linearized estimator from Lemma~\ref{lem:reverse_linearized_check_score_representation} and \(\cR^\star_s\) is the Bayes population DSM risk at diffusion time \(s\). Markers show the deterministic prediction from Theorem~\ref{thm:gf_train_test_equivalents_moderate_time}.
    }
    \label{fig:covariance_evolution}
\end{figure}

It remains to identify the linear process~\eqref{eq:linearized_reverse_SDE}, which is Gaussian with zero mean; its conditional covariance $\widetilde{\bSigma}_t \deq  \E[\widetilde{\bx}_t\widetilde{\bx}_t^\sT\mid \bX]$ satisfies, by It\^o's formula, the matrix ODE
\begin{equation}\label{eq:covariance_ODE_linearized_reverse}
    \frac{\d}{\d t}\widetilde{\bSigma}_t = -2\widetilde{\bSigma}_t - 2\E[\widetilde{\bS}_t(\widetilde{\bx}_t)\widetilde{\bx}_t^\sT\mid \bX] - 2\E[\widetilde{\bx}_t\widetilde{\bS}_t(\widetilde{\bx}_t)^\sT\mid \bX] - 2\bI_d ,\qquad \widetilde{\bSigma}_T = \bI_d.
\end{equation}

The covariance \(\widetilde\bSigma_t\) is well-approximated by an auxiliary covariance flow with coefficient matrices that are spectral functions of the sample covariance \(\widehat\bSigma=\bX\bX^\sT/n\) (see Lemma~\ref{lem:linearized_covariance_approximation}). Consequently, the
covariance ODE of this auxiliary process diagonalizes in an eigenbasis of \(\widehat\bSigma\) and reduces to a scalar ODE along its eigenvalues. Lemma~\ref{lem:spectral_linearized_covariance_deterministic_equivalent} then identifies its deterministic equivalent in terms of the matrix-valued measure \(\bOmega\) from \eqref{eq:resolvent_fixed_point_measure}. Let \(\psi_t(x,\sft_t)\) be the time-dependent analogue of the spectral function \(\psi(x,\sft)\) defined in~\eqref{eq:gf_lin_profiles}.
We define the \emph{effective covariance}
\begin{equation}
    \bar{\bSigma}_t = \int_\R  \bar{\Lambda}_t(\lambda)\, \bOmega(\d \lambda),
\end{equation}
where the scalar profile $\bar\Lambda_t(x)$ solves the backward ODE
\begin{equation}\label{eq:s_t_x_ode_def}
    \frac{\d}{\d t}\bar{\Lambda}_t(x)
    =
    -2\left(1+\frac{2}{\sigma_t}\psi_t(x,\sft_t)\right)\bar{\Lambda}_t(x)-2,
    \qquad \bar{\Lambda}_T(x)=1.
\end{equation}

\begin{theorem}\label{thm:effective_covariance_bound}
    Let $T>0$ and $t_0\in (0,T)$ be constants, and let $t\mapsto\sft_t$ be a deterministic piecewise-continuous training-time schedule. Suppose that Assumptions~\ref{ass:high_dim},~\ref{ass:input_dist}, and~\ref{ass:kernel_reverse} hold. Then, for every deterministic $0\preceq \bA\in \R^{d\times d}$ such that $\|\bA\|_\ast\leq 1$ and constants $\epsilon,D>0$, there exists a constant $C>0$ such that
    \[
        \sup_{t\in [t_0,T]}\left|\Tr\bigl(\bA(\widetilde{\bSigma}_t-\bar{\bSigma}_t)\bigr)\right| \lesssim d^{\epsilon - \frac{1}{2}}\left( 1 +  d^{-1}\sup_{t\in[t_0,T]}\sft_t \right)
    \]
    with probability at least $1-Cd^{-D}$.
\end{theorem}

The proof of this theorem can be found in Appendix~\ref{sec:proof-effective-covariance}.

Theorems~\ref{thm:kl_bound_empirical_linearized_reverse} and~\ref{thm:effective_covariance_bound} together describe the output distribution for early-stopped gradient flow score estimators. Conditionally on the training data, the generated samples are approximately Gaussian, whose limiting covariance has the deterministic equivalent given by $\bar\bSigma_{t_0}$. Equivalently, every fixed deterministic projection of the generated samples is asymptotically Gaussian with covariance $\bar\bSigma_{t_0}$, where $\bar\bSigma_{t_0}$ is the composition of the Marchenko--Pastur-type map $\bSigma\mapsto\bOmega$ with the ODE flow~\eqref{eq:s_t_x_ode_def}.

As discussed in the introduction, combining the results of Sections \ref{sec:gf} and \ref{sec:reverse} allows us to examine the relationship between generalization to the population DSM objective and generalization to the target distribution. In Figure \ref{fig:covariance_evolution}, we consider a Gaussian target distribution $\pi_0=\cN(0,\bSigma)$, where $\bSigma$ has two eigenspaces and condition number $51$, under three early-stopping schedules (left panel). The middle panel shows the evolution, throughout the reverse process, of the learned covariance $\widetilde{\bSigma}_t$ projected onto each eigenspace for each schedule. The right panel shows the cumulative test score MSE and the KL divergence $\KL\bigl(\pi_0\mmid\cN(0,\widetilde{\bSigma}_t)\bigr)$ along the reverse diffusion process. Notably, an early-stopping schedule that is suboptimal in terms of the DSM objective can nevertheless yield a (slightly) smaller KL divergence when the reverse process is stopped early. We consider exploring further the relationship between score-matching performance and generation quality to be an important future research direction.

\section{Conclusion and future directions}
\label{sec:discussion}

Diffusion models present a statistical puzzle: their training objective is minimized by a memorizing solution, and available sample sizes are far too small for consistent estimation of an unrestricted score function. Yet models trained in practice generalize, producing genuinely new samples that reproduce both global and local statistics of the target distribution. Resolving this puzzle requires characterizing what practical training algorithms learn: which score functions they fit and which distributions the resulting reverse diffusion dynamics generate.

Motivated by analogous successes in the theory of supervised learning, we studied lazy denoising score matching in the proportional high-dimensional regime. At each noise level, we fit a kernel model corresponding to a neural network trained in the lazy regime. We characterize the estimator learned by gradient flow and derive precise asymptotics for its training and test errors throughout the optimization trajectory. The estimator passes through three distinct phases (Figure \ref{fig:risk_three_regimes}):
\[
    \underbrace{\text{$\sft = \Theta(d)$: generalization}}_{\substack{  \text{Phase I: } \widehat \boldf \approx \psi^{\mathrm{opt}}(\bX\bX^{\sT}/n)\,\bu \\ \text{covariance learning}}
    }
    \quad\longrightarrow\quad
    \underbrace{\text{$\sft = d^{1+\Theta(1)}$: overfitting}}_{\substack{ \text{Phase II: } \widehat \boldf \approx -\bu/\sigma \\ \text{noise score}} }
    \quad\longrightarrow\quad
    \underbrace{\text{$\sft = d^{\omega(1)}$: memorization} }_{\substack{\text{Phase III: } \widehat \boldf \approx (\rho \bx_{0,i^\star(\bu)} - \bu)/\sigma \\ \text{empirical Bayes denoiser}} }.
\]
The population risk initially decreases and attains an optimum at stopping time $\sft^\star=\Theta(d)$, corresponding to a spectral estimator applied to the empirical data covariance. It then rises to $\tfrac{\rho^2}{\sigma^2}\tau_{\bSigma}$ as the estimator loses its dependence on the data and approaches the pure-noise score, before ultimately reaching $2\tfrac{\rho^2}{\sigma^2}\tau_{\bSigma}$ upon memorizing the training set.

Our analysis separates three phenomena that are often conflated: generalization in sampling, overfitting (vanishing training objective), and memorization (reproducing training samples). In particular, overfitting does not imply memorization: in the lazy high-dimensional regime, these phenomena are separated by an exponential gap in training time and a model can interpolate its training objective while continuing to generate genuinely new samples. At the level of score estimation, however, interpolation requires the lazy model to discard some of the statistical structure learned from the data. Overfitting therefore degrades population-level generalization, and the separation between the generalization and overfitting timescales is governed by the kernel nonlinearity and the data covariance. At the level of generation, by contrast, a form of benign overfitting remains possible. The learned score decomposes into a global linear component and localized corrections around the training samples. Because the reverse diffusion trajectory is delocalized, it typically avoids these corrections, so the SDE stays effectively linear and produces a Gaussian distribution.

Our results suggest several natural directions for future research: 
\begin{enumerate}
    \item[(1)] \emph{Beyond the proportional regime.}
    Our analysis focuses on $n\asymp d$, where the learned score is asymptotically at most linear. In supervised learning, the polynomial scaling $n\asymp d^\ell$ gives rise to a precise hierarchy: inner-product kernel methods learn exactly the degree-$\leq\ell$ polynomial component of the target~\cite{ghorbani2021linearized,mei2022hypercontractivity,misiakiewicz2024six}. We expect an analogous hierarchy for lazy generative models, in which the generated distribution progressively acquires non-Gaussian corrections and captures higher-order moments of $\pi_0$ as $\ell$ increases. Establishing this generative hierarchy---and, in particular, extending the delocalization argument to scores with higher-degree components---is a natural next step. Although these polynomial sample-size regimes may be less directly relevant to practice, they would help delineate the capabilities of lazy training from those of feature learning.
    
    \item[(2)] \emph{Feature learning.}
At proportional sample sizes, any non-Gaussian structure captured by the generated distribution must arise from feature learning or architectural inductive bias. Solvable analyses of feature learning in shallow autoencoders~\cite{CuiZde23Denoising,Cui+24LearningGenerative,CuiPehLu25Solvable} and the analysis of convolutional score models~\cite{kamb2024analytic} provide important first steps. A theory coupling feature learning to the reverse process analysis developed here could determine which learned features are actually encountered by the sampler and, consequently, which non-Gaussian statistics are acquired first and at what cost in samples and optimization time.

    \item[(3)] \emph{Low-dimensional and structured data.}
Our assumptions exclude data supported on, or concentrated near, low-dimensional manifolds, where memorization and generalization interact with the geometry of the support~\cite{pidstrigach2022score,achilli2025memorization,GeoMac26Manifold,wang2025diffusion,ross2025geometric}. A natural next step is to extend our analysis to these structured data settings.


    \item[(4)] \emph{Coupled noise levels and practical samplers.} We train an independent RKHS model at each noise level, whereas practical diffusion models use a single time-conditioned network. Parameter sharing couples the learning problems across noise levels and understanding what changes in this case is an important direction.
\end{enumerate}

\section*{Acknowledgments and use of AI}

J.S.~was supported in part by the Yale Office of the Provost. We used AI tools to assist in constructing the example of an unstable ridgeless limit in Appendix~\ref{sec:ridgeless-instability} and in writing code to generate the figures. Their use was otherwise limited to light editing and proofreading. We take full responsibility for the content of the paper.








\bibliographystyle{amsalpha}
\bibliography{biblio}

\clearpage

\appendix

\section{Ridge-regularized denoising score matching}
\label{app:ridge}

This appendix presents the exact asymptotics for the train and test risks of the kernel ridge estimator $\widehat f_{\kappa,\bv}$ of~\eqref{eq:opt_denoiser} at a fixed noise level $\rho,\sigma>0$, across the three regimes of the (deterministic) penalty $\kappa$. Under the calibration $\kappa = d/\sft$, the three regimes correspond to the three phases of Section~\ref{sec:gf}: moderate $\kappa$ (equivalently $\sft \lesssim d^{1+c}$), polynomially vanishing $\kappa$ (polynomial $\sft$; training loss interpolation), and the ridgeless limit $\kappa \to 0^+$ (the memorization endpoint in Theorem~\ref{thm:ridgeless_main}, proved in Appendix~\ref{app:interpolation-limit}). 

\subsection{Moderate regularization}
\label{app:ridge-moderate}

For $\kappa = \Omega_d(d^{-c})$ with a small constant $c > 0$, the ridge estimator is stable and the kernel linearization (Proposition~\ref{prop:linearization_kernel_op}) of Section~\ref{sec:linearization} applies directly.
 We work under Assumptions~\ref{ass:high_dim},~\ref{ass:input_dist}, and~\ref{ass:kernel_func_constant_kappa}, and recall $\tau_\bSigma = \Tr(\bSigma)/d$ and the effective operators~\eqref{eq:K_hat_op}. By Proposition~\ref{prop:bayes_opt_denoiser_linearization}, the irreducible train error is exponentially small.

We first define the deterministic equivalents. Let $\bM_\kappa\in \R^{d\times d}$ and $\mu_\kappa > 0$ solve the system
\begin{equation}\label{eq:fixed_point_equation}
    \bM_\kappa \deq  \left(\frac{\bSigma}{1+\mu_\kappa} + \kappa^\star_{\kappa}\bI_d\right)^{-1},\qquad \mu_\kappa = \frac{1}{n}\Tr(\bSigma \bM_\kappa),
\end{equation}
where the auxiliary parameters $\kappa^\star_{\kappa}$ and $\chi_\kappa$ are
\begin{equation}\label{eq:auxiliary_parameters}
    \chi_\kappa \deq  \frac{\kappa}{d}+\frac{h(\rho^2\tau_\bSigma)-\rho^2\tau_\bSigma}{n} = \frac{\kappa}{d} + \delta_n,\qquad \kappa^\star_{\kappa} \deq  \frac{d(\sigma^2+\kappa)\chi_\kappa}{\rho^2\kappa} = \frac{\sigma^2+\kappa}{\rho^2}\left(1 + \frac{d}{n}\frac{h(\rho^2\tau_\bSigma)-\rho^2\tau_\bSigma}{\kappa}\right).
\end{equation}
Note that \(\kappa_\kappa^\star \geq \sigma^2/\rho^2>0\) is bounded away from \(0\) whenever \(\sigma\) is. Existence and uniqueness of the solution to~\eqref{eq:fixed_point_equation} follow from standard interference-function arguments~\cite{yates1995framework,couillet2011random}. Define further the second-order equivalent
\begin{equation}
    \bM^{(2)}_\kappa \deq  \bM^2_\kappa + \frac{\frac{1}{n}\Tr(\bSigma\bM^2_\kappa)}{(1+\mu_\kappa)^2-\frac{1}{n}\Tr(\bSigma^2\bM^2_\kappa)}\,\bM_\kappa\bSigma\bM_\kappa .
\end{equation}

The deterministic equivalent train and test errors are given by
\begin{align}
    \cL_{\mathrm{det}}^{(\kappa)} &\deq  1 - \frac{\sigma^2}{\rho^2\kappa}\left(2\chi_\kappa - \frac{\kappa}{d}\right)\Tr(\bM_\kappa) + \frac{\sigma^2\chi_\kappa(\sigma^2\chi_\kappa d - \kappa(\sigma^2+\kappa))}{\rho^4\kappa^2} \Tr(\bM^{(2)}_\kappa), \label{eq:train_error_equivalent}\\
    \cR_{\mathrm{det}}^{(\kappa)} &\deq  1 - \frac{2\sigma^2\chi_\kappa}{\rho^2\kappa }\Tr(\bM_\kappa) + \frac{\sigma^2\chi_\kappa^2d}{\rho^4\kappa^2 }\Tr(\bSigma_\sigma \bM^{(2)}_\kappa), \label{eq:test_error_equivalent}
\end{align}
where $\bSigma_\sigma = \rho^2\bSigma + \sigma^2\bI_d$. Both quantities are explicit functions of the spectrum of $\bSigma$, computable by solving the scalar fixed point for $\mu_\kappa$. These asymptotics are equivalent to the asymptotics in the equivalent linearized risk \eqref{eq:effective-train-risk-linear} and \eqref{eq:effective-test-risk-linear} with linearized kernel $\widehat\cK_{\mathrm{eff}}$ and linearized RKHS $\cH_{\mathrm{eff}}$; the nonlinear part of the kernel only enters through the shift $\delta_n$ in $\chi_\kappa$. 

\begin{theorem}[Ridge DSM with moderate regularization]\label{thm:train_test_equivalents_constant_kappa}
    Suppose that Assumptions~\ref{ass:high_dim},~\ref{ass:input_dist} and~\ref{ass:kernel_func_constant_kappa} hold. Then,
    \[
        |\cL(\widehat{\boldf}_{\kappa};0)- \cL_{\mathrm{det}}^{(\kappa)}| \prec \frac{\kappa+1}{\kappa \sqrt{d}},\qquad  |\cR(\widehat{\boldf}_{\kappa})- \cR_{\mathrm{det}}^{(\kappa)}| \prec \frac{\kappa^3+1}{\kappa^3 \sqrt{d}}.
    \]
    Furthermore, if $h(x) = x$ for all $x\in \R$, there exists a constant $c>0$ such that
    \[
        |\cL(\widehat{\boldf}_{\kappa};0) - \cL^{(\kappa)}_{\mathrm{det}}| \prec \frac{1}{\sqrt{d}}, \qquad |\cR(\widehat{\boldf}_{\kappa}) - \cR_{\mathrm{det}}^{(\kappa)}| \prec \frac{e^{-d^c}}{\kappa^2} + \frac{1}{\sqrt{d}}.
    \]
\end{theorem}

The proof is given in Appendix~\ref{app:train-test}; it proceeds by linearizing the empirical risk via propositions~\ref{prop:linearization_kernel_op} and~\ref{prop:bayes_opt_denoiser_linearization} and applying deterministic-equivalent estimates to the associated resolvents. The bounds are non-vacuous provided $\kappa \geq d^{c}$ for some $c > -1/6$; below this scale, the empirical kernel resolvent is ill-conditioned and one cannot substitute $\widehat\cK$ by $\widehat\cK_{\mathrm{eff}}$.

\subsection{Vanishing regularization: overfitting and memorization}
\label{app:ridge-vanishing}

For $\kappa = O_d(d^{-1/6})$ and non-linear kernel function \(h\), the direct comparison to the linearized kernel breaks down. Instead we use Proposition~\ref{prop:localized_bumps_properties} which shows that, under Assumption \ref{ass:kernel_func_vanishing_kappa}, we can construct localized bumps that approximately interpolate the empirical Bayes denoiser, with training error $O_d(d^{-k_0+2})$, while having RKHS norm bounded by $O_d(d)$. The optimality of the ridge estimator directly implies $\cL(\widehat{\boldf}_{\kappa};\kappa) \prec d^{-k_0 + 2} + \kappa$ and the training error vanishes $\cL(\widehat{\boldf}_{\kappa};0)=o_d(1)$. We show that this forces $\widehat{\boldf}_\kappa (\bu)\approx - \bu/\sigma$ on test inputs, and the population risk stabilizes at the pure-noise plateau uniformly over the polynomial range of $\kappa$, similarly to Phase II of the gradient flow estimator (Theorem \ref{thm:gf_train_test_equivalents_extended_time}).

\begin{theorem}[Ridge DSM with polynomially vanishing regularization]\label{thm:train_test_equivalents_vanishing_kappa}
    Suppose that Assumptions~\ref{ass:high_dim},~\ref{ass:input_dist}, and~\ref{ass:kernel_func_vanishing_kappa} hold. Suppose further that there exists an $\epsilon\in(0,k_0-2)$ such that $d^{-k_0+2} \lesssim \kappa \lesssim d^{-\epsilon}$, where $k_0$ is given in Assumption~\ref{ass:kernel_func_vanishing_kappa}. Then, \(|\cL(\widehat{\boldf}_{\kappa};0)| \prec  \kappa\). Furthermore, for every constant $c \in (0, (\epsilon \wedge 1)/2)$ and $\zeta >0$ given in Assumption~\ref{ass:kernel_func_vanishing_kappa},
    \[
         \left| \cR(\widehat{\boldf}_{\kappa}) - \frac{\rho^2}{\sigma^2} \frac{\Tr(\bSigma)}{d} \right| \lesssim (C_{\zeta} \vee 1) d^{c-\frac{\epsilon \wedge 1}{2}}  + \sqrt{1+(C_{\zeta} \vee 1) d^{c-\frac{\epsilon \wedge 1}{2}}}\,d^{\frac{c}{2}-\frac{\epsilon \wedge 1}{4}}
    \]
    with probability at least $1-\zeta-e^{-d^{c'}}$ for all $d \geq C$, where $C_\zeta$ is defined in \eqref{eq:C_delta_def}.
\end{theorem}

The proof is given in Appendix~\ref{app:train-test}. 

Finally, the ridgeless limit estimator $\kappa \to 0^+$ with \(n,d\) fixed coincides with infinite-time limit of the Gradient Flow estimator $\sft \to \infty$, and share the same limiting test error with doubled risk of Theorem~\ref{thm:ridgeless_main} (after the mild post-processing step). This is studied in Appendix \ref{app:interpolation-limit}.

\clearpage

\section{Technical preliminaries}
\label{app:tech-back}

In this appendix, we collect some technical preliminaries that will be used in the proofs of the main results.

\subsection{Consequences of the Logarithmic Sobolev Inequality}

We review some standard consequences of the Logarithmic Sobolev inequality (LSI), which is satisfied by the target data distribution $\pi_0$ (Assumption \ref{ass:input_dist}.(b)). The standard Gaussian measure \(\cN(0,\bI_d)\) satisfies an LSI with constant \(1\) \cite{gross1975logarithmic}. By the tensorization property of LSI \cite[Proposition 5.2.7]{bakry2014analysis}, the product measure \(\pi_d = \pi_0 \otimes \cN(0,\bI_d)\) satisfies an LSI with constant given by the maximum of the marginal constants, which is \(1 \vee C_{\mathrm{LSI}} = C_{\mathrm{LSI}}\). Since the pushforward \((\bx,\bz)\to \rho\bx+\sigma\bz\), viewed as a function from \(\R^{2d}\) to \(\R^d\), has operator norm bounded by \(1\), the Lipschitz contraction property of LSI implies that \(\mu_d\) also satisfies an LSI with constant \(C_{\mathrm{LSI}}\). The LSI inequality in Assumption~\ref{ass:input_dist}.(b) implies a Poincaré inequality
\[
    \Var_{\pi_0}(f)\leq C_{\mathrm{LSI}} \E_{\pi_0}[\|\nabla f\|_2^2],
\]
for every \(f\in W^{1,2}(\pi_0)\). Applying this to the linear function \(\bx\mapsto \<\bx,\bv\>\), 
\begin{equation}\label{eq:cov_op_bound}
    \|\bSigma\|_\op \leq C_{\mathrm{LSI}}.
\end{equation}

Suppose that \(\pi\) on \(\R^d\) satisfies LSI with constant \(C>0\). By~\cite[(5.4.2)]{bakry2014analysis},
\begin{equation}\label{eq:Gaussian_Lipschitz_concentration}
    \P_{\bx\sim \pi}(|f(\bx) - \E_{\pi}[f]| > t) \leq 2\exp\left(-\frac{t^2}{2C\|f\|_{\mathrm{Lip}}^2}\right)
\end{equation}
for any Lipschitz function \(f:\R^d\to\R\) and any \(t > 0\). By~\cite[Theorem 2.3]{adamczak2015note}, for any \(\bA\in \R^{d\times d}\) and \(t\geq 0\),
\begin{equation}
    \P_{\bx\sim \pi_0}\left(|\bx^\sT \bA \bx - \Tr(\bA\bSigma)| \geq t\right) \leq 2\exp\left(-\frac{1}{C}\min\left\{\frac{t^2}{2\|\bA\|_\frob^2},\frac{t}{\|\bA\|_\op}\right\}\right)
\end{equation}
where \(C>0\) is some constant depending only on the LSI constant of \(\pi_0\) and thus
\begin{equation}\label{eq:norm_concentration_overlaps}
    \max_{i\in [n]}|\|\bx_{0,i}\|^2 - \Tr(\bSigma)| \prec \sqrt{d},\qquad  \max_{i\neq j\in [n]}|\<\bx_{0,i},\bx_{0,j}\>| \prec \sqrt{d}.
\end{equation}
The analogous \(L^p\) bounds also hold for any \(p \geq 1\):
\begin{equation}\label{eq:norm_concentration_overlaps_lp}
    \begin{gathered}
        \E_{\bx\sim \pi_0}[|\|\bx\|^2-\Tr(\bSigma)|^p] \leq C_p d^{p/2},\qquad \E_{\bx'\sim \pi_0}[|\<\bx,\bx'\>|^p]^{1/p} \prec \sqrt{d},
     \\
     \E_{\bx,\bx'\sim \pi_0}[|\<\bx,\bx'\>|^p]^{1/p} \leq C_{p}' \sqrt{d},
    \end{gathered}
\end{equation}
for some constants \(C_p,C_p'>0\) depending only on \(p\) and \(C_{\mathrm{LSI}}\).
Since \(\mu_d\) satisfies the same LSI inequality as \(\pi_0\), similar bounds also hold for \(\mu_d\). The mixed moments bounds also hold for the overlap between \(\cN(0,\bI_d)\) and \(\pi_0\). We will use the following lemma to control the exponential moments of the overlaps between independent random vectors from \(\pi_0\) and \(\cN(0,\bI_d)\).

\begin{lemma}
\label{lem:exp_noisy_overlap_moments}
Let \(\pi\) be a distribution on \(\R^d\), and suppose that \(\pi\) satisfies LSI with constant \(L\) and that \( \|\E_{\bx\sim\pi}\bx\|_2^2\le Md\) for some constant \(M>0\). Then, for every \(s>0\),
\begin{equation}
\label{eq:conditional_lsi_overlap}
    \E_{\bx\sim\pi}\exp\left(\frac{s}{2d}\|\bx\|_2^2\right)\leq 3e^{s(L+M)}
\end{equation}
whenever \(d \ge 4Ls\).
\end{lemma}

\begin{proof}
    The LSI concentration inequality~\eqref{eq:Gaussian_Lipschitz_concentration} applied to \(\bx\mapsto\|\bx\|_2\) gives
    \[
        \P_{\bx\sim\pi}(|\|\bx\|_2-\E\|\bx\|_2|>t)\le 2\exp\left(-\frac{t^2}{2L}\right),
    \]
    while Poincaré inequality gives \(0\preceq \Cov_\pi(\bx) \preceq L \bI_d\) and thus \(\E_{ \pi}[\|\bx\|_2^2]\leq (L + M) d\). Because \(\|\bx\|_2\leq \E[\|\bx\|_2] + R\) with \(R=(\|\bx\|_2-\E[\|\bx\|_2])_+\), we have \(\|\bx\|_2^2\le 2(L+M) d+2R^2\) by Young's inequality. Integrating the Gaussian tail of \(R\) concludes the proof.
\end{proof}

Finally, a standard covering-net and Bernstein argument for independent sub-Gaussian vectors, in the form of the usual sample-covariance bound in~\cite[Theorem 6.5]{wainwright2019high}, gives
\begin{equation}\label{eq:cov_concentration}
\left\|\frac{1}{n}\sum_{i=1}^{n} \bx_{0,i}\bx_{0,i}^\sT - \bSigma\right\|_\op \lesssim \sqrt{\frac{d}{n}} + \frac{d}{n}
\end{equation}
with very high probability.

We apply the Poincaré inequality to bound the second moment of the order-\(k\) tensor polynomial. For every integer \(k\geq 0\), denote by \((\R^d)^{\otimes k}\) the set of tensors of order \(k\) on \(\R^d\), and by \(\mathrm{Sym}^k(\R^d)\) the subset of (totally) symmetric tensors. That is, \(\bA\in \mathrm{Sym}^k(\R^d)\) if \(\bA \in (\R^d)^{\otimes k}\) and \(\bA_{i_1,\ldots, i_k} = \bA_{i_{\mathfrak{s}(1)},\ldots, i_{\mathfrak{s}(k)}}\) for every permutation \(\mathfrak{s}\) of \(\{1,\ldots, k\}\).

\begin{lemma}\label{lem:moment_bound}
    Suppose that Assumption~\ref{ass:input_dist} holds. For every integer \(k\geq 1\) and symmetric tensor \(\bA\in \mathrm{Sym}^{k}(\R^d)\),
    \[
        \E[\<\bA,\bx^{\otimes k}\>^2] \leq (2C_{\mathrm{LSI}}^2k^2)^{k} d^{\frac{k}{2}} \|\bA\|_\frob^2,
    \]
    where the expectation is taken over \(\bx\sim \pi_0\) or \(\bx\sim\mu_d\).
\end{lemma}

\begin{proof}
    We prove the result for \(\pi_0\). The same proof applies to \(\mu_d\) since both measures satisfy the same Poincaré inequality.
    The proof proceeds by induction on \(k\). For \(k=1\), \(\bA \in \R^d\) is a vector and
    \[
        \E_{\bx\sim \pi_0}\left[\<\bA,\bx\>^2\right] = \bA^\sT \bSigma \bA \leq \|\bSigma\|_\op \|\bA\|_2^2.
    \]
    The bound holds with \(C_1 = C_{\mathrm{LSI}}\geq \|\bSigma\|_\op\).
    
    We now turn to the inductive step. Let \(k\geq 2\). We decompose the second moment of our order-\(k\) tensor polynomial into its variance and squared mean:
    \begin{equation}\label{eq:bias_var_decomp}
        \E_{\bx\sim \pi_0}[\<\bA,\bx^{\otimes k}\>^2] = \Var(\<\bA,\bx^{\otimes k}\>) + \E_{\bx\sim \pi_0}[\<\bA,\bx^{\otimes k}\>]^2.
    \end{equation}
    To bound the variance, we apply the Poincaré inequality (see Assumption~\ref{ass:input_dist}) to the function \(f(\bx) = \<\bA,\bx^{\otimes k}\>\) to obtain
    \begin{align*}
        \Var_{\bx\sim \pi_0}(\<\bA,\bx^{\otimes k}\>) &\leq C_{\mathrm{LSI}}\E_{\bx\sim \pi_0}[\|\nabla \<\bA,\bx^{\otimes k}\>\|_2^2]  = C_{\mathrm{LSI}}k^2 \sum_{i=1}^{d} \E_{\bx\sim \pi_0}[\<\bB_i,\bx^{\otimes (k-1)}\>^2],
    \end{align*}
    where \(\bB_i\in \mathrm{Sym}^{k-1}(\R^d)\) is the \((k-1)\)-order slice defined by \((\bB_i)_{j_1,\ldots,j_{k-1}} \deq  \bA_{i,j_1,\ldots,j_{k-1}}\). Applying the inductive hypothesis to each slice \(\bB_i\),
    \[
        \E_{\bx\sim \pi_0}\left[\<\bB_i,\bx^{\otimes (k-1)}\>^2\right] \leq C_{k-1} d^{ \frac{k-1}{2}} \|\bB_i\|_\frob^2
    \]
    where \(C_{k-1} = (2C_{\mathrm{LSI}}^2(k-1)^2)^{k-1}\). Summing over all coordinate slices using \(\sum_{i=1}^{d} \|\bB_i\|_\frob^2 = \|\bA\|_\frob^2\), we obtain
    \begin{equation}\label{eq:var_bound}
        \Var(\<\bA,\bx^{\otimes k}\>) \leq C_{\mathrm{LSI}} k^2 C_{k-1} d^{\frac{k-1}{2}} \|\bA\|_\frob^2.
    \end{equation}

    We now turn to the squared mean term in \eqref{eq:bias_var_decomp}. We decompose the order-\(k\) tensor \(\bA\) into its coordinate slices as \(\bA = \sum_{i=1}^{d} e_i \otimes \bB_i\), where \(e_i\) is the \(i\)-th standard basis vector in \(\R^d\). Since \(\E_{\bx\sim \pi_0}[\bx]=0\) by Assumption~\ref{ass:input_dist}, we apply Cauchy-Schwarz inequality to obtain
    \begin{align*}
       \E_{\bx\sim \pi_0}[\<\bA,\bx^{\otimes k}\>]^2 & = \left(\sum_{i=1}^{d}\E_{\bx\sim \pi_0}[x_i\<\bB_i,\bx^{\otimes (k-1)}\>]\right)^2 \leq \Tr(\bSigma)\sum_{i=1}^{d}\Var(\<\bB_i,\bx^{\otimes (k-1)}\>). 
    \end{align*}
    On one hand, \(\Tr(\bSigma)\leq d\|\bSigma\|_\op\leq dC_{\mathrm{LSI}}\) by~\eqref{eq:cov_op_bound}. On the other hand, we use~\eqref{eq:var_bound} to get
    \[
        \sum_{i=1}^{d}\Var(\<\bB_i,\bx^{\otimes (k-1)}\>) \leq C_{\mathrm{LSI}}(k-1)^2C_{k-2}d^{\frac{k-2}{2}}\sum_{i=1}^{d}\|\bB_i\|_\frob^2 = C_{\mathrm{LSI}}(k-1)^2C_{k-2}d^{\frac{k-2}{2}}\|\bA\|_\frob^2.
    \]
    For \(k=2\), we have \(C_0 = 1\) and the bound above is valid.
    Combining the bounds above,
    \[
        \E_{\bx\sim \pi_0}[\<\bA,\bx^{\otimes k}\>]^2 \leq C_{\mathrm{LSI}}^2(k-1)^2C_{k-2}d^{\frac{k}{2}}\|\bA\|_\frob^2
    \]
    and
    \begin{align*}
        \E_{\bx\sim \pi_0}[\<\bA,\bx^{\otimes k}\>^2] 
         \leq 2C_{\mathrm{LSI}}^2k^2 C_{k-1} d^{\frac{k}{2}}\|\bA\|_\frob^2.
    \end{align*}
    This completes the inductive step and the proof.
\end{proof}

We now apply the moment bound in Lemma~\ref{lem:moment_bound} and the Rosenthal bound in Lemma~\ref{lem:rosenthal} to obtain a high-probability bound on the Frobenius norm of the empirical average of the order-\(k\) tensor polynomials \(\E_{\bz\sim \cN(0,\bI_d)}[(\rho\bx_{0,i}+\sigma\bz)^{\otimes k}]\).

\begin{lemma}\label{lem:moment_bound_empirical}
    Suppose that Assumptions~\ref{ass:high_dim} and~\ref{ass:input_dist} hold. Then, for every \(\zeta\in (0,1)\),
    \[
        \left\|\frac{1}{n}\sum_{i=1}^{n}\E_{\bz\sim \cN(0,\bI_d)}[(\rho\bx_{0,i}+\sigma\bz)^{\otimes k}]\right\|_\frob \lesssim C_{\zeta,k} d^{\frac{k-1}{2}}
    \]
    for every \(k\in \N\) with probability at least \(1-\zeta\), where
    \begin{equation}\label{eq:C_delta_k_constant}
        C_{\zeta,k}\deq (2C_{\mathrm{LSI}}^2((4k-\ln\zeta+2)k)^2)^{k/2}(4k-\ln\zeta).
    \end{equation}
\end{lemma}

\begin{proof}
    Let \(\bT_i = \E_{\bz\sim \cN(0,\bI_d)}[(\rho\bx_{0,i}+\sigma\bz)^{\otimes k}]\).
    By the triangle inequality,
    \begin{align*}
        \left\|\frac{1}{n}\sum_{i=1}^{n}\bT_i\right\|_\frob & \leq \left\|\frac{1}{n}\sum_{i=1}^{n}\left(\bT_i - \E_{\bu\sim\mu_d}[\bu^{\otimes k}]\right) \right\|_\frob + \left\|\E_{\bu\sim\mu_d}[\bu^{\otimes k}] \right\|_\frob.
    \end{align*}
    Using the variational representation of the Frobenius norm, the deterministic bias term is bounded by
    \[
      \left\|\E_{\bu\sim\mu_d}[\bu^{\otimes k}] \right\|_\frob = \sup_{\|\bA\|_\frob\leq 1}\E[\<\bA,\bu^{\otimes k}\>_\frob]\leq  (2C_{\mathrm{LSI}}^2k^2)^{\frac{k}{2}} d^{\frac{k}{4}}
    \]
    by taking the square root of the bound established in Lemma~\ref{lem:moment_bound}. Next, we apply the Rosenthal bound in Lemma~\ref{lem:rosenthal} to obtain
    \begin{align*}
        \left\|\left\|\frac{1}{n}\sum_{i=1}^{n}\left(\bT_i - \E_{\bu\sim\mu_d}[\bu^{\otimes k}]\right) \right\|_\frob\right\|_{L^p} & \leq \sqrt{\frac{p-1}{n}}\|\|\bT_i - \E_{\bu\sim\mu_d}[\bu^{\otimes k}]\|_\frob\|_{L^2}
        \\ &\qquad + \frac{\sqrt{(p-1)(p-2)}}{n}\left[\frac{(d^k+1)n}{2}\right]^{\frac{1}{p}} \|\|\bT_i - \E_{\bu\sim\mu_d}[\bu^{\otimes k}]\|_\frob\|_{L^p}
    \end{align*}
    for every even \(p\geq 2\). By Jensen's inequality and the triangle inequality,
    \begin{align*}
        \|\|\bT_i - \E_{\bu\sim\mu_d}[\bu^{\otimes k}]\|_\frob\|_{L^p} & \leq \|\|\bT_i\|_\frob\|_{L^p} + \|\|\E_{\bu\sim\mu_d}[\bu^{\otimes k}]\|_\frob\|_{L^p}
         \leq 2\|\|\bT_i\|_\frob\|_{L^p} \leq 2\|\|\bu\|_2^{k}\|_{L^p(\mu_d)}.
    \end{align*}
    By the monotonicity of \(L^p\) with respect to \(p\),
    \begin{align*}
         \left\|\left\|\frac{1}{n}\sum_{i=1}^{n}\left(\bT_i - \E_{\bu\sim\mu_d}[\bu^{\otimes k}]\right) \right\|_\frob\right\|_{L^p} & \leq 2\sqrt{\frac{p-1}{n}}\left(1 + \sqrt{\frac{(p-2)}{n}}\left[\frac{(d^k+1)n}{2}\right]^{\frac{1}{p}}\right)\|\|\bu\|_2^{k}\|_{L^{2p}(\mu_d)}.
    \end{align*}
    As
    \[
        \|\|\bu\|_2^{k}\|_{L^{2p}(\mu_d)} = \E_{\bu \sim \mu_d}[\|\bu\|_2^{2pk}]^{\frac{1}{2p}} = \E_{\bu \sim \mu_d}[\<\bI_d^{\otimes pk/2},\bu^{\otimes pk}\>^2]^{\frac{1}{2p}},
    \]
    where \(\bI_d^{\otimes pk/2}\) is the order-\(pk\) identity tensor defined such that \(\<\bI_d^{\otimes pk/2},\bA\otimes \bB\>_\frob = \<\bA,\bB\>_\frob\) for every \(\bA,\bB\in (\R^d)^{\otimes pk/2}\), we can apply Lemma~\ref{lem:moment_bound} with \(\bA = \proj_{\mathrm{Sym}^{pk}(\R^d)}(\bI_d^{\otimes pk/2})\) to obtain
    \begin{align*}
        \|\|\bu\|_2^{k}\|_{L^{2p}(\mu_d)}\leq (2C_{\mathrm{LSI}}^2(pk)^2)^{\frac{k}{2}} d^{\frac{k}{4}} \|\bA\|_\frob^{1/p}.
    \end{align*}
    Here, \(\proj_{\mathrm{Sym}^{pk}(\R^d)}\) denotes the orthogonal projection \((\R^d)^{\otimes pk}\mapsto \mathrm{Sym}^{pk}(\R^d)\) with respect to the Frobenius inner product. Since \(\|\bI_d^{\otimes pk/2}\|_\frob=d^{pk/4}\),
    \begin{align*}
         \left\|\left\|\frac{1}{n}\sum_{i=1}^{n}\left(\bT_i - \E_{\bu\sim\mu_d}[\bu^{\otimes k}]\right) \right\|_\frob\right\|_{L^p} & \leq 2(2C_{\mathrm{LSI}}^2(pk)^2)^{\frac{k}{2}} \sqrt{\frac{p-1}{n}}\left(1 + \sqrt{\frac{(p-2)}{n}}\left[\frac{(d^k+1)n}{2}\right]^{\frac{1}{p}}\right) d^{\frac{k}{2}}
    \end{align*}
    and
    \begin{align*}
         \left\|\left\|\frac{1}{n}\sum_{i=1}^{n}\bT_i \right\|_\frob\right\|_{L^p} \leq 3(2C_{\mathrm{LSI}}^2(pk)^2)^{\frac{k}{2}} \sqrt{\frac{p-1}{n}}\left(1 + \sqrt{\frac{(p-2)}{n}}\left[\frac{(d^k+1)n}{2}\right]^{\frac{1}{p}}\right) d^{\frac{k}{2}}
    \end{align*}
    for every even \(p\geq 2\).

    To obtain a tail bound, let \(\zeta\in (0,1)\) and, for every \(k\in \N\) let \(p_k\) be the smallest even integer such that \(p_k\geq 4k - \ln \zeta\). By Markov's inequality, it follows that
    \[
        \P\left(\left\|\frac{1}{n}\sum_{i=1}^{n}\bT_i \right\|_\frob\geq t_k\right) \leq \frac{\left\|\left\|\frac{1}{n}\sum_{i=1}^{n}\bT_i \right\|_\frob\right\|_{L^{p_k}}^{p_k}}{t_k^{p_k}} \leq e^{-4k + \ln \zeta} \leq \zeta 2^{-k}
    \]
    with
    \[
        t_k = 3e(2C_{\mathrm{LSI}}^2(p_k k)^2)^{\frac{k}{2}} \sqrt{\frac{p_k-1}{n}}\left(1 + \sqrt{\frac{(p_k-2)}{n}}\left[\frac{(d^k+1)n}{2}\right]^{\frac{1}{p_k}}\right) d^{\frac{k}{2}}.
    \]
    In particular, with probability at least \(1-\zeta\),
    \[
       \left\|\frac{1}{n}\sum_{i=1}^{n}\bT_i \right\|_\frob\leq 3e(2C_{\mathrm{LSI}}^2(p_k k)^2)^{\frac{k}{2}} \sqrt{\frac{p_k-1}{n}}\left(1 + \sqrt{\frac{(p_k-2)}{n}}\left[\frac{(d^k+1)n}{2}\right]^{\frac{1}{p_k}}\right) d^{\frac{k}{2}}
    \]
    for every \(k\in \N\). Observe that \(p_k \geq 4k\) and thus
    \[
        \sqrt{\frac{(p_k-2)}{n}}\left[\frac{(d^k+1)n}{2}\right]^{\frac{1}{p_k}} \lesssim \sqrt{p_k-2}\left(\frac{n}{d}\right)^{\frac{1}{4k}}
    \]
    for every \(k\in \N\). This completes the proof.
\end{proof}

We establish the Rosenthal-type inequality used above for sums of independent random tensors. The proof relies on the matrix Rosenthal inequality established by~\cite[Fact B.3]{tropp2026universality}.

\begin{lemma}[Rosenthal inequality for sums of random tensors]\label{lem:rosenthal}
    Suppose that \(\{\bT_i\}_{i=1}^{n}\) are independent random order-\(k\) tensors satisfying \(\E[\bT_i]=0\) for all \(i\in [n]\), and let \(\bT=\sum_{i=1}^{n}\bT_i\). Then, for every \(p=1/2\), \(p=1\) and \(p\geq 3/2\),
    \[
        \left\|\|\bT\|_\frob\right\|_{L^{4p}} \leq \sqrt{4p-1} \left(\sum_{i=1}^{n}\E\|\bT_i\|_\frob^2\right)^{1/2}  + \sqrt{(4p-1)(4p-2)} \left(\frac{d^k+1}{2}\right)^{\frac{1}{4p}} n^{\frac{1}{4p}} \max_{i\in [n]} \|\|\bT_i\|_\frob\|_{L^{4p}}.
    \]
\end{lemma}

\begin{proof}
    The proof relies on a simple application of the matrix Rosenthal inequality of~\cite[Fact B.3]{tropp2026universality} to the symmetric dilation of the vectorization of the random tensors \(\bT_i\).
    Define \(\bM_i\in \R^{(d^k+1)\times (d^k+1)}\) as the symmetric dilation of \(\mathrm{Vec}(\bT_i)\in \R^{d^k}\). In particular, \(\bM_i\) is a block matrix of the form
    \[
        \bM_i = \begin{bmatrix}
            0 & \mathrm{Vec}(\bT_i)^\sT \\
            \mathrm{Vec}(\bT_i) & 0
        \end{bmatrix}
    \]
    where \(\mathrm{Vec}(\bT_i)\) is the vectorization of the order-\(k\) tensor \(\bT_i\). Let \(\bM=\sum_{i=1}^{n}\bM_i\) and notice that
    \[
        \E[\Tr(|\bM|^{p})]^{\frac{1}{p}} = \left(2\E\|\mathrm{Vec}(\bT)\|_2^{p}\right)^{1/p} = 2^{1/p}\left\|\|\bT\|_\frob\right\|_{L^{p}}
    \]
    for every \(p\in \N\). Therefore, applying the matrix Rosenthal inequality of~\cite[Fact B.3]{tropp2026universality} to the sum of symmetric dilations \(\bM\) yields the following bound on the \(L^p\) moment of the Frobenius norm of \(\bT\):
    \begin{align*}
        \left\|\|\bT\|_\frob\right\|_{L^{4p}}  & = 2^{-\frac{1}{4p}}\E[\Tr(|\bM|^{4p})]^{\frac{1}{4p}} 
        \\ & \leq 2^{-\frac{1}{4p}}\bigg( \sqrt{4p-1} \Tr\left((\E[\bM^2])^{2p}\right)^{\frac{1}{4p}}  + \sqrt{(4p-1)(4p-2)} (d^k+1)^{\frac{1}{4p}} \left\| \max_{i\in[n]} \|\bM_i\|_\op \right\|_{L^{4p}} \bigg)
    \end{align*}
    for \(p=1/2\), \(p=1\) and \(p\geq 3/2\). Since the summands \(\bM_i\) are independent and centered, \(\E[\bM^2] = \sum_{i=1}^{n}\E[\bM_i^2]\) and \(\Tr\left((\E[\bM^2])^{2p}\right)^{\frac{1}{4p}} \leq 2^{\frac{1}{4p}} \left(\sum_{i=1}^{n}\E\|\bT_i\|_\frob^2\right)^{1/2}\). Furthermore, \(\|\bM_i\|_\op = \|\mathrm{Vec}(\bT_i)\|_2 = \|\bT_i\|_\frob\) for every \(i\in [n]\). Therefore, \(\left\| \max_{i\in[n]} \|\bM_i\|_\op \right\|_{L^{4p}}\leq n^{\frac{1}{4p}} \max_{i\in [n]} \|\|\bT_i\|_\frob\|_{L^{4p}}\). Combining the bounds above concludes the proof.
\end{proof}

\subsection{Deterministic equivalents for the resolvent matrix}

Define the \emph{empirical resolvent matrix} as
\begin{equation}\label{eq:emp_kernel_resolvent_matrices}
    \widehat{\bR}_\kappa \deq  \left(\frac{\bX\bX^\sT}{n} + \kappa^\star\bI_d\right)^{-1} \in \R^{d\times d},
\end{equation}
where we recall that \(\kappa^\star \equiv \kappa^\star_\kappa\) is defined explicitly in~\eqref{eq:auxiliary_parameters}. In the regime of Theorem~\ref{thm:train_test_equivalents_constant_kappa}, \(\kappa^\star\) is bounded away from zero by assumption. However, we do not have an upper bound on \(\kappa^\star_\kappa\), and thus we will need to allow for the possibility that \(\kappa^\star\) diverges with \(d\).

We momentarily write
\begin{gather*}
    \widetilde{\bR}(\vartheta) \deq  \frac{1}{\kappa^\star}\left(\frac{\bX\bX^\sT}{\kappa^\star n} + (1+\vartheta)\bI_d\right)^{-1}
    ,\quad \widetilde{\bM}(\vartheta) = \frac{1}{\kappa^\star}\left(\frac{\bSigma}{\kappa^\star (1+\widetilde{\mu}(\vartheta))} + (1+\vartheta)\bI_d\right)^{-1},
    \\ \widetilde{\mu}(\vartheta) = \frac{1}{n}\Tr(\bSigma \widetilde{\bM}(\vartheta))
\end{gather*}
such that \(\widehat{\bR}_\kappa = \widetilde{\bR}(0)\) and \(\bM_\kappa = \widetilde{\bM}(0)\). While we are ultimately interested in the case \(\vartheta=0\), we will show the following more general result, which will be used to obtain a second-order deterministic equivalent for \(\widehat{\bR}_\kappa^2\) in Lemma~\ref{lem:second_order_deterministic_equivalent}.

\begin{lemma}\label{lem:first_order_deterministic_equivalent}
    Suppose that Assumptions~\ref{ass:high_dim} and~\ref{ass:input_dist} hold, and that \(\kappa^\star \geq c\) for some constant \(c>0\). 
    Uniformly over \(\vartheta \in [0,1]\),
    \[
        \|\E[\widetilde{\bR}(\vartheta)] - \widetilde{\bM}(\vartheta)\|_\frob  \lesssim \frac{1}{\kappa^\star\sqrt{d}} 
    \]
    Furthermore, for every deterministic \(\bA\in \R^{d\times d}\) with \(\|\bA\|_\frob\leq 1\),
    \[
      \left|\Tr(\bA(\widetilde{\bR}(\vartheta)-\widetilde{\bM}(\vartheta)))\right| \prec \frac{1}{\kappa^\star\sqrt{d}}.
    \]
\end{lemma}

\begin{proof}
    The first part follows directly from \cite[Theorem 2.3]{chouard2022quantitative} and the fact that \(\kappa^\star\geq C\) for some constant \(C>0\). For the concentration, let \(\bX_1,\bX_2\in \R^{d\times n}\) and denote \(\widetilde{\bR}_i=(\bX_i\bX_i^\sT/n + \kappa^\star(1+\vartheta)\bI_d)^{-1}\). Using the resolvent identity \(\widetilde{\bR}_1-\widetilde{\bR}_2 = \frac{1}{n}\widetilde{\bR}_1(\bX_2\bX_2^\sT - \bX_1\bX_1^\sT)\widetilde{\bR}_2\), we have
    \begin{align*}
        \|\widetilde{\bR}_1-\widetilde{\bR}_2\|_\frob & \leq  \frac{1}{n}\|\widetilde{\bR}_1\|_\op \|\bX_2^\sT\widetilde{\bR}_2\|_\op \|\bX_1-\bX_2\|_\frob + \frac{1}{n}\|\widetilde{\bR}_1\bX_1\|_\op\|\widetilde{\bR}_2\|_\op\|\bX_1-\bX_2\|_\frob .
    \end{align*}
    Note that \(\|\widetilde{\bR}_i\|_\op\leq 1/\kappa^\star\) and, by applying the scalar bound
    \[
        \frac{s}{s^2/n+\kappa^\star(1+\vartheta)}
        \leq \frac{\sqrt n}{2\sqrt{\kappa^\star(1+\vartheta)}}
    \]
    to the singular values of \(\bX_i\), we have \(\|\widetilde{\bR}_i\bX_i\|_\op\lesssim\sqrt n\), using that \(\kappa^\star\) is bounded below.
    Substituting this back yields
    \[
        \|\widetilde{\bR}_1-\widetilde{\bR}_2\|_\frob \lesssim \frac{1}{\kappa^\star \sqrt{n}}\|\bX_1-\bX_2\|_\frob.
    \]
    In particular, for any \(\bA\in \R^{d\times d}\), the function \(\bX\mapsto \Tr(\bA\widetilde{\bR}(\vartheta))\) is Lipschitz with respect to the Frobenius norm with Lipschitz constant \(\|\bA\|_\frob / (\kappa^\star \sqrt{ n})\). By the tensorization property of the logarithmic Sobolev inequality, \(\pi_0^{\otimes n}\) admits a logarithmic Sobolev inequality with the same constant \(C_{\mathrm{LSI}}\) as \(\pi_0\). Therefore, by the Gaussian concentration of Lipschitz observables \eqref{eq:Gaussian_Lipschitz_concentration}, for every \(t>0\),
    \[
        \P\left(\left|\Tr(\bA\widetilde{\bR}(\vartheta)) - \E[\Tr(\bA\widetilde{\bR}(\vartheta))]\right|\geq t\right) \leq  2\exp\left(-\frac{t^2 n(\kappa^\star_\kappa)^2}{2C_{\mathrm{LSI}}\|\bA\|_\frob^2}\right).
    \]
    We conclude the proof using the bound on the bias established in the first part of the lemma, along with the definition of stochastic domination.
\end{proof}

Our analysis also requires a second-order deterministic equivalent for the squared resolvent, \(\widehat{\bR}_\kappa^2\). Observe that these squared matrices emerge naturally as the negative derivatives of their first-order counterparts evaluated at zero. Consequently, we can establish this second-order equivalent by applying a finite-difference approximation to the first-order bounds established in Lemma~\ref{lem:first_order_deterministic_equivalent}.

\begin{lemma}\label{lem:derivative_equiv_squared_resolvent}
    We have
    \[
        \partial_\vartheta \widetilde{\bR}(\vartheta) = -\kappa^\star\widetilde{\bR}(\vartheta)^2,\qquad \partial_\vartheta \widetilde{\bM}(\vartheta) = -\kappa^\star\widetilde{\bM}(\vartheta)\left(\frac{\frac{1}{n}\Tr(\bSigma\widetilde{\bM}(\vartheta)^2)}{(1+\widetilde{\mu}(\vartheta))^2 - \frac{1}{n}\Tr(\bSigma^2\widetilde{\bM}(\vartheta)^2)} \bSigma + \bI_d\right)\widetilde{\bM}(\vartheta)
    \]
    for every \(\vartheta>0\). In particular, \(\widehat{\bR}_\kappa^2 = -(\kappa^\star)^{-1}\partial_\vartheta\widetilde{\bR}(0)\) and \(\bM_\kappa^{(2)}=-(\kappa^\star)^{-1}\partial_\vartheta \widetilde{\bM}(0)\).
\end{lemma}

\begin{proof}
    The expression for \(\partial_\vartheta \widetilde{\bR}(\vartheta)\) follows directly from the standard identity for the derivative of a matrix inverse. For \(\partial_\vartheta \widetilde{\bM}(\vartheta)\), we first compute \(\partial_\vartheta \widetilde{\mu}(\vartheta)\) by differentiating the fixed-point equation defining \(\widetilde{\mu}(\vartheta)\):
    \begin{align*}
        \partial_\vartheta \widetilde{\mu}(\vartheta) & = - \frac{\kappa^\star}{ n}\Tr\left(\bSigma\widetilde{\bM}(\vartheta)\left(-\frac{\bSigma}{\kappa^\star(1+\widetilde{\mu}(\vartheta))^2}\partial_\vartheta \widetilde{\mu}(\vartheta) + \bI_d\right)\widetilde{\bM}(\vartheta)\right) 
        \\ & = \frac{1}{ n}\Tr(\bSigma\widetilde{\bM}(\vartheta)\bSigma\widetilde{\bM}(\vartheta))\frac{\partial_\vartheta \widetilde{\mu}(\vartheta)}{(1+\widetilde{\mu}(\vartheta))^2} - \frac{\kappa^\star}{n}\Tr(\bSigma\widetilde{\bM}(\vartheta)^2).
    \end{align*}
    By the definition of \(\widetilde{\bM}(\vartheta)\), we have the matrix inequality \(0 \prec \frac{ \bSigma}{1+\widetilde{\mu}(\vartheta)}\widetilde{\bM}(\vartheta) \prec \bI_d\). Because \(\bSigma\) and \(\widetilde{\bM}(\vartheta)\) commute, it follows that
    \[
        0 \leq \frac{\Tr(\bSigma^2\widetilde{\bM}(\vartheta)^2)}{n(1+\widetilde{\mu}(\vartheta))^2} \leq \frac{1}{n(1+\widetilde{\mu}(\vartheta))}\Tr(\bSigma\widetilde{\bM}(\vartheta)) = \frac{\widetilde{\mu}(\vartheta)}{1+\widetilde{\mu}(\vartheta)} < 1.
    \]
    Thus, the coefficient of \(\partial_\vartheta \widetilde{\mu}(\vartheta)\) is strictly positive, and we can invert the expression to obtain
    \[
        \partial_\vartheta \widetilde{\mu}(\vartheta) = -\kappa^\star\frac{\frac{1}{n}\Tr(\bSigma\widetilde{\bM}(\vartheta)^2)}{1 - \frac{1}{n}\Tr(\bSigma^2\widetilde{\bM}(\vartheta)^2)/(1+\widetilde{\mu}(\vartheta))^2}.
    \]
    Substituting this back into the chain rule expansion for \(\partial_\vartheta \widetilde{\bM}(\vartheta)\) yields the desired expression.
\end{proof}

In order to establish the second-order deterministic equivalent, we need as a preliminary step to bound the higher-order derivatives of \(\E[\widetilde{\bR}(\vartheta)]\) and \(\widetilde{\bM}(\vartheta)\) with respect to \(\vartheta\).

\begin{lemma}\label{lem:higher_order_derivative_bounds}
    For every fixed \(\vartheta \in [0,C)\) and every \(k \in \N\),
    \[
        \|\partial_\vartheta^k \E[\widetilde{\bR}(\vartheta)]\|_\op \leq \frac{k!}{\kappa^\star_\kappa}, \qquad \|\partial_\vartheta^k \widetilde{\bM}(\vartheta)\|_\op \leq \frac{k! 2^{k+2}}{\kappa^\star_\kappa}.
    \]
\end{lemma}

\begin{proof}
    We first establish a bound on the higher-order derivatives of \(\E[\widetilde{\bR}(\vartheta)]\). For every \(k\in \N\), we use the dominated convergence theorem to exchange the order of differentiation and expectation, and then apply the standard formula for the derivative of a matrix inverse to obtain
    \[
        \partial_\vartheta^k \E[\widetilde{\bR}(\vartheta)] = \frac{(-1)^k k!}{\kappa^\star} \E[(\kappa^\star\widetilde{\bR}(\vartheta))^{k+1}].
    \]
    The operator norm of the resolvent is bounded by the scalar shift \(\|\kappa^\star\widetilde{\bR}(\vartheta)\|_\op \leq 1\). By Jensen's inequality and the sub-multiplicativity of the operator norm, it immediately follows that
    \[
        \|\partial_\vartheta^k \E[\widetilde{\bR}(\vartheta)]\|_\op \leq\frac{ k! }{\kappa^\star}\E[\|\kappa^\star\widetilde{\bR}(\vartheta)\|_\op^{k+1}] \leq \frac{k!}{\kappa^\star}.
    \]

    For the higher-order derivatives of \(\widetilde{\bM}(\vartheta)\), let \(\bM(z)\) denote the solution of the fixed-point equation~\eqref{eq:resolvent_fixed_point}. By
\eqref{eq:resolvent_fixed_point_measure},
\[
    \bM(z)
    =
    \int_{\R}\frac{1}{\lambda-z}\,\bOmega(\de\lambda),
    \qquad
    z\in\C\setminus\operatorname{supp}(\bOmega),
\]
where \(\bOmega\) is a positive matrix-valued measure satisfying
\(\bOmega(\R)=\bI_d\) and
\(\operatorname{supp}(\bOmega)\subseteq[0,\infty)\).
This representation extends \(\bM\) analytically across the
negative real axis and satisfies the Schwarz symmetry
\(\bM(\overline z)=\bM(z)^\ast\).

For every real \(\vartheta\geq0\), the matrix \(\widetilde{\bM}(\vartheta)\) defined above is
the restriction of this resolvent to the negative real axis:
\[
    \widetilde{\bM}(\vartheta)
    =
    \bM\bigl(-\kappa^\star(1+\vartheta)\bigr)
    =
    \int_{\R}
    \frac{1}{\lambda+\kappa^\star(1+\vartheta)}
    \,\bOmega(\de\lambda).
\]
Indeed, both sides solve the same fixed-point equation, whose admissible
solution is unique. Differentiating under the integral therefore gives
\[
    \partial_\vartheta^k\widetilde{\bM}(\vartheta)
    =
    (-1)^k k!(\kappa^\star)^k
    \int_{\R}
    \frac{\bOmega(\de\lambda)}
    {(\lambda+\kappa^\star(1+\vartheta))^{k+1}}.
\]
Since \(\lambda\geq0\) on the support of \(\bOmega\) and
\(\bOmega(\R)=\bI_d\), it follows that
\[
    \|\partial_\vartheta^k\widetilde{\bM}(\vartheta)\|_{\op}
    \leq
    \frac{k!}{\kappa^\star(1+\vartheta)^{k+1}}
    \leq
    \frac{k!}{\kappa^\star},
\]
which proves the desired bound.
\end{proof}


\begin{lemma}\label{lem:second_order_deterministic_equivalent}
    Suppose that Assumptions~\ref{ass:high_dim} and~\ref{ass:input_dist} hold, and that \(\kappa^\star \geq c\) for some constant \(c>0\). Then,
    \[
        \|\E[\widehat{\bR}_\kappa^2] - \bM^{(2)}_\kappa\|_\frob \prec \frac{1}{(\kappa^\star_\kappa)^2\sqrt{d}}
    \]
    Furthermore, for any deterministic \(\bA \in \R^{d\times d}\) with \(\|\bA\|_\frob \leq 1\),
    \[
       \left|\Tr(\bA(\widehat{\bR}_\kappa^2-\bM_\kappa^{(2)}))\right| \prec \frac{1}{(\kappa^\star_\kappa)^2\sqrt{d}}.
    \]
\end{lemma}

\begin{proof}
    Let \(\bA \in \R^{d\times d}\) with \(\|\bA\|_\frob\leq 1\) and define \(f_\bA(\vartheta) = \Tr(\bA(\E[\widetilde{\bR}(\vartheta)] - \widetilde{\bM}(\vartheta)))\). Let \(k\in \N\). Using the standard \((k+1)\)-point forward difference formula for the derivative at \(\vartheta=0\) with nodes \(\vartheta_j=j\Delta/k\) for \(j=0,\ldots, k\), we obtain
    \begin{equation}\label{eq:forward_diff_derivative}
        |f_\bA'(0)| \leq \frac{k 2^{k+1}}{\Delta}\max_{j\in \{0,\ldots, k\}}|f_\bA(\vartheta_j)| + \frac{\Delta^k}{(k+1)k^k}\max_{\vartheta\in [0,\Delta]}|f_\bA^{(k+1)}(\vartheta)|,
    \end{equation}
    where we bounded the sum of the absolute values of the Lagrange basis derivatives as \(\sum_{j=0}^{k}|L_j'(0)|\leq k 2^{k+1}/\Delta\) (e.g., see \cite[Lemma 6.2]{cheng2024dimension}).

    By Lemma~\ref{lem:first_order_deterministic_equivalent}, the function values are bounded by 
    \[
         \frac{k 2^{k+1}}{\Delta}\max_{j\in \{0,\ldots, k\}}|f_\bA(\vartheta_j)| \lesssim \frac{k 2^{k+1}}{\Delta \kappa^\star\sqrt{d}}.
    \]
    For the remainder term, we bridge the trace to the operator norm using the inequality \(|\Tr(\bA\bB)| \leq \|\bA\|_\frob \|\bB\|_\frob \leq \sqrt{d} \|\bA\|_\frob \|\bB\|_\op\). Applying Lemma~\ref{lem:higher_order_derivative_bounds} to the \((k+1)\)-th derivative yields
    \[
        |f_\bA^{(k+1)}(\vartheta)| \leq \sqrt{d} \left( \|\partial_\vartheta^{k+1}\E[\widetilde{\bR}(\vartheta)]\|_\op + \|\partial_\vartheta^{k+1}\widetilde{\bM}(\vartheta)\|_\op \right) \leq \frac{\sqrt{d} (k+1)! 2^{k+4}}{\kappa^\star}.
    \]
    Plugging this into the second term of~\eqref{eq:forward_diff_derivative}, the \((k+1)\) factor in the denominator cancels, giving
    \[
        \frac{\Delta^k}{(k+1)k^k}\max_{\vartheta\in [0,\Delta]}|f_\bA^{(k+1)}(\vartheta)| \leq \frac{\Delta^k k! 2^{k+4} \sqrt{d}}{k^k \kappa^\star}.
    \]
    Bringing both bounds together, we have shown that
    \[
         |f_\bA'(0)| \lesssim \frac{C_k}{\kappa^\star}\left(\frac{1}{\Delta \sqrt{d}} + \Delta^k \sqrt{d}\right)
    \]
    for some constant \(C_k > 0\) depending on \(k\). This proves the first part of the lemma by choosing \(\Delta = d^{-\epsilon}\) for some sufficiently small \(\epsilon > 0\) and \(k=\lceil \epsilon^{-1}\rceil\), and then using Lemma~\ref{lem:derivative_equiv_squared_resolvent} to divide \(f_\bA'(0)\) by \(-\kappa^\star\) and identify \(\Tr(\bA(\E[\widehat{\bR}_\kappa^2] - \bM_\kappa^{(2)}))\).

    We apply a similar argument as in the first-order case to establish the concentration of \(\Tr(\bA\widehat{\bR}_\kappa^2)\) around its expectation. Let \(\bX_1,\bX_2\in \R^{d\times n}\) and denote \(\widehat{\bR}_i=(\bX_i\bX_i^\sT/n + \kappa^\star\bI_d)^{-1}\). Then,
    \[
        \|\widehat{\bR}_1^2 - \widehat{\bR}_2^2\|_\frob  \leq (\|\widehat{\bR}_1\|_\op + \|\widehat{\bR}_2\|_\op) \|\widehat{\bR}_1 - \widehat{\bR}_2\|_\frob \leq \frac{2}{\kappa^\star_\kappa}\|\widehat{\bR}_1 - \widehat{\bR}_2\|_\frob.
    \]
    We established in the proof of Lemma~\ref{lem:first_order_deterministic_equivalent} that \(\|\widehat{\bR}_1 - \widehat{\bR}_2\|_\frob \leq \frac{1}{\sqrt{n}\kappa^\star_\kappa}\|\bX_1-\bX_2\|_\frob\). Therefore, the function \(\bX\mapsto \Tr(\bA\widehat{\bR}_\kappa^2)\) is Lipschitz with respect to the Frobenius norm with Lipschitz constant \(2\|\bA\|_\frob / (\sqrt{n}(\kappa^\star_\kappa)^2)\). We conclude by following the exact same concentration argument as in the first-order case.
\end{proof}

Finally, we need a quantitative local law that characterizes the distance between the empirical resolvent and its deterministic equivalent outside the spectral support. Recall the empirical resolvent
\begin{equation}
    \widehat{\bR}(z) \deq  \left(\frac{\bX\bX^\sT}{n} - z\bI_d\right)^{-1} \in \C^{d\times d},
\end{equation}
which is well-defined for all \(z \in \C\) outside the spectrum of \(\bX\bX^\sT/n\). The corresponding deterministic equivalent is given by \(\bM(z) = (\tfrac{\bSigma}{1+\mu(z)} - z\bI_d)^{-1}\), where \(\mu(z)\) is the unique solution to the fixed-point equation \(\mu(z) = \tfrac{1}{n}\Tr(\bSigma\bM(z))\) as defined in~\eqref{eq:resolvent_fixed_point}.

The following proposition, which follows directly from the local law for sample covariance matrices established in \cite[Theorem 2.10]{fan2026anisotropic}, provides a uniform bound on the distance between the empirical resolvent and its deterministic equivalent outside the spectral support. Note that by Assumption~\ref{ass:input_dist}.(b),  \(c \leq \tau_\bSigma \leq  C_{\mathrm{LSI}} \), and therefore there exists $c>0$ such that at most $(1-c)d$ eigenvalues of $\bSigma$ belong to the interval $[0,c]$.

\begin{proposition}[{Outside-support anisotropic local law \cite[Theorem 2.10]{fan2026anisotropic}}]\label{prop:outside_local_law}
Suppose that Assumptions~\ref{ass:high_dim} and~\ref{ass:input_dist} hold.
For every fixed \(\delta>0\), every deterministic \(0\preceq \bA\in\R^{d\times d}\) with
\(\|\bA\|_\ast\le1\), and every \(\epsilon,D>0\), there exists \(C>0\) such that, with probability at least \(1-Cd^{-D}\), uniformly over all \(z\in\C\) at distance at least \(\delta\) from the support of the limiting spectral distribution of $\bX \bX^\sT /n$ and satisfying \(|z|\ge\delta\), \(|\Re z|\le\delta^{-1}\), and \(|\Im z|\le1\), we have
\[
    \left|\Tr\!\left(\bA(\widehat{\bR}(z)-\bM(z))\right)\right|
    \le d^{\epsilon-\frac12}.
\]
\end{proposition}

\subsection{Power series decomposition of the RKHS}
\label{sec:rkhs_decomposition}

Under the regularity conditions of Assumption~\ref{ass:kernel_func_vanishing_kappa}, the base function admits a power series expansion \(h(x) = \sum_{k=0}^\infty\alpha_k x^k\) with non-negative coefficients. This analytic structure naturally lifts to the inner-product kernel \(K(\bx, \bx')\), inducing a direct sum decomposition of the associated RKHS into orthogonal subspaces of active homogeneous-polynomial degrees
\[
    \mathcal A_h\deq \{k\ge3:\alpha_k>0\}.
\]

\begin{lemma}[RKHS Functional Decomposition]\label{lem:rkhs_decomposition}
    Suppose that Assumption~\ref{ass:kernel_func_vanishing_kappa} holds. Then, every function $f \in \cH$ can be uniquely represented as
    \[
        f(\bx) = \<\btheta, \bx\> + \sum_{k\in\mathcal A_h} \sqrt{\frac{\alpha_k}{d^k}} \<\bW_k, \bx^{\otimes k}\>,
    \]
    where \(\btheta \in \R^d\) and \(\bW_k \in \mathrm{Sym}^k(\R^d)\) is a symmetric tensor for all \(k \geq 3\). Furthermore, the associated RKHS norm of \(f\) is given by
    \[
        \|f\|_{\cH}^2 = d\|\btheta\|_2^2 +  \sum_{k\in\mathcal A_h} \|\bW_k\|_{\frob}^2.
    \]
\end{lemma}

\begin{proof}
    By Assumption~\ref{ass:kernel_func_vanishing_kappa}, the base function $h$ admits a power series expansion \(h(x) = \sum_{k=0}^{\infty}\alpha_k x^k\) with \(\alpha_0 = \alpha_2 = 0\) and \(\alpha_1 = 1\). Evaluating the kernel \(K(\bx, \by) = h(\<\bx, \by\>/d)\) and using the standard tensor product identity \(\<\bx, \by\>^k = \<\bx^{\otimes k}, \by^{\otimes k}\>_\frob\), we can expand it as
    \[
        K(\bx, \by) = \frac{1}{d}\<\bx, \by\> +  \sum_{k\in\mathcal A_h} \left\< \sqrt{\frac{\alpha_k}{d^k}}\bx^{\otimes k}, \sqrt{\frac{\alpha_k}{d^k}}\by^{\otimes k} \right\>_\frob.
    \]
    This defines a natural feature map \(\phi: \R^d \to \cF\), where the feature space is the direct sum of active symmetric tensor spaces \(\cF = \R^d \oplus \bigoplus_{k\in \cA_h} \mathrm{Sym}^k(\R^d)\), explicitly given by
    \[
        \phi(\bx) = \left( \frac{1}{\sqrt{d}}\bx,\left(\sqrt{\frac{\alpha_k}{d^k}}\bx^{\otimes k}\right)_{k\in \cA_h} \right).
    \]
    By construction, the kernel is exactly the inner product in this feature space, \(K(\bx, \by) = \<\phi(\bx), \phi(\by)\>_{\cF}\). By the standard isometry between the RKHS and the feature space image (see, for instance,~\cite[Theorem 4.21]{steinwart2008support}), for any function \(f \in \cH\) we can write
    \[
        f(\bx) = \<\phi(\bx), \bw\>_{\cF} = \frac{1}{\sqrt{d}}\< \bx, \bw_1 \> + \sum_{k\in\mathcal A_h} \sqrt{\frac{\alpha_k}{d^k}}\< \bx^{\otimes k}, \bW_k \>_\frob
    \]
    for some \(\bw = (\bw_1, \bW_3, \bW_4, \dots)\in \cF\), where \(\bw_1 \in \R^d\) and \(\bW_k \in \mathrm{Sym}^k(\R^d)\). We match the linear term to the form in the lemma by defining \(\btheta = \bw_1 / \sqrt{d}\), which gives the desired functional form. Furthermore, we can write the RKHS norm of \(f\in \cH\) as
    \[
        \|f\|_\cH^2 \deq  \inf\{\|\bw\|^2_\cF : \bw\in \cF,\, f = \<\phi(\cdot), \bw\>_\cF\} = d\|\btheta\|_2^2 + \sum_{k\in\mathcal A_h} \|\bW_k\|_\frob^2.
    \]
    Here, we have used the fact that the RKHS norm of \(f\) is defined as the minimum \(\cF\)-norm of any weight sequence \(\bw\) that represents \(f\) through the feature map. This infimum is uniquely attained by the weight sequence $\bw$ that lies in the closure of the linear span of the feature map $\phi(\R^d)$.
\end{proof}

Next, we bound the contribution of the higher-order terms in the power series decomposition of the RKHS. We will define
\begin{equation}\label{eq:Hge3}
    \cH_{\geq 3} \deq  \{f\in \cH : f(\bx) = \sum_{k \in \cA_h} \sqrt{\frac{\alpha_k}{d^k}} \<\bW_k, \bx^{\otimes k}\>\}
\end{equation}
as the subspace of the RKHS consisting of functions with no linear component. Before stating the bound, we briefly detour to introduce the necessary machinery regarding multivariate Hermite tensors and the Wiener chaos expansion, which will provide the required isometries. Let $\gamma_d (\bx) = (2\pi)^{-\frac{d}{2}} e^{- \| \bx\|_2^2/2} $ denote the density of a standard Gaussian vector in $\R^d$. For any \(d\geq 1\) and $k \geq 0$, define $\He_{k} : \R^d \to \mathrm{Sym}^k (\R^d)$ to be the normalized degree-$k$ Hermite tensor in $\R^d$ given by
\begin{equation}\label{eq:Hermite-tensor}
\He_{k} (\bx) \deq  \frac{(-1)^k}{\sqrt{k!}} \frac{\nabla^k \gamma_d (\bx)}{\gamma_d (\bx)},
\end{equation}
where $\nabla^k \gamma_d(\bx)$ denotes the $k$-th derivative of $\gamma_d$, viewed as a symmetric $k$-tensor. In particular, \(\He_0(\bx) = 1\) and \(\He_1(\bx) = \bx\).

The Hermite tensors realize the classical isometry between $\mathrm{Sym}^k (\R^d)$ and the $k$-th Wiener chaos:  for all $\bA \in \mathrm{Sym}^k (\R^d)$ and $ \bB \in \mathrm{Sym}^j (\R^d)$,
\begin{equation}
    \E_{\gamma_d} [ \< \bA,\He_k (\bx) \>_\frob\< \bB,\He_j (\bx) \>_\frob ] = \ind_{j=k} \< \bA, \bB\>_\frob.
\end{equation}
Hence any $f\in L^2(\gamma_d)$ admits the Wiener chaos expansion
\begin{equation*}
    f(\bx)  = \sum_{k = 0}^\infty \< \bA_k, \He_k (\bx) \>_\frob , \qquad \bA_k \deq  \E_{\gamma_d} [ f(\bx) \He_k (\bx)] \in \mathrm{Sym}_k (\R^d),
\end{equation*}
with convergence in $L^2(\gamma_d)$. Of particular interest to us will be the Hermite coefficients
\begin{equation}\label{eq:Gamma_k_def}
    \Gamma_k f(\bx) \deq  \E_{\bz\sim\cN(0,\bI_d)} [ f(\rho\bx + \sigma\bz) \He_k (\bz) ].
\end{equation}
In light of the general Wiener chaos expansion, the tensor $\Gamma_k f(\bx)$ is exactly the $k$-th Hermite coefficient of the shifted and scaled function $\bz \mapsto f(\rho\bx + \sigma\bz)$ for a fixed $\bx$. Consequently, we can decompose this function over the Gaussian noise $\bz$ as
\begin{equation}
    f(\rho\bx + \sigma\bz) = \sum_{k=0}^\infty \< \Gamma_k f(\bx), \He_k (\bz) \>_\frob.
\end{equation}

\begin{lemma}\label{lem:high_order_term_bound}
    Suppose that Assumptions~\ref{ass:high_dim},~\ref{ass:input_dist}, and \ref{ass:kernel_func_vanishing_kappa} hold. Then,
    \begin{equation}\label{eq:C_delta_def}
        \|g\|_{L^2(\mu_d)}\lesssim d^{-3/4}\|g\|_\cH,\qquad
        \left\|\frac{1}{n}\sum_{i=1}^{n}\Gamma_1 g(\bx_{0,i})\right\|_2 \lesssim \frac{C_\zeta}{d}\|g\|_\cH, \qquad C_{\zeta}^2 \deq  \sum_{k=3}^{\infty} k^2 \alpha_k C_{\zeta,k}^2,
    \end{equation}
    for every \(g\in\cH_{\geq 3}\) with power series representation \eqref{eq:Hge3}, with probability at least \(1-\zeta\), where $\zeta$ and $C_{\zeta,k}$ are given in Assumption \ref{ass:kernel_func_vanishing_kappa}.
\end{lemma}

\begin{proof}  We first establish the result for finite truncations of the series. Let \(g_L\) be the polynomial of maximum degree \(L\geq 3\), defined as
    \[
        g_L(\bx) = \sum_{k=3}^{L} \sqrt{\frac{\alpha_k}{d^k}} \<\bW_k, \bx^{\otimes k}\>.
    \]
    By Lemma~\ref{lem:rkhs_decomposition},
    \[
        \|g_L\|_\cH^2 = \sum_{k=3}^{L} \|\bW_k\|_\frob^2.
    \]
    Let \(m,k\) be arbitrary positive integers. By Cauchy-Schwarz inequality and Lemma~\ref{lem:moment_bound},
    \begin{align*}
        \E_{\bx\sim \mu_d}[\<\bW_m,\bx^{\otimes m}\>\<\bW_k,\bx^{\otimes k}\>] & \leq \sqrt{\E_{\bx\sim \mu_d}[\<\bW_m,\bx^{\otimes m}\>^2]\E_{\bx\sim \mu_d}[\<\bW_k,\bx^{\otimes k}\>^2]}
        \\ & \leq B_{m}B_{k}d^{\frac{m+k}{4}}\|\bW_m\|_\frob\|\bW_k\|_\frob
    \end{align*}
    where we defined \( B_{k}^2\deq (2C_{\mathrm{LSI}}^2k^2)^{k}\). Thus,
    \begin{align*}
        \|g_L\|_{L^2(\mu_d)}^2 &  \leq  \sum_{m=3}^{L}\sum_{k=3}^{L} \sqrt{\alpha_m\alpha_k}  B_{m}B_{k}d^{-\frac{m+k}{4}}\|\bW_m\|_\frob\|\bW_k\|_\frob
        \\ & \leq d^{-\frac{3}{2}}\left(\sum_{k=3}^{L}\sqrt{\alpha_k}B_k\|\bW_k\|_\frob\right)^2
         \leq d^{-\frac{3}{2}}\|g_L\|_\cH^2\sum_{k=3}^{L}\alpha_k B^2_k,
    \end{align*}
    where the last inequality follows from Cauchy-Schwarz. Since \(B_k\leq C_{\zeta,k}\) for any \(\zeta\in (0,1)\), it follows from Assumption \ref{ass:kernel_func_vanishing_kappa} that \(\sum_{k=3}^{L}\alpha_k B^2_k\leq \sum_{k=3}^{\infty}\alpha_k B^2_k<\infty\) and hence  \( \|g_L\|_{L^2(\mu_d)}^2 \lesssim d^{-3/2}\|g_L\|_\cH^2\). Since \(\|g_L-g\|_\cH\to 0\), the reproducing property gives \(g_L(\bx)\to g(\bx)\) pointwise. Fatou's lemma and \(\|g_L\|_\cH\to\|g\|_\cH\) therefore yield
    \[
        \|g\|_{L^2(\mu_d)}^2
        \leq \liminf_{L\to\infty}\|g_L\|_{L^2(\mu_d)}^2
        \lesssim d^{-3/2}\|g\|_\cH^2,
    \]
    proving the first inequality.

    We now turn to the second inequality. Let \(k\in [L]\) and define \(f_k(\bx)=\<\bW_k,\bx^{\otimes k}\>_\frob\). Then, by Gaussian integration by parts,
    \[
        \Gamma_1 f_k(\bx) = \E_{\bz\sim\cN(0,\bI_d)}[\nabla_{\bz}\<\bW_k,(\rho\bx+\sigma\bz)^{\otimes k}\>_\frob] = \sigma k \bW_k \otimes_{k-1} \E_{\bz\sim \cN(0,\bI_d)}[(\rho\bx+\sigma\bz)^{\otimes (k-1)}]
    \]
    where, for any \(\bA\in \mathrm{Sym}_k(\R^d),\bB\in \mathrm{Sym}_{k-1}(\R^d)\), \(\bA\otimes_{k-1}\bB\in \R^{d}\) denotes the partial contraction along any \(k-1\) indices. By triangle inequality, sub-multiplicativity of the Frobenius norm, Lemma~\ref{lem:moment_bound_empirical} and Cauchy-Schwarz inequality,
    \begin{align*}
        \left\|\frac{1}{n}\sum_{i=1}^{n}\Gamma_1 g_L(\bx_{0,i})\right\|_2 & \leq \sum_{k=3}^{L} \sigma k \sqrt{\frac{\alpha_k}{d^k}}\|\bW_k\|_\frob \left\|\frac{1}{n}\sum_{i=1}^{n}\E_{\bz\sim \cN(0,\bI_d)}[(\rho\bx_{0,i}+\sigma\bz)^{\otimes (k-1)}]\right\|_\frob
        \\ & \lesssim \sigma\sum_{k=3}^{L} k \sqrt{\frac{\alpha_k}{d^k}}\|\bW_k\|_\frob C_{\zeta,k-1}d^{\frac{k-2}{2}}
        \leq \frac{\sigma}{d}\|g_L\|_\cH \left(\sum_{k=3}^{L} k^2 \alpha_k C_{\zeta,k-1}^2\right)^{1/2}
    \end{align*}
    with probability at least \(1-\zeta\). Since \(C_{\zeta,k-1}\leq C_{\zeta,k}\), it follows from Assumption \ref{ass:kernel_func_vanishing_kappa} that \(\sum_{k=3}^{L} k^2 \alpha_k C_{\zeta,k}^2\leq\sum_{k=3}^{\infty} k^2 \alpha_k C_{\zeta,k}^2= C_\zeta^2 <\infty\). The event in Lemma~\ref{lem:moment_bound_empirical} is simultaneous over all \(k\), so the preceding estimate holds for every \(L\) on the same event. Moreover, for each fixed training sample, the reproducing property and Cauchy--Schwarz give
    \[
        \|\Gamma_1(g_L-g)(\bx_{0,i})\|_2
        \leq \|g_L-g\|_\cH
        \left(\E_{\bz}\!\left[K(\rho\bx_{0,i}+\sigma\bz,\rho\bx_{0,i}+\sigma\bz)\|\bz\|_2^2\right]\right)^{1/2}
        \longrightarrow 0,
    \]
    where the expectation is finite under Assumption~\ref{ass:kernel_func_vanishing_kappa}. We may therefore take \(L\to\infty\) in the preceding estimate, which proves the second inequality.
\end{proof}

\begin{remark}[Degree-damped power-series kernels are admissible]\label{rmk:admissible-kernels}
Assumption~\ref{ass:kernel_func_vanishing_kappa} is satisfied by a broad family of degree-damped power-series kernels. This includes, for instance, kernel functions of the form
\[
    h(x)=x+\sum_{k=3}^{\infty}
    \frac{b_k}{(k!)^q}x^k, \qquad q >4,\qquad 0\leq b_k\leq e^{O(k)} \text{ for } k\geq 3,\qquad b_k>0 \text{ for } k_0\leq k\leq 2k_0.
\]
\end{remark}

\clearpage

\section{Linearization of the kernel and localized bumps}
\label{app:lin_bumps}

This appendix proves the kernel linearization results and constructs localized RKHS bumps. Throughout, the diffusion time \(t>0\) is fixed, so we suppress the dependence on \(t\), and treat \(\sigma,\rho\) as positive constants.

\subsection{Linearization of the kernels and Bayes denoiser}
\label{app:lin}

In this section, we establish the linearization results stated in propositions~\ref{prop:linearization_kernel_op} and~\ref{prop:bayes_opt_denoiser_linearization}.

\begin{proof}[Proof of Proposition~\ref{prop:linearization_kernel_op}]
    We provide a complete proof of the first bound in Proposition~\ref{prop:linearization_kernel_op}. The second bound follows by a similar argument and we only sketch the main differences at the end of the proof.

    \paragraph*{Decomposition.}

    For every \(i,j\in [n]\) and \(\bz,\bz'\in \R^d\), define \(U_{i,j}(\bz,\bz') \deq  \<\rho \bx_{0,i}+\sigma \bz,\rho\bx_{0,j}+\sigma\bz'\>/d\),
    \[
        D_{i,j}(\bz,\bz') \deq  h(U_{i,j}(\bz,\bz')) - U_{i,j}(\bz,\bz') - \ind_{i=j}(h(\rho^2\tau_{\bSigma}) - \rho^2\tau_{\bSigma}),
    \]
    and recall that \(\tau_{\bSigma} = \Tr(\bSigma)/d\) is the average eigenvalue of \(\bSigma\). Let \(f \in L^2(\widehat{\pi}_{d,n})\) be an arbitrary function with \(\|f\|_{L^2(\widehat{\pi}_{d,n})}=1\).  We decompose the operator difference into its off-diagonal and diagonal components: \((\widehat{\cK} - \widehat{\cK}_{\mathrm{eff}})f = \Delta_{\mathrm{od}}f + \Delta_{\mathrm{d}}f\), where
    \begin{align*}
        \Delta_{\mathrm{od}}f(\bx_{0,i},\bz) = \frac{1}{n}\sum_{j\neq i}\E_{\bz'}\left[D_{i,j}(\bz,\bz') f(\bx_{0,j}, \bz')\right], \quad \Delta_{\mathrm{d}}f(\bx_{0,i},\bz) = \frac{1}{n}\E_{\bz'}\left[D_{i,i}(\bz,\bz') f(\bx_{0,i}, \bz')\right]
    \end{align*}
    for every \(i\in [n]\) and \(\bz\in \R^d\). By the triangle inequality, \(\|(\widehat{\cK} - \widehat{\cK}_{\mathrm{eff}})f\|_{L^2(\widehat{\pi}_{d,n})} \leq \|\Delta_{\mathrm{od}}f\|_{L^2(\widehat{\pi}_{d,n})} + \|\Delta_{\mathrm{d}}f\|_{L^2(\widehat{\pi}_{d,n})}\). We bound each term separately.

    \paragraph*{Off-diagonal term.}

    For the off-diagonal operator, we apply Cauchy-Schwarz over the joint measure for \(j \neq i\):
    \begin{align*}
        \|\Delta_{\mathrm{od}}f\|_{L^2(\widehat{\pi}_{d,n})}^2 & = \frac{1}{n} \sum_{i=1}^{n}\E_{\bz}\Bigg[ \Bigg( \frac{1}{n}\sum_{j\neq i}\E_{\bz'}\left[D_{i,j}(\bz,\bz') f(\bx_{0,j}, \bz')\right]\Bigg)^2 \Bigg]
        \\ & \leq \frac{1}{n^2}\sum_{i\neq j}\E_{\bz,\bz'}\left[D_{i,j}(\bz,\bz')^2\right] = \frac{1}{n^2}\sum_{i\neq j}\E_{\bz,\bz'}\left[\left(h\left(U_{i,j}(\bz,\bz')\right) - U_{i,j}(\bz,\bz')\right)^2\right].
    \end{align*}
    Taylor expanding \(h\) around \(0\),
    \[
        h(U_{i,j}) - U_{i,j} = \frac{1}{6}h'''(\xi_{i,j})U_{i,j}^3
    \]
    for some \(\xi_{i,j}\) between \(0\) and \(U_{i,j}\). The exponential-growth bound in Assumption~\ref{ass:kernel_func_constant_kappa} along with Young's inequality and Lemma~\ref{lem:exp_noisy_overlap_moments} on the event \(\max_i\|\bx_{0,i}\|_2^2/d\lesssim 1\) gives
    \[
        \max_{i,j}\E_{\bz,\bz'}\left[|h'''(\xi_{i,j})|^4\right]\lesssim 1
    \]
    with very high probability. 
    Hence, 
    \[
        \E_{\bz,\bz'}\left[\left(h(U_{i,j})-U_{i,j}\right)^2\right]
        \lesssim
        \E_{\bz,\bz'}\left[|U_{i,j}|^{12}\right]^{1/2}.
    \]
    By~\eqref{eq:norm_concentration_overlaps} and~\eqref{eq:norm_concentration_overlaps_lp} (which hold under Assumption~\ref{ass:input_dist}) and standard properties of Gaussian vectors,
    \begin{equation}\label{eq:standard_overlap_bounds}
        \begin{gathered}
        \max_{i\neq j}|\<\bx_{0,i},\bx_{0,j}\>|\prec \sqrt{d},\qquad \max_{i\in [n]}\E_{\bz\sim\cN(0,\bI_d)}[|\<\bx_{0,i},\bz\>|^k]^{1/k}\prec \sqrt{d},\\ \E_{(\bz,\bz')\sim \cN(0,\bI_d)^{\otimes 2}}[|\<\bz,\bz'\>|^k]^{1/k}\lesssim \sqrt{d},
        \end{gathered}
    \end{equation}
    for every positive integer \(k\). It follows by expanding the definition of \(U_{i,j}\) and Assumption~\ref{ass:high_dim} that \( \|\Delta_{\mathrm{od}}f\|_{L^2(\widehat{\pi}_{d,n})} \prec d^{-3/2}\). Since \(f\in L^2(\widehat{\pi}_{d,n})\) is arbitrary, we conclude that \( \|\Delta_{\mathrm{od}}\|_{\op} \prec d^{-3/2}\) by the variational characterization of the operator norm.

    \paragraph*{Diagonal term.}
    For the diagonal operator, we apply Cauchy-Schwarz solely to the inner expectation over \(\bz'\):
    \begin{align*}
        \|\Delta_{\mathrm{d}}f\|_{L^2(\widehat{\pi}_{d,n})}^2 & = \frac{1}{n} \sum_{i=1}^{n}\E_{\bz}\left[ \left( \frac{1}{n}\E_{\bz'}\left[D_{i,i}(\bz,\bz') f(\bx_{0,i}, \bz')\right] \right)^2 \right]
         \leq \frac{1}{n^2}\max_{i\in [n]}\E_{\bz,\bz'}[D_{i,i}(\bz,\bz')^2].
    \end{align*}
    Taylor expanding \(h(x)-x\) around \(\rho^2\tau_\bSigma\), we have \(D_{i,i}(\bz,\bz') = (h'(\xi_{i})-1)(U_{i,i}(\bz,\bz')-\rho^2\tau_\bSigma)\) for some \(\xi_i\) between \(U_{i,i}(\bz,\bz')\) and \(\rho^2\tau_\bSigma\). By Cauchy-Schwarz,
    \[
        \E_{\bz,\bz'}[D_{i,i}(\bz,\bz')^2] \leq \E_{\bz,\bz'}[(h'(\xi_i)-1)^4]^{1/2} \E_{\bz,\bz'}[(U_{i,i}(\bz,\bz')-\rho^2\tau_\bSigma)^4]^{1/2}.
    \]
    Recalling the standard identities in~\eqref{eq:standard_overlap_bounds} and \(\max_{i\in [n]}|\|\bx_{0,i}\|^2 - \Tr(\bSigma)|\prec \sqrt{d}\) by~\eqref{eq:norm_concentration_overlaps}, it follows that \(\E_{\bz,\bz'}[|U_{i,i}(\bz,\bz')-\rho^2\tau_\bSigma|^k]^{1/k} \prec d^{-1/2}\) for every positive integer \(k\). Combining this with~\eqref{eq:cov_op_bound} and the triangle inequality, \(\E_{\bz,\bz'}[|U_{i,i}(\bz,\bz')|^k]^{1/k}\lesssim 1\) with very high probability for every positive integer \(k\). Integrating the exponential-growth bound on \(h'''\) and using \(h'(0)=1\) gives \(|h'(u)-1|\le C\exp(C|u|)\). Since \(\xi_i\) lies on the segment between \(U_{i,i}(\bz,\bz')\) and \(\rho^2\tau_\bSigma\),~\eqref{eq:cov_op_bound}, Lemma~\ref{lem:exp_noisy_overlap_moments} and the convexity of \(x\mapsto e^x\) yields \(\E_{\bz,\bz'}[(h'(\xi_i)-1)^4]^{1/2}\lesssim 1\), uniformly in \(i\), with very high probability. The proportional regime (see Assumption~\ref{ass:high_dim}) yields \( \|\Delta_{\mathrm{d}}f\|_{L^2(\widehat{\pi}_{d,n})} \prec d^{-3/2}\) uniformly over \(f\in L^2(\widehat{\pi}_{d,n})\). This concludes the proof of the first part of Proposition~\ref{prop:linearization_kernel_op}.

    \paragraph*{Population kernel.} In order to adapt the argument to control \(\|\cK-\cK_{\mathrm{eff}}\|_\op\), for every \(i\in [n]\), \(\bx,\bz,\bz'\in\R^d\) define \(U_{i}(\bx,\bz,\bz')\deq \<\rho\bx_{0,i}+\sigma\bz',\rho\bx+\sigma\bz\>/d\) along with
    \[
        D_{i}(\bx,\bz,\bz') \deq  h(U_i(\bx,\bz,\bz'))-U_i(\bx,\bz,\bz').
    \]
    Then, for every \(f\in L^2(\widehat{\pi}_{d,n})\) with \(\|f\|_{L^2(\widehat{\pi}_{d,n})}\leq 1\), we apply Cauchy-Schwarz inequality to obtain
    \begin{align*}
        \|(\cK-\cK_{\mathrm{eff}})f\|^2_{L^2(\pi_d)} 
         \leq  \frac{1}{n}\sum_{i=1}^{n}\E_{(\bx,\bz)\sim \pi_d,\, \bz'\sim\cN(0,\bI_d)}[D_i(\bx,\bz,\bz')^2].
    \end{align*}
    Taylor expanding \(h(x)=x+h'''(\xi_i)x^3/6\) around \(0\), and applying the same H\"older argument, with Assumption~\ref{ass:kernel_func_constant_kappa} and Lemma~\ref{lem:exp_noisy_overlap_moments} controlling \(h'''(\xi_i)\), gives \(\E[D_i(\bx,\bz,\bz')^2]\prec d^{-3}\) uniformly in \(i\). Hence \(\|(\cK-\cK_{\mathrm{eff}})f\|_{L^2(\pi_d)} \prec d^{-3/2}\). 
\end{proof}

Next, we prove Proposition~\ref{prop:bayes_opt_denoiser_linearization}. We recall that \(\widehat{f}^\star_\bv\in L^2(\widehat{\pi}_{d,n})\) is defined via the pullback of~\eqref{eq:emp_bayes_opt_denoiser}.

\begin{proof}[Proof of Proposition~\ref{prop:bayes_opt_denoiser_linearization}]
    By definition (see~\eqref{eq:emp_bayes_opt_denoiser}), we have
    \begin{align*}
        \|\widehat{f}^\star_\bv -\sfz_\bv\|_{L^2(\widehat{\pi}_{d,n})}^2 
         = \frac{\rho^2}{\sigma^2n}\sum_{i=1}^{n}\E_{\bz\sim \cN(0,\bI_d)}\Bigg[\Bigg(\sum_{k\neq i}\omega_k(\rho\bx_{0,i}+\sigma\bz)\<\bx_{0,i}-\bx_{0,k},\bv\>\Bigg)^2\Bigg]
    \end{align*}
    where the weights \(\omega_k\) are defined in~\eqref{eq:emp_bayes_opt_denoiser_weights}. Because \(\sum_{k\neq i} \omega_k \leq 1\), it follows from Jensen's inequality that
    \[
        \Bigg(\sum_{k\neq i}\omega_k(\rho\bx_{0,i}+\sigma\bz)\<\bx_{0,i}-\bx_{0,k},\bv\>\Bigg)^2 \leq \sum_{k\neq i}\omega_k(\rho\bx_{0,i}+\sigma\bz)\<\bx_{0,i}-\bx_{0,k},\bv\>^2,
    \]
    and thus
    \[
        \|\widehat{f}^\star_\bv - \sfz_\bv\|_{L^2(\widehat{\pi}_{d,n})}^2 \leq \frac{\rho^2}{\sigma^2n}\sum_{i=1}^{n}\sum_{k\neq i}\E_{\bz\sim \cN(0,\bI_d)}\left[\omega_k(\rho\bx_{0,i}+\sigma\bz)\right]\<\bx_{0,i}-\bx_{0,k},\bv\>^2.
    \]
    By Assumption~\ref{ass:input_dist} and~\eqref{eq:norm_concentration_overlaps}, for every \(c\in (0,1/2)\),
    \[
        \|\bx_{0,k}-\bx_{0,i}\|_2^2 \geq d^{1-c}, \qquad \<\bx_{0,i}-\bx_{0,k},\bv\>^2 \leq d^{1+c}
    \]
    for every \(i\neq k\in [n]\) with very high probability. We denote this high-probability event for the training data by \(\cE_{\bx}\) and condition on it for the rest of the proof, supposing that \(d\geq C\) is large enough. We also define a ``typical noise'' event \(\cZ_{i,k}\) for each pair \(i\neq k\) as
    \[
        \cZ_{i,k} \deq  \left\{\bz \in \R^d : \frac{\rho}{\sigma}|\<\bx_{0,k}-\bx_{0,i},\bz\>| \leq \frac{\rho^2}{4\sigma^2}\|\bx_{0,k}-\bx_{0,i}\|_2^2 \right\}.
    \]
    By standard Gaussian tail bounds, \(\cZ_{i,k}\) holds with very high probability over the randomness of the noise \(\bz\).
    We use the event \(\cZ_{i,k}\) to safely integrate over the noise \(\bz\). Indeed, for every \(i\neq k\) and \(\bz\in \cZ_{i,k}\), we isolate the \(j=i\) term in the denominator of \(\omega_k\) to obtain the deterministic bound
    \begin{align*}
        \omega_k(\rho\bx_{0,i} + \sigma\bz) & \leq \exp\left(-\frac{\rho^2}{2\sigma^2}\|\bx_{0,k}-\bx_{0,i}\|^2 + \frac{\rho}{\sigma}\<\bx_{0,k}-\bx_{0,i},\bz\>\right)
        \\ & \leq \exp\left(-\frac{\rho^2}{4\sigma^2}\|\bx_{0,k}-\bx_{0,i}\|^2\right) \leq \exp\left(-\frac{\rho^2d^{1-c}}{4\sigma^2}\right).
    \end{align*}
    On the complement \(\cZ_{i,k}^c\), we simply use the trivial bound \(\omega_k \leq 1\). Combining these two bounds yields
    \begin{align*}
        \E_{\bz}\left[\omega_k(\rho\bx_{0,i}+\sigma\bz)\right] & = \E_{\bz}\left[\omega_k \ind_{\cZ_{i,k}}\right] + \E_{\bz}\left[\omega_k \ind_{\cZ_{i,k}^c}\right] 
         \leq \exp\left(-\frac{\rho^2d^{1-c}}{4\sigma^2}\right) + \P_{\bz}(\cZ_{i,k}^c) \leq 2\exp(-d^{c'})
    \end{align*}
    for some \(c' > 0\) and sufficiently large \(d\). Substituting this back into the expression for the squared norm, we obtain
    \[
        \|\widehat{f}^\star_{\bv} - \sfz_\bv\|_{L^2(\widehat{\pi}_{d,n})}^2 \leq \frac{2\rho^2}{\sigma^2n}\sum_{i=1}^{n}\sum_{k\neq i}\exp(-d^{c'})\<\bx_{0,i}-\bx_{0,k},\bv\>^2 \leq \frac{2\rho^2nd^{1+c}}{\sigma^2}\exp(-d^{c'}),
    \]
    where we have used the bound on the overlaps \(\<\bx_{0,i}-\bx_{0,k},\bv\>^2 \leq d^{1+c}\) for every \(i\neq k\in [n]\) on the event \(\cE_{\bx}\). Because \(n \propto d\), the polynomial prefactor is exponentially dominated for sufficiently large \(d\). We conclude that \(\|\widehat{f}^\star_\bv - \sfz_\bv\|_{L^2(\widehat{\pi}_{d,n})}^2 \leq \exp(-d^{c''})\) for some \(c'' > 0\).
\end{proof}

\begin{remark}[Kernel linearization with $h(0), h''(0) \neq 0$]\label{rmk:constant-contribution}
    The conditions $h(0) = h''(0) = 0$ are added for convenience to simplify the analysis. If these conditions do not hold, the kernel linearization will include additional constant-function contributions. Specifically, let $\widehat\cJ$ be the orthogonal projection onto the constant functions in $L^2(\widehat\pi_{d,n})$, let $\cJ:L^2(\widehat\pi_{d,n})\to L^2(\pi_d)$ map a function to its constant empirical mean, and let $ \cE f(\bx_{0,i},\bz)\deq \E_{\bz'}[f(\bx_{0,i},\bz')]$. Then 
\[
 \|\widehat{\cK} - \widehat{\cK}_{\{0,2\}}\|_{\op}
        \vee \|\cK - \cK_{\{0,2\}}\|_{\op}
        \prec d^{-\frac{3}{2}},
\]
where
\[
    \widehat\cK_{\{0,2\}}\deq \widehat\cK_{\mathrm{eff}}
       + h(0) \widehat\cJ + \frac{h''(0)}{2} \nu_d\left(\widehat\cJ-\frac1n\cE\right),
    \qquad
    \cK_{\{0,2\}}\deq \cK_{\mathrm{eff}}+\Big(h(0) + \frac{h''(0)}{2} \nu_d\Big)\cJ,
\]
where $\nu_d\deq \frac{\Tr(\bSigma_\sigma^2)}{d^2}=O_d(d^{-1})$. Now, $\delta_n = (h (\rho^2\tau_\bSigma) - h(0) - \rho^2 \tau_{\bSigma})/n$.  These contributions can be handled separately without affecting the main results.
\end{remark}

\begin{remark}[Non-negativity of \(\delta_{n}\)]\label{rmk:delta_n_nonnegativity}
    The self-induced regularization term \(\delta_n=(h(\rho^2\tau_\bSigma)-\rho^2\tau_\bSigma)/n\), which appears in the definition of the effective kernel operator \(\widehat{\cK}_{\mathrm{eff}}\) in~\eqref{eq:K_hat_op} as a consequence of the kernel linearization, is non-negative under Assumption~\ref{ass:kernel_func_constant_kappa}.

    Indeed, fix \(\bu\in \R^d\), let \(u=\|\bu\|^2/d\), and take \(\epsilon>0\). Since the kernel is positive definite, the Gram matrix associated with the points \(\bu\) and \(\sqrt{\epsilon}\bu\) satisfies
    \[
        \begin{bmatrix}
            h(u) & h(\sqrt{\epsilon} u) \\
            h(\sqrt{\epsilon} u) & h(\epsilon u)
        \end{bmatrix}
        \succeq 0.
    \]
    In particular, its determinant is non-negative, and hence \( h(u)h(\epsilon u) \geq h(\sqrt{\epsilon}\,u)^2\). Dividing by \(\epsilon\) and taking \(\epsilon\to 0\) using \(h(0)=0\) and \(h'(0)=1\), we obtain \(h(u)u\geq u^2\). Thus, \(h(u)\geq u\) for every \(u>0\). Taking \(u=\rho^2\tau_\bSigma\) proves that \(\delta_n\geq0\).

    Moreover, equality at one positive point forces equality everywhere.
    Suppose that \(h(u_0)=u_0\) for some \(u_0>0\), and choose
    \(\bu\in\R^d\) such that \(\|\bu\|_2^2/d=u_0\). Fix \(t\in\R\) and set
    \(\by=(t/u_0)\bu\). For \(\epsilon>0\), positive definiteness applied
    to \(\bu,\by,\epsilon\bu\) gives
    \[
        \begin{bmatrix}
            h(u_0) & h(t) & h(\epsilon u_0)\\
            h(t) & h(t^2/u_0) & h(\epsilon t)\\
            h(\epsilon u_0) & h(\epsilon t) & h(\epsilon^2u_0)
        \end{bmatrix}
        \succeq0.
    \]
    Since \(h(\epsilon^2u_0)>0\) for sufficiently small \(\epsilon\), its
    Schur complement gives
    \[
        \begin{bmatrix}
            h(u_0)&h(t)\\
            h(t)&h(t^2/u_0)
        \end{bmatrix}
        -
        \frac{1}{h(\epsilon^2u_0)}
        \begin{bmatrix}
            h(\epsilon u_0)\\
            h(\epsilon t)
        \end{bmatrix}
        \begin{bmatrix}
            h(\epsilon u_0)&h(\epsilon t)
        \end{bmatrix}
        \succeq0.
    \]
    Letting \(\epsilon\to0\) and using \(h(s)=s+o(s)\) at the origin yields
    \[
        \begin{bmatrix}
            h(u_0)-u_0 & h(t)-t\\
            h(t)-t & h(t^2/u_0)-t^2/u_0
        \end{bmatrix}
        \succeq0.
    \]
    Since \(h(u_0)=u_0\), the upper-left entry vanishes. Non-negativity of
    the determinant therefore gives
    \[
        -(h(t)-t)^2\ge0,
    \]
    and hence \(h(t)=t\). Since \(t\in\R\) was arbitrary, \(h(t)=t\) for
    every \(t\in\R\).
\end{remark}

\subsection{Construction of the localized bumps}\label{app:localized_bumps}

In this section, we construct the localized bump functions used in the proof of Theorem~\ref{thm:train_test_equivalents_vanishing_kappa}. Fix \(m\in\{0,1,2,\ldots\}\) such that \(2m+6 = k_0\). Define
\begin{equation}\label{eq:poly_def}
    p_m(x) \deq  x^{k_0} \sum_{j=0}^{k_0}(-1)^j \binom{k_0+j-1}{j} (x-1)^j
\end{equation}
where the convention is that \((x-1)^0 = 1\) even when \(x=1\). We have the following simple properties of the polynomial \(p_m\).

\begin{lemma}\label{lem:poly_properties}
    Let \(m\in\{0,1,2,\ldots\}\), \(k_0=2m+6\) and \(p_m\) be defined as in~\eqref{eq:poly_def}. Then, \(p_m\) is a polynomial of degree \(2k_0\) with no monomial of degree less than \(k_0\). Moreover, \(p_m(1) = 1\), \(p^{(k-1)}_m(0) = 0\) and \(p^{(k)}_m(1) = 0\) for every \(k\in \{1,\ldots,k_0\}\).
\end{lemma}

\begin{proof}
    Momentarily define \(f(x)=x^{k_0}\) and \(g(x) = \sum_{j=0}^{k_0}(-1)^j \binom{k_0+j-1}{j} (x-1)^j\). Note that \(g\) is precisely the \(k_0\)-th order Taylor expansion of the function \(x\mapsto x^{-k_0}\) around \(x=1\), so \(g^{(k)}(1)=(1/f)^{(k)}(1)\) for every \(k\in \{0,\ldots, k_0\}\). In particular, \(g(1) = 1/f(1) = 1\).
    Furthermore, note that \(p^{(k)}_m(0)=0\) for every \(k\in \{0,\ldots,k_0-1\}\) since the lowest degree monomial in the expansion of \(p_m\) is of degree \(k_0\).
    Finally, by the general Leibniz rule, for every \(k\in \{1,\ldots,k_0\}\),
    \[
        p^{(k)}_m(1) = \sum_{i=0}^{k} \binom{k}{i} f^{(i)}(1) g^{(k-i)}(1) = \sum_{i=0}^{k} \binom{k}{i} f^{(i)}(1) (1/f)^{(k-i)}(1) = (f/f)^{(k)}(1) = 0. \qedhere
    \]
\end{proof}

For every \(i\in [n]\), define
\begin{equation}
    \phi_{m,i}(\bx) \deq  p_m\left(\frac{\<\bx,\bx_{0,i}\>}{\rho \tau_\bSigma d}\right).
\end{equation}
Since \(p_m\) is a polynomial with terms of degree at least \(k_0\geq 6\), \(\phi_{m,i}\in \cH_{\geq 3}\) for every \(i\in [n]\), with \(\cH_{\geq 3}\) defined in~\eqref{eq:Hge3}. Furthermore, for every \(\bz\in\R^d\), let \(\bPhi_m(\bz)\in \R^{n\times n}\) be the matrix with entries \(\bPhi_{m,ij}(\bz) = \phi_{m,j}(\rho\bx_{0,i}+\sigma\bz)\) for every \(i,j\in [n]\). The following lemma shows that the \emph{mean interpolation matrix} \(\E_{\bz\sim\cN(0,\bI_d)}[\bPhi_m(\bz)]\) is close to the identity matrix in operator norm.

\begin{lemma}\label{lem:mean_interpolation_matrix}
    Under Assumptions~\ref{ass:high_dim} and~\ref{ass:input_dist},
    \[
        \|\E_{\bz\sim\cN(0,\bI_d)}[\bPhi_m(\bz)] -  \bI_n\|_\op  \prec d^{-2}.
    \]
\end{lemma}

\begin{proof}
    By~\eqref{eq:norm_concentration_overlaps} and~\eqref{eq:norm_concentration_overlaps_lp}, for every \(k\in \N\) we have
    \[
        \max_{i,j\in [n]}\E_{\bz\sim\cN(0,\bI_d)}\left[\left|\frac{\<\rho\bx_{0,i} + \sigma\bz,\bx_{0,j}\>}{\rho\tau_\bSigma d}-\ind_{i=j}\right|^{k}\right]^{1/k} \prec \frac{1}{\sqrt{d}}.
    \]
    Here, Assumption~\ref{ass:input_dist} ensures that \(\tau_{\bSigma}>c\) for some constant \(c>0\). By flatness of \(p_m\) at \(0\) and \(1\) (see Lemma~\ref{lem:poly_properties}), there exists a constant \(C_m>0\) such that
    \[
        |p_m(x)-1| \leq C_m |x-1|^{k_0+1}(1+|x-1|^{k_0-1}),\qquad |p_m(x)| \leq C_m |x|^{k_0}(1+|x|^{k_0})
    \]
    for all \(x\in \R\). Hence, by Jensen's inequality and Cauchy-Schwarz inequality,
    \begin{equation}\label{eq:Q_entrywise}
        \max_{i\in [n]}|\E[\bPhi_{m,ii}(\bz)]-1| 
         \prec d^{- \frac{k_0+1}{2}},
         \qquad   \max_{i\neq j\in [n]}|\E[\bPhi_{m,ij}(\bz)]| \prec d^{-\frac{k_0}{2}}.
    \end{equation}

    By Schur's test, \(\|\E[\bPhi_{m}(\bz)] - \bI_n\|_\op^2 \leq \|\E[\bPhi_{m}(\bz)] - \bI_n\|_1\|\E[\bPhi_{m}(\bz)] - \bI_n\|_\infty\) where
    \[
        \|\bA\|_1 \deq  \max_{j}\sum_{i=1}^{n}|\bA_{i,j}|,\qquad \|\bA\|_\infty \deq  \max_{i}\sum_{j=1}^{n}|\bA_{i,j}|
    \]
    are the maximum row/column sum, respectively. By~\eqref{eq:Q_entrywise},
    \[
        \|\E[\bPhi_{m}(\bz)] - \bI_n\|_1 \leq \max_{j}\bigg(|\E[\bPhi_{m,jj}(\bz)]-1| + \sum_{i\neq j}|\E[\bPhi_{m,ij}(\bz)]|\bigg) \prec \frac{1}{d^{\frac{k_0}{2}-1}},
    \]
    and similarly \(\|\E[\bPhi_{m}(\bz)] - \bI_n\|_\infty \prec d^{-\frac{k_0}{2}+1}\). Since \(k_0\geq 6\), we conclude that \(\|\E[\bPhi_{m}(\bz)] - \bI_n\|_\op \prec d^{-2}\).
\end{proof}

Next, define the fluctuation matrix \(\bR_m(\bz) \deq  \bPhi_m(\bz) - \E_{\bz\sim\cN(0,\bI_d)}[\bPhi_m(\bz)]\).

\begin{lemma}\label{lem:fluctuation_bound}
    Under Assumptions~\ref{ass:high_dim} and~\ref{ass:input_dist}, for any vector \(\bbeta\in \R^n\),
    \[
        \frac{1}{n}\E_{\bz\sim\cN(0,\bI_d)}[\|\bR_m(\bz)\bbeta\|_2^2] \prec d^{-k_0+1} \|\bbeta\|_2^2.
    \]
\end{lemma}

\begin{proof}
    Using the high-probability bounds on the overlaps in~\eqref{eq:norm_concentration_overlaps}, it follows from a similar argument as in the proof of Lemma~\ref{lem:mean_interpolation_matrix} that
    \[
        \max_{i\in [n]}\E_{\bz\sim\cN(0,\bI_d)}\left[\bR_{m,ii}(\bz)^2\right]^{1/2} \prec d^{-\frac{k_0+1}{2}},\qquad \max_{i\neq j\in [n]}\E_{\bz\sim\cN(0,\bI_d)}\left[\bR_{m,ij}(\bz)^2\right]^{1/2} \prec d^{-\frac{k_0}{2}}.
    \]
    The details are similar so we omit them. For any \(\bbeta\in \R^n\), we first apply Minkowski's inequality to pull the sum out of the expectation, and then apply the Cauchy-Schwarz inequality to separate the fluctuations from the vector components:
    \begin{align*}
        \frac{1}{n}\E_{\bz\sim\cN(0,\bI_d)}[\|\bR_m\bbeta\|_2^2] &  \leq \frac{1}{n}\sum_{i=1}^{n}\Bigg(\sum_{j=1}^{n}\E_{\bz\sim\cN(0,\bI_d)}\left[\bR_{m,ij}(\bz)^2\right]^{1/2} |\bbeta_j|\Bigg)^{2}
        \\ & \leq \frac{1}{n}\sum_{i=1}^{n}\Bigg( \sum_{j=1}^{n}\E_{\bz\sim\cN(0,\bI_d)}\left[\bR_{m,ij}(\bz)^2\right] \Bigg) \Bigg( \sum_{j=1}^{n}\bbeta_j^2 \Bigg).
    \end{align*}
    Using the entry-wise bounds and \(n = O_d(d)\) by Assumption~\ref{ass:high_dim}, the sum of variances for any fixed row \(i\) is bounded by
    \[
        \sum_{j=1}^{n}\E_{\bz\sim\cN(0,\bI_d)}\left[\bR_{m,ij}(\bz)^2\right] = \E_{\bz}\left[\bR_{m,ii}(\bz)^2\right] + \sum_{j\neq i}\E_{\bz}\left[\bR_{m,ij}(\bz)^2\right] \prec d^{-k_0+1}.
    \]
    Substituting this back into the sequence of inequalities and noting that \(\sum_{j=1}^n \bbeta_j^2 = \|\bbeta\|_2^2\), we obtain
    \[
        \frac{1}{n}\E_{\bz\sim\cN(0,\bI_d)}[\|\bR_m\bbeta\|_2^2] \prec \frac{1}{n}\sum_{i=1}^{n} \left( d^{-k_0+1} \right) \|\bbeta\|_2^2 = d^{-k_0+1}\|\bbeta\|_2^2. \qedhere
    \]
\end{proof}

Define the localized bump functions as
\begin{equation}\label{eq:localized_bumps}
    f^{\mathrm{loc}}_{m,\bv}(\bx) \deq  -\frac{\<\bx,\bv\>}{\sigma} + \sum_{i=1}^{n}\beta_{m,i} \phi_{m,i}(\bx), \qquad \bbeta_{m,\bv} \deq  \frac{\rho}{\sigma}\E_{\bz\sim\cN(0,\bI_d)}[\bPhi_m(\bz)]^{-1}\bX^\top \bv.
\end{equation}

\begin{proof}[Proof of Proposition~\ref{prop:localized_bumps_properties}]

     By Lemma~\ref{lem:mean_interpolation_matrix}, \(\|\E_{\bz}[\bPhi_m(\bz)] - \bI_n\|_\op \prec d^{-2}\). On the resulting very high probability event, \(\E_{\bz}[\bPhi_m(\bz)]\) is invertible and \(\|\E_{\bz}[\bPhi_m(\bz)]^{-1}\|_\op \leq 2\). Hence, by~\eqref{eq:cov_op_bound} and~\eqref{eq:cov_concentration},
     \[
        \|\bbeta_{m,\bv}\|_2^2 \leq \|\E_{\bz}[\bPhi_m(\bz)]^{-1}\|_\op^2 \frac{\rho^2}{\sigma^2}\bv^\sT\bX\bX^\top \bv \lesssim \bv^\sT\bX\bX^\top \bv \prec d\|\bv\|_2^2 = d.
     \]

    Next, we bound the loss. By~\eqref{eq:emp_risk_decoupled},
    \begin{align*}
        \cL_{\bv}(f^{\mathrm{loc}}_{m,\bv};0)  & = \frac{1}{n}\sum_{j=1}^n \E_{\bz}\left[\bigl|\be_j^\top \bPhi_{m}(\bz) \bbeta_{m,\bv}-\frac{\rho}{\sigma}\<\bx_{0,j} ,\bv\>\bigr|^2\right].
    \end{align*}
    Note that by the definition of \(\bbeta_{m,\bv}\),
    \begin{align*}
        \E_{\bz\sim\cN(0,\bI_d)}\left[\be_j^\top \bPhi_{m}(\bz) \bbeta_{m,\bv}\right] & = \frac{\rho}{\sigma}\be_j^\top \bX^\top \bv = \frac{\rho}{\sigma}\<\bx_{0,j},\bv\>,
    \end{align*}
    so
    \[
        \cL_{\bv}(f^{\mathrm{loc}}_{m,\bv};0) = \frac{1}{n}\E_{\bz\sim\cN(0,\bI_d)}\left[\|\bR_m(\bz)\bbeta_{m,\bv}\|_2^2\right]
    \]
    where we recall that \(\bR_m(\bz)  = \bPhi_m(\bz) - \E_{\bz}[\bPhi_m(\bz)]\) is the matrix of centered residuals. By Lemma~\ref{lem:fluctuation_bound},
    \[
        \cL_{\bv}(f^{\mathrm{loc}}_{m,\bv};0) \prec d^{-k_0+1}\|\bbeta_{m,\bv}\|_2^2 \prec  d^{-k_0+2}.
    \]

    Finally, we control the RKHS norm. By Young's inequality,
    \[
        \| f^{\mathrm{loc}}_{m,\bv}\|_\cH^2 \leq \frac{2d}{\sigma^2} + 2\left\|\sum_{i=1}^{n}\beta_{m,i} \phi_{m,i}\right\|_\cH^2.
    \]
    We can expand the polynomial function as
    \begin{align*}
        \phi_{m,i}(\bx) & = p_m\left(\frac{\<\bx,\bx_{0,i}\>}{\rho\tau_{\bSigma} d}\right) 
         = \sum_{k=k_0}^{2k_0} c_{m,k} \left(\frac{\<\bx,\bx_{0,i}\>}{\rho \tau_{\bSigma} d}\right)^k
         = \sum_{k=k_0}^{2k_0} \sqrt{\frac{\alpha_k}{d^k}} \<\bW_{i,k},\bx^{\otimes k}\>,
    \end{align*}
    where we define \(\bW_{i,k} \deq  \sqrt{\frac{d^k}{\alpha_k}} \frac{c_{m,k}}{(\rho\tau_\bSigma d)^k} \bx_{0,i}^{\otimes k}\).
    
    By Lemma~\ref{lem:rkhs_decomposition}, the RKHS norm of this sum is exactly the sum of the squared Frobenius norms of these scaled tensors. To bound this, let \(\bG^{(k)} \in \R^{n \times n}\) be the Gram matrix with entries \(\bG^{(k)}_{ij} = \<\bW_{i,k}, \bW_{j,k}\>\). We can rewrite the squared norm for each degree \(k\) as a quadratic form:
    \begin{align*}
        \left\|\sum_{i=1}^{n}\beta_{m,i} \bW_{i,k}\right\|_\frob^2 & = \bbeta_{m,\bv}^\sT \bG^{(k)} \bbeta_{m,\bv} \leq \|\bG^{(k)}\|_\op \|\bbeta_{m,\bv}\|_2^2.
    \end{align*}
    By Schur's test, \(\|\bG^{(k)}\|_\op \leq \max_i \sum_{j=1}^n |\bG^{(k)}_{ij}|\). The entries of the Gram matrix evaluate directly to
    \[
        \bG^{(k)}_{ij} = \frac{d^k}{\alpha_k} \frac{c_{m,k}^2}{(\rho\tau_\bSigma d)^{2k}} \<\bx_{0,i}, \bx_{0,j}\>^k = \frac{c_{m,k}^2}{\alpha_k (\rho\tau_{\bSigma})^{2k}} \left(\frac{\<\bx_{0,i}, \bx_{0,j}\>}{d}\right)^k.
    \]
    By~\eqref{eq:norm_concentration_overlaps}, for the diagonal entries we have \((\<\bx_{0,i}, \bx_{0,i}\>/d)^k \prec 1\). Since \(\alpha_k > 0\) is a constant by Assumption~\ref{ass:kernel_func_vanishing_kappa}, the diagonal elements satisfy \(\bG^{(k)}_{ii} \prec 1\). 
    
    For the off-diagonal entries (\(i \neq j\)), the cross-terms decay rapidly since \(|\<\bx_{0,i}, \bx_{0,j}\>/d|^k \prec d^{-k/2}\). Because \(n \asymp d\) and the active degrees in the bump construction satisfy \(k \geq k_0 \geq 6\), the sum over the \(n-1\) off-diagonal elements vanishes asymptotically:
    \[
        \sum_{j\neq i} |\bG^{(k)}_{ij}| \prec n d^{-k/2} \asymp d^{1-k/2} \prec 1.
    \]
    This yields \(\|\bG^{(k)}\|_\op \prec 1\). Therefore,
    \begin{align*}
        \Bigg\|\sum_{i=1}^{n}\beta_{m,i} \phi_{m,i}\Bigg\|_\cH^2 & = \sum_{k=k_0}^{2k_0} \Bigg\|\sum_{i=1}^{n}\beta_{m,i} \bW_{i,k}\Bigg\|_\frob^2 \leq \sum_{k=k_0}^{2k_0} \|\bG^{(k)}\|_\op \|\bbeta_{m,\bv}\|_2^2 \prec \|\bbeta_{m,\bv}\|_2^2 \prec d.
    \end{align*}
    Substituting this into the initial inequality gives \(\| f^{\mathrm{loc}}_{m,\bv}\|_\cH^2 \prec d\), concluding the proof.
\end{proof}

\clearpage

\section{Proofs of the kernel ridge regression results}\label{app:train-test}

This section is devoted to proving the asymptotic risk equivalents stated in Theorems~\ref{thm:train_test_equivalents_constant_kappa} and~\ref{thm:train_test_equivalents_vanishing_kappa}. It is implicitly assumed that the diffusion time \(t>0\) is fixed, implying that \(\sigma,\rho\) are positive constants.

\subsection{Proof of Theorem~\ref{thm:train_test_equivalents_constant_kappa}}

We proceed through a sequence of intermediate results, beginning with the linearization of the training and test errors via Proposition~\ref{prop:linearization_kernel_op} and Proposition~\ref{prop:bayes_opt_denoiser_linearization}. For notational convenience, let
\begin{equation}
    \widehat{\cG}^{(\kappa)} \deq  \left(\widehat{\cK} + \frac{\kappa}{d}\id\right)^{-1},\qquad \widehat{\cG}^{(\kappa)}_{\mathrm{eff}} \deq  \left(\widehat{\cK}_{\mathrm{eff}} + \frac{\kappa}{d}\id\right)^{-1}.
\end{equation}
Since \(\widehat{\cK}\) and \(\widehat{\cK}_{\mathrm{eff}}\) are positive semi-definite, \(\|\widehat{\cG}^{(\kappa)}\|_\op\vee \|\widehat{\cG}_{\mathrm{eff}}^{(\kappa)}\|_\op \leq d/\kappa\). Define the \emph{linearized training and test errors} coordinate-wise
as
\begin{equation}\label{eq:linearized_train_test}
    \cL_{\mathrm{lin},\bv}^{(\kappa)} \deq  \frac{\kappa^2}{d^2}\left\|\widehat{\cG}^{(\kappa)}_{\mathrm{eff}}\sfz_\bv\right\|_{L^2(\widehat{\pi}_{d,n})}^2, \qquad \cR_{\mathrm{lin},\bv}^{(\kappa)} \deq  \left\|\cK_{\mathrm{eff}}\widehat{\cG}^{(\kappa)}_{\mathrm{eff}}\sfz_\bv - \sfz_\bv\right\|_{L^2(\pi_d)}^2.
\end{equation}
The full linearized training and test errors are then defined by averaging over any orthonormal basis of \(\R^d\). For simplicity, we drop the dependence of the various quantities on \(\kappa\) in the notation. The following lemma controls the operator norm of the kernel operators \(\cK,\cK_{\mathrm{eff}}\).

\begin{lemma}\label{lem:kernel_op_norm_bound}
    Suppose that Assumptions~\ref{ass:high_dim},~\ref{ass:input_dist} and~\ref{ass:kernel_func_constant_kappa} hold. Then, \(\|\cK\|_\op \vee \|\cK_{\mathrm{eff}}\|_\op \prec d^{-1}\).
\end{lemma}

\begin{proof}
    For every \(f\in L^2(\widehat{\pi}_{d,n})\),
    \begin{align*}
       \|\cK_{\mathrm{eff}} f \|_{L^2(\pi_d)}^2 & = \E_{(\bx,\bz)\sim \pi_d}\left[\E_{(\bx',\bz')\sim \widehat{\pi}_{d,n}}\left[\frac{\<\rho\bx'+\sigma\bz',\rho\bx+\sigma\bz\>}{d}f(\bx',\bz')\right]^2\right] 
        \\ & = \E_{(\bx,\bz),(\bx',\bz')\sim \widehat{\pi}_{d,n}^{\otimes 2}}\left[\frac{\<\rho\bx'+\sigma\bz',\bSigma_\sigma (\rho\bx+\sigma\bz)\>}{d^2}f(\bx,\bz)f(\bx',\bz')\right]
        \\ & 
        \leq \frac{\|\bSigma_\sigma\|_\op}{d^2}\left\|\E_{(\bx,\bz)\sim \widehat{\pi}_{d,n}}[f(\bx,\bz)(\rho\bx+\sigma\bz)]\right\|_2^2,
    \end{align*}
    where we recall that \(\bSigma_\sigma = \rho^2\bSigma + \sigma^2\bI_d\). By~\eqref{eq:cov_op_bound}, \(\|\bSigma_\sigma\|_\op \leq \rho^2C_{\mathrm{LSI}} + \sigma^2 \leq C_{\mathrm{LSI}}\) for \(C_{\mathrm{LSI}}\) the Poincar\'e constant of the input distribution (see Assumption~\ref{ass:input_dist}). On the other hand, by Cauchy-Schwarz,
    \begin{align*}
        \left\|\E_{(\bx,\bz)\sim \widehat{\pi}_{d,n}}[f(\bx,\bz)(\rho\bx+\sigma\bz)]\right\|_2^2 & = \sup_{\|\bu\|_2=1} \E_{(\bx,\bz)\sim \widehat{\pi}_{d,n}}[f(\bx,\bz)\<\rho\bx+\sigma\bz,\bu\>]^2
        \\ & \leq \|f\|_{L^2(\widehat{\pi}_{d,n})}^2\sup_{\|\bu\|_2=1}\E_{(\bx,\bz)\sim\widehat{\pi}_{d,n}}[\<\rho\bx+\sigma\bz,\bu\>^2] 
        \\ & \leq \|f\|_{L^2(\widehat{\pi}_{d,n})}^2\left\|\E_{(\bx,\bz)\sim\widehat{\pi}_{d,n}}[(\rho\bx+\sigma\bz)(\rho\bx+\sigma\bz)^\sT]\right\|_\op.
    \end{align*}
    By~\eqref{eq:cov_op_bound} and~\eqref{eq:cov_concentration},
    \begin{align*}
        \left\|\E_{(\bx,\bz)\sim\widehat{\pi}_{d,n}}[(\rho\bx+\sigma\bz)(\rho\bx+\sigma\bz)^\sT]\right\|_\op = \left\|\frac{\rho^2\bX\bX^\sT}{n} + \sigma^2\bI_d\right\|_\op \prec 1
    \end{align*}
    and thus we may combine the bounds to get \(\|\cK_{\mathrm{eff}}f\|_{L^2(\pi_d)}^2 \prec d^{-2}\|f\|_{L^2(\widehat{\pi}_{d,n})}^2\). By the variational characterization of the operator norm, \(\|\cK_{\mathrm{eff}}\|_\op \prec d^{-1}\). We conclude the proof using Proposition~\ref{prop:linearization_kernel_op}.
\end{proof}

We have the following linearization result for the training and test errors.

\begin{lemma}\label{lem:train_test_error_linearization}
    Suppose that Assumptions~\ref{ass:high_dim},~\ref{ass:input_dist} and~\ref{ass:kernel_func_constant_kappa} hold.
    Then, there exists \(c>0\) such that for every \(\bv\in \R^d\) with \(\|\bv\|_2=1\),
    \[
        |\cL_{\bv}(\widehat{f}_{\kappa,\bv};0) - \cL^{(\kappa)}_{\mathrm{lin},\bv}| \prec \frac{1}{\kappa \sqrt{d}} + e^{-d^c} ,\qquad |\cR_{\bv}(\widehat{f}_{\kappa,\bv}) - \cR_{\mathrm{lin},\bv}^{(\kappa)}| \prec \frac{(\kappa + 1)^2}{\kappa^3\sqrt{d}}.
    \]
\end{lemma}

\begin{proof}
    We first consider the training error. By Proposition~\ref{prop:bayes_opt_denoiser_linearization} there exists a constant \(c>0\) such that 
     \[
        \|\widehat{f}^\star_\bv - \sfz_\bv\|_{L^2(\widehat{\pi}_{d,n})} \leq e^{-d^{c}}
    \]
    with very high probability, which implies that
    \[
        \left|\cL_{\bv}(\widehat{f}_{\kappa,\bv};0)- \frac{\kappa^2}{d^2}\|\widehat{\cG}\sfz_\bv\|_{L^2(\widehat{\pi}_{d,n})}^2\right|
         \lesssim e^{-d^c}.
    \]
    Next, using the scalar inequality \(|\|x\|^2 - \|y\|^2|\leq \|x-y\|(\|x\|+\|y\|)\) for every \(x,y\) in a normed space, we have
    \begin{align*}
        \left|\|\widehat{\cG}\sfz_\bv\|_{L^2(\widehat{\pi}_{d,n})}^2 - \|\widehat{\cG}_{\mathrm{eff}}\sfz_\bv\|_{L^2(\widehat{\pi}_{d,n})}^2\right| & \leq \|\widehat{\cG}\sfz_\bv - \widehat{\cG}_{\mathrm{eff}}\sfz_\bv\|_{L^2(\widehat{\pi}_{d,n})}\left(\|\widehat{\cG}\sfz_\bv\|_{L^2(\widehat{\pi}_{d,n})} + \|\widehat{\cG}_{\mathrm{eff}}\sfz_\bv\|_{L^2(\widehat{\pi}_{d,n})}\right).
    \end{align*}
    By the resolvent identity, the norm bounds \(\|\widehat{\cG}^{(\kappa)}\|_\op\vee \|\widehat{\cG}_{\mathrm{eff}}^{(\kappa)}\|_\op \leq d/\kappa\) and Proposition~\ref{prop:linearization_kernel_op},
    \begin{equation}\label{eq:resolvent_identity_training_error}
        \|\widehat{\cG} - \widehat{\cG}_{\mathrm{eff}}\|_\op \leq \|\widehat{\cG}\|_\op\|\widehat{\cK}-\widehat{\cK}_{\mathrm{eff}}\|_\op \|\widehat{\cG}_{\mathrm{eff}}\|_\op \prec \frac{\sqrt{d}}{\kappa^2}.
    \end{equation}
    Furthermore, \(\|\sfz_\bv\|_{L^2(\widehat{\pi}_{d,n})} = 1\) and so
    \[
        \left|\frac{\kappa^2}{d^2}\|\widehat{\cG}\sfz_\bv\|_{L^2(\widehat{\pi}_{d,n})}^2 - \frac{\kappa^2}{d^2}\|\widehat{\cG}_{\mathrm{eff}}\sfz_\bv\|_{L^2(\widehat{\pi}_{d,n})}^2\right| \prec \frac{1}{\kappa \sqrt{d}}.
    \]
    This proves the first part of the lemma. 

    We use a similar argument to control the test error. We swap the empirical Bayes estimator \(\widehat{f}^\star_{\bv}\) with \(\sfz_\bv\) using Proposition~\ref{prop:bayes_opt_denoiser_linearization}:
    \begin{align*}
        \left|\cR_{\bv}(\widehat{f}_{\kappa,\bv}) - \left\|\cK\widehat{\cG}\sfz_{\bv} - \sfz_\bv\right\|_{L^2(\pi_d)}^2\right| & \leq  \|\cK\widehat{\cG}\|_\op\|\widehat{f}^\star_\bv - \sfz_\bv\|_{L^2(\widehat{\pi}_{d,n})}\left(2+ \|\cK\widehat{\cG}\|_\op\|\widehat{f}^\star_\bv\|_{L^2(\widehat{\pi}_{d,n})}+ \|\cK\widehat{\cG}\|_\op\right).
    \end{align*}
    Here, we have used the triangle inequality and the fact that \(\|\sfz_\bv\|_{L^2(\pi_d)} = 1\). By Lemma~\ref{lem:kernel_op_norm_bound} and the norm bounds on the resolvents, \(\|\cK\widehat{\cG}\|_\op \leq \|\cK\|_\op\|\widehat{\cG}\|_\op \prec 1/\kappa\). Hence, there exists a constant \(c > 0\) such that
    \[
        \left|\cR_{\bv}(\widehat{f}_{\kappa,\bv}) - \left\|\cK\widehat{\cG}\sfz_{\bv} - \sfz_\bv\right\|_{L^2(\pi_d)}^2\right| \prec \frac{(\kappa+1)e^{-d^c}}{\kappa^2}.
    \]
    
    We can now compare \(\|\cK\widehat{\cG}\sfz_\bv - \sfz_\bv\|_{L^2(\pi_d)}^2\) with \(\|\cK_{\mathrm{eff}}\widehat{\cG}_{\mathrm{eff}}\sfz_\bv - \sfz_\bv\|_{L^2(\pi_d)}^2\) using a similar argument as for the training error. We have
    \begin{align*}
        \left|\left\|\cK\widehat{\cG}\sfz_{\bv} - \sfz_\bv\right\|_{L^2(\pi_d)}^2 - \cR_{\mathrm{lin},\bv}^{(\kappa)}\right| & \leq  \left\|\cK\widehat{\cG}\sfz_{\bv}- \cK_{\mathrm{eff}}\widehat{\cG}_{\mathrm{eff}}\sfz_\bv\right\|_{L^2(\pi_d)}\left(2 + \|\cK\widehat{\cG}\sfz_\bv\|_{L^2(\pi_d)} + \|\cK_{\mathrm{eff}}\widehat{\cG}_{\mathrm{eff}}\sfz_\bv\|_{L^2(\pi_d)}\right)
        \\ & \prec \frac{\kappa + 1}{\kappa}\left(\frac{1}{d}\|\widehat{\cG}-\widehat{\cG}_{\mathrm{eff}}\|_\op + \frac{d}{\kappa}\|\cK-\cK_{\mathrm{eff}}\|_\op\right).
    \end{align*}
    By Proposition~\ref{prop:linearization_kernel_op} and~\eqref{eq:resolvent_identity_training_error}, we conclude that
    \[
        \left|\left\|\cK\widehat{\cG}\sfz_{\bv} - \sfz_\bv\right\|_{L^2(\pi_d)}^2 - \cR_{\mathrm{lin},\bv}^{(\kappa)}\right| \prec \frac{\kappa+1}{\kappa^3 \sqrt{d}}+\frac{\kappa+1}{\kappa^2 \sqrt{d}}.
    \]
    We absorb the exponentially small term into the polynomially small term and conclude the proof.
\end{proof}

Having approximated the training and test errors by their linearized counterparts, we now derive more explicit representations of the linearized estimators \(\widehat{\cG}_{\mathrm{eff}}\sfz_\bv\) and \(\cK_{\mathrm{eff}}\widehat{\cG}_{\mathrm{eff}}\sfz_\bv\). This will allow us to compute the asymptotic limits of the linearized training and test errors.  Recall the definition of the resolvent matrix \(\widehat{\bR}\) in~\eqref{eq:emp_kernel_resolvent_matrices}. Given fixed \(\bw_\bx,\bw_\bz\in \R^d\), we define the linear function \(L_{\bw_\bx,\bw_\bz}: \R^d\times \R^d \to \R\) as \(L_{\bw_\bx,\bw_\bz}(\bx,\bz) = \<\bw_\bx,\bx\> + \<\bw_\bz,\bz\>\) for every \(\bx,\bz\in \R^d\). We will also write \(L_{\bw}(\bx,\bz)\) for \(\bw\in \R^{2d}\), in which case \(\bw_\bx\) and \(\bw_\bz\) are the first and last \(d\) coordinates of \(\bw\), respectively. We have the following representation.

\begin{lemma}\label{lem:linearized_estimators_representation}
    For every \(\bv\in \R^d\) with \(\|\bv\|_2=1\), we have \(\widehat{\cG}_{\mathrm{eff}}\sfz_\bv = L_{\widehat{\bw}}\) and \(\cK_{\mathrm{eff}}\widehat{\cG}_{\mathrm{eff}}\sfz_\bv = L_{\bw}\) where
    \[
        \widehat{\bw} \deq   \begin{bmatrix}
            \frac{\sigma d}{\rho \kappa}\widehat{\bR}\bv \\ -\left(\frac{d}{\sigma^2+\kappa}\bI_d + \frac{\sigma^2d}{\kappa(\sigma^2+\kappa)n} \bX\bX^\sT\widehat{\bR}\right)\bv
        \end{bmatrix},\qquad 
        \bw =  
            - \frac{\kappa^\star}{\sigma^2+\kappa}
            \begin{bmatrix}
                \rho \sigma\widehat{\bR}\bv \\
                \sigma^2\widehat{\bR}\bv
            \end{bmatrix}.
    \]
\end{lemma}

\begin{proof}
    Let \(\bw_\bx,\bw_\bz\in \R^d\) be arbitrary and denote \(\bw = [\bw_\bx^\sT;\bw_\bz^\sT]^\sT\in \R^{2d}\). By definition of the effective kernel operator \(\widehat{\cK}_{\mathrm{eff}}\)~\eqref{eq:K_hat_op}, we readily verify that \( \widehat{\cK}_{\mathrm{eff}}L_{\bw}  = L_{\widehat{\bK}\bw}\), where the empirical kernel matrix \(\widehat{\bK}\) is defined as
    \[
        \widehat{\bK} \deq  
        \begin{bmatrix} 
            \frac{\rho^2}{dn}\bX\bX^\sT + \delta_n \bI_d & \frac{\rho \sigma}{d}\bI_d \\ 
            \frac{\rho \sigma}{dn}\bX\bX^\sT & \frac{\sigma^2}{d}\bI_d
        \end{bmatrix}.
    \]
    Therefore, \(\widehat{\cK}_{\mathrm{eff}} + \frac{\kappa}{d}\id\) acts as the linear operator on the space of linear functions given by the matrix \(\widehat{\bK} + \frac{\kappa}{d}\bI_{2d}\), and thus \(\widehat{\cG}_{\mathrm{eff}} L_{\bw} = L_{(\widehat{\bK} + \frac{\kappa}{d}\bI_{2d})^{-1}\bw}\). Using the block structure of \(\widehat{\bK}\) and the definition of \(\widehat{\bR}\), we verify that
    \[
        \left(\widehat{\bK} + \frac{\kappa}{d}\bI_{2d}\right)^{-1} = \begin{bmatrix}
            \frac{d(\sigma^2+\kappa)}{\rho^2\kappa}\widehat{\bR} & - \frac{\sigma d}{\rho \kappa}\widehat{\bR} \\
            -\frac{\sigma d}{\rho \kappa n}\bX\bX^\sT \widehat{\bR} & \frac{d}{\sigma^2+\kappa}\bI_d + \frac{\sigma^2 d}{\kappa(\sigma^2+\kappa)n} \bX\bX^\sT\widehat{\bR}
        \end{bmatrix}. 
    \]

    Plugging \(\bw_x=0\) and \(\bw_z = -\bv\) into the expression above, we obtain
    \[
        \widehat{\cG}_{\mathrm{eff}}\sfz_\bv = L_{\widehat{\bw}}(\bx,\bz),\qquad \widehat{\bw} \deq   \begin{bmatrix}
            \frac{\sigma d}{\rho \kappa}\widehat{\bR}\bv \\ -\left(\frac{d}{\sigma^2+\kappa}\bI_d + \frac{\sigma^2 d}{\kappa(\sigma^2+\kappa)n} \bX\bX^\sT\widehat{\bR}\right)\bv
        \end{bmatrix}.
    \]
    
    By a similar argument, we can also verify that \(\cK_{\mathrm{eff}}L_{\bw}(\bx,\bz) = L_{\bK\bw}(\bx,\bz)\) where \(\bK\in \R^{2d \times 2d}\) is the kernel matrix defined as
    \[
         \bK \deq  
        \begin{bmatrix} 
            \frac{\rho^2}{dn}\bX\bX^\sT  & \frac{\rho \sigma}{d}\bI_d \\ 
            \frac{\rho \sigma}{dn}\bX\bX^\sT & \frac{\sigma^2}{d}\bI_d
        \end{bmatrix}
        .
    \]
    Therefore, \(\cK_{\mathrm{eff}}\widehat{\cG}_{\mathrm{eff}}\sfz_\bv = L_{\bK\widehat{\bw}}\) which simplifies to
    \begin{align*}
        \bK\widehat{\bw} =  
            - \frac{\kappa^\star}{\sigma^2+\kappa}
            \begin{bmatrix}
                \rho \sigma\widehat{\bR}\bv \\
                \sigma^2\widehat{\bR}\bv
            \end{bmatrix},
    \end{align*}
    which is what we wanted to show.
\end{proof}

We can now compute the asymptotic limits of the linearized training and test errors using the explicit representations of the linearized estimators and the first- and second-order deterministic equivalents of the resolvent matrix \(\widehat{\bR}\) established in Lemmas~\ref{lem:first_order_deterministic_equivalent} and~\ref{lem:second_order_deterministic_equivalent}.

\begin{lemma}\label{lem:train_test_error_deterministic_equivalent}
    Suppose that Assumptions~\ref{ass:high_dim} and~\ref{ass:input_dist} hold. Then,
    \[
        |\cL_{\mathrm{lin}}^{(\kappa)}-\cL_{\mathrm{det}}^{(\kappa)}|\vee |\cR_{\mathrm{lin}}^{(\kappa)}-\cR_{\mathrm{det}}^{(\kappa)}| \prec \frac{1}{\sqrt{d}}.
    \]
\end{lemma}

\begin{proof}
    We first consider the training error. Let \(\bv\in \R^d\) with \(\|\bv\|_2=1\). By Lemma~\ref{lem:linearized_estimators_representation} and the identity \(n^{-1}\bX\bX^\sT\widehat{\bR} = \bI_d - \kappa^\star \widehat{\bR}\)
    \begin{align*}
        \cL_{\mathrm{lin},\bv}^{(\kappa)}  & = \frac{\kappa^2}{d^2}\left(\frac{\sigma^2d^2}{\rho^2\kappa^2}\bv^\sT\widehat{\bR} ( \bI_d - \kappa^\star \widehat{\bR})\bv + \left\|\left(\frac{d}{\sigma^2+\kappa}\bI_d + \frac{\sigma^2d}{\kappa(\sigma^2+\kappa)} ( \bI_d - \kappa^\star \widehat{\bR})\right)\bv\right\|_2^2\right)
        \\ & = 1 + \left( \frac{\sigma^2}{\rho^2} - \frac{2\sigma^2\kappa^\star}{\sigma^2+\kappa} \right) \bv^\sT\widehat{\bR}\bv + \left( \frac{\sigma^4(\kappa^\star)^2}{(\sigma^2+\kappa)^2} - \frac{\sigma^2\kappa^\star}{\rho^2} \right) \bv^\sT\widehat{\bR}^2\bv.
    \end{align*}
    Summing over an orthonormal basis \(\{\bv_{1},\ldots,\bv_d\}\) of \(\R^d\), we obtain
    \begin{equation}\label{eq:linearized_train_explicit}
        \cL^{(\kappa)}_{\mathrm{lin}} = 1 + \left( \frac{\sigma^2}{\rho^2} - \frac{2\sigma^2\kappa^\star}{\sigma^2+\kappa} \right) \frac{\Tr(\widehat{\bR})}{d} + \left( \frac{(\sigma^2)^2(\kappa^\star)^2}{(\sigma^2+\kappa)^2} - \frac{\sigma^2\kappa^\star}{\rho^2} \right) \frac{\Tr(\widehat{\bR}^2)}{d}.
    \end{equation}
    By Lemmas~\ref{lem:first_order_deterministic_equivalent} and~\ref{lem:second_order_deterministic_equivalent},
    \begin{equation}\label{eq:res_det_equiv}
        \frac{1}{d}\left|\Tr(\widehat{\bR}-\bM)\right| \prec \frac{1}{\kappa^\star\sqrt{d}},\qquad \frac{1}{d}\left|\Tr(\widehat{\bR}^2-\bM^{(2)})\right| \prec \frac{1}{(\kappa^\star)^2\sqrt{d}}.
    \end{equation}
    Hence,
    \begin{align*}
        \cL^{(\kappa)}_{\mathrm{lin}} = 1 + \left( \frac{\sigma^2}{\rho^2} - \frac{2\sigma^2\kappa^\star}{\sigma^2+\kappa} \right) \frac{\Tr(\bM)}{d} + \left( \frac{(\sigma^2)^2(\kappa^\star)^2}{(\sigma^2+\kappa)^2} - \frac{\sigma^2\kappa^\star}{\rho^2} \right) \frac{\Tr(\bM^{(2)})}{d} + O_{\prec}(d^{-1/2}).
    \end{align*}
    Substituting \(\kappa^\star_{\kappa} = \frac{d(\sigma^2+\kappa)\chi_\kappa}{\rho^2\kappa}\), we recover the expression for \(\cL_{\mathrm{det}}^{(\kappa)}\) in~\eqref{eq:train_error_equivalent}.

    For the test error, we use the representation of the linearized estimator \(\cK_{\mathrm{eff}}\widehat{\cG}_{\mathrm{eff}}\sfz_\bv\) in Lemma~\ref{lem:linearized_estimators_representation} to verify that
    \begin{equation}\label{eq:linearized_test_explicit}
        \cR_{\mathrm{lin},\bv}^{(\kappa)}  = 1 - \frac{2\kappa^\star \sigma^2}{\sigma^2+\kappa}\bv^\sT\widehat{\bR}\bv + \left(\frac{\kappa^\star}{\sigma^2+\kappa}\right)^2\sigma^2\bv^\sT\widehat{\bR}\bSigma_\sigma \widehat{\bR}\bv.
    \end{equation}
    Summing over an orthonormal basis \(\{\bv_1,\ldots,\bv_d\}\) of \(\R^d\) and using~\eqref{eq:res_det_equiv}, we obtain
    \begin{align*}
        \cR_{\mathrm{lin}}^{(\kappa)} & = 1 - \frac{2\kappa^\star \sigma^2}{d(\sigma^2+\kappa)}\Tr(\bM) + \left(\frac{\kappa^\star}{\sigma^2+\kappa}\right)^2\frac{\sigma^2}{d}\Tr(\bM^{(2)}\bSigma_\sigma)
        + O_{\prec}(d^{-1/2})
         = \cR^{(\kappa)}_{\mathrm{det}} + O_{\prec}(d^{-1/2}),
    \end{align*}
    where we have substituted the expression for \(\kappa^\star\) to recover the expression for \(\cR_{\mathrm{det}}^{(\kappa)}\) in~\eqref{eq:test_error_equivalent}.
\end{proof}

We are now ready to prove Theorem~\ref{thm:train_test_equivalents_constant_kappa}.

\begin{proof}[Proof of Theorem~\ref{thm:train_test_equivalents_constant_kappa}]
    Let \(\{\bv_j\}_{j=1}^d\) be an orthonormal basis of \(\R^d\). Combining Lemma~\ref{lem:train_test_error_linearization} and Lemma~\ref{lem:train_test_error_deterministic_equivalent}, we conclude that
    \begin{align*}
        |\cL(\widehat{\boldf}_{\kappa};0) - \cL_{\mathrm{det}}^{(\kappa)}| & \leq \frac{1}{d}\sum_{j=1}^{d}|\cL_{\bv_j}(\widehat{f}_{\kappa,\bv_j};0) - \cL_{\mathrm{lin},\bv_j}^{(\kappa)}| + |\cL_{\mathrm{lin}}^{(\kappa)} - \cL_{\mathrm{det}}^{(\kappa)}| \prec \frac{\kappa+1}{\kappa \sqrt{d}},
    \end{align*}
    where we have absorbed the exponentially small term into the polynomially small term.
    Similarly,
    \begin{align*}
         |\cR(\widehat{\boldf}_{\kappa}) - \cR_{\mathrm{det}}^{(\kappa)}| & \prec \frac{\kappa^3  + 1}{\kappa^3\sqrt{d}},
    \end{align*}
    where we absorbed the term \(1/(\kappa^2\sqrt{d})\) into the term \((\kappa^3 + 1)/(\kappa^3\sqrt{d})\). This concludes the proof of the first part of the theorem.

    For the second part, when \(h(x) = x\) for all \(x \in \R\), \(\cK = \cK_{\mathrm{eff}}\) and \(\widehat{\cK} = \widehat{\cK}_{\mathrm{eff}}\). Adapting Lemma~\ref{lem:train_test_error_linearization}, we find that there exists a constant \(c > 0\) such that for any unit vector \(\bv \in \R^d\),
    \[
        \bigl|\cL_{\bv}(\widehat{f}_{\kappa,\bv};0) - \cL^{(\kappa)}_{\mathrm{lin},\bv}\bigr| \prec e^{-d^c}, \qquad \text{and} \qquad \bigl|\cR_{\bv}(\widehat{f}_{\kappa,\bv}) - \cR_{\mathrm{lin},\bv}^{(\kappa)}\bigr| \prec \frac{(\kappa+1)e^{-d^c}}{\kappa^2}.
    \]
    The rest of the proof is then identical to the previous case.
\end{proof}
\subsection{Proof of Theorem~\ref{thm:train_test_equivalents_vanishing_kappa}}

The proof of Theorem~\ref{thm:train_test_equivalents_vanishing_kappa} relies on the construction of a localized bump function, with $n$ bumps that interpolate the empirical Bayes denoiser $\widehat{f}^\star_\bv$ around each sample $\bx_{0,i}$, with small training error and RKHS norm. This construction is done in Appendix~\ref{app:localized_bumps}. Theorem~\ref{thm:train_test_equivalents_vanishing_kappa} is then proved by comparing the ridge estimator with the localized bump function, and then using the structure of the RKHS to show that the linear coefficients of the ridge estimator are close to the target vector \(\bv\).

\begin{proof}[Proof of Theorem~\ref{thm:train_test_equivalents_vanishing_kappa}]
    Let \(m\in\{0,1,2,\ldots\}\) and \(k_0=2m+6\) be guaranteed by Assumption~\ref{ass:kernel_func_vanishing_kappa}, and let \(\bv\in \S^{d-1}\). By Proposition~\ref{prop:localized_bumps_properties}, the localized bump function \(f^{\mathrm{loc}}_{m,\bv}\) defined in~\eqref{eq:localized_bumps} satisfies
    \[
         \cL_{\bv}(f^{\mathrm{loc}}_{m,\bv};0) \prec  d^{-k_0+2},\qquad  \|f^{\mathrm{loc}}_{m,\bv}\|_\cH^2 \prec d.
    \]
    The regularized empirical risk is defined as \(\cL_{\bv}(f;\kappa) = \cL_{\bv}(f;0) + \frac{\kappa}{d} \|f\|_\cH^2\). By the optimality of the ridge estimator \( \widehat{f}_{\kappa,\bv}\), its risk must be upper-bounded by that of the localized bump function. Since we assume \(d^{-k_0+2} \lesssim \kappa\), this yields
    \[
        \cL_{\bv}(\widehat{f}_{\kappa,\bv};\kappa) \leq  \cL_{\bv}(f^{\mathrm{loc}}_{m,\bv};\kappa) \prec  d^{-k_0+2} + \kappa \lesssim \kappa.
    \]
    Consequently, we obtain both the training loss bound \( \cL_{\bv}(\widehat{f}_{\kappa,\bv};0) \prec  \kappa\) and the RKHS norm bound \(\|\widehat{f}_{\kappa,\bv}\|_\cH^2 \prec  d\). 
    
    We decompose the estimator into its linear and higher-order components as \(\widehat{f}_{\kappa,\bv}(\bx) = \<\widehat{\btheta}_{\kappa,\bv},\bx\> + \widehat{g}_{\kappa,\bv}(\bx)\), where \(\widehat{g}_{\kappa,\bv}\in \cH_{\geq 3}\). By the norm decomposition in Lemma~\ref{lem:rkhs_decomposition}, \(\|\widehat{f}_{\kappa,\bv}\|_\cH^2 = d\|\widehat{\btheta}_{\kappa,\bv}\|_2^2 + \|\widehat{g}_{\kappa,\bv}\|_\cH^2\). Thus, the RKHS bound implies \(\|\widehat{\btheta}_{\kappa,\bv}\|_2 \prec 1\) and \(\|\widehat{g}_{\kappa,\bv}\|_\cH \prec \sqrt{d}\).

    Recall the definition of the Hermite coefficients in~\eqref{eq:Gamma_k_def}. By convexity and Parseval's identity,
    \begin{align*}
        \left\| \sigma \widehat{\btheta}_{\kappa,\bv}+\bv + \frac{1}{n}\sum_{j=1}^{n}\Gamma_1\widehat{g}_{\kappa,\bv}(\bx_{0,j})  \right\|_2^2 & \leq \frac{1}{n}\sum_{j=1}^{n} \left\| \sigma \widehat{\btheta}_{\kappa,\bv}+\bv + \Gamma_1\widehat{g}_{\kappa,\bv}(\bx_{0,j})  \right\|_2^2 
         \leq \cL_{\bv}(\widehat{f}_{\kappa,\bv};0) \prec \kappa.
    \end{align*}
    By the triangle inequality, Lemma \ref{lem:high_order_term_bound}, and Proposition~\ref{prop:localized_bumps_properties}, we bound the distance of the linear coefficients to the target vector. Since \(\kappa \lesssim d^{-\epsilon}\) for some \(\epsilon > 0\), for every \(c \in (0, (\epsilon \wedge 1)/2)\) there exist constants \(c', C > 0\) such that
    \begin{equation}\label{eq:linear_coefficients_bound_vanishing_kappa}
        \|\sigma\widehat{\btheta}_{\kappa,\bv}+\bv\|_2 \lesssim d^c\sqrt{\kappa} + d^c\frac{C_{\zeta}}{\sqrt{d}} \lesssim d^{c-\epsilon/2} + C_{\zeta} d^{c-1/2}
    \end{equation}
    with probability at least \(1-\zeta-e^{-d^{c'}}\) for all \(d \geq C\), where $C_\zeta$ is defined in \eqref{eq:C_delta_def}. We place ourselves on this high-probability event for the remainder of the proof. Furthermore, still by Lemma \ref{lem:high_order_term_bound}, the non-linear residual vanishes in \(L^2(\pi_d)\), yielding \(\|\widehat{g}_{\kappa,\bv}\|_{L^2(\mu_d)}^2 \prec d^{-1/2}\).

    Finally, we expand the population risk:
    \begin{align*}
        \cR_{\bv}(\widehat{f}_{\kappa,\bv}) & = \E_{(\bx_0,\bz)}\left[\bigl|\<\widehat{\btheta}_{\kappa,\bv},\rho\bx_0 + \sigma\bz\> + \widehat{g}_{\kappa,\bv}(\rho\bx_0 + \sigma\bz) + \<\bv,\bz\>\bigr|^2\right]
        \\ & = \widehat{\btheta}_{\kappa,\bv}^\sT \bSigma_\sigma \widehat{\btheta}_{\kappa,\bv} + 2\sigma\widehat{\btheta}_{\kappa,\bv}^\sT\bv + 1 + \|\widehat{g}_{\kappa,\bv}\|_{L^2(\mu_d)}^2 
        \\ & \qquad + 2\E_{(\bx_0,\bz)}\left[\<\widehat{\btheta}_{\kappa,\bv},\rho\bx_0 + \sigma\bz\>\widehat{g}_{\kappa,\bv}(\rho\bx_0 + \sigma\bz)\right] + 2\E_{(\bx_0,\bz)}\left[\<\bv,\bz\>\widehat{g}_{\kappa,\bv}(\rho\bx_0 + \sigma\bz)\right].
    \end{align*}
    We treat the terms separately. By~\eqref{eq:linear_coefficients_bound_vanishing_kappa} and~\eqref{eq:cov_op_bound}, the deviation of the linear component from its target is bounded by
    \[
        \left|\widehat{\btheta}_{\kappa,\bv}^\sT \bSigma_\sigma \widehat{\btheta}_{\kappa,\bv} - \frac{1}{\sigma^2}\bv^\sT\bSigma_\sigma\bv\right| \vee \left|\widehat{\btheta}_{\kappa,\bv}^\sT\bv + \frac{1}{\sigma}\right| \lesssim d^{c-\epsilon/2} + C_{\zeta} d^{c-1/2}.
    \]
    An application of the Cauchy-Schwarz inequality allows us to bound the cross-terms using the linear coefficients and the non-linear residual directly:
    \begin{align*}
        \left|\E_{(\bx_0,\bz)}\left[\<\widehat{\btheta}_{\kappa,\bv},\rho\bx_0 + \sigma\bz\>\widehat{g}_{\kappa,\bv}(\rho\bx_0 + \sigma\bz)\right]\right| & \leq \sqrt{\E_{(\bx_0,\bz)}\left[\<\widehat{\btheta}_{\kappa,\bv},\rho\bx_0 + \sigma\bz\>^2\right]} \sqrt{\E_{(\bx_0,\bz)}\left[\widehat{g}_{\kappa,\bv}(\rho\bx_0 + \sigma\bz)^2\right]} 
        \\ & \lesssim \sqrt{1 + d^{c-\epsilon/2} + C_{\zeta} d^{c-1/2}} \, d^{c/2-1/4}.
    \end{align*}
    A similar bound holds for the second cross-term involving \(\<\bv,\bz\>\). Combining all the components, and using \(\bSigma_\sigma = \rho^2 \bSigma + \sigma^2 \bI_d\) alongside \(\|\bv\|_2^2 = 1\), we conclude that
    \begin{align*}
        \left| \cR_{\bv}(\widehat{f}_{\kappa,\bv}) - \frac{\rho^2}{\sigma^2}\bv^\sT \bSigma \bv \right| & \lesssim d^{c-\epsilon/2} + C_{\zeta} d^{c-1/2} + \sqrt{1 + d^{c-\epsilon/2} + C_{\zeta} d^{c-1/2}} \, d^{c/2-1/4}.
    \end{align*}
    We recover the final statement of the theorem by averaging the expected risk over an orthonormal basis \(\{\bv_i\}_{i=1}^d\) of \(\R^d\).
\end{proof}

\clearpage

\section{Proofs of the gradient flow results}\label{app:gradient_flow}

This section is devoted to proving the asymptotic risk equivalents stated in Theorems~\ref{thm:gf_train_test_equivalents_moderate_time} and~\ref{thm:gf_train_test_equivalents_extended_time}. The proofs are similar to those in Appendix~\ref{app:train-test}, and we similarly suppose that the diffusion time \(t>0\) is fixed throughout this appendix.

\subsection{Proof of Theorem~\ref{thm:gf_train_test_equivalents_moderate_time}}
\label{app:gf_train_test_equivalents_moderate_time}

The argument mirrors the proof of Theorem~\ref{thm:train_test_equivalents_constant_kappa} for ridge regression. Analogously, we define the linearized training and test errors as
\begin{equation}
    \cL_{\mathrm{lin},\bv}^{(\sft)} \deq  \left\|\exp(-\sft \widehat{\cK}_{\mathrm{eff}})\sfz_\bv\right\|_{L^2(\widehat{\pi}_{d,n})}^2, \qquad 
    \cR_{\mathrm{lin},\bv}^{(\sft)} \deq   \left\|\cK_{\mathrm{eff}}\int_0^{\sft} \exp(-\sfs \widehat{\cK}_{\mathrm{eff}})\sfz_{\bv} \de \sfs - \sfz_\bv\right\|_{L^2(\pi_d)}^2.
\end{equation}
We then have the following linearization result, which is similar to Lemma~\ref{lem:train_test_error_linearization}.

\begin{lemma}\label{lem:train_test_error_linearization_gf}
    Suppose that Assumptions~\ref{ass:high_dim},~\ref{ass:input_dist} and~\ref{ass:kernel_func_constant_kappa} hold. Then, there exists \(c>0\) such that for every \(\bv\in \R^d\) with \(\|\bv\|_2=1\),
    \[
        |\cL_{\bv}(\widehat{f}_{\sft,\bv};0) - \cL^{(\sft)}_{\mathrm{lin},\bv}| \prec \frac{\sft}{d^{\frac{3}{2}}} + e^{-d^c} ,\qquad |\cR_{\bv}(\widehat{f}_{\sft,\bv}) - \cR_{\mathrm{lin},\bv}^{(\sft)}| \prec
        \frac{\sft}{d^{\frac{3}{2}}}\left(1 + \frac{\sft}{d}\right)^2
    \]
    uniformly over \(\sft \geq 0\).
\end{lemma}

\begin{proof}
    Note that \(\|\exp(-\sft \widehat{\cK})\|_\op\vee \|\exp(-\sft \widehat{\cK}_{\mathrm{eff}})\|_\op \leq 1\) since \(\widehat{\cK}\) and \(\widehat{\cK}_{\mathrm{eff}}\) are positive semidefinite operators. By Proposition~\ref{prop:bayes_opt_denoiser_linearization}, there exists a constant \(c>0\) such that 
    \[
        \|\widehat{f}^\star_\bv - \sfz_\bv\|_{L^2(\widehat{\pi}_{d,n})} \leq e^{-d^{c}}
    \]
    with very high probability, which implies that replacing the target noise with the linear signal introduces an exponentially small error, i.e.,
    \[
        \left|\cL_{\bv}(\widehat{f}_{\sft,\bv};0) - \|\exp(-\sft \widehat{\cK})\sfz_\bv\|_{L^2(\widehat{\pi}_{d,n})}^2\right| \lesssim e^{-d^{c}}.
    \]
    Next, using the scalar inequality \(|x^2 - y^2| \leq |x-y|(|x|+|y|)\) alongside the reverse triangle inequality yields
    \[
        \left|\|\exp(-\sft \widehat{\cK})\sfz_\bv\|_{L^2(\widehat{\pi}_{d,n})}^2 - \cL_{\mathrm{lin},\bv}^{(\sft)}\right| \leq 2  \|\exp(-\sft \widehat{\cK}) - \exp(-\sft \widehat{\cK}_{\mathrm{eff}})\|_\op,
    \]
    where we used the fact that \(\|\exp(-\sft \widehat{\cK})\|_\op\vee \|\exp(-\sft \widehat{\cK}_{\mathrm{eff}})\|_\op\vee \|\sfz_\bv\|_{L^2(\widehat{\pi}_{d,n})} \leq 1\). To bound the operator difference, define the operator-valued mapping \(\sfs\mapsto \Phi(\sfs)\) on the interval \([0,\sft]\) as
    \[
        \Phi(\sfs) \deq  \exp(-(\sft-\sfs)\widehat{\cK}) \exp(-\sfs\widehat{\cK}_{\mathrm{eff}}).
    \]
    Differentiating with respect to \(\sfs\) via the operator product rule yields
    \[
        \frac{\de}{\de\sfs}\Phi(\sfs) = \exp(-(\sft-\sfs)\widehat{\cK})\widehat{\cK} \exp(-\sfs\widehat{\cK}_{\mathrm{eff}}) - \exp(-(\sft-\sfs)\widehat{\cK}) \widehat{\cK}_{\mathrm{eff}}\exp(-\sfs\widehat{\cK}_{\mathrm{eff}}).
    \]
    The space of bounded linear operators on \(L^2(\widehat{\pi}_{d,n})\), equipped with the uniform operator norm, is a Banach space. By the Fundamental Theorem of Calculus for Bochner integrals, integrating over \(\sfs\) gives \(\Phi(\sft) - \Phi(0)\), which evaluates to
    \[
        \exp(-\sft\widehat{\cK}_{\mathrm{eff}}) - \exp(-\sft\widehat{\cK}) = \int_0^{\sft} \exp(-(\sft-\sfs)\widehat{\cK}) (\widehat{\cK} - \widehat{\cK}_{\mathrm{eff}}) \exp(-\sfs\widehat{\cK}_{\mathrm{eff}}) \de\sfs.
    \]
    Taking the operator norm of both sides and using submultiplicativity, we bound the right-hand side by \(\sft\|\widehat{\cK} - \widehat{\cK}_{\mathrm{eff}}\|_\op\). Combined with Proposition~\ref{prop:linearization_kernel_op}, it follows that
    \[
        \left|\|\exp(-\sft \widehat{\cK})\sfz_\bv\|_{L^2(\widehat{\pi}_{d,n})}^2 - \cL_{\mathrm{lin},\bv}^{(\sft)}\right| \prec \sft\|\widehat{\cK} - \widehat{\cK}_{\mathrm{eff}}\|_\op \prec \frac{\sft}{d^{\frac{3}{2}}}.
    \]

    For the test error, note that
    \[
        \cK\widehat{\cK}^\dagger (\id_{L^2(\widehat{\pi}_{d,n})} - \exp(-\sft \widehat{\cK})) = \cK\int_0^{\sft} \exp(-\sfs \widehat{\cK}) \de \sfs
    \]
    since \(\cK\) annihilates the null space of \(\widehat{\cK}\).
    We apply Lemma~\ref{lem:kernel_op_norm_bound} along with the integration over time to bound the operator norm:
    \[
       \left\| \cK\int_0^{\sft} \exp(-\sfs \widehat{\cK}) \de \sfs \right\|_\op \leq \sft \|\cK\|_{\op} \prec \frac{\sft}{d},
    \]
    and a similar bound holds for the effective kernel operators. Consequently, swapping \(\widehat{f}^\star_\bv\) for \(\sfz_\bv\) yields
    \[
        \left| \cR_{\bv}(\widehat{f}_{\sft,\bv}) - \left\|\cK\int_0^{\sft} \exp(-\sfs \widehat{\cK})\sfz_{\bv} \de \sfs - \sfz_\bv\right\|_{L^2(\pi_d)}^2\right| \prec \frac{\sft}{d}e^{-d^c}\left(1 + \frac{\sft}{d}\right).
    \]
    By the triangle inequality,
    \begin{align*}
       & \left\|\cK\int_0^{\sft} \exp(-\sfs \widehat{\cK}) \de \sfs - \cK_{\mathrm{eff}}\int_0^{\sft} \exp(-\sfs \widehat{\cK}_{\mathrm{eff}}) \de \sfs \right\|_\op 
        \\ & \qquad \leq \|\cK\|_\op \left\|\int_0^{\sft} \left( \exp(-\sfs \widehat{\cK}) - \exp(-\sfs \widehat{\cK}_{\mathrm{eff}}) \right) \de \sfs\right\|_\op + \|\cK - \cK_{\mathrm{eff}}\|_\op \left\|\int_0^{\sft} \exp(-\sfs \widehat{\cK}_{\mathrm{eff}}) \de \sfs\right\|_\op
        \\ & \qquad  \prec \frac{\sft^2}{d^{\frac{5}{2}}} + \frac{\sft}{d^{\frac{3}{2}}}.
    \end{align*}
    The last line follows from Proposition~\ref{prop:linearization_kernel_op}, Lemma~\ref{lem:kernel_op_norm_bound} and the exponential difference bound derived above for the training error. We obtain the final approximation error
    \begin{align*}
        \left| \left\|\cK\int_0^{\sft} \exp(-\sfs \widehat{\cK})\sfz_{\bv} \de \sfs - \sfz_\bv\right\|_{L^2(\pi_d)}^2 - \cR_{\mathrm{lin},\bv}^{(\sft)} \right| & \prec \left(\frac{\sft}{d^{\frac{3}{2}}} + \frac{\sft^2}{d^{\frac{5}{2}}}\right)\left(1 + \frac{\sft}{d}\right).
    \end{align*}
    This concludes the proof.    
\end{proof}

Having approximated the training and test errors by their linearized counterparts, we now derive explicit finite-dimensional representations of the linearized training and test errors. This reduction will allow us to compute the asymptotic limits via random matrix theory. Recall the definition of the linear functionals \(L_{\bw_\bx,\bw_\bz}\) for \(\bw_\bx,\bw_\bz\in \R^d\) introduced in the paragraph preceding Lemma~\ref{lem:linearized_estimators_representation}.

\begin{lemma}[Exact spectral representation of the linearized errors]
\label{lem:gf_linearized_train_test_Q_representation}
    Let \(\sft\geq0\). Then,
    \[
         \cL_{\mathrm{lin}}^{(\sft)}
        =\frac{1}{d}\Tr\left(
        \ell_{\mathrm{det}}^{(\sft)}(\widehat{\bSigma})\right),\qquad \cR_{\mathrm{lin}}^{(\sft)}
        =\frac{1}{d}\Tr\left(\bSigma_\sigma\psi(\widehat{\bSigma},\sft)^2\right)
        +\frac{2\sigma}{d}\Tr\left(\psi(\widehat{\bSigma},\sft)\right)+1,
    \]
    where we recall the definitions of \(\psi\) in~\eqref{eq:gf_lin_profiles} and \(\ell_{\mathrm{det}}^{(\sft)}\) in~\eqref{eq:train_test_error_equivalents_gf}.
\end{lemma}

\begin{proof}
    Let \(\bv\in \S^{d-1}\). By Lemma~\ref{lem:linearized_estimators_representation}, the effective operators preserve the subspace of linear functions and act on their coefficient vectors through the matrices
    \[
        \widehat{\bK}\deq 
        \begin{bmatrix}
            \dfrac{\rho^2}{d}\widehat{\bSigma}+\delta_n\bI_d
            & \dfrac{\rho\sigma}{d}\bI_d
            \\
            \dfrac{\rho\sigma}{d}\widehat{\bSigma}
            & \dfrac{\sigma^2}{d}\bI_d
        \end{bmatrix},
        \qquad
        \bK\deq 
        \begin{bmatrix}
            \dfrac{\rho^2}{d}\widehat{\bSigma}
            & \dfrac{\rho\sigma}{d}\bI_d
            \\
            \dfrac{\rho\sigma}{d}\widehat{\bSigma}
            & \dfrac{\sigma^2}{d}\bI_d
        \end{bmatrix}.
    \]
    More precisely, the power-series representation of the exponential gives
    \begin{equation}\label{eq:gf_Q_operator_coefficient_intertwining}
        e^{-\sft\widehat{\cK}_{\mathrm{eff}}}L_{\bw}
        =L_{e^{-\sft\widehat{\bK}}\bw},
        \qquad
        \cK_{\mathrm{eff}}\int_0^{\sft}
        e^{-\sfs\widehat{\cK}_{\mathrm{eff}}}L_{\bw}\,\de\sfs
        =L_{\bK\int_0^{\sft}e^{-\sfs\widehat{\bK}}\bw\,\de\sfs}.
    \end{equation}
    Note that the matrix \(\widehat{\bK}\) is not symmetric in the raw coefficient coordinates. Define
    \[
        \widehat{\bJ}\deq 
        \begin{bmatrix}
            \widehat{\bSigma}^{1/2}&0\\
            0&\bI_d
        \end{bmatrix},
        \qquad
        \widehat{\bQ}\deq 
        \begin{bmatrix}
            \dfrac{\rho^2}{d}\widehat{\bSigma}+\delta_n\bI_d
            & \dfrac{\rho\sigma}{d}\widehat{\bSigma}^{1/2}
            \\
            \dfrac{\rho\sigma}{d}\widehat{\bSigma}^{1/2}
            & \dfrac{\sigma^2}{d}\bI_d
        \end{bmatrix},
    \]
    and note via direct multiplication that \(\widehat{\bJ}\widehat{\bK}=\widehat{\bQ}\widehat{\bJ}\). Applying this identity term by term in the exponential series yields
    \begin{equation}\label{eq:gf_Q_matrix_intertwining}
        \widehat{\bJ}e^{-\sft\widehat{\bK}}
        =e^{-\sft\widehat{\bQ}}\widehat{\bJ}.
    \end{equation}
    In an eigenbasis of
    \(\widehat{\bSigma}\), the matrix \(\widehat{\bQ}\) decomposes into the
    two-dimensional blocks \(\bQ(x)\) defined in~\eqref{eq:gf_lin_profiles}.

    Define the block embedding
    \[
        \bE_2\deq 
        \begin{bmatrix}0\\\bI_d\end{bmatrix}.
    \]
    Since \(\sfz_\bv=L_{-\bE_2\bv}\),
    \eqref{eq:gf_Q_operator_coefficient_intertwining} and
    \eqref{eq:gf_Q_matrix_intertwining} give
    \[
        e^{-\sft\widehat{\cK}_{\mathrm{eff}}}\sfz_\bv
        =L_{\widehat{\bw}_{\sft}},
        \qquad
        \widehat{\bJ}\widehat{\bw}_{\sft}
        =-e^{-\sft\widehat{\bQ}}\bE_2\bv.
    \]

    For every \(\bw=(\bw_\bx,\bw_\bz)\in\R^{2d}\), independence and
    centering of the Gaussian noise imply
    \[
        \|L_{\bw}\|_{L^2(\widehat\pi_{d,n})}^2
        =\bw_\bx^\sT\widehat{\bSigma}\bw_\bx
        +\|\bw_\bz\|_2^2
        =\|\widehat{\bJ}\bw\|_2^2.
    \]
    Therefore, using the symmetry of \(\widehat{\bQ}\),
    \begin{align*}
        \cL_{\mathrm{lin},\bv}^{(\sft)}
        &=\left\|e^{-\sft\widehat{\bQ}}\bE_2\bv\right\|_2^2
        =\bv^\sT\bE_2^\sT e^{-2\sft\widehat{\bQ}}\bE_2\bv
        =\bv^\sT\ell_{\mathrm{det}}^{(\sft)}
        (\widehat{\bSigma})\bv.
    \end{align*}

    For the test error, we may factor the matrix \(\bK\) as
    \[
        \bK=\begin{bmatrix}\rho\bI_d\\\sigma\bI_d\end{bmatrix}\begin{bmatrix}
            \dfrac{\rho}{d}\widehat{\bSigma}
            &\dfrac{\sigma}{d}\bI_d
        \end{bmatrix}
        = \frac{1}{\sigma}\begin{bmatrix} \rho\bI_d\\ \sigma\bI_d\end{bmatrix}\bE_2^\sT
        \widehat{\bQ}\widehat{\bJ}
    \]
    Hence, using~\eqref{eq:gf_Q_matrix_intertwining} and
    \(\widehat{\bJ}\bE_2=\bE_2\),
    \begin{align*}
        \bK\int_0^{\sft}e^{-\sfs\widehat{\bK}}
        (-\bE_2\bv)\,\de\sfs
        =-\frac{1}{\sigma}\begin{bmatrix} \rho\bI_d\\ \sigma\bI_d\end{bmatrix}\bE_2^\sT\widehat{\bQ}
        \int_0^{\sft}e^{-\sfs\widehat{\bQ}}
        \bE_2\bv\,\de\sfs
        & =-\frac{1}{\sigma}\begin{bmatrix} \rho\bI_d\\ \sigma\bI_d\end{bmatrix}\bE_2^\sT
        \left(\bI_{2d}-e^{-\sft\widehat{\bQ}}\right)
        \bE_2\bv
   \\ & =\begin{bmatrix} \rho\bI_d\\ \sigma\bI_d\end{bmatrix}\psi(\widehat{\bSigma},\sft)\bv.
    \end{align*}
    Evaluating this linear function on an independent test pair
    \((\bx,\bz)\sim\pi_d\) gives
    \begin{align*}
        \cR_{\mathrm{lin},\bv}^{(\sft)}
        =\bv^\sT\psi(\widehat{\bSigma},\sft)
        \bSigma_\sigma\psi(\widehat{\bSigma},\sft)\bv
        +2\sigma\bv^\sT\psi(\widehat{\bSigma},\sft)\bv+1.
    \end{align*}

    Averaging over an orthonormal basis of \(\R^d\) and using cyclicity of the trace yields the result.
\end{proof}

\begin{remark}[Explicit form of the spectral profiles]
\label{rmk:gf_explicit_profiles}
    We can express the spectral profiles \(\ell_{\mathrm{det}}^{(\sft)}\) and \(\psi\) in a more explicit closed form.
    Let
    \[
        m(x)\deq \frac{\rho^2x}{d}+\delta_n+\frac{\sigma^2}{d},
        \qquad
        \gamma_\pm(x)\deq \frac12\left(
        m(x)\pm\sqrt{m(x)^2-\frac{4\sigma^2\delta_n}{d}}\right), \qquad I_\pm(x,\sft)\deq \int_0^{\sft}
        e^{-\sfs\gamma_\pm(x)}\de\sfs
        ,
    \]
    Since \(\Tr\bQ(x)=m(x)\) and
    \(\det\bQ(x)=\sigma^2\delta_n/d\), the numbers
    \(\gamma_+(x)\geq\gamma_-(x)\) are the two eigenvalues of
    \(\bQ(x)\). Whenever \(\gamma_+(x)>\gamma_-(x)\), Sylvester's
    formula gives
    \[
        e^{-\sft\bQ(x)}
        =\frac{
        e^{-\sft\gamma_+(x)}
        \bigl(\bQ(x)-\gamma_-(x)\bI_2\bigr)
        -e^{-\sft\gamma_-(x)}
        \bigl(\bQ(x)-\gamma_+(x)\bI_2\bigr)
        }{\gamma_+(x)-\gamma_-(x)}.
    \]
    Applying this identity to \(\be_2\) and taking the squared Euclidean
    norm yields
    \begin{align*}
        \ell_{\mathrm{det}}^{(\sft)}(x)
        ={}&\frac{\rho^2\sigma^2x}{d^2}
        \left(
        \frac{e^{-\sft\gamma_-(x)}-e^{-\sft\gamma_+(x)}}
        {\gamma_+(x)-\gamma_-(x)}
        \right)^2
        +\left(
        \frac{
        \left(\frac{\sigma^2}{d}-\gamma_+(x)\right)
        e^{-\sft\gamma_-(x)}
        -\left(\frac{\sigma^2}{d}-\gamma_-(x)\right)
        e^{-\sft\gamma_+(x)}
        }{\gamma_+(x)-\gamma_-(x)}
        \right)^2.
    \end{align*}
    Similarly, using
    \(\gamma_+(x)+\gamma_-(x)=m(x)\) and
    \(\gamma_+(x)\gamma_-(x)=\sigma^2\delta_n/d\), the \((2,2)\)
    entry of the same identity gives
    \[
        \psi(x,\sft)
        =\frac{\sigma}{d}
        \frac{
        \bigl(\gamma_-(x)-\delta_n\bigr)I_-(x,\sft)
        -\bigl(\gamma_+(x)-\delta_n\bigr)I_+(x,\sft)
        }{\gamma_+(x)-\gamma_-(x)}.
    \]
    The apparent singularities in these expressions are removable. If
    \(\gamma_+(x)=\gamma_-(x)=\gamma(x)\), symmetry forces
    \(\bQ(x)=\gamma(x)\bI_2\), and the continuous extensions are
    \[
        \ell_{\mathrm{det}}^{(\sft)}(x)=e^{-2\sft\gamma(x)},
        \qquad
        \psi(x,\sft)=-\frac{1-e^{-\sft\gamma(x)}}{\sigma}.
    \]
\end{remark}

\begin{remark}[Spectral form of the linearized denoiser]
    The proof of Lemma~\ref{lem:gf_linearized_train_test_Q_representation} also shows that the linearized denoiser admits the explicit spectral representation
    \[
        \widehat{\boldf}^{\mathrm{lin}}_{\sft}(\bu)
        =\psi(\widehat{\bSigma},\sft)\bu.
    \]
\end{remark}

\begin{lemma}\label{lem:gf_train_test_error_deterministic_equivalent}
    Suppose that Assumptions~\ref{ass:high_dim},~\ref{ass:input_dist}, and
    \ref{ass:kernel_func_constant_kappa} hold. Then, for every \(\epsilon,D>0\), there exists a constant \(C>0\) such that
    \[
        \sup_{\sft\geq0}\left(
        |\cL_{\mathrm{lin}}^{(\sft)}-\cL_{\mathrm{det}}^{(\sft)}|
        \vee
        |\cR_{\mathrm{lin}}^{(\sft)}-\cR_{\mathrm{det}}^{(\sft)}|
        \right)
        \lesssim d^{\epsilon-\frac{1}{2}}
    \]
    with probability at least \(1-Cd^{-D}\).
\end{lemma}

\begin{proof}
    Set \(q\deq d\delta_n\), which satisfies \(0\le q\lesssim1\) by
    Remark~\ref{rmk:delta_n_nonnegativity}, \(n\asymp d\), and the bounds on
    \(\tau_\bSigma\). The equality characterization in the same remark,
    together with the standing bounds on \(\rho^2\tau_\bSigma\) and \(d/n\),
    gives the following dichotomy: either \(q=0\), or there exists a
    deterministic constant \(c_q>0\) such that \(q\geq c_q\) uniformly in
    \(d\).

    The symmetric matrix \(\bQ(x)\) is convenient on the nonnegative real
    spectrum, but its square-root dependence on \(x\) is not convenient for
    holomorphic functional calculus. For \(z\in\C\), and \(x>0\), let
    \[
        \bB(z)\deq 
        \begin{bmatrix}
            \rho^2z+q&\rho\sigma\\
            \rho\sigma z&\sigma^2
        \end{bmatrix},\qquad  \bD(x)\deq \begin{bmatrix}\sqrt{x}&0\\0&1\end{bmatrix},
    \]
    such that \(\bD(x)\bB(x)=d\bQ(x)\bD(x)\).
    Hence, term by term in the exponential series,
    \[
        \bD(x)e^{-u\bB(x)}=e^{-du\bQ(x)}\bD(x),
        \qquad u\geq0.
    \]
    Since \(\bD(x)\be_2=\be_2\), the
    definitions in
    Lemma~\ref{lem:gf_linearized_train_test_Q_representation} imply
    \begin{align}
        \ell_{\mathrm{det}}^{(\sft)}(x)
        &=
        \begin{bmatrix}0&1\end{bmatrix}
        e^{-\frac{\sft}{d}\bB(x)^\sT}
        \begin{bmatrix}x&0\\0&1\end{bmatrix}
        e^{-\frac{\sft}{d}\bB(x)}
        \begin{bmatrix}0\\1\end{bmatrix},
        \label{eq:linearized_train_symbol_matrix_representation}\\
        \psi(x,\sft)
        &=
        -\begin{bmatrix}\rho x&\sigma\end{bmatrix}
        \int_0^{\sft/d}e^{-u\bB(x)}\,\d u
        \begin{bmatrix}0\\1\end{bmatrix}.
        \label{eq:linearized_test_symbol_matrix_representation}
    \end{align}
    For the second identity, we additionally used
    \[
        \bI_2-e^{-\sft\bQ(x)}
        =\bQ(x)\int_0^{\sft}e^{-\sfs\bQ(x)}\,\de\sfs,
        \qquad
        \be_2^\sT\bQ(x)\bD(x)
        =\frac{\sigma}{d}\begin{bmatrix}\rho x&\sigma\end{bmatrix},
    \]
    followed by the change of variables \(u=\sfs/d\). Because
    \(\bB(z)\) is affine in \(z\), its exponential power series converges
    locally uniformly in \(z\), uniformly for \(u\) in compact intervals.
    Thus the right-hand sides of
    \eqref{eq:linearized_train_symbol_matrix_representation} and
    \eqref{eq:linearized_test_symbol_matrix_representation} are entire in
    \(z\) and provide the entire analytic continuations of the two profiles
    defined by \(\bQ(x)\) for \(x\geq0\).

    By~\eqref{eq:cov_concentration}, there exists a deterministic
    \(\Lambda<\infty\) such that the limiting spectral support is contained
    in \([0,\Lambda]\) and, with very high probability, so is
    \(\operatorname{spec}(\widehat\bSigma)\). Fix a rectangular contour
    \(\Gamma\) enclosing this interval, with horizontal edges
    \(\Im z=\pm1\) and vertical edges
    \(\Re z=-\eta\) and \(\Re z=\Lambda+\eta\), where \(\eta>0\) is small
    enough that
    \[
        \inf_{z\in\Gamma}\Re(\rho^2z+\sigma^2)\geq c>0
    \]
    for some constant \(c>0\). The characteristic roots of \(\bB(z)\) satisfy
    \[
        \lambda_+(z)+\lambda_-(z)=\rho^2z+\sigma^2+q,
        \qquad
        \lambda_+(z)\lambda_-(z)=\sigma^2q.
    \]

    Suppose that \(q>0\). Assume, by contradiction, that \(\Re\lambda_+(z)\le0\). Then
    \[
        \Re\lambda_-(z)
        =
        \sigma^2q\,
        \frac{\Re\lambda_+(z)}{|\lambda_+(z)|^2}
        \le0,
    \]
    contradicting
    \(\Re(\lambda_+(z)+\lambda_-(z))>0\). Hence both roots have positive
    real parts. In the present case, the dichotomy given by Remark~\ref{rmk:delta_n_nonnegativity} gives
    \(q\geq c_q>0\); since also \(q\lesssim1\) and \(\Gamma\) is compact,
    there exists \(c_0>0\) such that
    \(\Re\lambda_+(z)\wedge\Re\lambda_-(z)\ge c_0\) for every
    \(z\in\Gamma\). Let \(u\geq0\) and write
    \[
    e^{-u\bB(z)}
    =
    e^{-u\lambda_-(z)}\bI_2+
    \frac{e^{-u\lambda_+(z)}-e^{-u\lambda_-(z)}}
    {\lambda_+(z)-\lambda_-(z)}
    \bigl(\bB(z)-\lambda_-(z)\bI_2\bigr),
    \]
    where, when \(\lambda_+=\lambda_-=\lambda\), the quotient is defined by
    continuity as \(-u e^{-u\lambda}\). By the fundamental theorem of
    calculus,
    \[
        \left|
        \frac{e^{-u\lambda_+(z)}-e^{-u\lambda_-(z)}}
        {\lambda_+(z)-\lambda_-(z)}
        \right|
        \leq
        u\int_0^1
        e^{-u\Re\{(1-\theta)\lambda_-(z)+\theta\lambda_+(z)\}}
        \,\d\theta
        \leq ue^{-c_0u}\lesssim1.
    \]
    Since \(\bB(z)\) and its characteristic roots are uniformly bounded,
    it follows that
    \[
        \sup_{u\ge0}\sup_{z\in\Gamma}
        \|e^{-u\bB(z)}\|_{\op}\lesssim 1.
    \]

    Now suppose that \(q=0\). Then, the roots are \(0\) and
    \(\rho^2z+\sigma^2\), which has positive real part for \(z\in\Gamma\).
    Moreover,
    \[
        \bB(z)
        =
        \begin{bmatrix}\rho\\ \sigma\end{bmatrix}
        \begin{bmatrix}\rho z&\sigma\end{bmatrix},
        \qquad
        \bB(z)^2=(\rho^2z+\sigma^2)\bB(z).
    \]
    Consequently,
    \[
        e^{-u\bB(z)}
        =
        \bI_2+
        \frac{e^{-u(\rho^2z+\sigma^2)}-1}
            {\rho^2z+\sigma^2}\bB(z).
    \]
    Since \(\bB(z)\) is bounded on \(\Gamma\),
    \(|\rho^2z+\sigma^2|\geq c\), and
    \(|e^{-u(\rho^2z+\sigma^2)}|\leq1\),
    \[
        \sup_{u\ge0}\sup_{z\in\Gamma}
        \|e^{-u\bB(z)}\|_{\op}\lesssim 1.
    \]

    Combining the two cases, we conclude that
    \(\sup_{u\ge0}\sup_{z\in\Gamma}\|e^{-u\bB(z)}\|_{\op}\lesssim1\),
    and the same bound holds with \(\bB(z)\) replaced by \(\bB(z)^\sT\).
    Since \(\operatorname{diag}(z,1)\) is uniformly bounded on \(\Gamma\),
    \eqref{eq:linearized_train_symbol_matrix_representation} gives
    \[
        \sup_{\sft\geq0}\sup_{z\in\Gamma}
        |\ell_{\mathrm{det}}^{(\sft)}(z)|\lesssim1.
    \]

    It remains to bound the time integral in
    \eqref{eq:linearized_test_symbol_matrix_representation}. Let
    \(\bE_{11}\deq \operatorname{diag}(1,0)\). For \(q>0\),
    \[
        (\bB(z)-q\bE_{11})
        \int_0^{\sft/d}e^{-u\bB(z)}\,\d u
        =
        \bigl(\bI_2-q\bE_{11}\bB(z)^{-1}\bigr)
        \bigl(\bI_2-e^{-\frac{\sft}{d}\bB(z)}\bigr),
    \]
    where
    \[
        q\bE_{11}\bB(z)^{-1}
        =
        \begin{bmatrix}
            1&-\rho/\sigma\\
            0&0
        \end{bmatrix}.
    \]
    For \(q=0\),
    \[
        \bB(z)\int_0^{\sft/d}e^{-u\bB(z)}\,\d u
        =
        \bI_2-e^{-\frac{\sft}{d}\bB(z)}.
    \]
    In both cases,
    \[
        \bB(z)-q\bE_{11}
        =
        \begin{bmatrix}\rho\\\sigma\end{bmatrix}
        \begin{bmatrix}\rho z&\sigma\end{bmatrix},
        \qquad
        \left\|\begin{bmatrix}\rho\\\sigma\end{bmatrix}\right\|_2=1.
    \]
    It follows from~\eqref{eq:linearized_test_symbol_matrix_representation}
    and \(\rho^2+\sigma^2=1\) that
    \[
        |\psi(z,\sft)|
        =
        \left\|
        (\bB(z)-q\bE_{11})
        \int_0^{\sft/d}e^{-u\bB(z)}\,\d u
        \begin{bmatrix}0\\1\end{bmatrix}
        \right\|_2.
    \]
    The preceding identities and the semigroup bound therefore imply
    \begin{equation}\label{eq:psi_bound}
        \sup_{\sft\geq0}\sup_{z\in\Gamma}|\psi(z,\sft)|\lesssim1.
    \end{equation}

    Lemma~\ref{lem:gf_linearized_train_test_Q_representation} gives directly
    \[
        \cL_{\mathrm{lin}}^{(\sft)}
        =\frac1d\Tr\!\left(
        \ell_{\mathrm{det}}^{(\sft)}(\widehat\bSigma)\right),
    \]
    and
    \[
        \cR_{\mathrm{lin}}^{(\sft)}
        =
        \frac1d\Tr\!\left(
        \bSigma_\sigma\psi(\widehat\bSigma,\sft)^2\right)
        +\frac{2\sigma}{d}\Tr\!\left(
        \psi(\widehat\bSigma,\sft)\right)+1.
    \]
    Holomorphic functional calculus on \(\Gamma\) and
    \eqref{eq:resolvent_fixed_point_measure} therefore give
    \begin{align*}
        \left|
        \cL_{\mathrm{lin}}^{(\sft)}
        -\cL_{\mathrm{det}}^{(\sft)}
        \right|
        &\lesssim
        \sup_{z\in\Gamma}
        \left|
        \frac1d\Tr\!\left(
        \widehat\bR(z)-\bM(z)\right)
        \right|,\\
        \left|
        \cR_{\mathrm{lin}}^{(\sft)}
        -\cR_{\mathrm{det}}^{(\sft)}
        \right|
        &\lesssim
        \sup_{z\in\Gamma}
        \left|
        \frac1d\Tr\!\left(
        \bSigma_\sigma(\widehat\bR(z)-\bM(z))\right)
        \right|
        +
        \sup_{z\in\Gamma}
        \left|
        \frac1d\Tr\!\left(
        \widehat\bR(z)-\bM(z)\right)
        \right|.
    \end{align*}
    The contour \(\Gamma\) lies a fixed positive distance from the limiting
    spectral support, satisfies \(\inf_{z\in\Gamma}|z|>0\), has bounded
    real part, and obeys \(|\Im [z]|\le1\). Proposition~\ref{prop:outside_local_law},
    applied with \(\bA=\bI_d/d\) and
    \(\bA=\bSigma_\sigma/\Tr(\bSigma_\sigma)\), now proves the claim, since
    \(\Tr(\bSigma_\sigma)/d\asymp1\). Both the symbol bounds and the
    local-law event are independent of \(\sft\), so the estimates hold
    simultaneously for every \(\sft\geq0\).
\end{proof}

We can now complete the proof of Theorem~\ref{thm:gf_train_test_equivalents_moderate_time} by combining the above results.

\begin{proof}[Proof of Theorem~\ref{thm:gf_train_test_equivalents_moderate_time}]
    Let \(\{\bv_j\}_{j=1}^d\) be an orthonormal basis of \(\R^d\). Combining Lemma~\ref{lem:train_test_error_linearization_gf} and Lemma~\ref{lem:gf_train_test_error_deterministic_equivalent}, we conclude that for every \(\epsilon, D>0\) there exists a constant \(C > 0\) such that
    \begin{align*}
        |\cL(\widehat{\boldf}_{\sft};0) - \cL_{\mathrm{det}}^{(\sft)}| & \leq \frac{1}{d}\sum_{j=1}^{d}|\cL_{\bv_j}(\widehat{f}_{\sft,\bv_j};0) - \cL_{\mathrm{lin},\bv_j}^{(\sft)}| + |\cL_{\mathrm{lin}}^{(\sft)} - \cL_{\mathrm{det}}^{(\sft)}| 
        \lesssim   \left(1 + \frac{\sft}{d}\right) d^{\epsilon-\frac{1}{2}}
    \end{align*}
    with probability at least \(1-Cd^{-D}\).
    A similar argument gives \( |\cR(\widehat{\boldf}_{\sft}) - \cR_{\mathrm{det}}^{(\sft)}| \lesssim  (1 + \frac{\sft}{d})^3 d^{\epsilon-\frac{1}{2}}\) with probability at least \(1-Cd^{-D}\).
\end{proof}

\subsection{Proof of Theorem~\ref{thm:gf_train_test_equivalents_extended_time}}\label{app:gf_train_test_equivalents_extended_time}

\begin{lemma}[Gradient flow versus ridge]\label{lem:gf_ridge_comparison}
    Let \(\bv\in\R^d\) satisfy \(\|\bv\|_2=1\). For \(\sft>0\), \(\kappa_\sft=d/\sft\), we have
    \[
        \cL^{\mathrm{red}}_{\bv}(\widehat{f}_{\sft,\bv};0)
        \leq \cL^{\mathrm{red}}_{\bv}(\widehat{f}_{\kappa_\sft,\bv};0),
        \qquad
        \|\widehat{f}_{\sft,\bv}\|_{\cH}^2
        \leq 4\|\widehat{f}_{\kappa_\sft,\bv}\|_{\cH}^2 .
    \]
\end{lemma}

\begin{proof}
    Consider the (complete) eigendecomposition of the empirical kernel
    \[
        \widehat{\cK} = \sum_{j \geq 1} \lambda_j \phi_j \phi_j^*, \qquad \lambda_j \geq 0, \quad \<  \phi_j,\phi_{j'} \>_{L^2(\widehat{\pi}_{d,n})} = \ind_{j=j'},
    \]
    and define the weighted spectral measure
    \[
    \nu_\bv = \sum_{j \geq 1} \delta_{\lambda_j} \< \phi_j , \widehat{f}^\star_{\bv} \>^2_{L^2(\widehat{\pi}_{d,n})},
    \]
    where \(\delta_{\lambda_j}\) is the Dirac measure at \(\lambda_j\).
    By~\eqref{eq:gf_train_red_loss},
    \[
        \cL^{\mathrm{red}}_{\bv}(\widehat{f}_{\sft,\bv};0)
        = \left\|\exp(-\sft \widehat{\cK})\widehat{f}^\star_{\bv}\right\|_{L^2(\widehat{\pi}_{d,n})}^2 = \int e^{-2\sft x}\,\nu_\bv(\de x).
    \]
    On the other hand,~\eqref{eq:opt_train_red_loss} gives
    \[
        \cL^{\mathrm{red}}_{\bv}(\widehat{f}_{\kappa_\sft,\bv};0) = \int \left(\frac{x}{\kappa_\sft/d}+1\right)^{-2}\,\nu_\bv(\de x) = \int \left(\sft x+1\right)^{-2}\,\nu_\bv(\de x).
    \]
    Since \(e^{-2u}\leq (1+u)^{-2}\) for \(u\geq0\) and \(\nu_\bv\) is supported on \([0,\infty)\), the first claim follows.

    For the norm comparison, it follows by~\eqref{eq:gf_solution} and the definition of the empirical kernel \(\widehat{\cK}\) that
    \begin{align*}
        \|\widehat{f}_{\sft,\bv}\|_{\cH}^2 = \left\< \int_0^{\sft} \exp(-(\sft-\sfs)\widehat{\cK})\widehat{f}^\star_{\bv} \de \sfs, \widehat{\cK}\int_0^{\sft} \exp(-(\sft-\sfs)\widehat{\cK})\widehat{f}^\star_{\bv} \de \sfs\right\>_{L^2(\widehat{\pi}_{d,n})}.
    \end{align*}
    Define \(q_\sft:\R_{\geq 0}\to\R\) so that
    \begin{equation}\label{eq:q_sft_def}
        \int_0^{\sft} \exp(-(\sft-\sfs)\widehat{\cK}) \de \sfs = q_\sft(\widehat{\cK}),\qquad q_\sft(x) =
        \begin{cases*}
            \frac{1-e^{-\sft x}}{x}, & \(x>0\),\\
            \sft, & \(x=0\).
        \end{cases*}
    \end{equation}
    Then,
    \begin{align*}
        \|\widehat{f}_{\sft,\bv}\|_{\cH}^2 & = \left\< q_\sft(\widehat{\cK})\widehat{f}^\star_{\bv}, \widehat{\cK} q_\sft(\widehat{\cK})\widehat{f}^\star_{\bv} \right\>_{L^2(\widehat{\pi}_{d,n})} 
         = \int x q_\sft(x)^2\,\nu_\bv(\de x) = \int \frac{(1-e^{-\sft x})^2}{x}\,\nu_\bv(\de x)
    \end{align*}
    where the last integrand is set to zero at \(x=0\). On the other hand, by~\eqref{eq:opt_denoiser},
    \[
         \|\widehat{f}_{\kappa_\sft,\bv}\|_{\cH}^2 = \int \frac{x}{(x + \kappa_\sft/d)^2}\,\nu_\bv(\de x) = \int \frac{\sft^2 x}{(\sft x + 1)^2}\,\nu_\bv(\de x).
    \]
    Since \((1-e^{-u})^2/u \leq 4 u/(1+u)^2\) for \(u\geq0\), the second claim follows by plugging in \(u=\sft x\).
\end{proof}
\begin{proof}[Proof of Theorem~\ref{thm:gf_train_test_equivalents_extended_time}] The only properties of the ridge estimator used in the proof of Theorem \ref{thm:train_test_equivalents_vanishing_kappa} are the training-loss bound $ \cL_{\bv}(\widehat f_{\kappa,\bv};0)\prec \kappa$ and the RKHS norm bound $\| \widehat f_{\kappa,\bv} \|_{\cH}^2 \prec d$. By Lemma \ref{lem:gf_ridge_comparison}, these bounds also hold for the gradient flow estimator \(\widehat f_{\sft,\bv}\) with \(\kappa=d/\sft\). The rest of the proof then carries through identically as Theorem \ref{thm:train_test_equivalents_vanishing_kappa}.
\end{proof}

\clearpage

\section{The interpolation limit}\label{app:interpolation-limit}

This appendix is dedicated to the study of the gradient flow estimator \eqref{eq:gf_solution} when the dimension \(d\) and the number of samples \(n\) are fixed while the gradient flow time \(\sft\) tends to infinity. Throughout, the diffusion time \(t>0\), and so \(0<\rho,\sigma<1\), are fixed. We focus on the gradient flow estimator, but the same arguments can be used to show analogous results for the ridge estimator \eqref{eq:opt_denoiser}.

\subsection{Convergence of the test error under truncation}

We establish Proposition~\ref{prop:gf_fixed_d_endpoint} via three sublemmas. We first show that the reproducing kernel Hilbert space \(\cH\) is dense in \(L^2(\widehat\mu_{d,n})\) and \(L^2(\mu_d)\). This is a consequence of the universality of the kernel \(K\) and the fact that both measures have finite exponential moments.

\begin{lemma}[Density of the RKHS]\label{lem:ridgeless-rkhs-density}
Assume Assumption \ref{ass:universal-kernel} holds. Fix \(d,n\) and a realization \(\{\bx_{0,i}\}_{i=1}^{n}\). Then \(\cH\) is dense in \(L^2(\widehat\mu_{d,n})\). Moreover, under Assumption~\ref{ass:input_dist}, \(\cH\) is also dense in \(L^2(\mu_d)\).
\end{lemma}

\begin{proof}
    A straightforward adaptation of Lemma~\ref{lem:rkhs_decomposition} implies that every polynomial of the form \(\bx\mapsto \|\bx\|_2^{k_0}p(\bx)\), where \(k_0\ge0\) is given in Assumption~\ref{ass:universal-kernel} and \(p:\R^d\to\R\) is a polynomial, belongs to \(\cH\). It remains to recall why these polynomials are dense for the two measures at hand.

    The measure \(\widehat\mu_{d,n}\) is a finite mixture of Gaussian measures with common covariance \(\sigma^2\bI_d\), while \(\mu_d\) is the convolution of the law of \(\rho\bx_0\) with \(\cN(0,\sigma^2\bI_d)\). Thus, for either \(\nu\in\{\widehat\mu_{d,n},\mu_d\}\), it follows from Lemma~\ref{lem:exp_noisy_overlap_moments} that there exists \(a_\nu>0\) such that \( \int e^{a_\nu\|\bx\|_2^2}\nu(\d \bx) <\infty\).
    
    Define the finite measure \(\bar{\nu}(\d \bx)=\|\bx\|^{2k_0}\nu(\d \bx)\). Indeed, it follows from the inequality
    \[
        \|\bx\|^{2k_0} \exp(a_\nu\|\bx\|^2/2)\leq C e^{a_\nu\|\bx\|^2}
    \]
    for some constant \(C>0\) that \(\bar{\nu}\) inherits finite exponential moments from \(\nu\). If \(g\in L^2(\bar{\nu})\) and \(g\) is orthogonal to all polynomials, then the signed measure \(g\bar{\nu}\) has a (finite) Laplace transform \(\int \exp(\< t,\bx\>)g(\bx)\bar{\nu}(\d\bx)\) whose Taylor coefficients all vanish. The Laplace transform is therefore identically zero, and hence the signed measure \(g\bar{\nu}\) is zero. Thus, \(g=0\) in \(L^2(\bar{\nu})\), proving that polynomials are dense in \(L^2(\bar{\nu})\).

    Finally, since \(\sigma>0\), \(\|\bx\|^{k_0}\) is almost surely nonzero under \(\nu\) and \(f\mapsto \|\bx\|^{k_0}f\) is an isometry from \(L^2(\bar{\nu})\) to \(L^2(\nu)\). In addition, this isometry is onto. Therefore, the density of polynomials in \(L^2(\bar{\nu})\) implies the density of \(\cH\) in \(L^2(\nu)\).
\end{proof}


Next, we show that the gradient flow estimator converges to the empirical Bayes denoiser \(\widehat \boldf^\star\) in \(L^2(\widehat\pi_{d,n})\) as \(\sft\to\infty\), while the test error is asymptotically lower bounded by the Bayes risk \(\cR(\widehat \boldf^\star)\).

\begin{lemma}\label{lem:ridgeless-limit-test-continuity}
    Suppose that Assumption \ref{ass:input_dist} and Assumption \ref{ass:universal-kernel} hold. Fix $d,n$, the training samples \(\{\bx_{0,i}\}_{i=1}^{n}\) and \(\bv\in \S^{d-1}\). Then,
    \begin{equation}\label{eq:ridgeless-train-convergence}
        \lim_{\sft\to \infty}\|\cS\widehat f_{\sft,\bv}-\widehat\cP\widehat f^\star_{\bv}\|_{L^2(\widehat\pi_{d,n})}
        =
        \lim_{\sft\to \infty}\|\cS_0\widehat f_{\sft,\bv}-\widehat f^\star_{\bv}\|_{L^2(\widehat\mu_{d,n})} = 0,
    \end{equation}
    and
    \begin{equation}\label{eq:ridgeless-test-risk-liminf}
        \liminf_{\sft\to \infty}\cR_{\bv}(\widehat f_{\sft,\bv})
        \geq \cR_{\bv}(\widehat f^\star_{\bv}).
    \end{equation}
\end{lemma}

\begin{proof}
    Denote \(g_\bv \deq  \widehat{\cP}\widehat{f}_\bv^\star \in L^2(\widehat{\pi}_{d,n})\). By~\eqref{eq:gf_solution}, recall that
    \[
        \cS\widehat{f}_{\sft,\bv} = \bigl(\id_{L^2(\widehat{\pi}_{d,n})} - \exp(-\sft \widehat{\cK})\bigr) g_\bv.
    \]
    The operator \(\widehat\cK=\cS\cS^\ast\) is bounded, self-adjoint and nonnegative. Since \(1-e^{-\sft\lambda} \) is a bounded, continuous function of \(\lambda\) for every \(\sft>0\), and \(1-e^{-\sft\lambda} \to \ind_{(0,\infty)}(\lambda)\) as \(\sft\to\infty\), the spectral theorem and dominated convergence ensure that
    \[
        \lim_{\sft\to\infty}(\id_{L^2(\widehat{\pi}_{d,n})}-\exp(-\sft\widehat\cK)) = \proj_{\overline{\operatorname{Ran}(\widehat\cK)}},
    \]
    where the convergence is in the strong operator topology and \(\proj_{\overline{\operatorname{Ran}(\widehat\cK)}}\) is the orthogonal projection onto the closed subspace \(\overline{\operatorname{Ran}(\widehat\cK)}\subseteq L^2(\widehat{\pi}_{d,n})\). 
    
    Since \(\overline{\mathrm{Ran}(\widehat\cK)}=\overline{\mathrm{Ran}(\cS)}\), it remains to identify the closure of \(\mathrm{Ran}(\cS)\). By Lemma~\ref{lem:ridgeless-rkhs-density}, \(\cS_0\cH\) is dense in \(L^2(\widehat\mu_{d,n})\). Furthermore, \(\widehat{\cP}\) is an isometry. Hence \(\cS\cH=\widehat\cP\cS_0\cH\) is dense in \(\widehat\cP L^2(\widehat\mu_{d,n})\), and \(g_{\bv}=\widehat\cP\widehat f^\star_{\bv}\) belongs to this closed subspace. This proves~\eqref{eq:ridgeless-train-convergence}.

    Since \(\sigma>0\), both \(\widehat\mu_{d,n}\) and \(\mu_d\) have strictly positive continuous densities, say \(\widehat p_{d,n}\) and \(p_d\), with respect to  Lebesgue measure. Therefore, for every fixed \(R<\infty\),
    \[
        C_R\deq \sup_{\|\bu\|_2\le R}\frac{p_d(\bu)}{\widehat p_{d,n}(\bu)}<\infty .
    \]
    For \(B_R=\{\|\bu\|_2\le R\}\),
    \[
        \int_{B_R}
        \left|\cT_0\widehat f_{\sft,\bv}-\widehat f^\star_\bv\right|^2\,\de\mu_d
        \le
        C_R
        \left\|\cS_0\widehat f_{\sft,\bv}-\widehat f^\star_\bv
        \right\|^2_{L^2(\widehat\mu_{d,n})}
        \longrightarrow0 
    \]
    as \(\sft\to\infty\) by~\eqref{eq:ridgeless-train-convergence}. Thus, \(\cT_0\widehat f_{\sft,\bv}\to \widehat f^\star_\bv\) in probability under \(\mu_d\) as \(\sft\to\infty\). Since \(\mu_d\) is the pushforward of \(\pi_d\) under \((\bx,\bz)\mapsto\rho\bx+\sigma\bz\), it follows that
    \[
        |\cT\widehat f_{\sft,\bv}-\sfz_\bv|^2 \longrightarrow |\cP\widehat f^\star_\bv-\sfz_\bv|^2
    \]
    in probability under \(\pi_d\). Now choose a sequence \(\sft_k\to\infty\) along which the risks converge to their liminf. Passing to a further subsequence, the preceding convergence holds almost surely. Fatou's lemma then gives \eqref{eq:ridgeless-test-risk-liminf}.
\end{proof}

Finally, we show that the test error of the truncated gradient flow estimator converges to the test error of the empirical Bayes denoiser. We denote
\[
\Pi_A \boldf (\bu) = \frac{\proj_A D_{\boldf} (\bu) - \bu}{\sigma},
\]
where $D_{\boldf} (\bu) =  \bu + \sigma \boldf (\bu)$ and $\proj_A$ is the Euclidean projection onto the compact convex set $A$, which is assumed to contain all the sample vectors $\{\rho \bx_{0,i}\}_{i \in [n]}$.

\begin{lemma}\label{lem:conv-test-error-ridgeless-limit}
    Suppose that Assumptions~\ref{ass:input_dist} and \ref{ass:universal-kernel} hold. Fix \(d,n\) and a training realization \(\{\bx_{0,i}\}_{i=1}^n\). Then,
    \[
        \lim_{\sft\to \infty}\cR (\Pi_A \widehat \boldf_{\sft}) =   \cR (\widehat \boldf^\star).
    \]
\end{lemma}

\begin{proof}
    Note that 
    \[
     D_{\widehat{\boldf}^\star}(\bu)=\rho\sum_{i=1}^n\omega_i(\bu)\bx_{0,i} \in \rho\,\operatorname{conv}\{\bx_{0,1},\ldots,\bx_{0,n}\} \subseteq A,
    \]
     so $\Pi_A \widehat\boldf^\star = \widehat\boldf^\star$. Thus,
     \[
        \|\Pi_A\widehat{\boldf}_\sft(\bu)
       -\widehat{\boldf}^\star(\bu)\|_2 = \frac{1}{\sigma}\| \proj_A D_{\widehat{\boldf}_\sft}(\bu) - \proj_A D_{\widehat{\boldf}^\star}(\bu)\|_2.
     \]
     Since the Euclidean projection onto a convex set is non-expansive,
     \begin{equation}\label{eq:calibrated-pointwise-bounds}
 \|\Pi_A\widehat{\boldf}_\sft(\bu)
       -\widehat{\boldf}^\star(\bu)\|_2
 \le \|\widehat{\boldf}_\sft(\bu)
       -\widehat{\boldf}^\star(\bu)\|_2,
 \qquad
 \|\Pi_A\widehat{\boldf}_\sft(\bu)
       -\widehat{\boldf}^\star(\bu)\|_2
 \le \frac{\operatorname{diam}(A)}{\sigma}.
\end{equation}
 On a fixed ball $B_R$, using a similar density comparison as in the proof of Lemma~\ref{lem:ridgeless-limit-test-continuity},
\[
 \int_{B_R}\|\Pi_A\widehat{\boldf}_\sft(\bu)
       -\widehat{\boldf}^\star(\bu)\|_2^2\de\mu_d
 \le C_R\int\|\widehat{\boldf}_\sft(\bu)
       -\widehat{\boldf}^\star(\bu)\|_2^2\de\widehat\mu_{d,n}
 \longrightarrow0.
\]
The second bound in~\eqref{eq:calibrated-pointwise-bounds} gives
\[
 \int_{B_R^c}\|\Pi_A\widehat{\boldf}_\sft(\bu)
       -\widehat{\boldf}^\star(\bu)\|_2^2\,\de\mu_d
 \le \frac{\operatorname{diam}(A)^2}{\sigma^2}\,\mu_d(B_R^c).
\]
First take the limit $\sft \to \infty$ and then let $R\to\infty$ to obtain \(\Pi_A\widehat{\boldf}_\sft \to \widehat{\boldf}^\star\) in \(L^2(\mu_d)\). This implies the convergence of the test error.
\end{proof}

\begin{proof}[Proof of Proposition~\ref{prop:gf_fixed_d_endpoint}]
    The fact that the RKHS is dense in \(L^2(\widehat\mu_{d,n})\) and \(L^2(\mu_d)\) follows from Lemma~\ref{lem:ridgeless-rkhs-density}. The vector-valued convergence of the gradient flow estimator to the empirical Bayes denoiser in \(L^2(\widehat\mu_{d,n})\) and the lower bound on the test error follow from Lemma~\ref{lem:ridgeless-limit-test-continuity} by an averaging argument over an orthonormal basis of \(\R^d\) to decompose the vector-valued \(L^2\) norm into a sum of scalar \(L^2\) norms. The convergence of the test error follows from Lemma~\ref{lem:conv-test-error-ridgeless-limit}.
\end{proof}

\subsection{Instability of the test error at the ridgeless limit}
\label{sec:ridgeless-instability}

Before establishing the instability of the ridgeless limit, we record
some estimates for polynomials under Gaussian measures. For \(s>0\), \(N\in \N\), let \(\gamma_s := \cN(0,s^2\bI_d)\), and let \(\mathscr{P}_N\) denote the space of polynomials on \(\R^d\) of total degree at most \(N\). Every \(P\in \mathscr{P}_N\) admits a unique decomposition
\begin{equation}
\label{eq:polynomial-symmetric-tensor-representation}
    P(\bu)
    =
    \sum_{k=0}^N
    \< \bA_k,\bu^{\otimes k}\>_\frob,
    \qquad
    \bA_k\in\mathrm{Sym}^k(\R^d).
\end{equation}
Recall the normalized Hermite tensors \(\He_k\) from
\eqref{eq:Hermite-tensor}.

For the counterexample in this subsection, we specialize the population distribution to \(\pi_0=\gamma_1\). Since \(\rho^2+\sigma^2=1\), this gives \(\mu_d=\gamma_1\), with \(0<\sigma<1\) fixed throughout.

\begin{lemma}[Gaussian polynomial estimates]\label{lem:gaussian-polynomial-estimates}
Fix \(d\geq1\) and \(s>0\). Suppose that
\[
    P(\bx)=\sum_{k=0}^{N}\<\bA_k,\bx^{\otimes k}\>_\frob,\qquad \bA_k\in\mathrm{Sym}^k(\R^d).
\]
Then there are unique tensors \(\bB_k\in\mathrm{Sym}^k(\R^d)\) such that
\[
    P(\bx)=\sum_{k=0}^{N}\<\bB_k,\He_k(\bx/s)\>_\frob.
\]
Moreover, there exists \(C=C(d,s)<\infty\) such that
\[
    e^{-CN\log(N+2)}\sum_{k=0}^{N}\|\bA_k\|_\frob^2\leq\sum_{k=0}^{N}\|\bB_k\|_\frob^2 =  \|P\|_{L^2(\gamma_s)}^2\leq e^{CN\log(N+2)}\sum_{k=0}^{N}\|\bA_k\|_\frob^2.
\]
\end{lemma}

\begin{proof}
    We use the convention
    \[
        \bI_d^{\otimes j} \deq \sum_{i_1,\ldots,i_j=1}^d\left(\be_{i_1}\otimes\cdots\otimes\be_{i_j}\right)\otimes \left(\be_{i_1}\otimes\cdots\otimes\be_{i_j}\right)
    \]
    such that \( \<\bA\otimes\bB,\bI_d^{\otimes j}\>_\frob=\<\bA,\bB\>_\frob\) for \(\bA,\bB\in(\R^d)^{\otimes j}\). The standard Hermite identities read
    \begin{equation}\label{eq:Hermite-tensor-explicit}
        \He_k(\bx)
        =
        \frac{1}{\sqrt{k!}}
        \sum_{j=0}^{\lfloor k/2\rfloor}
        \frac{(-1)^jk!}{2^j j!(k-2j)!}
        \proj_{\mathrm{Sym}^k(\R^d)}
        \left(
            \bx^{\otimes(k-2j)}
            \otimes\bI_d^{\otimes j}
        \right),
    \end{equation}
    and, conversely,
    \begin{equation}\label{eq:inverse-Hermite-tensor}
        \bx^{\otimes k}
        =
        \sum_{j=0}^{\lfloor k/2\rfloor}
        \frac{k!}{2^j j!\sqrt{(k-2j)!}}
        \proj_{\mathrm{Sym}^k(\R^d)}
        \left(
            \He_{k-2j}(\bx)\otimes\bI_d^{\otimes j}
        \right).
    \end{equation}

    For \(\bA\in\mathrm{Sym}^k(\R^d)\), let
    \(\Tr^j(\bA)\in\mathrm{Sym}^{k-2j}(\R^d)\) denote its \(j\)-fold trace. Let \(\bg\sim\cN(0,\bI_d)\). Applying \eqref{eq:inverse-Hermite-tensor} to \(P(s\bg)\) gives
    \[
        P(s\bg)=\sum_{k=0}^N\<\bB_k,\He_k(\bg)\>_\frob,
    \]
    where
    \begin{equation}\label{eq:polynomial-tensor-chaos-expansion}
        \bB_k
        =
        \sum_{j=0}^{\lfloor(N-k)/2\rfloor}
        \frac{s^{k+2j}(k+2j)!}
            {2^j j!\sqrt{k!}}
        \Tr^j(\bA_{k+2j}).
    \end{equation}
    The Hermite-tensor isometry gives
    \begin{equation}\label{eq:tensor-Hermite-parseval}
        \|P\|_{L^2(\gamma_s)}^2
        =
        \sum_{k=0}^N\|\bB_k\|_\frob^2.
    \end{equation}
    Conversely, expanding the Hermite tensors by
    \eqref{eq:Hermite-tensor-explicit} gives
    \begin{equation}\label{eq:inverse-polynomial-tensor-chaos-expansion}
        \bA_k
        =
        \sum_{j=0}^{\lfloor(N-k)/2\rfloor}
        \frac{(-1)^j s^{-k}\sqrt{(k+2j)!}}
            {2^j j!k!}
        \Tr^j(\bB_{k+2j}).
    \end{equation}

    Since \(\|\Tr^j(\bA)\|_\frob \leq d^{j/2}\|\bA\|_\frob\),
    \[
        \|\bA_k\|_\frob \leq \sum_{j=0}^{\lfloor(N-k)/2\rfloor}
        \frac{s^{-k}\sqrt{(k+2j)!}}{2^j  j!k!}d^{j/2}\|\bB_{k+2j}\|_\frob \leq (1\vee s^{-1})^N N^{N/2}d^N\sum_{k=0}^{N}\|\bB_k\|_\frob.
    \] 
    Squaring and summing over \(k\) using Cauchy--Schwarz, and using a similar argument for \(\|\bB_k\|_\frob\) in terms of \(\|\bA_k\|_\frob\), gives the desired estimate.
\end{proof}

We also require an estimate for the complex evaluation of polynomials. Recall the elementary bound
\[
    |H_m(z)|
    \leq
    2^{m/2}\sqrt{m!}\exp\left(\sqrt{2m}|z|\right), \qquad z\in\C,
\]
for the physicists' Hermite polynomials (see, for instance, \cite[Eq.~(1.2)]{van1990new}). Since
\[
    \He_m(w)=\frac{2^{-m/2}H_m(w/\sqrt{2})}{\sqrt{m!}},
\]
it follows that \( |\He_m(w)|\leq\exp(\sqrt{m}|w|)\) for every \(w\in\C\). Then, for every \(P\in\mathscr{P}_N\) and \(\bz\in\C^d\), we apply Cauchy--Schwarz and expand the entries of \(\He_k(\bz/s)\) in the normalized product Hermite basis to obtain
\begin{align*}
    |P(\bz)|^2
    \leq
    \binom{N+d}{d}
    \exp\left(\frac{2\sqrt N}{s}\|\bz\|_2\right)
    \|P\|_{L^2(\gamma_s)}^2.
\end{align*}
Consequently, for every compact \(\cD\subset\C^d\), there exists \(C=C(\cD,d,s)<\infty\) such that
\begin{equation}\label{eq:polynomial-complex-evaluation}
    \sup_{\bz\in\cD}|P(\bz)|
    \leq
    C(N+1)^{d/2}
    \exp\!\left(C\sqrt{N+1}\right)
    \|P\|_{L^2(\gamma_s)}.
\end{equation}

\begin{lemma}
\label{lem:empirical-score-polynomial-extrapolation}
Fix \(d\geq1\), \(n\geq2\), and a training realization \(\{\bx_{0,i}\}_{i=1}^n\). Let \(\bv\in \S^{d-1}\) be such that the numbers \(\{\<\bv,\bx_{0,i}\>\}_{i=1}^{n}\) are pairwise distinct. For every \(N\geq 0\), let \(\proj_{\mathscr{P}_N}\) denote the \(L^2(\widehat\mu_{d,n})\)-orthogonal projection onto \(\mathscr{P}_N\).
Then, for every \(r>\sigma\),
\[
    \sup_{N\geq0}
    \|\proj_{\mathscr{P}_N}\widehat f_{\bv}^\star\|_{L^2(\gamma_r)}=\infty.
\]
\end{lemma}

\begin{proof}
Write
\[
    p_N\deq \proj_{\mathscr{P}_N}\widehat f_{\bv}^\star,
    \qquad
    h_N\deq p_N-p_{N-1},\qquad N\geq1.
\]
Fix \(r>\sigma\) and suppose, toward a contradiction, that
\[
    \sup_N\|p_N\|_{L^2(\gamma_r)}=: M < \infty,
\]
so that \(\|h_N\|_{L^2(\gamma_r)}\leq2M\). Choose \(0<s_-<\sigma<s_+<r\). Since
\(\widehat\mu_{d,n}\) is a finite mixture of translates of \(\gamma_\sigma\), there exist \(c_-,c_+>0\) such that \(c_-\gamma_{s_-}\leq\widehat\mu_{d,n}\leq c_+\gamma_{s_+}\) in the sense of pointwise comparison of densities. 
Write the top homogeneous part of \(h_N\) as \(\bu\mapsto \<\bA_N,\bu^{\otimes N}\>_\frob\).
By~\eqref{eq:Hermite-tensor-explicit}, the top homogeneous part of the polynomial \(s_+^N\sqrt{N!}\<\bA_N,\He_N(\bu/s_+)\>_\frob\) is precisely \(\<\bA_N,\bu^{\otimes N}\>_\frob\). Consequently,
\[
    s_+^N\sqrt{N!}
    \<\bA_N,\He_N(\cdot/s_+)\>_\frob-h_N
    \in \mathscr{P}_{N-1}.
\]
Since \(h_N\perp_{L^2(\widehat\mu_{d,n})}\mathscr{P}_{N-1}\),
\begin{align*}
    \|h_N\|_{L^2(\widehat\mu_{d,n})}
    &\leq
    \left\|
        s_+^N\sqrt{N!}
        \<\bA_N,\He_N(\cdot/s_+)\>_\frob
    \right\|_{L^2(\widehat\mu_{d,n})}\leq
    \sqrt{c_+}\,
    s_+^N\sqrt{N!}\|\bA_N\|_\frob,
\end{align*}
where the second inequality follows from
\(\widehat\mu_{d,n}\leq c_+\gamma_{s_+}\) and the Hermite-tensor isometry.

On the other hand, for every \(r>0\), the degree-\(N\) Wiener-chaos component of \(h_N\)
under \(\gamma_r\) is \( r^N\sqrt{N!}
    \<\bA_N,\He_N(\cdot/r)\>_\frob\),
so 
\[
    r^N\sqrt{N!}\|\bA_N\|_\frob
    \leq
    \|h_N\|_{L^2(\gamma_r)}.
\]
Combining the preceding two estimates yields
\[
    \|h_N\|_{L^2(\widehat\mu_{d,n})}
    \leq
    \sqrt{c_+}
    \left(\frac{s_+}{r}\right)^N
    \|h_N\|_{L^2(\gamma_r)}.
\]
Finally, using
\(c_-\gamma_{s_-}\leq\widehat\mu_{d,n}\), we obtain
\[
    \|h_N\|_{L^2(\gamma_{s_-})}
    \leq
    \sqrt{\frac{c_+}{c_-}}
    \left(\frac{s_+}{r}\right)^N
    \|h_N\|_{L^2(\gamma_r)} \leq 2\sqrt{\frac{c_+}{c_-}}
    \left(\frac{s_+}{r}\right)^N
    M.
\]

Let \(\cD\) be a compact subset of \(\C^d\). By~\eqref{eq:polynomial-complex-evaluation}, there exists a constant \(C=C(\cD,d,s_-)\) such that
\[
    \sup_{\bz\in\cD}|h_N(\bz)|
    \leq
    C(N+1)^{d/2}
    \exp\left(C\sqrt{N+1}\right)
    2\sqrt{\frac{c_+}{c_-}}
    \left(\frac{s_+}{r}\right)^N M.
\]
The right-hand side is summable in \(N\), so the Weierstrass convergence theorem shows that \(p_N=p_0+\sum_{k=1}^Nh_k\) converges locally uniformly on \(\C^d\) to an entire function \(G\). On the other hand, polynomial density (see Lemma~\ref{lem:ridgeless-rkhs-density}) gives \(p_N\to\widehat f^\star_{\bv}\) in \(L^2(\widehat\mu_{d,n})\). Passing to an almost-everywhere convergent subsequence therefore shows that \(G=\widehat f^\star_{\bv}\), \(\widehat\mu_{d,n}\)-almost everywhere on \(\R^d\). Since both functions are continuous and \(\widehat\mu_{d,n}\) has a strictly positive density, they agree everywhere on \(\R^d\).

Define
\[
    F_{\bv}(z) \deq  \sum_{i=1}^n \exp\left(-\frac{\rho^2\|\bx_{0,i}\|_2^2}{2\sigma^2} + \frac{\rho\<\bv,\bx_{0,i}\>}{\sigma^2}z\right),\qquad z\in\C.
\]
For every \(t\in\R\), the empirical-score formula gives
\[
    \widehat{f}_{\bv}^\star(t\bv) = - \frac{t}{\sigma} + \sigma \frac{F_{\bv}'(t)}{F_{\bv}(t)}.
\]
Since \(G(t\bv)=\widehat f^\star_{\bv}(t\bv)\), multiplying by \(\sigma F_{\bv}(t)\) yields
\[
    F_{\bv}(t)(\sigma G(t\bv) + t) =\sigma^2 F_{\bv}'(t)
\]
for every \(t\in\R\). Both sides are restrictions to \(\R\) of entire functions of the complex variable \(z\), so the identity theorem gives
\[
    F_{\bv}(z)(\sigma G(z\bv) + z) =\sigma^2 F_{\bv}'(z),
    \qquad z\in\C.
\]
By assumption, \(F_{\bv}\) is an exponential sum with at least two distinct frequencies, and therefore has a complex zero. Indeed, \(F_{\bv}\) is entire of exponential type, so if it had no zeros, Hadamard factorization would give \(F_{\bv}(z)=e^{az+b}\), whereas \(F_{\bv}'(t)/F_{\bv}(t)\) converges as \(t\to\pm\infty\) to the largest and smallest frequencies, respectively, which are distinct. 
Suppose that \(F_{\bv}\) has a zero of order \(m\geq 1\) at \(z_0\in\C\). Since \(\sigma G(z\bv)+z\) is entire, the left-hand side has a zero of order at least \(m\) at \(z_0\), while the right-hand side has a zero of exactly \(m-1\). This is a contradiction. Since \(r>\sigma\) was arbitrary, the claim follows.
\end{proof}

Fix \(\eta>4\) and set, for every \(N\geq 2\),
\begin{equation}\label{eq:lacunary-kernel}
    h(s) \deq  \sum_{k=0}^\infty \alpha_ks^k,\qquad h_{\leq N}(s) \deq  \sum_{k=0}^N \alpha_k s^k,\qquad \alpha_0=\alpha_1=1,\qquad \alpha_k\deq e^{-\eta^k},\quad k\geq2.
\end{equation}
Let \(K,K_{\leq N}:\R^d\times\R^d\to\R\) be the inner-product kernels associated with \(h\) and \(h_{\leq N}\), respectively, and let \(\cH\) and \(\cH_{\leq N}\subset\cH\) be the corresponding RKHSs.
For \(\nu\in\{\mu_d,\widehat\mu_{d,n}\}\), Gaussian-mixture moment bounds give
\begin{equation}\label{eq:lacunary-kernel-admissibility}
    \int K(\bu,\bu)\nu(\de\bu)\leq 1+C+\sum_{k\geq2}e^{-\eta^k}(Ck)^k<\infty.
\end{equation}
Hence, \(K\) is an admissible kernel, and the inclusion operators \(\cS_0:\cH\to L^2(\widehat\mu_{d,n})\) and \(\cT_0:\cH\to L^2(\mu_d)\) are bounded. Let \(\cS_{0,\leq N}\) and \(\cT_{0,\leq N}\) be their restrictions to \(\cH_{\leq N}\), and define
\[
    \widehat\cK_0\deq \cS_0\cS_0^\ast,
    \qquad
    \cK_0\deq \cT_0\cS_0^\ast,
    \qquad
    \widehat\cK_{0,\leq N}\deq \cS_{0,\leq N}\cS_{0,\leq N}^\ast,
    \qquad
    \cK_{0,\leq N}\deq \cT_{0,\leq N}\cS_{0,\leq N}^\ast.
\]
These are respectively the full and truncated empirical and population-transfer kernel operators.

\begin{lemma}\label{lem:lacunary-kernel-polynomial-scales}
    Suppose that \(\pi_0=\gamma_1\). Fix \(d\geq 1\), \(n\geq 2\), and a training realization \(\{\bx_{0,i}\}_{i=1}^n\). For every \(N\geq 2\), let \(\lambda_N\) denote the smallest positive eigenvalue of \(\widehat\cK_{0,\leq N}\), and
    \[
        I_N
        \deq 
        \sup_{0\neq P\in\mathscr{P}_N}
        \frac{\|P\|_{L^2(\mu_d)}}
            {\|P\|_{L^2(\widehat\mu_{d,n})}}.
    \]
    Then \(\operatorname{Ran}(\widehat\cK_{0,\leq N})=\mathscr P_N\), and there exists a constant \(C=C(d,\sigma,\bX)>0\) independent of \(N\) such that
    \[
        e^{-\eta^N-CN\log(N+2)}\leq\lambda_N\leq e^{-\eta^N+CN\log(N+2)},\qquad I_N\leq e^{CN\log(N+2)}.
    \]
    Furthermore,
    \[
        \|\widehat{\cK}_0 - \widehat{\cK}_{0,\leq N}\|_{\op} + \|\cK_0 - \cK_{0,\leq N}\|_{\op} \leq e^{-\eta^{N+1}+CN\log(N+2)}.
    \]
\end{lemma}

\begin{proof}
    Using the tensor feature representation as in Lemma~\ref{lem:rkhs_decomposition}, set
    \[
        \cF_{\leq N}\deq \bigoplus_{k=0}^N\mathrm{Sym}^k(\R^d),
        \qquad
        \phi_{\leq N}(\bu)\deq \left(\sqrt{\frac{\alpha_k}{d^k}}\bu^{\otimes k}\right)_{k=0}^N,
    \]
    such that
    \[
        K_{\leq N}(\bu,\bu')=\<\phi_{\leq N}(\bu),\phi_{\leq N}(\bu')\>_{\cF_{\leq N}}.
    \]
    Let \(\Phi_N:\cF_{\leq N}\to L^2(\widehat\mu_{d,n})\) be the feature operator \((\Phi_N(\bW))(\bu)\deq \<\bW,\phi_{\leq N}(\bu)\>_{\cF_{\leq N}}\). Then, \(\Phi_N^\ast f = \E_{\bu\sim\widehat{\mu}_{d,n}}[f(\bu)\phi_{\leq N}(\bu)]\) and \(\widehat{\cK}_{0,\leq N}=\Phi_N\Phi_N^\ast\). Thus,
    \[
        \bC_{N} \deq  \Phi_N^\ast\Phi_N = \E_{\bu\sim\widehat{\mu}_{d,n}}[\phi_{\leq N}(\bu)\otimes\phi_{\leq N}(\bu)]
    \]
    has the same nonzero spectrum as \(\widehat{\cK}_{0,\leq N}\). Note that, for every \(\bW=(\bW_k)_{k=0}^N\),
    \[
    \<\bW,\bC_N\bW\>
    =
    \E_{\bu\sim\widehat\mu_{d,n}}
    \left[
    \left|\<\bW,\phi_{\leq N}(\bu)\>\right|^2
    \right]
    =
    \|P_{\bW}\|_{L^2(\widehat\mu_{d,n})}^2,
    \]
    where
    \begin{equation}\label{eq:polynomial-from-tensor}
        P_{\bW}(\bu)
        \deq 
        \<\bW,\phi_{\leq N}(\bu)\>_{\cF_{\leq N}}
        =
        \sum_{k=0}^N
        \<\bA_k,\bu^{\otimes k}\>_\frob,
        \qquad
        \bA_k
        \deq 
        \sqrt{\frac{\alpha_k}{d^k}}\bW_k.
    \end{equation}
    If the quadratic form is \(0\), the fact that \(\widehat\mu_{d,n}\) has a strictly positive density implies that the polynomial \(P_{\bW}\) vanishes identically and hence that \(\bW=0\). Therefore, \(\bC_N\) is positive definite. Since \(\Phi_N\) has finite-dimensional domain, it follows that \(\operatorname{Ran}(\widehat\cK_{0,\leq N}) =  \operatorname{Ran}(\Phi_N) = \mathscr{P}_N\).
    Moreover,
    \[
    \lambda_N
    =
    \min_{\|\bW\|_{\cF_{\leq N}}=1}
    \<\bW,\bC_N\bW\>.
    \]

    Let \(\bW=(\bW_k)_{k=0}^N\) be a unit vector in \(\cF_{\leq N}\) and \(P_{\bW}\) the corresponding polynomial from \eqref{eq:polynomial-from-tensor}. Since \(N\geq2\) and \(\alpha_kd^{-k}\geq e^{-\eta^N}d^{-N}\) for every \(0\leq k\leq N\),
    \[
        \sum_{k=0}^N\|\bA_k\|_\frob^2
        =
        \sum_{k=0}^N\frac{\alpha_k}{d^k}\|\bW_k\|_\frob^2
        \geq
        e^{-\eta^N}d^{-N}.
    \]
    Choose \(0<s_-<\sigma<s_+<1\). As in the proof of Lemma~\ref{lem:empirical-score-polynomial-extrapolation}, there exist \(c_-,c_+>0\) such that \(c_-\gamma_{s_-}\leq\widehat\mu_{d,n}\leq c_+\gamma_{s_+}\) in the sense of pointwise comparison of densities. By Lemma~\ref{lem:gaussian-polynomial-estimates}, there exists a constant \(C=C(d,s_-)\) such that
    \begin{align*}
        \<\bW,\bC_N\bW\>_{\cF_{\leq N}}= \|P_{\bW}\|_{L^2(\widehat\mu_{d,n})}^2\geq c_-\|P_{\bW}\|_{L^2(\gamma_{s_-})}^2\geq e^{-\eta^N-CN\log(N+2)},
    \end{align*}
    where we absorbed \(c_-\) and \(d^{-N}\) into the last exponential. Taking the infimum over unit vectors \(\bW\) proves the lower bound for \(\lambda_N\).

    For the reverse bound, take \(\bW_N=\be_1^{\otimes N}\) and \(\bW_k=0\) for \(k<N\). This is a unit vector in \(\cF_{\leq N}\). Using \(\widehat\mu_{d,n}\leq c_+\gamma_{s_+}\) and the upper bound in Lemma~\ref{lem:gaussian-polynomial-estimates}, we obtain
    \[
        \lambda_N \leq \|P_{\bW}\|_{L^2(\widehat\mu_{d,n})}^2 \leq e^{-\eta^N+CN\log(N+2)}
    \]
    where we potentially increase the constant \(C\). This proves the upper bound for \(\lambda_N\).

    Next, let \(P\in\mathscr{P}_N\setminus\{0\}\). We combine the estimates obtained from Lemma~\ref{lem:gaussian-polynomial-estimates} to upper bound \(\|P\|_{L^2(\mu_d)}=\|P\|_{L^2(\gamma_1)}\) and lower bound \(\|P\|_{L^2(\widehat\mu_{d,n})}\gtrsim\|P\|_{L^2(\gamma_{s_-})}\) to obtain
    \[
        \|P\|_{L^2(\mu_d)}^2\leq e^{CN\log(N+2)}\|P\|_{L^2(\widehat\mu_{d,n})}^2.
    \]
    This proves the bound for \(I_N\).

    It remains to control the operator differences. Define
    \[
        R_N(\bu,\bu')\deq \sum_{k>N}e^{-\eta^k}\frac{\<\bu,\bu'\>^k}{d^k},\qquad K(\bu,\bu')-K_{\leq N}(\bu,\bu')=R_N(\bu,\bu').
    \]
    For \(\widetilde\mu\in\{\mu_d,\widehat\mu_{d,n}\}\), let
    \(\bu\sim\widetilde\mu\) and
    \(\bu'\sim\widehat\mu_{d,n}\) be independent. Sub-Gaussian moment bounds give
    \[
        d^{-1}\|\<\bu,\bu'\>\|_{L^{2k}(\widetilde\mu\otimes\widehat\mu_{d,n})}\leq  d^{-1}\|\|\bu\|_2\|_{L^{2k}(\widetilde\mu)}\|\|\bu'\|_2\|_{L^{2k}(\widehat\mu_{d,n})}\leq
        Ck
    \]
    for some \(C=C(d,\sigma,\bX)\). Then, by Cauchy-Schwarz inequality and Minkowski's inequality,
    \[
        \|\widehat{\cK}_0-\widehat{\cK}_{0,\leq N}\|_\op \leq \|R_N\|_{L^2(\widehat\mu_{d,n}\otimes\widehat\mu_{d,n})} \leq \sum_{k>N} e^{-\eta^k}(Ck)^k \leq e^{-\eta^{N+1}+CN\log(N+2)}\,.
    \]
    A similar argument gives the analogous bound for \(\|\cK_0-\cK_{0,\leq N}\|_\op\).
\end{proof}

Recall that \(\cS=\widehat\cP\cS_0\) and \(\cT=\cP\cT_0\), where \(\widehat\cP\) and \(\cP\) are pullback operators from \(L^2(\widehat\mu_{d,n})\) to \(L^2(\widehat{\pi}_{d,n})\) and from \(L^2(\mu_d)\) to \(L^2(\pi_d)\), respectively. The next lemma shows that, at training time \(\sft_N=\lambda_N^{-2}\), gradient flow shadows the degree-\(N\) polynomial projection of its target.

\begin{lemma}\label{lem:gradient-flow-shadows-polynomial-projections}
    Under the specialization \(\pi_0=\gamma_1\), fix \(d\geq 1\), \(n\geq 2\), and a training realization \(\{\bx_{0,i}\}_{i=1}^n\). For every \(N\geq2\), let \(\lambda_N\) denote the smallest positive eigenvalue of \(\widehat\cK_{0,\leq N}\), and let \(\proj_{\mathscr{P}_N}\) denote the \(L^2(\widehat\mu_{d,n})\)-orthogonal projection onto \(\mathscr{P}_N\). Then, for every \(\bv\in \S^{d-1}\),
    \[
        \lim_{N\to\infty}\|\cT_0\widehat{f}_{\sft_N,\bv} - \proj_{\mathscr{P}_N}\widehat{f}_{\bv}^\star\|_{L^2(\mu_d)}=\lim_{N\to\infty}\|\cT \widehat{f}_{\sft_N,\bv} - \cP\proj_{\mathscr{P}_N}\widehat{f}_{\bv}^\star\|_{L^2(\pi_d)}=0,\qquad \sft_N\deq \lambda_N^{-2}.
    \]
\end{lemma}

\begin{proof}

    The diagonal-integrability estimate~\eqref{eq:lacunary-kernel-admissibility} ensures that \(\cS_0\) and \(\cT_0\) are bounded. Moreover, the tail bound in Lemma~\ref{lem:lacunary-kernel-polynomial-scales} gives
    \[
        \sup_{N\geq2}\|\cK_{0,\leq N}\|_\op\leq\|\cK_0\|_\op+\sup_{N\geq2}\|\cK_0-\cK_{0,\leq N}\|_\op <\infty.
    \]
    Since \(\widehat\cP\) is an isometry, we may rewrite~\eqref{eq:gradient_flow} as
    \[
        \partial_\sft\widehat f_{\sft,\bv}=-\cS_0^\ast(\cS_0\widehat f_{\sft,\bv}-\widehat f^\star_\bv).
    \]
    Recall the spectral filter \(q_\sft\) from~\eqref{eq:q_sft_def}. The solution formula therefore gives \( \cT_0 \widehat{f}_{\sft_N,\bv} = \cK_0 q_{\sft_N}(\widehat{\cK}_0)\widehat{f}_{\bv}^\star\)
    and we may decompose
    \[
        \cK_0q_{\sft_N}(\widehat{\cK}_0) - \cK_{0,\leq N}q_{\sft_N}(\widehat{\cK}_{0,\leq N}) = (\cK_0-\cK_{0,\leq N})q_{\sft_N}(\widehat{\cK}_0) + \cK_{0,\leq N}(q_{\sft_N}(\widehat{\cK}_0)-q_{\sft_N}(\widehat{\cK}_{0,\leq N})).
    \]
    By~\eqref{eq:variation_of_parameters},
    \[
        \|e^{-\sfs \widehat{\cK}_0} - e^{-\sfs \widehat{\cK}_{0,\leq N}}\|_\op \leq \sfs\|\widehat{\cK}_0 - \widehat{\cK}_{0,\leq N}\|_\op
    \]
    for every \(\sfs\geq0\). Since \(q_{\sft_N}(\widehat{\cK}_0) = \int_0^{\sft_N} e^{-\sfs \widehat{\cK}_0}\,d\sfs\), it follows that
    \[
        \|q_{\sft_N}(\widehat{\cK}_0)-q_{\sft_N}(\widehat{\cK}_{0,\leq N})\|_\op \leq \sft_N^2\|\widehat{\cK}_0 - \widehat{\cK}_{0,\leq N}\|_\op,\qquad \|q_{\sft_N}(\widehat{\cK}_0)\|_\op\leq \sft_N.
    \]
    Then, by Lemma~\ref{lem:lacunary-kernel-polynomial-scales} and the fact that \(\|\cK_{0,\leq N}\|_\op\) is bounded uniformly in \(N\),
    \begin{align*}
        \| \cK_0q_{\sft_N}(\widehat{\cK}_0)- \cK_{0,\leq N}q_{\sft_N}(\widehat{\cK}_{0,\leq N})\|_\op & \leq \sft_N \|\cK_0-\cK_{0,\leq N}\|_\op+ \|\cK_{0,\leq N}\|_\op \sft_N^2
        \|\widehat{\cK}_0-\widehat{\cK}_{0,\leq N}\|_\op 
        \\ &  \lesssim e^{(4-\eta)\eta^N + CN\log(N+2)} \to 0
    \end{align*}
    as \(N\to\infty\). Since \(\|\widehat{f}_{\bv}^\star\|_{L^2(\widehat{\mu}_{d,n})}\leq 1\) by conditional Jensen's inequality, 
    \begin{equation}\label{eq:full-to-truncated-operator-bound}
        \|\cK_0q_{\sft_N}(\widehat{\cK}_0)\widehat{f}_{\bv}^\star - \cK_{0,\leq N}q_{\sft_N}(\widehat{\cK}_{0,\leq N})\widehat{f}_{\bv}^\star\|_{L^2(\mu_d)}\to 0
    \end{equation}
    as \(N\to\infty\).

    It remains to compare the truncated flow with the polynomial projection. Note that
    \[
        \widehat{\cK}_{0,\leq N}q_{\sft_N}(\widehat{\cK}_{0,\leq N})
        =\id-e^{-\sft_N\widehat{\cK}_{0,\leq N}}.
    \]
    Since \(\operatorname{Ran}(\widehat{\cK}_{0,\leq N})=\mathscr{P}_N\) and \(\widehat{\cK}_{0,\leq N}\) is self-adjoint, \(\ker(\widehat{\cK}_{0,\leq N})=\mathscr{P}_N^\perp\), and therefore
    \begin{align*}
        \widehat{\cK}_{0,\leq N}q_{\sft_N}(\widehat{\cK}_{0,\leq N})\widehat{f}^{\star}_\bv- \proj_{\mathscr{P}_N}\widehat{f}^{\star}_\bv= - e^{-\sft_N\widehat{\cK}_{0,\leq N}}\proj_{\mathscr{P}_N}\widehat{f}^{\star}_\bv.
    \end{align*}
    By definition of \(\lambda_N\), it follows that
    \[
        \|\widehat{\cK}_{0,\leq N}q_{\sft_N}(\widehat{\cK}_{0,\leq N})\widehat{f}^{\star}_\bv - \proj_{\mathscr{P}_N}\widehat{f}^{\star}_\bv\|_{L^2(\widehat{\mu}_{d,n})} \leq e^{-\sft_N\lambda_N}\|\widehat{f}^{\star}_\bv\|_{L^2(\widehat{\mu}_{d,n})} \leq e^{-\sft_N\lambda_N}.
    \]
    Recall the definition of \(I_N\) from Lemma~\ref{lem:lacunary-kernel-polynomial-scales}. More precisely, if
    \[
        g_N\deq q_{\sft_N}(\widehat\cK_{0,\leq N})\widehat f^\star_\bv,
        \qquad
        P_N(\bu)\deq \int K_{\leq N}(\bu,\bu')g_N(\bu')\,\widehat\mu_{d,n}(\de\bu'),
    \]
    then \(P_N\in\mathscr P_N\), while \(\widehat\cK_{0,\leq N}g_N\) and \(\cK_{0,\leq N}g_N\) are the realizations of \(P_N\) in \(L^2(\widehat\mu_{d,n})\) and \(L^2(\mu_d)\), respectively. Identifying \(\proj_{\mathscr P_N}\widehat f^\star_\bv\) with its (unique) polynomial representative, the definition of \(I_N\), applied to \(P_N-\proj_{\mathscr P_N}\widehat f^\star_\bv\), gives
    \begin{align*}
        \|\cK_{0,\leq N}q_{\sft_N}(\widehat{\cK}_{0,\leq N})\widehat{f}^{\star}_\bv
        - \proj_{\mathscr{P}_N}\widehat{f}^{\star}_\bv\|_{L^2(\mu_d)} & \leq I_N e^{-\sft_N\lambda_N}
       \\ &  \leq \exp\left(CN\log(N+2)- e^{\eta^N-CN\log(N+2)}\right) \to 0
    \end{align*}
    as \(N\to\infty\), where the last inequality follows from Lemma~\ref{lem:lacunary-kernel-polynomial-scales}. Combining this with~\eqref{eq:full-to-truncated-operator-bound} yields
    \[
        \|\cT_0\widehat{f}_{\sft_N,\bv} - \proj_{\mathscr{P}_N}\widehat{f}_\bv^\star\|_{L^2(\mu_d)} \to 0.
    \]
    The upper bound on \(\lambda_N\) in Lemma~\ref{lem:lacunary-kernel-polynomial-scales} also gives \(\lambda_N\to0\), and hence \(\sft_N=\lambda_N^{-2}\to\infty\).
    Since \(\cT=\cP\cT_0\) and \(\cP\) is an isometry, the second convergence follows immediately.
\end{proof}

We are finally ready to prove the main result of this subsection, which shows that the gradient-flow risk can be unbounded for a universal kernel.

\begin{proposition}\label{prop:gradient-flow-risk-diverges}
     Suppose that \(\pi_0=\gamma_1\). Fix \(d\geq 1\) and \(n\geq 2\). Let \(K\) be the inner-product kernel associated with~\eqref{eq:lacunary-kernel}. Then \(K\) satisfies Assumption~\ref{ass:universal-kernel}, and, for every fixed \(\bv\in\S^{d-1}\), almost surely over \(\{\bx_{0,i}\}_{i=1}^n\sim\pi_0^{\otimes n}\),
     \[
        \limsup_{\sft\to\infty}\cR_\bv(\widehat{f}_{\sft,\bv}) = \infty.
     \]
\end{proposition}

\begin{proof}
    All Taylor coefficients of \(h\) are strictly positive, so Assumption~\ref{ass:universal-kernel} holds. Let \(\bv\in\S^{d-1}\). Since \(\pi_0=\gamma_1\), almost surely \(\{\<\bx_{0,i},\bv\>\}_{i=1}^n\) are pairwise distinct. Conditional on such a training realization, Lemma~\ref{lem:empirical-score-polynomial-extrapolation}, applied with \(r=1\), gives
    \[
        \sup_N \|\proj_{\mathscr{P}_N}\widehat{f}_\bv^\star\|_{L^2(\mu_d)} = \infty.
    \]
    On the other hand, Lemma~\ref{lem:gradient-flow-shadows-polynomial-projections} gives a sequence \(\sft_N\to\infty\) such that
    \[
        \|\cT\widehat{f}_{\sft_N,\bv} - \cP\proj_{\mathscr{P}_N}\widehat{f}_\bv^\star\|_{L^2(\pi_d)} \to 0.
    \]
    Choose a subsequence \(N_\ell\) along which \(\|\proj_{\mathscr{P}_{N_\ell}}\widehat{f}_\bv^\star\|_{L^2(\mu_d)}\to\infty\). Since \(\cP:L^2(\mu_d)\to L^2(\pi_d)\) is an isometry,
    \[
        \|\cT \widehat{f}_{\sft_{N_\ell},\bv}\|_{L^{2}(\pi_d)} \geq \|\cP\proj_{\mathscr{P}_{N_\ell}}\widehat{f}_\bv^\star\|_{L^2(\pi_d)} - \|\cT \widehat{f}_{\sft_{N_\ell},\bv} - \cP\proj_{\mathscr{P}_{N_\ell}}\widehat{f}_\bv^\star\|_{L^2(\pi_d)} \to \infty
    \]
    as \(\ell\to\infty\).
    Finally, \(\|\sfz_\bv\|_{L^2(\pi_d)}=1\), and hence
    \[
        \cR_\bv(\widehat f_{\sft_{N_\ell},\bv})^{1/2}
        =
        \|\cT\widehat f_{\sft_{N_\ell},\bv}-\sfz_\bv\|_{L^2(\pi_d)}
        \geq
        \|\cT\widehat f_{\sft_{N_\ell},\bv}\|_{L^2(\pi_d)}-1
        \to \infty
    \]
    as \(\ell\to\infty\).
\end{proof}

\subsection{Test error of the empirical Bayes denoiser}

We compute the population risk of the empirical Bayes denoiser~\eqref{eq:emp_bayes_opt_denoiser} itself. To do so, we use the following ``winner-take-all'' lemma, which shows that the empirical Bayes weights are asymptotically concentrated on a single training point.

\begin{lemma}\label{lem:wta-under-lsi}
Suppose that Assumptions~\ref{ass:high_dim} and~\ref{ass:input_dist} hold.
Then,
\begin{equation}\label{eq:wta-under-lsi}
    \E_{(\bx,\bz)\sim\pi_d}\left[1-\sum_{i=1}^n\omega_i(\rho\bx+\sigma\bz)^2\mid \bX\right]
    \prec \frac{1}{\sqrt{d}}.
\end{equation}
\end{lemma}

\begin{proof}
    Fix \(\epsilon\in (0,1/2)\), and let \(\cE_d\) be the training-sample event on which
\[
    \max_{i\in[n]}\left|\|\bx_{0,i}\|_2^2-\Tr(\bSigma)\right|\leq d^{1/2+\epsilon},
    \qquad
    \max_{i\neq j}|\< \bx_{0,i},\bx_{0,j}\>|\leq d^{1/2+\epsilon} .
\]
By the concentration bounds in~\eqref{eq:norm_concentration_overlaps}, \(\cE_d\) holds with very high probability. On \(\cE_d\), Assumption~\ref{ass:input_dist} gives
\(\Tr(\bSigma)\geq c d\) for all large \(d\), and therefore
\begin{equation}\label{eq:pairwise-distance-lower}
    \min_{i\neq j}\|\bx_{0,i}-\bx_{0,j}\|_2^2\geq c' d
\end{equation}
for some constant \(c'>0\). Write
\[
    \omega_k(\rho \bx+\sigma\bz) = \frac{\exp\left(-\frac{\rho^2}{2\sigma^2}\|\bx_{0,k}\|^2+\frac{\rho^2}{\sigma^2}\<\bx_{0,k},\bx\>+\frac{\rho}{\sigma}\<\bx_{0,k},\bz\>\right)}{\sum_{j=1}^{n}\exp\left(-\frac{\rho^2}{2\sigma^2}\|\bx_{0,j}\|^2+\frac{\rho^2}{\sigma^2}\<\bx_{0,j},\bx\>+\frac{\rho}{\sigma}\<\bx_{0,j},\bz\>\right)}
\]
and set, for every \(k \in[n]\),
\[
    S_k\equiv S_k(\bz;\bx,\{\bx_{0,i}\}_{i=1}^{n}) \deq  -\frac{\rho^2}{2\sigma^2}\|\bx_{0,k}\|^2+\frac{\rho^2}{\sigma^2}\<\bx_{0,k},\bx\>+\frac{\rho}{\sigma}\<\bx_{0,k},\bz\>.
\]
Let \(S_{(1)}\geq S_{(2)}\) be the two largest values of \(\{S_k\}_{k=1}^n\) and set \(\Delta\deq S_{(1)}-S_{(2)}\). By~\eqref{eq:pairwise-distance-lower}, the two largest values are attained at unique indices almost surely. Randomly partition \([n]\), independently of \(\bz\), into two balanced sets \(\cA,\cB\). Note that the two indices corresponding to \(S_{(1)}\) and \(S_{(2)}\) belong to different sets with probability at least \(1/2\). Therefore, for every \(r\geq 0\),
\[
    \P(\Delta \leq r\mid \bx,\{\bx_{0,i}\}_{i=1}^n) \leq 2\E_{\cA,\cB}\P\left(\left|\max_{k\in \cA}S_k - \max_{j\in \cB}S_j\right|\leq r \mid \bx,\{\bx_{0,i}\}_{i=1}^n\right).
\]
Conditionally on \(\bx\) and \(\{\bx_{0,i}\}_{i=1}^n\), the random variables \(\{S_k\}_{k\in [n]}\) are jointly Gaussian with
\begin{gather*}
     \E[S_k \mid \bx,\{\bx_{0,i}\}_{i=1}^n] = -\frac{\rho^2}{2\sigma^2}\|\bx_{0,k}\|^2+\frac{\rho^2}{\sigma^2}\<\bx_{0,k},\bx\>, \qquad \Var(S_k \mid \bx,\{\bx_{0,i}\}_{i=1}^n) = \frac{\rho^2}{\sigma^2}\|\bx_{0,k}\|^2, 
     \\ \Cov(S_k,S_j\mid \bx,\{\bx_{0,i}\}_{i=1}^n) = \frac{\rho^2}{\sigma^2}\<\bx_{0,k},\bx_{0,j}\>.
\end{gather*}
In particular, on the event \(\cE_d\),
\begin{align*}
    \left|
    \frac{\Cov(S_k,S_j\mid \bx,\{\bx_{0,i}\}_{i=1}^n)}
    {\sqrt{\Var(S_k\mid \bx,\{\bx_{0,i}\}_{i=1}^n)
    \Var(S_j\mid \bx,\{\bx_{0,i}\}_{i=1}^n)}}
    \right|
    &=
    \frac{|\<\bx_{0,k},\bx_{0,j}\>|}
    {\|\bx_{0,k}\|_2\|\bx_{0,j}\|_2}
    \lesssim d^{\epsilon-\frac{1}{2}} < 1
\end{align*}
for all \(k\neq j\). Furthermore, for \(k\neq j\), the conditional correlation is bounded away from \(1\), and
\begin{align*}
    \sqrt{\Var(S_k\mid \bx,\{\bx_{0,i}\}_{i=1}^n)}-\frac{\Cov(S_k,S_j\mid \bx,\{\bx_{0,i}\}_{i=1}^n)}{\sqrt{\Var(S_k\mid \bx,\{\bx_{0,i}\}_{i=1}^n)}}&=\frac{\rho}{\sigma}\left(\|\bx_{0,k}\|_2-\frac{\<\bx_{0,k},\bx_{0,j}\>}{\|\bx_{0,k}\|_2}\right)\gtrsim \frac{\rho}{\sigma}\sqrt{d}.
\end{align*}
The ratio in the first line equals one when \(k=j\). Therefore, using the anticoncentration inequality of~\cite[Theorem 2.4]{belloni2024anti}, it follows that, for every \(r\geq 0\),
\begin{align*}
     \P\left(\left|\max_{k\in \cA}S_k - \max_{j\in \cB}S_j\right|\leq r \mid \bx,\{\bx_{0,i}\}_{i=1}^n\right)
     & \lesssim \frac{\sigma}{\rho}\frac{\E[\max_{k\in [n]}|\<\bx_{0,k},\bz\>| \mid \bx,\{\bx_{0,i}\}_{i=1}^n]}{d}r \lesssim \frac{\sigma}{\rho}\sqrt{\frac{\log(2n)}{d}}\,r.
\end{align*}
The constants are independent of \(\bx\), \(\{\bx_{0,i}\}_{i=1}^n\), and the balanced partition, and we applied the Gaussian maximal inequality to bound the expected maximum. Taking \(r=L\log(d)\) for a fixed constant \(L>2\), and using \(n\asymp d\) and \(\rho/\sigma\geq c>0\) for some constant \(c>0\), we obtain
\[
    \P(\Delta \leq r\mid \bx,\{\bx_{0,i}\}_{i=1}^n)
    \lesssim \frac{(\log d)^{3/2}}{\sqrt d}
    \prec \frac{1}{\sqrt d}.
\]
On the event \(\{\Delta(\bz;\bx) > r\}\), we have
\[
    1 - \sum_{i=1}^n \omega_i(\rho\bx+\sigma\bz)^2 \leq 1 - \omega_{i_\star}(\rho\bx+\sigma\bz)^2 \leq 2ne^{-r},
\]
where \(i_\star\) is an index attaining the maximum (which depends on \(\bz\), \(\bx\), and the training sample). Since \(0\leq 1-\sum_{i=1}^n \omega_i(\rho\bx+\sigma\bz)^2\leq 1\), it follows that
\[
    \E\left[1 - \sum_{i=1}^n \omega_i(\rho\bx+\sigma\bz)^2\mid \bx, \{\bx_{0,i}\}_{i=1}^n\right]
    \lesssim \frac{(\log d)^{3/2}}{\sqrt d} + d^{1-L}
    \prec \frac{1}{\sqrt d}.
\]
The bound is uniform in \(\bx\). Averaging over \(\bx\sim\pi_0\) and using that \(\cE_d\) holds with very high probability concludes the proof.
\end{proof}

We are now ready to prove Theorem~\ref{thm:ridgeless_main}.

\begin{proof}[Proof of Theorem~\ref{thm:ridgeless_main}]
    Let \((\bx,\bz)\sim\pi_d\), set \(\bu=\rho\bx+\sigma\bz\), and write
\[
    \bM(\bu)\deq \sum_{i=1}^n \omega_i(\bu)\bx_{0,i} .
\]
By definition,
\begin{equation}\label{eq:empirical-score-risk-identity}
    \widehat{\boldsymbol f}^{\star}(\rho\bx+\sigma\bz)+\bz
    =\frac{\rho}{\sigma}\bigl(\bM(\rho\bx+\sigma\bz)-\bx\bigr).
\end{equation}
Therefore
\begin{equation}\label{eq:risk-expanded-M}
    \cR(\widehat{\boldsymbol f}^{\star})
    =\frac{\rho^2}{\sigma^2}
    \left\{
        \frac{1}{d}\E\|\bM(\bu)\|_2^2
        -\frac{2}{d}\E\< \bM(\bu),\bx\>
        +\frac{1}{d}\E\|\bx\|_2^2
    \right\},
\end{equation}
where the expectation is over the fresh pair \((\bx,\bz)\), conditionally on the training data. The last term equals \(\tau_\bSigma\). We show that the first term is \(\tau_\bSigma+o_{d,\P}(1)\) and the middle term is \(o_{d,\P}(1)\). Define
\[
    \varepsilon_d
    \deq \max_{i\in[n]}\left|\frac{\|\bx_{0,i}\|_2^2}{d}-\tau_\bSigma\right|
      +\max_{i\neq j}\frac{|\< \bx_{0,i},\bx_{0,j}\>|}{d},\qquad \varepsilon_d' \deq  
    \E_{(\bx,\bz)}\left[1-\sum_{i=1}^n\omega_i(\bu)^2\mid\{\bx_{0,i}\}_{i=1}^n\right].
\]
By~\eqref{eq:norm_concentration_overlaps} and Lemma~\ref{lem:wta-under-lsi}, \(\varepsilon_d,\varepsilon_d'\prec d^{-1/2}\). Since the weights are non-negative and sum to one,
\begin{align*}
    \left|\frac{\|\bM(\bu)\|_2^2}{d}-\tau_\bSigma\sum_{i=1}^n\omega_i(\bu)^2\right| & \leq \max_{i\in[n]}\left|\frac{\|\bx_{0,i}\|_2^2}{d}-\tau_\bSigma\right|\sum_{i=1}^n\omega_i(\bu)^2 + \sum_{i\neq j}\frac{|\< \bx_{0,i},\bx_{0,j}\>|}{d}\omega_i(\bu)\omega_j(\bu)\\
    &\leq \varepsilon_d
\end{align*}
and therefore
\begin{align*}
    \left|
        \frac1d\E\!\left[\|\bM(\bu)\|_2^2
        \,\middle|\,\{\bx_{0,i}\}_{i=1}^n\right]
        -\tau_{\bSigma}
    \right|
    &\le
    \varepsilon_d
    +
    \tau_{\bSigma}
    \E\!\left[
        1-\sum_{i=1}^n\omega_i(\bu)^2
        \,\middle|\,\{\bx_{0,i}\}_{i=1}^n
    \right]  =\varepsilon_d+\tau_{\bSigma}\varepsilon_d'
    \prec d^{-1/2}.
\end{align*}

It remains to control the cross term. Conditionally on the training data,
\[
    \left|\frac{1}{d}\E \big[ \< \bM(\bu),\bx\> \mid \{\bx_{0,i}\}_{i=1}^n\big]\right|
    \leq \frac{1}{d}\E_\bx \bigg[\max_{i\in[n]}|\< \bx_{0,i},\bx\>|\bigg] .
\]
Since \(\E\bx=0\), the logarithmic Sobolev inequality and the Herbst argument give, conditionally on the training data,
\[
    \E_{\bx}\exp\left(\lambda\<\bx_{0,i},\bx\>\right)\leq \exp\left(
        \frac{C_{\mathrm{LSI}}\lambda^2
        \|\bx_{0,i}\|_2^2}{2}
    \right).
\]
Applying the standard maximal inequality yields
\[
    \E_{\bx}
    \left[
        \max_{i\in[n]}
        |\<\bx_{0,i},\bx\>|
    \right]
    \lesssim
    (\max_{i}\|\bx_{0,i}\|_2)\sqrt{\log(n)} \prec \sqrt{d}.
\]
Here, we have used \eqref{eq:norm_concentration_overlaps}, \(\Tr(\bSigma)\lesssim d\) and the proportional limiting regime \(n\asymp d\) to obtain the last bound. Substituting these estimates into~\eqref{eq:risk-expanded-M} gives
\[
    \left|\cR(\widehat{\boldsymbol f}^{\star}) - 2\frac{\rho^2}{\sigma^2}\tau_{\bSigma}\right| \prec \frac{1}{\sqrt{d}}.
\]

Under Assumption~\ref{ass:universal-kernel},
Proposition~\ref{prop:gf_fixed_d_endpoint} gives, for every fixed
\(d,n\) and training realization,
\[
    \lim_{\sft\to\infty}
    \cR(\widetilde{\boldsymbol f}_{\sft})
    =
    \cR(\widehat{\boldsymbol f}^{\star}).
\]
Combining this identity with the preceding high-dimensional asymptotic
proves the second assertion.
\end{proof}

\clearpage

\section{Reverse process dynamics}\label{app:reverse_process}

In this section, we study the dynamics of the reverse diffusion process and its linearization. We suppose throughout that \(0<t_0<T\) are fixed, that \(t\in [t_0,T]\mapsto\sft_t\in \R_{\geq 0}\) is a deterministic piecewise-continuous schedule, and that \(\widetilde{\bx}_t\) solves~\eqref{eq:linearized_reverse_SDE}.

\subsection{Properties of the linearized reverse process}

Unlike \(\widehat{\cK}_{\mathrm{eff},t}\), whose self-induced regularization is identical across training samples, the sample-corrected operator \(\widetilde{\cK}_t\) contains the sample-dependent diagonal corrections \(\widetilde{\delta}_{i,t}\). Consequently, it cannot be reduced to a spectral function of the empirical covariance \(\widehat\bSigma\), and the \(2d\)-dimensional representation of Lemma~\ref{lem:gf_linearized_train_test_Q_representation} no longer applies. Nevertheless, \(\widetilde{\cK}_t\) preserves a larger \((n+d)\)-dimensional subspace, on which it admits an explicit matrix representation.

Define the subspace
\begin{equation}\label{eq:tilde-V-def}
    \widetilde{\cV} \deq  \{F_{\ba,\bb}:\ba\in\R^n,\ \bb\in\R^d\} \subset L^2(\widehat\pi_{d,n}),\qquad F_{\ba,\bb}(\bx_{0,i},\bz) \deq  \sqrt{n}a_i + \<\bb,\bz\>.
\end{equation}

\begin{lemma}\label{lem:reverse_linearized_check_score_representation}
    For every \(t\in [t_0,T]\), \(\widetilde{\cK}_t\) maps \(\widetilde{\cV}\) to itself, and the linearized score function \(\widetilde\bS_t\) admits the representation
    \[
        \widetilde\bS_t(\bx)
        =\frac{1}{\sigma_t}\widetilde\bPsi_t\bx
    \]
    where
    \[
        \widetilde{\bPsi}_t \deq  
         -\frac{1}{\sigma_t}
        \begin{bmatrix}0&\bI_d\end{bmatrix}
        \left(\bI_{n+d}-e^{-\sft_t\widetilde{\bK}_t}\right)
        \begin{bmatrix}0\\\bI_d\end{bmatrix}, \qquad \widetilde{\bK}_t \deq  \begin{bmatrix}
            \widetilde\bD_t+\dfrac{\rho_t^2}{dn}\bX^\sT\bX
            &\dfrac{\rho_t\sigma_t}{d\sqrt n}\bX^\sT\\[1.2ex]
            \dfrac{\rho_t\sigma_t}{d\sqrt n}\bX
            &\dfrac{\sigma_t^2}{d}\bI_d
        \end{bmatrix},
    \]
    and \(\widetilde\bD_t \deq  \operatorname{diag}(\widetilde\delta_{1,t},\ldots,\widetilde\delta_{n,t})\). Furthermore, \(-\sigma_t^{-1}\bI_{d}\preceq \widetilde\bPsi_t\preceq 0\).
\end{lemma}

\begin{proof}
    For every \(F_{\ba,\bb}\in \widetilde{\cV}\),
    \begin{align*}
        \widetilde{\cK}_tF_{\ba,\bb}(\bx_{0,i},\bz) 
         = 
        \sqrt{n}\left(\widetilde{\delta}_{i,t}a_i+\frac{\rho_t^2}{dn}\<\bx_{0,i},\bX\ba\>+\frac{\rho_t\sigma_t}{d\sqrt n}\<\bx_{0,i},\bb\>\right)
        +\left\<\frac{\rho_t\sigma_t}{d\sqrt n}\bX\ba+\frac{\sigma_t^2}{d}\bb,\bz\right\>.
    \end{align*}
    In other words, \(\widetilde{\cK}_t\) maps \(\widetilde{\cV}\) to itself: \(\widetilde{\cK}_tF_{\ba,\bb} = F_{\ba',\bb'}\) with \(((\ba')^\sT,(\bb')^\sT)^\sT = \widetilde{\bK}_t((\ba)^\sT,(\bb)^\sT)^\sT\). 

    Let \(\bv\in \R^d\) and \(\sfs\geq0\). Note that \(\sfz_\bv = F_{0,-\bv}\in \widetilde{\cV}\), and therefore
    \[
        e^{-\sfs \widetilde{\cK}_t}\sfz_\bv = F_{\ba_{t,\sfs}(\bv),\bb_{t,\sfs}(\bv)}, \qquad 
        \begin{bmatrix}
            \ba_{t,\sfs}(\bv)\\
            \bb_{t,\sfs}(\bv)
        \end{bmatrix}
        =
        e^{-\sfs\widetilde{\bK}_t}
        \begin{bmatrix}0\\-\bv\end{bmatrix}.
    \]
    Since \(F_{\ba,\bb}\) depends linearly on \((\ba,\bb)\), we can push the integral inside the finite-dimensional representation:
    \[
        \int_{0}^{\sft_t}e^{-\sfs \widetilde{\cK}_t}\sfz_\bv\,\d\sfs
        =
        F_{\overline{\ba}_t(\bv),\overline{\bb}_t(\bv)}, \qquad
        \begin{bmatrix}
            \overline{\ba}_t(\bv)\\
            \overline{\bb}_t(\bv)
        \end{bmatrix}
        =
        \int_0^{\sft_t}e^{-\sfs\widetilde{\bK}_t}\,\d\sfs
        \begin{bmatrix}0\\-\bv\end{bmatrix}.
    \]
    For every \(F_{\ba,\bb}\in \widetilde{\cV}\),
    \[
        \cS^\ast_{\mathrm{eff},t}F_{\ba,\bb}(\bx)
        = \frac1d\<\bx,\frac{\rho_t}{\sqrt n}\bX\ba+\sigma_t\bb\>.
    \]
    Applying this to the integrated finite-dimensional representation gives
    \[
        \widetilde{S}_{t,\bv}(\bx)
        =
        \frac{1}{\sigma_t d}\left\<\bx,
        \begin{bmatrix}
            \frac{\rho_t}{\sqrt n}\bX&\sigma_t\bI_d
        \end{bmatrix}
        \int_0^{\sft_t}e^{-\sfs\widetilde{\bK}_t}\,\d\sfs
        \begin{bmatrix}0\\-\bv\end{bmatrix}
        \right\>\,.
    \]
    Summing over an orthonormal basis \(\{\bv_{\ell}\}_{\ell=1}^d\) of \(\R^d\) gives the matrix representation \(\widetilde\bS_t(\bx)=\sigma_t^{-1}\widetilde\bPsi_t\bx\) with
    \[
         \widetilde{\bPsi}_t \deq  \frac{1}{d}
        \begin{bmatrix}\dfrac{\rho_t}{\sqrt n}\bX&\sigma_t\bI_d\end{bmatrix}
        \int_0^{\sft_t}e^{-\sfs\widetilde{\bK}_t}\,\de\sfs
        \begin{bmatrix}0\\-\bI_d\end{bmatrix}.
    \]
    Observe that
    \[
        \begin{bmatrix}
            \frac{\rho_t}{\sqrt{n}}\bX&\sigma_t\bI_d
        \end{bmatrix}
        =
        \frac{d}{\sigma_t}
        \begin{bmatrix}0&\bI_d\end{bmatrix}
        \widetilde{\bK}_t,
    \]
    and therefore we may integrate the matrix exponential to obtain the alternative representation.

    It is immediate from the representation that \(\widetilde\bPsi_t\) is symmetric. In order to establish the eigenvalue bounds, we decompose \(\widetilde{\bK}_t\) as
    \[
        \widetilde{\bK}_t = \frac{1}{d}
        \begin{bmatrix}
            \dfrac{\rho_t}{\sqrt n}\bX^\sT\\
            \sigma_t\bI_d
        \end{bmatrix}
        \begin{bmatrix}
            \dfrac{\rho_t}{\sqrt n}\bX&
            \sigma_t\bI_d
        \end{bmatrix}
        +
        \begin{bmatrix}
            \widetilde\bD_t&0\\
            0&0
        \end{bmatrix}.
    \]
    Since \(\widetilde{\delta}_{i,t}\geq0\) for all \(i\in[n]\) by Remark~\ref{rmk:delta_n_nonnegativity}, it follows that \(\widetilde{\bK}_t\succeq0\). Consequently, for every \(\sfs\geq0\), \(0\preceq e^{-\sfs\widetilde{\bK}_t}\preceq\bI_{n+d}\), and the result follows from the representation of \(\widetilde{\bPsi}_t\).
\end{proof}

As a consequence of the representation in Lemma~\ref{lem:reverse_linearized_check_score_representation}, it follows that, conditionally on the training data \(\bX\), the linearized reverse process \(\widetilde{\bx}_t\) is a centered Gaussian process.

Next, we establish uniform bounds on the linearized reverse process \(\widetilde{\bx}_t\) and its inner products with the training data \(\{\bx_{0,i}\}_{i=1}^{n}\).

\begin{lemma}[Boundedness and delocalization of the linearized process]\label{lem:linearized_SDE_bounds}
    Suppose that
    Assumptions~\ref{ass:high_dim},~\ref{ass:input_dist}, and
    \ref{ass:kernel_reverse} hold. Then, conditionally on the training
    data, for every fixed $k\in\N$,
    \[
        \E\left[\sup_{t\in[t_0,T]}\|\widetilde\bx_t\|_2^k\mid\bX\right]^{1/k}\lesssim\sqrt d,
        \qquad
        \E\left[\sup_{t\in[t_0,T]}\max_{i\le n}|\<\widetilde\bx_t,\bx_{0,i}\>|^k\mid\bX\right]^{1/k}\prec\sqrt d .
    \]

    Furthermore, there exists a constant \(C>0\) such that for every \(s>0\),
    \[
        \E\left[\exp\left(\frac{s}{d}\sup_{t\in[t_0,T]}\|\widetilde\bx_t\|_2^2\right)\mid\bX\right]\le 3e^{2sC}
    \]
    whenever \(d\geq 8sC\).
\end{lemma}

\begin{proof}
    Reverse time by setting $s=T-t\in[0,T-t_0]$ and $\by_s\deq \widetilde\bx_{T-s}$. Then
    \[
        \d\by_s=\bD_{T-s}\by_s\,\d s+\sqrt2\,\d\bW_s,
        \qquad
        \bD_t\deq \bI_d+\frac{2}{\sigma_t}\widetilde\bPsi_t,
        \qquad \by_0\sim\cN(0,\bI_d).
    \]
    Let \(\bPhi(s,u)\) denote the state-transition matrix for this forward-time linear drift:
    \[
        \partial_s\bPhi(s,u)=\bD_{T-s}\bPhi(s,u),
        \qquad \bPhi(u,u)=\bI_d.
    \]
    By the spectral bound on \(\widetilde\bPsi_t\) from Lemma~\ref{lem:reverse_linearized_check_score_representation}, \(\sup_{t\in[t_0,T]}\|\bD_t\|_{\op}\leq 1+2\sigma_{t_0}^{-2}=: C_D\). By Gronwall's inequality,
    \begin{equation}\label{eq:reverse_transition_bound}
        \|\bPhi(s,u)\|_{\op}\le e^{C_D(s-u)},\qquad 0\le u\le s\le T-t_0 .
    \end{equation}
    The process admits the variation-of-constants representation
    \begin{equation}\label{eq:reverse_transition_representation}
        \by_s=\bPhi(s,0)\by_0+\sqrt2\int_0^s\bPhi(s,u)\,\d\bW_u .
    \end{equation}
    Applying the It\^o isometry to~\eqref{eq:reverse_transition_representation},
    \[
        \E[\by_s\by_s^\sT\mid\bX]
        =\bPhi(s,0)\bPhi(s,0)^\sT+2\int_0^s\bPhi(s,u)\bPhi(s,u)^\sT\,\d u
    \]
    and hence
    \begin{equation}\label{eq:linearized_reverse_cov_bound}
        \sup_{s\in[0,T-t_0]}\|\E[\by_s\by_s^\sT\mid\bX]\|_{\op}
        \le e^{2C_DT}+2T e^{2C_DT}=:C_{\mathrm{cov}} .
    \end{equation}
    In particular, the covariance of \(\widetilde{\bx}_t\), conditionally on the input data \(\{\bx_{0,i}\}_{i=1}^{n}\), is bounded uniformly in \(t\) and \(d\).

    We next prove delocalization. Consider the centered Gaussian process
    \[
        G(i,s)\deq \<\by_s,\bx_{0,i}\>,\qquad (i,s)\in[n]\times[0,T-t_0],
    \]
    conditionally on \(\bX\), and let \(\dist_G\) be its canonical metric. We first control temporal increments. For \(0\le u\le s\le T-t_0\),
    \[
        \by_s=\bPhi(s,u)\by_u+\sqrt2\int_u^s\bPhi(s,r)\,\d\bW_r,
    \]
    and the martingale increment is independent of \(\by_u\). Thus
    \begin{align*}
        \dist_G((i,s),(i,u))^2
        &=\E\left[\left|\<(\bPhi(s,u)-\bI_d)\by_u,\bx_{0,i}\>
        +\sqrt2\int_u^s\<\bPhi(s,r)^\sT\bx_{0,i},\d\bW_r\>\right|^2\,\middle|\,\bX\right]\\
        &\le \|\bx_{0,i}\|_2^2
        \left(\|\bPhi(s,u)-\bI_d\|_{\op}^2 C_{\mathrm{cov}}
        +2\int_u^s\|\bPhi(s,r)\|_{\op}^2\,\d r\right).
    \end{align*}
    By the fundamental theorem of calculus and~\eqref{eq:reverse_transition_bound},
    \[
        \|\bPhi(s,u)-\bI_d\|_{\op}
        \le \int_u^s \|\bD_{T-r}\|_{\op}\|\bPhi(r,u)\|_{\op}\,\d r
        \le C_D e^{C_DT}(s-u).
    \]
     By~\eqref{eq:cov_op_bound} and~\eqref{eq:norm_concentration_overlaps}, there exists a constant \(C>0\) such that the event \(\max_{i\le n}\|\bx_{0,i}\|_2^2\le Cd\)
    holds with very high probability. We condition on this event.
    Since \((s-u)^2\le T(s-u)\), the preceding bounds imply
    \begin{equation}\label{eq:reverse_temporal_metric_bound}
        \dist_G((i,s),(i,u))\le C_{\mathrm{time}}\sqrt{d}|s-u|^{1/2}
    \end{equation}
    for some constant \(0<C_{\mathrm{time}}<\infty\).
    For spatial increments at fixed time,
    \begin{equation}\label{eq:reverse_spatial_metric_bound}
        \dist_G((i,u),(j,u))^2
        =\E[\<\by_u,\bx_{0,i}-\bx_{0,j}\>^2\mid\bX]
        \le C_{\mathrm{cov}}\|\bx_{0,i}-\bx_{0,j}\|_2^2 .
    \end{equation}
    Combining~\eqref{eq:reverse_temporal_metric_bound} and~\eqref{eq:reverse_spatial_metric_bound}, \(\dist_G\) is dominated by the metric
    \[
        \rho((i,s),(j,u))
        \deq C\left(\sqrt{d}|s-u|^{1/2}+\|\bx_{0,i}-\bx_{0,j}\|_2\right),
    \]
    where \(C\) is a constant depending on \(C_{\mathrm{cov}}\) and \(C_{\mathrm{time}}\).
    Its diameter is \(O_d(\sqrt{d})\), and its covering numbers obey
    \[
        N([n]\times[0,T-t_0],\rho,r)
        \le n\left(1+\frac{Cd}{r^2}\right).
    \]
    Dudley's entropy integral therefore gives
    \[
        \E\left[\sup_{i,s}|G(i,s)|\mid\bX\right]
        \lesssim \sqrt{d}\sqrt{\log d}.
    \]
    Moreover, \(\sup_{i,s}\E[G(i,s)^2\mid\bX]\le C d\). Borell--TIS upgrades the expectation bound to every fixed moment:
    \[
        \E\left[\sup_{s\in[0,T-t_0]}\max_{i\le n}|G(i,s)|^k\mid\bX\right]^{1/k}
        \le C_k \sqrt{d}\sqrt{\log d}.
    \]
    This implies the stated stochastic domination for the overlaps \(\<\widetilde\bx_t,\bx_{0,i}\>\).

    The norm bound follows from the same chaining argument applied to
    \(G'(\bv,s)\deq \<\bv,\by_s\>\), indexed by
    \(\mathbb S^{d-1}\times[0,T-t_0]\). The canonical metric is bounded by
    \[
        \dist_{G'}((\bv,s),(\bu,u))
        \le C\left(|s-u|^{1/2}+\|\bv-\bu\|_2\right),
    \]
    and the standard volumetric covering estimate for \(\mathbb S^{d-1}\) gives
    \[
        \E\left[\sup_{\bv\in\mathbb S^{d-1},\,s\in[0,T-t_0]}G'(\bv,s)\mid\bX\right]
        \lesssim \sqrt d .
    \]
    Together with the variance bound \(\sup_{\bv,s}\E[G'(\bv,s)^2\mid\bX]\le C_{\mathrm{cov}}\), Borell--TIS gives the claimed fixed-moment bound for \(\sup_s\|\by_s\|_2\), equivalently for \(\sup_t\|\widetilde\bx_t\|_2\).

    Borell--TIS also gives the probability bound
    \begin{equation}\label{eq:linearized_reverse_sup_norm_tail}
        \P\left(\sup_{s\in[0,T-t_0]}\|\by_s\|_2 \geq \E\left[\sup_{s\in[0,T-t_0]}\|\by_s\|_2\mid \bX\right] + u \mid \bX\right) \leq 2 \exp\left(-\frac{u^2}{2C_{\mathrm{cov}}}\right)
    \end{equation}
    for all \(u>0\). Applying the same proof as in Lemma~\ref{lem:exp_noisy_overlap_moments}, replacing the covariance bound~\eqref{eq:cov_op_bound} by~\eqref{eq:linearized_reverse_cov_bound}, the Gaussian tails by~\eqref{eq:linearized_reverse_sup_norm_tail} and the expectation bound by the chaining bound, we obtain the exponential moment bound.
\end{proof}

\subsection{Uniform linearization of the empirical kernel operator}

The fixed-time linearization lemmas control the training and test risks. For the
reverse process we need the corresponding statement along the random, but
delocalized, Gaussian path generated by~\eqref{eq:linearized_reverse_SDE}.

It will be convenient to introduce the following notation. For any \(t>0\), \(\bx,\bx',\bz\in \R^d\), define
\begin{equation}\label{eq:phi_definitions}
    \phi_{t,\bx}(\bx',\bz)\deq K(\bx,\rho_t\bx'+\sigma_t\bz),
    \qquad
    \phi_{\mathrm{eff},t,\bx}(\bx',\bz)\deq \frac{\< \bx,\rho_t\bx'+\sigma_t\bz\>}{d}.
\end{equation}
The next lemma shows that, along the delocalized path \(\widetilde\bx_t\), the feature map \(\phi_{t,\widetilde\bx_t}\) is well-approximated by the linearized feature map \(\phi_{\mathrm{eff},t,\widetilde\bx_t}\).

\begin{lemma}\label{lem:reverse_feature_linearization}
    Suppose that
    Assumptions~\ref{ass:high_dim},~\ref{ass:input_dist}, and
    \ref{ass:kernel_reverse} hold.
    Then, for every fixed \(k\in\N\),
    \begin{equation}\label{eq:reverse_feature_linearization_bound}
        \E\left[
        \sup_{t\in[t_0,T]}
        \|\phi_{t,\widetilde\bx_t}-\phi_{\mathrm{eff},t,\widetilde\bx_t}\|_{L^2(\widehat\pi_{d,n})}^{2k}
        \mid \bX\right]^{1/k}
        \prec d^{-5}.
    \end{equation}
\end{lemma}

\begin{proof}
    Let \(U_{i,t}(\bz)\deq \<\widetilde\bx_t,\rho_t\bx_{0,i}+\sigma_t\bz\>/d\). Taylor's theorem and Assumption~\ref{ass:kernel_reverse} give
    \[
        |h(u)-u|^2\leq C e^{C|u|}|u|^{10}
    \]
    for some constant \(C>0\), and hence Cauchy-Schwarz gives
    \begin{equation}\label{eq:reverse_feature_linearization_bound_intermediate}
        \sup_{t\in[t_0,T]}
        \|\phi_{t,\widetilde\bx_t}-\phi_{\mathrm{eff},t,\widetilde\bx_t}\|_{L^2(\widehat\pi_{d,n})}^2 \leq C\sup_{t\in[t_0,T]\, ,i\in [n]}\E_{\bz}[e^{2C|U_{i,t}(\bz)|}]^{1/2} \sup_{t\in[t_0,T]\, ,i\in [n]}\E_{\bz}[|U_{i,t}(\bz)|^{20}]^{1/2}.
    \end{equation}

    For the exponential term, conditionally on \(\widetilde\bx_t\), \(\<\widetilde\bx_t,\bz\>\) is Gaussian with variance \(\|\widetilde\bx_t\|_2^2\) so
    \[
        \sup_{t\in[t_0,T]\, ,i\in [n]}\E_{\bz}[e^{2C|U_{i,t}(\bz)|}] \lesssim \sup_{t\in[t_0,T]\, ,i\in [n]}\exp\left(\frac{C\rho_t|\<\bx_{0,i},\widetilde{\bx}_t\>|}{d}+\frac{2C^2\sigma_t^2\|\widetilde{\bx}_t\|_2^2}{d^2}\right).
    \]
    By Young's inequality and, conditionally on the very high probability event \(\{\max_{i\le n}\|\bx_{0,i}\|_2^2\lesssim d\}\),
    \[
        \exp\left(\frac{C\rho_t|\<\bx_{0,i},\widetilde{\bx}_t\>|}{d}\right)\lesssim \exp\left(\frac{C'\|\widetilde{\bx}_t\|^2_2}{2d}\right)
    \]
    for some constant \(C'>0\). By the exponential bound in Lemma~\ref{lem:linearized_SDE_bounds}, \(\E[\exp(\sup_{t}u\|\widetilde{\bx}_t\|_2^2/d)\mid\bX]\) is bounded for every fixed \(u>0\) and \(d\) large enough depending on \(u\). Hence, for every fixed \(m\in\N\), we get
    \[
        \E\left[\sup_{t\in[t_0,T]\, ,i\in [n]}\E_{\bz}[e^{2C|U_{i,t}(\bz)|}]^{m/2}\mid \bX\right] \lesssim 1.
    \]

    For the polynomial term, conditionally on \(\widetilde\bx_t\), \(\<\widetilde\bx_t,\bz\>\) is Gaussian with variance \(\|\widetilde\bx_t\|_2^2\) so
    \[
        \sup_{t\in[t_0,T]\, ,i\in [n]}\E_{\bz}[|U_{i,t}(\bz)|^{20}]^{1/20} \lesssim \sup_{t\in[t_0,T],\, i\in [n]}\left(\frac{\rho_t|\<\bx_{0,i},\widetilde{\bx}_t\>|}{d}+\frac{\sigma_t\|\widetilde{\bx}_t\|_2}{d}\right). 
    \]
    For every fixed \(m\in\N\), by Lemma~\ref{lem:linearized_SDE_bounds}, it follows that
    \[
        \E\left[\sup_{t\in[t_0,T]\, ,i\in [n]}\E_{\bz}[|U_{i,t}(\bz)|^{20}]^{m/2}\mid \bX\right] \prec d^{-5m}.
    \]
    Taking \(m=2k\) in the preceding two estimates, raising \eqref{eq:reverse_feature_linearization_bound_intermediate} to the \(k\)-th power, and applying conditional Cauchy--Schwarz yields the desired bound.
\end{proof}

In addition to Lemma~\ref{lem:reverse_feature_linearization}, we will also need a uniform-in-time linearization of the empirical kernel operator \(\widehat\cK_t=\cS_t\cS_t^*\), with \(\cS_t\) defined in~\eqref{eq:forward_operators} for diffusion time \(t>0\). 
For every \(t>0\), decompose
\begin{equation}\label{eq:reverse_kernel_operator_decomposition}
    \widehat\cK_t-\widetilde\cK_{t}
    =
    \Delta_{\mathrm{od},t}+\Delta_{\mathrm{d},t},
\end{equation}
where the operators \(\Delta_{\mathrm{od},t},\Delta_{\mathrm{d},t}: L^2(\widehat\pi_{d,n})\to L^2(\widehat\pi_{d,n})\) are defined by
\begin{align*}
    \Delta_{\mathrm{od},t}f(\bx_{0,i},\bz) = \frac{1}{n}\sum_{j\neq i}\E_{\bz'}\left[D_{ij,t}(\bz,\bz') f(\bx_{0,j}, \bz')\right], \quad \Delta_{\mathrm{d},t}f(\bx_{0,i},\bz) = \frac{1}{n}\E_{\bz'}\left[D_{ii,t}(\bz,\bz') f(\bx_{0,i}, \bz')\right]
\end{align*}
with
\begin{equation}\label{eq:reverse_kernel_difference}
    D_{ij,t}(\bz,\bz') \deq  h(U_{ij,t}(\bz,\bz')) - U_{ij,t}(\bz,\bz') - \ind_{i=j}n\widetilde{\delta}_{i,t},\qquad  U_{ij,t}(\bz,\bz')
    \deq \frac{\< \rho_t\bx_{0,i}+\sigma_t\bz,\,
    \rho_t\bx_{0,j}+\sigma_t\bz'\>}{d}.
\end{equation}
We have the following uniform-in-time linearization lemma for the diagonal and off-diagonal kernel operator differences.

\begin{lemma}\label{lem:reverse_operator_linearization_high_order}
    Suppose that Assumptions~\ref{ass:high_dim},~\ref{ass:input_dist}, and
    \ref{ass:kernel_reverse} hold.
    Then
    \[
        \sup_{t>0}
        \|\Delta_{\mathrm{od},t}\|_{\op}
        \prec d^{-5/2},
        \qquad
        \sup_{t>0}
        \|\Delta_{\mathrm{d},t}\|_{\op}
        \prec d^{-3/2}.
    \]
\end{lemma}

\begin{proof}
    We repeat the proof of Proposition~\ref{prop:linearization_kernel_op},
    keeping the estimates uniform in \(t\). This uniformity follows
    from \(0\le\rho_t,\sigma_t\le1\). Indeed, for \(i\ne j\),
    \[
        \sup_{t>0}|U_{ij,t}(\bz,\bz')|
        \le
        \frac{
        |\<\bx_{0,i},\bx_{0,j}\>|
        +|\<\bx_{0,i},\bz'\>|
        +|\<\bz,\bx_{0,j}\>|
        +|\<\bz,\bz'\>|
        }{d}.
    \]
    Thus the standard overlap and exponential-moment bounds used in
    the proof of Proposition~\ref{prop:linearization_kernel_op} hold with a
    common, time-independent majorant.

    For \(i\ne j\), Assumption~\ref{ass:kernel_reverse} gives
    \[
        h(U_{ij,t})-U_{ij,t}
        =
        \frac{h^{(5)}(\xi_{ij,t})}{5!}U_{ij,t}^5.
    \]
    Consequently,
    \[
        \sup_{t>0}\max_{i\ne j}
        \E_{\bz,\bz'}
        |h(U_{ij,t})-U_{ij,t}|^2
        \prec d^{-5}.
    \]
    Applying the off-diagonal Cauchy--Schwarz estimate from the proof
    of Proposition~\ref{prop:linearization_kernel_op} yields
    \[
        \sup_{t>0}
        \|\Delta_{\mathrm{od},t}\|_{\op}
        \prec d^{-5/2}.
    \]

    For the diagonal blocks, we apply the mean-value theorem to obtain
    \begin{align*}
         D_{ii,t}(\bz,\bz') 
         =
        (h'(\xi_{i,t}(\bz,\bz')) - 1) \left(\frac{\rho_t\sigma_t}{d}\<\bx_{0,i},\bz\>+\frac{\rho_t\sigma_t}{d}\<\bx_{0,i},\bz'\>+\frac{\sigma_t^2}{d}\<\bz,\bz'\>\right),
    \end{align*}
    where \(\xi_{i,t}\) lies between \(U_{ii,t}\) and \(\rho_t^2\|\bx_{0,i}\|_2^2/d\). For every fixed \(m\geq 2\),
    \[
        \sup_{t >0}\max_{i\leq n} \E_{\bz,\bz'}\left[\left|\frac{\rho_t\sigma_t}{d}\<\bx_{0,i},\bz\>+\frac{\rho_t\sigma_t}{d}\<\bx_{0,i},\bz'\>+\frac{\sigma_t^2}{d}\<\bz,\bz'\>\right|^m\right]^{1/m} \prec d^{-1/2}
    \]
    by the same argument as in the proof of Proposition~\ref{prop:linearization_kernel_op}. Moreover, since \(\xi_{i,t}\) lies between \(U_{ii,t}\) and \(\rho_t^2\|\bx_{0,i}\|_2^2/d\),
    \[
        |\xi_{i,t}|\leq\rho_t^2\frac{\|\bx_{0,i}\|_2^2}{d}+\left|U_{ii,t}-\rho_t^2\frac{\|\bx_{0,i}\|_2^2}{d}\right|.
    \]
    Therefore, the exponential-moment bound from Lemma~\ref{lem:exp_noisy_overlap_moments} and \(\max_i\|\bx_{0,i}\|_2^2/d\lesssim 1\) with very high probability imply that, for every fixed \(s>0\),
    \[
        \sup_{t>0}\max_{i\le n}
        \E_{\bz,\bz'}e^{s|\xi_{i,t}(\bz,\bz')|}
        \lesssim 1
    \]
    with very high probability.
    Using \(|h'(u)-1|\le Ce^{C|u|}\) by Assumption~\ref{ass:kernel_reverse}, H\"older's inequality, and the preceding moment bound, we get
    \[
        \sup_{t>0}\max_{i\leq n} \E_{\bz,\bz'}|D_{ii,t}(\bz,\bz')|^2 \prec d^{-1}.
    \]
    Finally, \(\Delta_{\mathrm d,t}\) is block diagonal in \(i\), so
    \[
        \sup_{t>0}\|\Delta_{\mathrm{d},t}\|_{\op}\leq\frac{1}{n}\sup_{t>0}\max_{i\leq n}\left(\E_{\bz,\bz'}|D_{ii,t}(\bz,\bz')|^2\right)^{1/2}\prec d^{-3/2},
    \]
    where we have used \(n\asymp d\) from Assumption~\ref{ass:high_dim}.
\end{proof}

Lemma~\ref{lem:reverse_operator_linearization_high_order} provides uniform-in-time operator-norm bounds for the off-diagonal and diagonal remainders. The off-diagonal bound is sufficiently strong to control its contribution in the subsequent analysis. The diagonal operator-norm bound is too weak for the gradient-flow times considered here. Nevertheless, the diagonal remainder has additional structure, which we exploit to obtain a stronger estimate on the invariant subspace \(\widetilde{\cV}\).

\begin{lemma}\label{lem:reverse_diagonal_subspace_kernel}
    Suppose that Assumptions~\ref{ass:high_dim},~\ref{ass:input_dist}, and
    \ref{ass:kernel_reverse} hold. Let \(\proj_{\widetilde\cV}\) denote the orthogonal projection onto the subspace \(\widetilde\cV\subseteq L^2(\widehat{\pi}_{d,n})\) defined in~\eqref{eq:tilde-V-def}. Then,
    \[
        \sup_{t\in[t_0,T]}
        \|\proj_{\widetilde\cV}\Delta_{\mathrm d,t}\proj_{\widetilde\cV}\|_{\op} \prec d^{-2}.
    \]
\end{lemma}

\begin{proof}
    Consider \(F_{\ba,\bb},F_{\ba',\bb'}\in\widetilde\cV\) with \(\|F_{\ba,\bb}\|_{L^2(\widehat\pi_{d,n})}\vee \|F_{\ba',\bb'}\|_{L^2(\widehat\pi_{d,n})}\le 1\), and write
    \[
        \<F_{\ba,\bb},\Delta_{\mathrm{d},t}F_{\ba',\bb'}\>_{L^2(\widehat\pi_{d,n})} = \frac{1}{n^2}\sum_{i=1}^n \E_{\bz,\bz'}\left[D_{ii,t}(\bz,\bz')\left(\sqrt{n}a_i+\<\bb,\bz\>\right)\left(\sqrt{n}a_i'+\<\bb',\bz'\>\right)\right].
    \]
    Taylor expansion of \(h(u)-u\) at \(\rho_t^2\|\bx_{0,i}\|_2^2/d\) gives
    \[
        D_{ii,t} = \left(h'\left(\frac{\rho_t^2\|\bx_{0,i}\|_2^2}{d}\right)-1\right)
        \left(U_{ii,t}-\frac{\rho_t^2\|\bx_{0,i}\|_2^2}{d}\right)
         + \frac{1}{2}h''(\xi_{i,t})\left(U_{ii,t}-\frac{\rho_t^2\|\bx_{0,i}\|_2^2}{d}\right)^2,
    \]
    where \(\xi_{i,t}\) lies between \(U_{ii,t}\) and \(\rho_t^2\|\bx_{0,i}\|_2^2/d\). We decompose the bilinear form into the linear and quadratic Taylor terms:
    \begin{align*}
        \<F_{\ba,\bb},\Delta_{\mathrm{d},t}F_{\ba',\bb'}\>_{L^2(\widehat\pi_{d,n})} = E_{1} + E_{2},
    \end{align*}
    where, with \(\eta_{i,t}=\rho_t^2\|\bx_{0,i}\|_2^2/d\),
    \begin{align*}
        E_{1} &\deq  \frac{1}{n^2}\sum_{i=1}^n \E_{\bz,\bz'}\left[\left(h'(\eta_{i,t})-1\right)
        \left(U_{ii,t}-\eta_{i,t}\right)
        \left(\sqrt{n}a_i+\<\bb,\bz\>\right)\left(\sqrt{n}a_i'+\<\bb',\bz'\>\right)\right],\\
        E_{2} &\deq  \frac{1}{2n^2}\sum_{i=1}^n \E_{\bz,\bz'}\left[h''(\xi_{i,t}(\bz,\bz'))\left(U_{ii,t}-\eta_{i,t}\right)^2
        \left(\sqrt{n}a_i+\<\bb,\bz\>\right)\left(\sqrt{n}a_i'+\<\bb',\bz'\>\right)\right].
    \end{align*}

    The uniform overlap and exponential-moment bounds used in the proof of
    Lemma~\ref{lem:reverse_operator_linearization_high_order} imply
    \[
        \sup_{\substack{t\in[t_0,T]\\i\le n}}
        \E_{\bz,\bz'}\left[
        |h''(\xi_{i,t})|^2
        \left|
        U_{ii,t}-\rho_t^2\frac{\|\bx_{0,i}\|_2^2}{d}
        \right|^4
        \right]^{1/2}
        \prec d^{-1},
    \]
    and therefore
    \[
        |E_{2}|\leq \frac1{2n}
        \sup_{\substack{t\in[t_0,T]\\i\le n}}
        \E_{\bz,\bz'}\left[
        |h''(\xi_{i,t})|^2
        \left|
        U_{ii,t}-\rho_t^2\frac{\|\bx_{0,i}\|_2^2}{d}
        \right|^4
        \right]^{1/2}
        \prec d^{-2},
    \]
    where we have used \(n\asymp d\) from Assumption~\ref{ass:high_dim}.

    It remains to control the linear Taylor term. Using
    \[
        U_{ii,t}-\rho_t^2\frac{\|\bx_{0,i}\|_2^2}{d}
        =
        \frac{\rho_t\sigma_t}{d}\<\bx_{0,i},\bz\>
        +\frac{\rho_t\sigma_t}{d}\<\bx_{0,i},\bz'\>
        +\frac{\sigma_t^2}{d}\<\bz,\bz'\>
    \]
    and the independence and standard Gaussian moments of \((\bz,\bz')\), its contribution to the bilinear form is
    \begin{equation}\label{eq:reverse_diagonal_subspace_kernel_linear_term}
        E_1 = \frac{1}{n^2}\sum_{i=1}^n \left(h'\left(\eta_{i,t}\right)-1\right) \left(\frac{\rho_t\sigma_t\sqrt n}{d} a_i'\<\bb,\bx_{0,i}\> + \frac{\rho_t\sigma_t\sqrt n}{d} a_i\<\bx_{0,i},\bb'\> + \frac{\sigma_t^2}{d}\<\bb,\bb'\> \right).
    \end{equation}
    On an event of very high probability,
    \[
        \sup_{t\in[t_0,T]}\max_{i\le n}
        \left|
        h'\left(\rho_t^2\frac{\|\bx_{0,i}\|_2^2}{d}\right)-1
        \right|\lesssim1,
        \qquad
        \|\bX\|_{\op}\lesssim\sqrt n.
    \]
    Consequently, 
    \begin{align*}
        |E_1| & \lesssim  \frac{1}{n^2}\left(\frac{\rho_t\sigma_t \sqrt{n}}{d}\bb^\sT\bX\bH\ba' + \frac{\rho_t\sigma_t \sqrt{n}}{d}\bb'^\sT\bX\bH\ba + \frac{\sigma_t^2n}{d}\<\bb,\bb'\>\right) 
        \\ & \lesssim \frac{1}{d^2}(\|\bb\|_2\|\ba'\|_2+\|\bb'\|_2\|\ba\|_2+\|\bb\|_2\|\bb'\|_2)
    \end{align*}
    where \(\bH:=\diag(h'(\rho_t^2\|\bx_{0,i}\|_2^2/d)-1)_{i\le n}\).
    By Cauchy--Schwarz inequality applied to \(\R^{2(n+d)}\),
    \[
        \|\bb\|_2\|\ba'\|_2 + \|\ba\|_2\|\bb'\|_2 \leq \sqrt{\|\ba\|_2^2+\|\bb\|_2^2}\sqrt{\|\ba'\|_2^2+\|\bb'\|_2^2}.
    \]
    Bounding \(\|\bb\|_2\|\bb'\|_2 \leq \sqrt{\|\ba\|_2^2+\|\bb\|_2^2}\sqrt{\|\ba'\|_2^2+\|\bb'\|_2^2}\), we get that
    \[
        |E_1| \lesssim d^{-2}\sqrt{\|\ba\|_2^2+\|\bb\|_2^2}\sqrt{\|\ba'\|_2^2+\|\bb'\|_2^2}.
    \]
    Finally, \((\ba,\bb)\mapsto F_{\ba,\bb}\) is an isometry from
    \(\R^{n+d}\) onto \(\widetilde\cV\). Taking the supremum over unit
    \(F_{\ba,\bb}\) and \(F_{\ba',\bb'}\), and combining the linear and
    quadratic contributions concludes the proof.
\end{proof}

\begin{remark}[Role of the sample correction]
    The sample-dependent correction \(\widetilde{\delta}_{i,t}\) is essential to obtain the \(d^{-2}\) bound in Lemma~\ref{lem:reverse_diagonal_subspace_kernel}. Using the common correction \(\delta_{n,t}\) instead would yield an extra term:
    \[
        U_{ii,t}-\rho_t^2\tau_{\bSigma}=\rho_t^2\left(\frac{\|\bx_{0,i}\|^2}{d}-\tau_\bSigma\right)+ \frac{\rho_t\sigma_t}{d}\<\bx_{0,i},\bz\>+\frac{\rho_t\sigma_t}{d}\<\bx_{0,i},\bz'\>+\frac{\sigma_t^2}{d}\<\bz,\bz'\>.
    \]
    Repeating the linear-term computation~\eqref{eq:reverse_diagonal_subspace_kernel_linear_term} after Taylor expanding around \(\rho_t^2\tau_{\bSigma}\) produces the additional contribution
    \[
        \frac{\rho_t^2}{n}(h'(\rho_t^2\tau_\bSigma)-1)\sum_{i=1}^n\left(\frac{\|\bx_{0,i}\|_2^2}{d}-\tau_\bSigma\right)a_i a_i'.
    \]
    Taking the supremum over unit \(F_{\ba,\bb}\) and \(F_{\ba',\bb'}\), this term is at most
    \[
        \frac{\rho_t^2}{n}(h'(\rho_t^2\tau_\bSigma)-1)\max_{i\le n}\left|\frac{\|\bx_{0,i}\|_2^2}{d}-\tau_\bSigma\right| \prec d^{-3/2},
    \]
    uniformly over \(t\in[t_0,T]\). Thus, the common correction would only yield a \(d^{-3/2}\) bound, whereas the sample-dependent correction cancels this term and yields the stronger \(d^{-2}\) estimate.
\end{remark}

\subsection{Proof of Theorem~\ref{thm:kl_bound_empirical_linearized_reverse}}

We next establish a uniform-in-time bound on the discrepancy between the empirical score and its linearized approximation, evaluated along the linearized reverse process. This estimate will control the KL divergence between the path laws of the empirical reverse SDE and its linearized counterpart.

\begin{lemma}[Reverse-time score linearization]\label{lem:reverse_score_linearization}
    Suppose that Assumptions~\ref{ass:high_dim},~\ref{ass:input_dist}, and
    \ref{ass:kernel_reverse} hold. Then,
    \[
        \E\left[\sup_{t\in[t_0,T]}\|\widehat{\bS}_t(\widetilde\bx_t)-\widetilde{\bS}_t(\widetilde\bx_t)\|_2^2\mid\bX\right]
        \prec \sup_{t\in[t_0,T]}\sft_t^2d^{-5}\left(1+\sft_t^2+\sft_t^4d^{-2}\right).
    \]
\end{lemma}

\begin{proof}
    Write
    \[
        \widehat\cA_t\deq \int_0^{\sft_t}e^{-s\widehat\cK_t}\,\d s,
        \qquad
        \widetilde\cA_t\deq \int_0^{\sft_t}e^{-s\widetilde\cK_t}\,\d s,
    \]
    and fix an orthonormal basis \(\{\bv_\ell\}_{\ell=1}^d\) of
    \(\R^d\). The gradient-flow representations give, for every unit
    vector \(\bv\),
    \begin{align*}
        \sigma_t\bigl(\widehat S_{t,\bv}(\bx)-\widetilde S_{t,\bv}(\bx)\bigr)
        &=\left\<\phi_{t,\bx}-\phi_{\mathrm{eff},t,\bx},
        \widehat\cA_t\widehat f^\star_{t,\bv}\right\>_{L^2(\widehat\pi_{d,n})}
        +\left\<\phi_{\mathrm{eff},t,\bx},
        (\widehat\cA_t-\widetilde\cA_t)\sfz_\bv\right\>_{L^2(\widehat\pi_{d,n})}\\
        &\quad+\left\<\phi_{\mathrm{eff},t,\bx},
        \widehat\cA_t(\widehat f^\star_{t,\bv}-\sfz_\bv)\right\>_{L^2(\widehat\pi_{d,n})},
    \end{align*}
    where \(\phi_{t,\bx}\) and \(\phi_{\mathrm{eff},t,\bx}\) are defined in~\eqref{eq:phi_definitions}. 
    
    Both \(\widehat\cK_t\) and \(\widetilde\cK_t\) are positive definite, and hence \(\|\widehat\cA_t\|_{\op}\vee\|\widetilde\cA_t\|_{\op}
    \le \sft_t\). By definition,
    \[
        \|\phi_{\mathrm{eff},t,\widetilde\bx_t}\|_{L^2(\widehat\pi_{d,n})}^2 = \frac{1}{d^2}\widetilde{\bx}_t^\sT(\rho_t^2\widehat\bSigma+\sigma_t^2\bI_d)\widetilde{\bx}_t.
    \]
    Lemma~\ref{lem:reverse_feature_linearization}, with the
    bounds for \(\widetilde\bx_t\) from Lemma~\ref{lem:linearized_SDE_bounds}, along with the operator-norm bounds from~\eqref{eq:cov_op_bound},~\eqref{eq:cov_concentration}
    gives
    \begin{equation}\label{eq:phi_eff_bounds}
        \E\left[\sup_{t\in [t_0,T]}
        \|\phi_{t,\widetilde\bx_t}-\phi_{\mathrm{eff},t,\widetilde\bx_t}\|_{L^2(\widehat\pi_{d,n})}^2
        \mid\bX\right]\prec d^{-5},
        \qquad
        \E\left[\sup_{t\in [t_0,T]}
        \|\phi_{\mathrm{eff},t,\widetilde\bx_t}\|_{L^2(\widehat\pi_{d,n})}^2
        \mid\bX\right] \lesssim d^{-1}
    \end{equation}
    with very high probability.
    Moreover, Proposition~\ref{prop:bayes_opt_denoiser_linearization} holds
    uniformly over \(t\in[t_0,T]\) and unit \(\bv\).

    We upper bound the square of the Euclidean norm of the score difference by the sum of three terms:
    \begin{align*}
        \|\widehat{\bS}_t(\widetilde\bx_t)-\widetilde{\bS}_t(\widetilde\bx_t)\|_2^2 \leq \frac{3}{\sigma_t^2}(E_{1,t}+E_{2,t}+E_{3,t}),
    \end{align*}
    where
    \begin{gather*}
         E_{1,t} \deq  \sum_{\ell=1}^d\left|\left\<
        \phi_{t,\widetilde\bx_t}-\phi_{\mathrm{eff},t,\widetilde\bx_t},
        \widehat\cA_t\widehat f^\star_{t,\bv_\ell}
        \right\>_{L^2(\widehat\pi_{d,n})}\right|^2, \qquad
        E_{2,t} \deq  \sum_{\ell=1}^d\left|\left\<
        \phi_{\mathrm{eff},t,\widetilde\bx_t},
        (\widehat\cA_t-\widetilde\cA_t)\sfz_{\bv_\ell}
        \right\>_{L^2(\widehat\pi_{d,n})}\right|^2, \\
        E_{3,t} \deq  \sum_{\ell=1}^d\left|\left\<
        \phi_{\mathrm{eff},t,\widetilde\bx_t},
        \widehat\cA_t(\widehat f^\star_{t,\bv_\ell}-\sfz_{\bv_\ell})
        \right\>_{L^2(\widehat\pi_{d,n})}\right|^2.
    \end{gather*}
    Let \(\cF_t:\R^d\to L^2(\widehat\pi_{d,n})\) denote the linear map \(\bv\mapsto\widehat f^\star_{t,\bv}\). Jensen's inequality gives \(\|\cF_t\|_{\op}\leq 1\), and
    \begin{align*}
        E_{1,t} = \sum_{\ell=1}^d\left|\left\<
        \cF_t^\ast\widehat{\cA}_t(\phi_{t,\widetilde\bx_t}-\phi_{\mathrm{eff},t,\widetilde\bx_t}),
        \bv_\ell \right\>\right|^2
        &  = \|\cF_t^\ast\widehat{\cA}_t(\phi_{t,\widetilde\bx_t}-\phi_{\mathrm{eff},t,\widetilde\bx_t})\|_2^2 
          \lesssim \sft^2_t\|\phi_{t,\widetilde\bx_t}-\phi_{\mathrm{eff},t,\widetilde\bx_t}\|_{L^2(\widehat\pi_{d,n})}^2,
    \end{align*}
    where the last inequality follows from the sub-multiplicativity of the norm and \(\|\widehat{\cA}_t\|_\op\leq \sft_t\). Therefore,
    \[
        \E\left[\sup_{t\in [t_0,T]}E_{1,t}\mid \bX\right] \prec \sup_{t\in [t_0,T]}\sft^2_t d^{-5}.
    \]
    A similar argument gives
    \[
        \E\left[\sup_{t\in [t_0,T]}E_{3,t}\mid \bX\right] \prec \sup_{t\in [t_0,T]}\sft^2_t d^{-1}e^{-2d^c}.
    \]

    It remains to control \(E_{2,t}\). Note that \(\{\sfz_{\bv_\ell}\}_{\ell=1}^d\) is an orthonormal family in \(L^2(\widehat\pi_{d,n})\), and that \(\sfz_{\bv_\ell},\phi_{\mathrm{eff},t,\widetilde\bx_t}\in\widetilde\cV\) for every \(\ell\) and \(t\). Hence, by Bessel's inequality,
    \begin{align*}
        E_{2,t} & = \sum_{\ell=1}^d\left|\left\<
        \phi_{\mathrm{eff},t,\widetilde\bx_t},
        \proj_{\widetilde{\cV}}(\widehat\cA_t-\widetilde\cA_t)\proj_{\widetilde{\cV}}\sfz_{\bv_\ell}
        \right\>_{L^2(\widehat\pi_{d,n})}\right|^2
         \leq \|\phi_{\mathrm{eff},t,\widetilde\bx_t}\|_{L^2(\widehat\pi_{d,n})}^2\|\proj_{\widetilde{\cV}}(\widehat\cA_t-\widetilde\cA_t)\proj_{\widetilde{\cV}}\|_{\op}^2.
    \end{align*}
    The variation-of-parameters
    identity gives
    \begin{equation}\label{eq:variation_of_parameters}
        e^{-s\widehat{\cK}_t}-e^{-s\widetilde{\cK}_t}
        =-\int_0^s e^{-(s-r)\widetilde{\cK}_t}
        (\Delta_{\mathrm{od},t}+\Delta_{\mathrm d,t})
        e^{-r\widehat{\cK}_t}\d r.
    \end{equation}
    The integrals are Bochner integrals in the Banach space of bounded linear operators on \(L^2(\widehat\pi_{d,n})\); see \cite[Corollary 1.7]{engel2000one}. Applying the same identity to the last \(e^{-r\widehat{\cK}_t}\) gives the exact expansion
    \begin{align*}
        \widehat\cA_t-\widetilde\cA_t
        &=-\int_0^{\sft_t}\int_0^s
        e^{-(s-r)\widetilde\cK_t}
        (\Delta_{\mathrm{od},t}+\Delta_{\mathrm d,t})
        e^{-r\widetilde\cK_t}\d r\d s
        \\ &\quad+\int_0^{\sft_t}\int_0^s\int_0^r
        e^{-(s-r)\widetilde\cK_t}
        (\Delta_{\mathrm{od},t}+\Delta_{\mathrm d,t})
        e^{-(r-u)\widetilde\cK_t}
        (\Delta_{\mathrm{od},t}+\Delta_{\mathrm d,t})
        e^{-u\widehat\cK_t}\d u\d r\d s.
    \end{align*}
    By the contraction property and the operator-norm bounds from Lemma~\ref{lem:reverse_operator_linearization_high_order},
    \begin{gather*}
        \left\|\int_0^{\sft_t}\int_0^s\int_0^r
        e^{-(s-r)\widetilde\cK_t}
        (\Delta_{\mathrm{od},t}+\Delta_{\mathrm d,t})
        e^{-(r-u)\widetilde\cK_t}
        (\Delta_{\mathrm{od},t}+\Delta_{\mathrm d,t})
        e^{-u\widehat\cK_t}\d u\d r\d s\right\|_\op  \prec \sft^3_t d^{-3},
        \\ \left\|\int_0^{\sft_t}\int_0^s
        e^{-(s-r)\widetilde\cK_t}
        \Delta_{\mathrm{od},t}
        e^{-r\widetilde\cK_t}\d r\d s\right\|_\op  \prec \sft_t^2 d^{-5/2}.
    \end{gather*}
    Since \(\widetilde\cK_t\) is self-adjoint and \(\widetilde\cV\)-invariant, \(\widetilde\cV\) is reducing and hence
    \[
    \proj_{\widetilde\cV}e^{-u\widetilde\cK_t} = e^{-u\widetilde\cK_t}\proj_{\widetilde\cV}.
    \]
     for every \(u\ge 0\). Lemma~\ref{lem:reverse_diagonal_subspace_kernel} thus gives
    \begin{align*}
        \left\|\proj_{\widetilde{\cV}}\int_0^{\sft_t}\int_0^s
        e^{-(s-r)\widetilde\cK_t}
        \Delta_{\mathrm d,t}
        e^{-r\widetilde\cK_t}\d r\d s\proj_{\widetilde{\cV}}\right\|_\op
        \prec \sft^2_t d^{-2}.
    \end{align*}
    Therefore, \(\|\proj_{\widetilde{\cV}}(\widehat\cA_t-\widetilde\cA_t)\proj_{\widetilde{\cV}}\|_{\op}\prec \sft^2_t d^{-2}+\sft^3_t d^{-3}\). Using~\eqref{eq:phi_eff_bounds},
    \[
        \E\left[\sup_{t\in [t_0,T]}E_{2,t}\mid \bX\right] \prec \sup_{t\in [t_0,T]}(\sft^4_t d^{-5}+\sft^6_t d^{-7}).
    \]
    Since \(\inf_{t\in[t_0,T]}\sigma_t>0\) and \(e^{-2d^c}\le d^{-4}\) for all sufficiently large \(d\), combining the preceding three estimates with the score decomposition proves the claim.
\end{proof}

\begin{proof}[Proof of Theorem~\ref{thm:kl_bound_empirical_linearized_reverse}]
    Fix the training sample; the following arguments are conditional on \(\bX\).

    Note that the drift of the linearized reverse SDE~\eqref{eq:linearized_reverse_SDE} has a uniformly bounded time-dependent linear coefficient and is therefore globally Lipschitz with linear growth. Hence, the linearized reverse SDE has a unique global strong solution.

    For the empirical reverse SDE~\eqref{eq:emp_reverse_SDE}, note that by~\eqref{eq:gf_solution},
    \[
        \sup_{t\in [t_0,T], \bv\in \S^{d-1}}\|\widetilde{\bS}_{t,\bv}\|_\cH\lesssim \sup_{t\in [t_0,T], \bv\in \S^{d-1}}\sft_t\|\cS_t\|_\op <\infty.
    \]
    Since \(h\in C^5(\R)\) by Assumption~\ref{ass:kernel_reverse}, it follows from the derivative reproducing property (e.g., \cite[Theorem 1]{zhou2008derivative}) applied to the restriction of \(K\) to compact balls that for every \(R>0\), there exists \(C_R>0\) such that
    \[
        \sup_{\|\bx\|_2\leq R}\|\nabla f(\bx)\|_2\leq C_R\|f\|_\cH
    \]
    for all \(f\in \cH\). Hence, there exists a constant \(C>0\) such that for all \(\bx,\by\in\R^d\) with \(\|\bx\|_2\vee\|\by\|_2\leq R\),
    \[
        \sup_{t\in [t_0,T]}\|\widehat{\bS}_t(\bx) - \widehat{\bS}_t(\by)\|_2 \leq C\|\bx-\by\|_2.
    \]
    Standard local SDE theory gives a unique maximal strong solution up to an explosion time. Setting the process equal to \(\infty\) after explosion yields a continuous \(\R_\infty^d\)-valued process on the full interval.
    
    Let \(\widetilde\P\) and \(\widehat\P\) denote the forward-time path laws of \(\{\widetilde\bx_{T-s}\}_{s\in[0,T-t_0]}\) and \(\{\widehat\bx_{T-s}\}_{s\in[0,T-t_0]}\), respectively, on \(C([0,T-t_0];\R_\infty^d)\), and let \(\by\) denote the canonical coordinate process. Since time reversal is a bimeasurable bijection,
    \[
        \KL(\P_{\widetilde\bx}\mmid\P_{\widehat\bx})
        =
        \KL(\widetilde\P\mmid\widehat\P).
    \]


    For \(R>0\), define the canonical stopping time
    \[
        \tau_R \deq \inf\left\{s\in[0,T-t_0]:\int_0^s\left\|\widehat{\bS}_{T-u}(\by_u)-\widetilde{\bS}_{T-u}(\by_u)\right\|_2^2\d u\ge R\right\}\wedge(T-t_0),
    \]
    with the convention that the integrand equals \(\infty\) when \(\by_u=\infty\).

    \quad

    \quad

    Since the linearized drift has the form \(\bD_{T-s}\by_s\), with \(\sup_t\|\bD_t\|_{\op}<\infty\), the empirical process satisfies, for \(s\) smaller than the explosion time \(\zeta\),
    \[
        \by_s =\by_0+\int_0^s\bD_{T-u}\by_u\d u+2\int_0^s(\widehat{\bS}_{T-u}(\by_u)-\widetilde{\bS}_{T-u}(\by_u))\d u+\sqrt{2}\bW_s.
    \]
    If \(\int_0^\zeta\|\Delta_u\|_2^2\,\d u<\infty\), then it follows from Cauchy-Schwarz and Gronwall's inequality that \(\sup_{s<\zeta}\|\by_s\|_2<\infty\), contradicting the explosion alternative. Consequently, \(\tau_R<\zeta\) on every explosive path, for every \(R<\infty\).

    Let \(\widetilde\P^R\) and \(\widehat\P^R\) denote the laws of the stopped canonical process \(\{\by_{s\wedge\tau_R}\}_{s\in[0,T-t_0]}\) under \(\widetilde\P\) and \(\widehat\P\), respectively. Girsanov's theorem gives
    \begin{align*}
        \KL(\widetilde\P^R\mmid\widehat\P^R)
        &\le \E_{\widetilde\P}\left[\int_0^{\tau_R}\left\|\widehat{\bS}_{T-s}(\by_s)-\widetilde{\bS}_{T-s}(\by_s)\right\|_2^2\d s\mid\bX\right]
        \\ & \le \E_{\widetilde\P}\left[\int_0^{T-t_0}\left\|\widehat{\bS}_{T-s}(\by_s)-\widetilde{\bS}_{T-s}(\by_s)\right\|_2^2\d s\mid\bX\right].
    \end{align*}

    Since the linearized process is nonexplosive and the scores are locally bounded, \(\tau_R\uparrow T-t_0\) almost surely under \(\widetilde{\P}\). Under \(\widehat{\P}\), \(\tau_R\) increases to \(T-t_0\) on nonexplosive paths and to the explosion time on explosive paths. In the latter case, the stopped paths converge to the absorbed path in the topology of $C([0,T-t_0]; \R_\infty^d)$. Hence,
    \[
        \widetilde\P^R\Longrightarrow\widetilde\P, \qquad \widehat\P^R\Longrightarrow\widehat\P.
    \]
    By the joint lower semicontinuity of relative entropy, the preceding estimate and Lemma~\ref{lem:reverse_score_linearization},
    \begin{align*}
        \KL(\P_{\widetilde\bx}\mmid\P_{\widehat\bx})=\KL(\widetilde\P\mmid\widehat\P)&\leq \int_{t_0}^T\E\left[\left\|\widehat{\bS}_t(\widetilde\bx_t)-\widetilde{\bS}_t(\widetilde\bx_t)\right\|_2^2\mid \bX\right]\d t
        \\ & \prec (T-t_0) \sup_{t\in[t_0,T]}\sft_t^2d^{-5}\left(1+\sft_t^2+\sft_t^4d^{-2}\right). \qedhere
    \end{align*} 
\end{proof}

\subsection{Proof of Theorem~\ref{thm:effective_covariance_bound}}
\label{sec:proof-effective-covariance}

We introduce an auxiliary spectral linearized process, which is analogous to the linearized reverse process \(\widetilde\bx_t\) but is obtained by replacing the kernel \(\widetilde{\cK}_t\) by the time-dependent effective empirical operator \(\widehat\cK_{\mathrm{eff},t}\). More particularly, for every \(t\in[t_0,T]\), let \(\check{\bS}_t\) denote the linear score estimator obtained by replacing \(\widetilde\cK_t\) with \(\widehat\cK_{\mathrm{eff},t}\) in~\eqref{eq:linearized_score_estimator}, where \(\widehat\cK_{\mathrm{eff},t}\) and \(\delta_{n,t}\) are defined in~\eqref{eq:K_hat_op} for time \(t\). Equivalently, unwinding the definitions using the proof of Lemma~\ref{lem:gf_linearized_train_test_Q_representation} shows that $\check \bS_t$ is an explicit linear shrinkage field:
\begin{equation}\label{eq:spectral_linear_score_gf_multiplier}
    \check{\bS}_{t}(\bx)=\frac1{\sigma_t}\,\psi_t(\widehat\bSigma,\sft_t)\bx .
\end{equation}
Define the auxiliary spectral linearized process by
\begin{equation}
    \de \check{\bx}_t = -(\check{\bx}_t + 2\check{\bS}_t(\check{\bx}_t))\d t + \sqrt{2}\d\bar{\bB}_t, \qquad \check{\bx}_T \sim \cN(0, \bI_d),
\end{equation}
and write
\begin{equation}
    \check{\bSigma}_t \deq  \E[\check{\bx}_t\check{\bx}_t^\sT\mid\bX].
\end{equation}
By Itô's lemma, \(\check{\bSigma}_t\) satisfies the ODE
\begin{equation}\label{eq:spectral_linearized_covariance_ode}
    \frac{\d}{\d t}\check{\bSigma}_t
    = -2\check{\bSigma}_t - \frac{2}{\sigma_t}\psi_t(\widehat\bSigma,\sft_t)\check{\bSigma}_t - \frac{2}{\sigma_t}\check{\bSigma}_t\psi_t(\widehat\bSigma,\sft_t) - 2\bI_d,
    \qquad
    \check{\bSigma}_T = \bI_d.
\end{equation}
The next lemma shows that the covariance of the linearized reverse process \(\widetilde{\bSigma}_t\) is close to the covariance of the spectral linearized process \(\check{\bSigma}_t\).

\begin{lemma}
\label{lem:linearized_covariance_approximation}
    Suppose that Assumptions~\ref{ass:high_dim},~\ref{ass:input_dist}, and
    \ref{ass:kernel_reverse} hold. Then
    \[
        \sup_{t\in[t_0,T]}
        \|\widetilde{\bSigma}_t-\check{\bSigma}_t\|_{\op}
        \prec
        \left(d^{-3/2}\sup_{t\in[t_0,T]}\sft_t\right)\wedge 1.
    \]
\end{lemma}

\begin{proof}
    Let
    \begin{equation}\label{eq:spectral_effective_kernel_matrix}
        \check{\bK}_t
        \deq 
        \widetilde{\bK}_t+
        \begin{bmatrix}
            \delta_{n,t}\bI_n-\widetilde{\bD}_t&0\\
            0 & 0
        \end{bmatrix}.
    \end{equation}
    Thus, \(\check{\bK}_t\) is the finite-dimensional representation of \(\widehat\cK_{\mathrm{eff},t}\) on the invariant subspace
    \(\widetilde\cV\). Both \(\widetilde{\bK}_t\) and \(\check{\bK}_t\) are positive semidefinite, since their diagonal corrections are nonnegative. Repeating the calculation in
    Lemma~\ref{lem:reverse_linearized_check_score_representation} with
    \(\widetilde{\bD}_t\) replaced by \(\delta_{n,t}\bI_n\) gives
    \begin{equation}\label{eq:spectral_linear_score_matrix_representation}
        \psi_t(\widehat\bSigma,\sft_t)
        =
        -\frac1{\sigma_t}
        \begin{bmatrix}0&\bI_d\end{bmatrix}
        \left(\bI_{n+d}-e^{-\sft_t\check{\bK}_t}\right)
        \begin{bmatrix}0\\\bI_d\end{bmatrix},\qquad -\sigma_t^{-1}\bI_d \preceq \psi_t(\widehat\bSigma,\sft_t)\preceq0.
    \end{equation}
    An argument analogous to the one used to
    obtain~\eqref{eq:linearized_reverse_cov_bound} shows that
    \begin{equation}\label{eq:linearized_covariances_uniform_bound}
        \sup_{t\in[t_0,T]}
        \left(
        \|\widetilde{\bSigma}_t\|_{\op}
        \vee\|\check{\bSigma}_t\|_{\op}
        \right)
        \lesssim1.
    \end{equation}

    On the very-high-probability event on which
    \(\max_{i\le n}\|\bx_{0,i}\|_2^2/d\) is bounded, the mean-value theorem and the definition of the two diagonal corrections yield
    \begin{align}\label{eq:effective_kernel_matrix_comparison}
        \sup_{t\in[t_0,T]}
        \|\widetilde{\bK}_t-\check{\bK}_t\|_{\op}
        &=
        \sup_{t\in[t_0,T]}\max_{i\le n}
        |\widetilde\delta_{i,t}-\delta_{n,t}|\lesssim
        \frac1n\max_{i\le n}
        \left|
        \frac{\|\bx_{0,i}\|_2^2}{d}-\tau_{\bSigma}
        \right|
        \prec d^{-3/2},
    \end{align}
    where the last estimate follows from \eqref{eq:norm_concentration_overlaps} and \(n\asymp d\). By~\eqref{eq:variation_of_parameters},
    \begin{equation}\label{eq:linearized_score_matrix_duhamel}
        e^{-\sft_t\widetilde{\bK}_t}-e^{-\sft_t\check{\bK}_t} = -\int_0^{\sft_t} e^{-(\sft_t-s)\widetilde{\bK}_t}(\widetilde{\bK}_t-\check{\bK}_t)e^{-s\check{\bK}_t}\d s.
    \end{equation}
    Hence, using~\eqref{eq:spectral_linear_score_matrix_representation} and
    \eqref{eq:effective_kernel_matrix_comparison}, together with the
    sub-multiplicativity of the operator norm and the fact that
    \(\|\exp(-s\widetilde{\bK}_t)\|_{\op}\vee
    \|\exp(-s\check{\bK}_t)\|_{\op}\le 1\) for every \(s\ge 0\),
    \begin{align}\label{eq:linearized_score_matrix_comparison}
        \sup_{t\in[t_0,T]}
        \|\widetilde\bPsi_t-
        \psi_t(\widehat\bSigma,\sft_t)\|_{\op}
        &\le
        \frac1{\sigma_{t_0}}
        \sup_{t\in[t_0,T]}
        \sft_t\|\widetilde{\bK}_t-\check{\bK}_t\|_{\op}\prec
        d^{-3/2}\sup_{t\in[t_0,T]}\sft_t.
    \end{align}
    
    Taking the difference between the covariance ODEs~\eqref{eq:covariance_ODE_linearized_reverse} and~\eqref{eq:spectral_linearized_covariance_ode}, using the linearity of the corresponding score functions, we have
    \begin{equation}
    \begin{aligned}\label{eq:linearized_covariance_difference_ode}
        \frac{\d}{\d t}
        (\widetilde\bSigma_t-\check{\bSigma}_t)
        &=-2(\widetilde\bSigma_t-\check{\bSigma}_t)-\frac2{\sigma_t}\big(
        \widetilde\bPsi_t(\widetilde\bSigma_t-\check{\bSigma}_t)
        +(\widetilde\bSigma_t-\check{\bSigma}_t)\widetilde\bPsi_t\big)
        \\ & \qquad -\frac2{\sigma_t}\big(
        (\widetilde\bPsi_t-\psi_t(\widehat\bSigma,\sft_t))
        \check{\bSigma}_t
        +\check{\bSigma}_t
        (\widetilde\bPsi_t-\psi_t(\widehat\bSigma,\sft_t))
        \big)
    \end{aligned}
    \end{equation}
    Integrating~\eqref{eq:linearized_covariance_difference_ode} from \(t\) to \(T\), using the common terminal condition, the bound on \(\|\widetilde{\bPsi}\|_\op\) from Lemma~\ref{lem:reverse_linearized_check_score_representation}, \eqref{eq:linearized_covariances_uniform_bound}, and taking the
    operator norm gives
    \begin{equation}\label{eq:linearized_covariance_difference_integral_bound}
        \|\widetilde{\bSigma}_t - \check{\bSigma}_t\|_\op
        \lesssim
        \int_{t}^T\left(
        \|\widetilde{\bSigma}_s - \check{\bSigma}_s\|_\op
        +\|\widetilde\bPsi_s-\psi_s(\widehat\bSigma,\sft_s)\|_\op
        \right)\d s.
    \end{equation}
    Applying backward Gronwall's inequality and~\eqref{eq:linearized_score_matrix_comparison}, yields
    \begin{equation}\label{eq:linearized_covariance_comparison_uncapped}
        \|\widetilde{\bSigma}_t - \check{\bSigma}_t\|_\op \lesssim\int_{t}^T \|\widetilde\bPsi_s-\psi_s(\widehat\bSigma,\sft_s)\|_\op\d s \prec d^{-3/2}\sup_{t\in[t_0,T]}\sft_t.
    \end{equation}
    Finally, \eqref{eq:linearized_covariances_uniform_bound} also gives
    \(\sup_t\|\widetilde\bSigma_t-\check{\bSigma}_t\|_{\op}\lesssim1\), which yields the minimum with \(1\).
\end{proof}

The covariance ODE associated with the spectral linearized process is diagonalized by the eigenvectors of \(\widehat\bSigma\). Hence, we may apply tools from random matrix theory to analyze the spectrum of \(\check{\bSigma}_t\).

\begin{lemma}[Deterministic equivalent of the spectral linearized covariance]
\label{lem:spectral_linearized_covariance_deterministic_equivalent}
    Suppose that Assumptions~\ref{ass:high_dim},~\ref{ass:input_dist}, and
    \ref{ass:kernel_reverse} hold. Then, for every deterministic
    \(0\preceq\bA\in\R^{d\times d}\) with \(\|\bA\|_*\le1\) and every constant \(\epsilon,D>0\), there exists a constant \(C>0\) such that
    \[
        \sup_{t\in[t_0,T]}
        \left|
        \Tr\!\left\{\bA(\check\bSigma_t-\bar\bSigma_t)\right\}
        \right|
        \leq d^{\epsilon-1/2}
    \]
    with probability at least \(1-Cd^{-D}\).
\end{lemma}

\begin{proof}
    Since all the coefficient matrices in
    \eqref{eq:spectral_linearized_covariance_ode} are spectral functions
    of \(\widehat\bSigma\) and \(\check\bSigma_T=\bI_d\),
    \begin{equation}\label{eq:spectral_covariance_functional_representation}
        \check\bSigma_t=\bar\Lambda_t(\widehat\bSigma),
    \end{equation}
    where \(\bar\Lambda_t\) is defined by~\eqref{eq:s_t_x_ode_def}.

    By~\eqref{eq:cov_concentration} and the compact support of \(\bOmega\),
    there is a deterministic \(\Lambda<\infty\) such that, with very high
    probability, both \(\operatorname{spec}(\widehat\bSigma)\) and
    \(\operatorname{supp}(\bOmega)\) are contained in \([0,\Lambda]\).
    Fix a positively oriented rectangular contour \(\Gamma\) enclosing this interval, with
    vertical edges \(\Re z=-\eta\) and \(\Re z=\Lambda+\eta\) and horizontal
    edges \(\Im z=\pm1\), where
    \(0<\eta<\min\{1,\sigma_{t_0}^2/4\}\). In particular,
    \begin{equation}\label{eq:reverse_contour_bound}
        \inf_{t\in[t_0,T]}\inf_{z\in\Gamma}
        \Re(\rho_t^2z+\sigma_t^2)>0.
    \end{equation}

    We show that the contour estimate~\eqref{eq:psi_bound} holds uniformly along the reverse-time schedule. The matrix-exponential representation \eqref{eq:linearized_test_symbol_matrix_representation}, applied with \((\rho,\sigma,\delta_n,\sft) = (\rho_t,\sigma_t,\delta_{n,t},\sft_t)\) shows that \(z\mapsto\psi_t(z,\sft_t)\) is entire for every \(t\). Also, since \(t\mapsto\sft_t\) is piecewise continuous and \(t\mapsto(\rho_t,\sigma_t,\delta_{n,t})\) is continuous, \(t\mapsto\psi_t(z,\sft_t)\) is piecewise continuous for every \(z\).

    It remains to verify that the estimates used in the proof of
    Lemma~\ref{lem:gf_train_test_error_deterministic_equivalent} are uniform over \(t\in[t_0,T]\). Assumption~\ref{ass:kernel_reverse} implies Assumption~\ref{ass:kernel_func_constant_kappa}, while \(0<\inf_{t\in[t_0,T]}\sigma_t\wedge \rho_t\leq\sup_{t\in[t_0,T]}\sigma_t\vee\rho_t<1\). Moreover, by Remark~\ref{rmk:delta_n_nonnegativity}, either \(h(u)=u\) for every \(u\), in which case \(d\delta_{n,t}=0\) for every \(t\), or \(d\delta_{n,t} \asymp 1\) uniformly over \(t\in[t_0,T]\). Indeed, in the latter case \(h(u)-u>0\) for every \(u>0\), while \(\rho_t^2\tau_{\bSigma}\) remains in a compact subset of \((0,\infty)\). Together with the contour bound~\eqref{eq:reverse_contour_bound}, these are precisely the uniform conditions used to derive \eqref{eq:psi_bound}. Consequently,
    \begin{equation}\label{eq:spectral_reverse_psi_contour_bound}
        \sup_{t\in[t_0,T]}\sup_{z\in\Gamma}|\psi_t(z,\sft_t)| \lesssim 1.
    \end{equation}

    Since \(t\mapsto\psi_t(z,\sft_t)\) is piecewise continuous and \(z\mapsto\psi_t(z,\sft_t)\) is entire, the integral form of \eqref{eq:s_t_x_ode_def} and Picard iteration show that \(z\mapsto\bar\Lambda_t(z)\) is entire for every \(t\). Moreover, \(\sigma_t\geq\sigma_{t_0}>0\) and \eqref{eq:spectral_reverse_psi_contour_bound} imply that the ODE coefficient is uniformly bounded on \([t_0,T]\times\Gamma\). Backward Gronwall's inequality therefore gives \begin{equation}\label{eq:spectral_reverse_covariance_contour_bound}
        \sup_{t\in[t_0,T]}\sup_{z\in\Gamma}|\bar\Lambda_t(z)|\lesssim 1.
    \end{equation}

    On the spectral-containment event above, Cauchy's formula and \eqref{eq:resolvent_fixed_point_measure} give
    \[
        \check\bSigma_t=-\frac{1}{2\pi i}\oint_\Gamma \bar\Lambda_t(z)\widehat\bR(z)\d z,
        \qquad
        \bar\bSigma_t=-\frac{1}{2\pi i}\oint_\Gamma\bar\Lambda_t(z)\bM(z)\d z.
    \]
    The contour \(\Gamma\) lies a fixed positive distance from
    \(\operatorname{supp}(\bOmega)\), satisfies
    \(\inf_{z\in\Gamma}|z|>0\), has bounded real part, and obeys
    \(|\Im z|\le1\), with horizontal edges \(\Im z=\pm1\). Hence
    Proposition~\ref{prop:outside_local_law} applies uniformly on \(\Gamma\). Together with~\eqref{eq:spectral_reverse_covariance_contour_bound}, this
    gives
    \begin{align*}
        \sup_{t\in[t_0,T]}
        \left|\Tr\!\left\{\bA(\check\bSigma_t-\bar\bSigma_t)\right\}\right|
        &\le \frac{|\Gamma|}{2\pi}
        \sup_{t\in[t_0,T]}\sup_{z\in\Gamma}|\bar\Lambda_t(z)|
        \sup_{z\in\Gamma}
        \left|\Tr\!\left\{\bA(\widehat\bR(z)-\bM(z))\right\}\right| \leq d^{\epsilon-1/2}
    \end{align*}
    with probability at least \(1-Cd^{-D}\) for some constant \(C>0\), where \(\epsilon,D>0\) are arbitrary.
\end{proof}

The proof of Theorem~\ref{thm:effective_covariance_bound} follows by combining Lemmas~\ref{lem:linearized_covariance_approximation} and~\ref{lem:spectral_linearized_covariance_deterministic_equivalent}.

\begin{proof}[Proof of Theorem~\ref{thm:effective_covariance_bound}]
    For every deterministic \(0\preceq\bA\in\R^{d\times d}\) with \(\|\bA\|_*\le1\),
    \begin{align*}
        \sup_{t\in[t_0,T]}
        \left|\Tr\!\left\{\bA(\widetilde\bSigma_t-\bar\bSigma_t)\right\}\right|
        &\le
        \sup_{t\in[t_0,T]}
        \left|\Tr\!\left\{\bA(\widetilde\bSigma_t-\check\bSigma_t)\right\}\right|
        +
        \sup_{t\in[t_0,T]}
        \left|\Tr\!\left\{\bA(\check\bSigma_t-\bar\bSigma_t)\right\}\right|\\
        &\le \sup_{t\in[t_0,T]}\|\widetilde\bSigma_t-\check\bSigma_t\|_{\op}
        +
        \sup_{t\in[t_0,T]}
        \left|\Tr\!\left\{\bA(\check\bSigma_t-\bar\bSigma_t)\right\}\right|.
    \end{align*}
    The first term is bounded by Lemma~\ref{lem:linearized_covariance_approximation} and the second term is bounded by Lemma~\ref{lem:spectral_linearized_covariance_deterministic_equivalent}.
\end{proof}

\end{document}